\documentclass{article}
\PassOptionsToPackage{numbers, compress}{natbib}
\usepackage[preprint]{neurips_2026}
\usepackage{afterpage}

\usepackage{times}
\usepackage{latexsym}
\usepackage[table]{xcolor}
\definecolor{Green}{RGB}{110,170,255}
\definecolor{BestGreen}{RGB}{200,230,255}
\definecolor{Red}{RGB}{255,120,140}
\definecolor{mycolor_1}{RGB}{220,255,255}  
\usepackage{placeins}
\usepackage{colortbl}
\usepackage{pgfmath}
\usepackage{hyperref}
\newcommand{\heat}[2]{%
\begingroup
\pgfmathparse{#1-#2}%
\edef\diff{\pgfmathresult}%
\pgfmathparse{min(40,max(0,round(3abs(#1-#2))))}%
\edef\strength{\pgfmathresult}%
\ifdim \diff pt > 0pt
\edef\ApplyColor{\noexpand\cellcolor{Green!\strength!white}}%
\ApplyColor #1%
\else
\ifdim \diff pt < 0pt
\pgfmathparse{min(80,max(0,round(3.4abs(#1-#2))))}%
\edef\strength{\pgfmathresult}%
\edef\ApplyColor{\noexpand\cellcolor{Red!\strength!white}}%
\ApplyColor #1%
\else
#1%
\fi
\fi
\endgroup
}
\newcommand{\heatbest}[2]{%
\begingroup
\pgfmathparse{#1-#2}%
\edef\diff{\pgfmathresult}%
\ifdim \diff pt > 0pt
\cellcolor{BestGreen}#1%
\else
\heat{#1}{#2}%
\fi
\endgroup
}
\newcommand{\heatsecond}[2]{%
\begingroup
\pgfmathparse{#1-#2}%
\edef\diff{\pgfmathresult}%
\ifdim \diff pt > 0pt
\cellcolor{Green!65!white}#1%
\else
\heat{#1}{#2}%
\fi
\endgroup
}
\newcommand{\heatlowest}[2]{%
\begingroup
\pgfmathparse{#1-#2}%
\edef\diff{\pgfmathresult}%
\ifdim \diff pt > 0pt
\cellcolor{Red!100!white}#1%
\else
\heat{#1}{#2}%
\fi
\endgroup
}
\usepackage[T1]{fontenc}
\usepackage[utf8]{inputenc}
\usepackage{microtype}
\usepackage{inconsolata}
\usepackage{graphicx}
\usepackage{subcaption}
\usepackage{enumitem}
\usepackage{float}
\usepackage{booktabs} 
\usepackage{multirow}       
\usepackage{algorithm}
\usepackage{algpseudocode}
\usepackage{etoc}
\usepackage{titletoc}
\usepackage{amsmath}
\usepackage{amssymb}
\usepackage{mathtools}
\usepackage{amsthm}

\usepackage{pifont}   

\usepackage[most]{tcolorbox}
\tcbuselibrary{skins, listings, breakable}
\definecolor{DeepBlue}{HTML}{003f5c}
\definecolor{WarmOrange}{HTML}{ffa600}
\definecolor{CreamBack}{HTML}{FFF8F0}
\tcbset{
styleA/.style={
enhanced,
breakable,
fonttitle=\bfseries\large,
fontupper=\ttfamily\footnotesize,
colframe=DeepBlue,
colback=CreamBack,
coltitle=white,
arc=3mm,
boxrule=1pt,
left=4mm, right=4mm, top=4mm, bottom=4mm,
listing only,
extras first={
  skin=enhanced,
  sharp corners=south,
  overlay={
    \draw[white,line width=2.5pt]
      (frame.south west) -- (frame.south east);
  },
},
extras middle={
  skin=enhanced,
  sharp corners,
  overlay={
    \draw[white,line width=2.5pt]
      (frame.north west) -- (frame.north east);
    \draw[white,line width=2.5pt]
      (frame.south west) -- (frame.south east);
  },
},
extras last={
  skin=enhanced,
  sharp corners=north,
  overlay={
    \draw[white,line width=2.5pt]
      (frame.north west) -- (frame.north east);
  },
},
listing options={
language=Python,
basicstyle=\ttfamily\footnotesize,
commentstyle=\ttfamily\footnotesize\upshape\color{gray},
breaklines=true,
keepspaces=true,
breakatwhitespace=false,
xleftmargin=0pt,
xrightmargin=0pt,
aboveskip=0pt,
belowskip=0pt,
tabsize=4,
columns=fullflexible,
breakindent=0pt,
showstringspaces=false,
escapeinside={(*@}{@*)},
}
}
}

\definecolor{AppendixPurple}{RGB}{125,95,170}
\hypersetup{
colorlinks=true,
linkcolor=AppendixPurple,
citecolor=blue,      
urlcolor=black
}
\titlecontents{section}
[0pt]
{\addvspace{0.5em}\bfseries\large\color{AppendixPurple}}
{\contentslabel{2.0em}}
{}
{\hfill\color{black}\contentspage}
\titlecontents{subsection}
[3.2em]
{\normalsize\color{AppendixPurple}}
{\contentslabel{3.0em}}
{}
{\titlerule*[0.7pc]{.}\color{black}\contentspage}

\usepackage[capitalize,noabbrev]{cleveref}

\theoremstyle{plain}
\newtheorem{theorem}{Theorem}[section]
\newtheorem{proposition}[theorem]{Proposition}

\newtheorem{corollary}[theorem]{Corollary}
\theoremstyle{definition}

\theoremstyle{remark}

\usepackage[textsize=tiny]{todonotes}
\tcbuselibrary{listings, breakable, skins}
\usepackage{listings}

\usepackage{fontawesome5}
\usepackage{svg}
\newtcolorbox{theorembox}{
  enhanced,
  breakable,
  colback=blue!3,
  colframe=blue!35!black,
  boxrule=0.8pt,
  arc=2mm,
  left=1.2mm,
  right=1.2mm,
  top=1mm,
  bottom=1mm
}

\definecolor{abstractblue}{RGB}{42, 91, 190}
\definecolor{abstractback}{RGB}{237,244,252}

\newtcolorbox{abstractbox}{
    colback=abstractback,   
    colframe=abstractblue,
    boxrule=0.8pt,
    arc=4mm,
    left=3mm,
    right=3mm,
    top=2mm,
    bottom=2mm,
    width=\textwidth,
    center,
}

\begin{document}

\newgeometry{
  letterpaper,
  left=0.8in,
  right=0.8in,
  top=0.6in,
  bottom=0.7in,
  footskip=0.3in
}
\includegraphics[height=40pt]{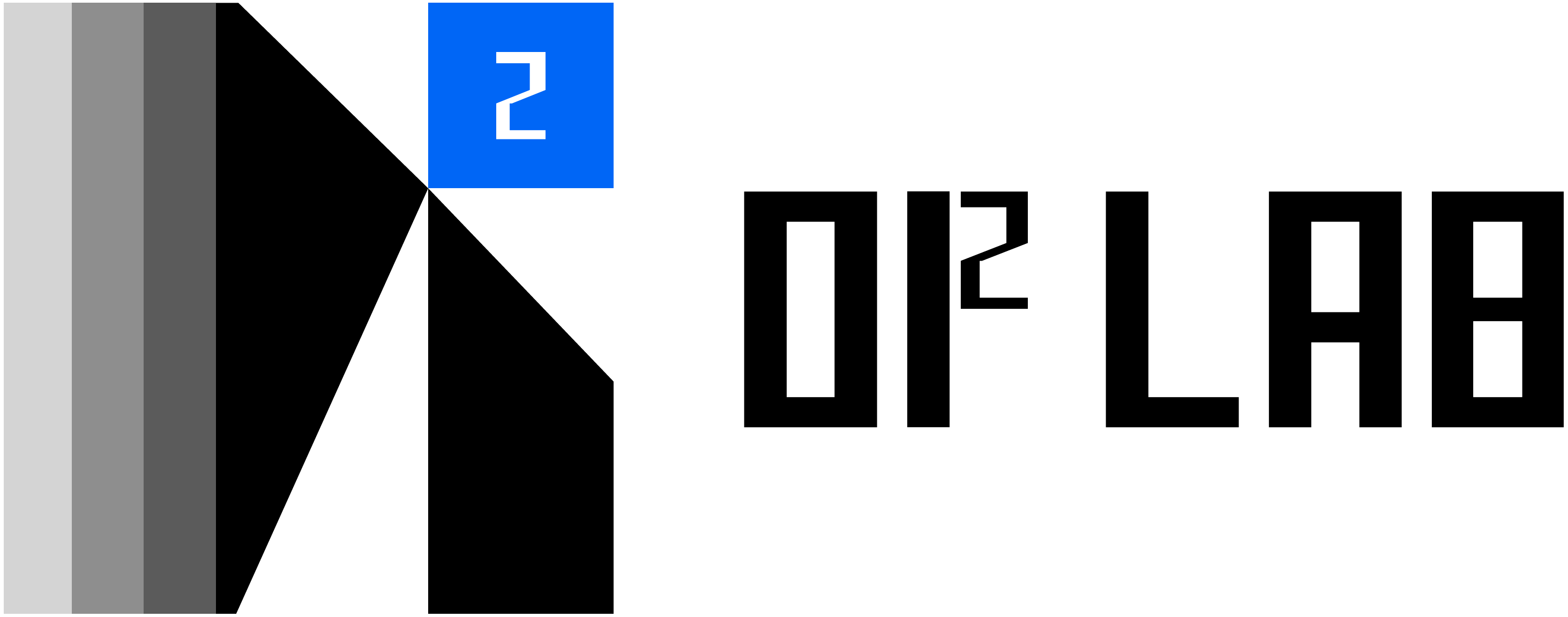} \quad

\vspace{1cm}

\title{CoSPlay: Cooperative Self-Play at Test-Time with Self-Generated \color{abstractblue}{Code} \color{black}{and} \color{orange}{Unit Test} }

%

\vspace{-1.2cm}

\author{
Zhangyi Hu\textsuperscript{1,*} \quad
Chenhui Liu\textsuperscript{1,*} \quad
Tian Huang\textsuperscript{1,*} \quad
\\
Jindong Li\textsuperscript{1} \quad
Yang Yang\textsuperscript{1} \quad
Jiemin Wu\textsuperscript{1} \quad
Zining Zhong\textsuperscript{1} \quad
Menglin Yang\textsuperscript{1} \quad
Yutao Yue\textsuperscript{1,2,\ensuremath{\dagger}}
\\
\textsuperscript{1}The Hong Kong University of Science and Technology (Guangzhou)
\\
\textsuperscript{2}Institute of Deep Perception Technology, JITRI, Wuxi, China
\\
\href{https://github.com/sanae-ai/CoSPlay}{
  \raisebox{-0.02em}{\faGithub}\ Github Page
}
\quad
\href{https://huggingface.co/datasets/yomi017/CoSPlay}{
  \raisebox{-0.12em}{\includegraphics[height=1em]{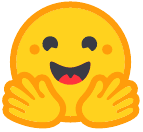}}\ Dataset \& Log
}
\quad
{
  \raisebox{-0.02em}{\faCalendar*}\ 2026-5-13
}
}

\maketitle

\vspace{-0.9cm}

\begin{abstractbox}
\begin{center}
    {\Large \textbf{\textcolor{abstractblue}{Abstract}}}
\end{center}

Recently, Reinforcement Learning with Verifiable Rewards (RLVR) and Test-Time Scaling (TTS) have advanced LLM code generation through executable verification. Yet Ground-Truth Unit Tests (GT UTs) remain a bottleneck: SOTA RLVR methods require them for costly training, while existing TTS methods lose competitiveness without them. This motivates GT-free TTS, where existing methods directly use self-generated UTs to refine and select code candidates. Yet such UTs are often noisy or spuriously coupled with wrong code, and UT quality in turn cannot be validated without reliable code. \textit{\textbf{The key challenge is therefore to jointly improve both.}}
To this end, we present \textbf{CoSPlay}, a \textbf{GT-free}, \textbf{training-free} framework that jointly improves codes and UTs through cooperative self-play. It first explores diverse solution ideas and identifies their potential failure modes to produce discriminative UT ideas. It then uses bidirectional pass-count signals from the Code-UT execution matrix to iteratively prune or fix weak codes and refresh or replace unreliable UTs, letting the two pools co-evolve. Finally, when multiple codes remain tied at the highest pass count, it picks the final code from the largest output-consensus cluster, since correct codes agree on the same inputs while wrong codes diverge.
Experiments on four challenging benchmarks show that CoSPlay on Qwen2.5-7B-Instruct improves average BoN from 22.1\% to \underline{\textbf{33.2\%}} and UT accuracy from 14.6\% to \underline{\textbf{78.3\%}}, matching or surpassing the RLVR model CURE-7B. When applied to CURE-7B, it further improves BoN by \underline{\textbf{5.7\%}}. CoSPlay also generalizes across diverse backbones and outperforms GT-free TTS baselines under comparable token budgets, with continued gains as the budget scales up. These results suggest a scalable inference strategy for competitive code generation without any GT data.



\end{abstractbox}

\begin{figure}[!htbp]
    \centering
    \includegraphics[width=\textwidth]{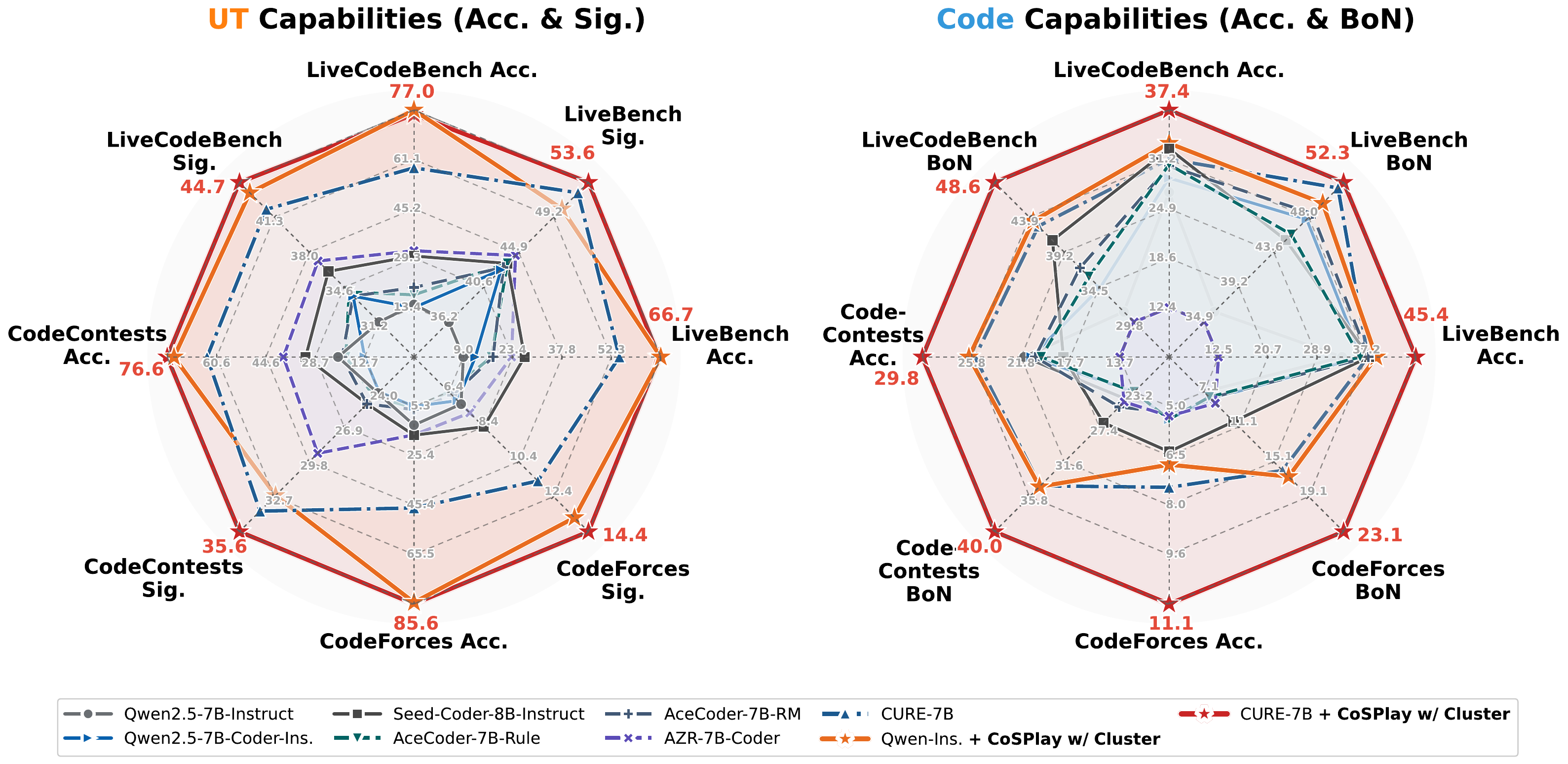}
    \caption{Performance comparison between our \textit{\textbf{Training-free}} and \textit{\textbf{GT-free}} CoSPlay and other RLVR methods that need costly weight updating (AZR-7B-Coder 0k) or massive GT data (AceCoder-7B-Rule 22k, AceCoder-7B-RM 329k, CURE-7B 4.5k).}
    \vspace{-1.5em}
    \label{fig:radar_figure}
\end{figure}

\section{Introduction}

\vspace{-0.3cm}

\begin{figure*}[!t]
    \centering
    \includegraphics[width=\textwidth]{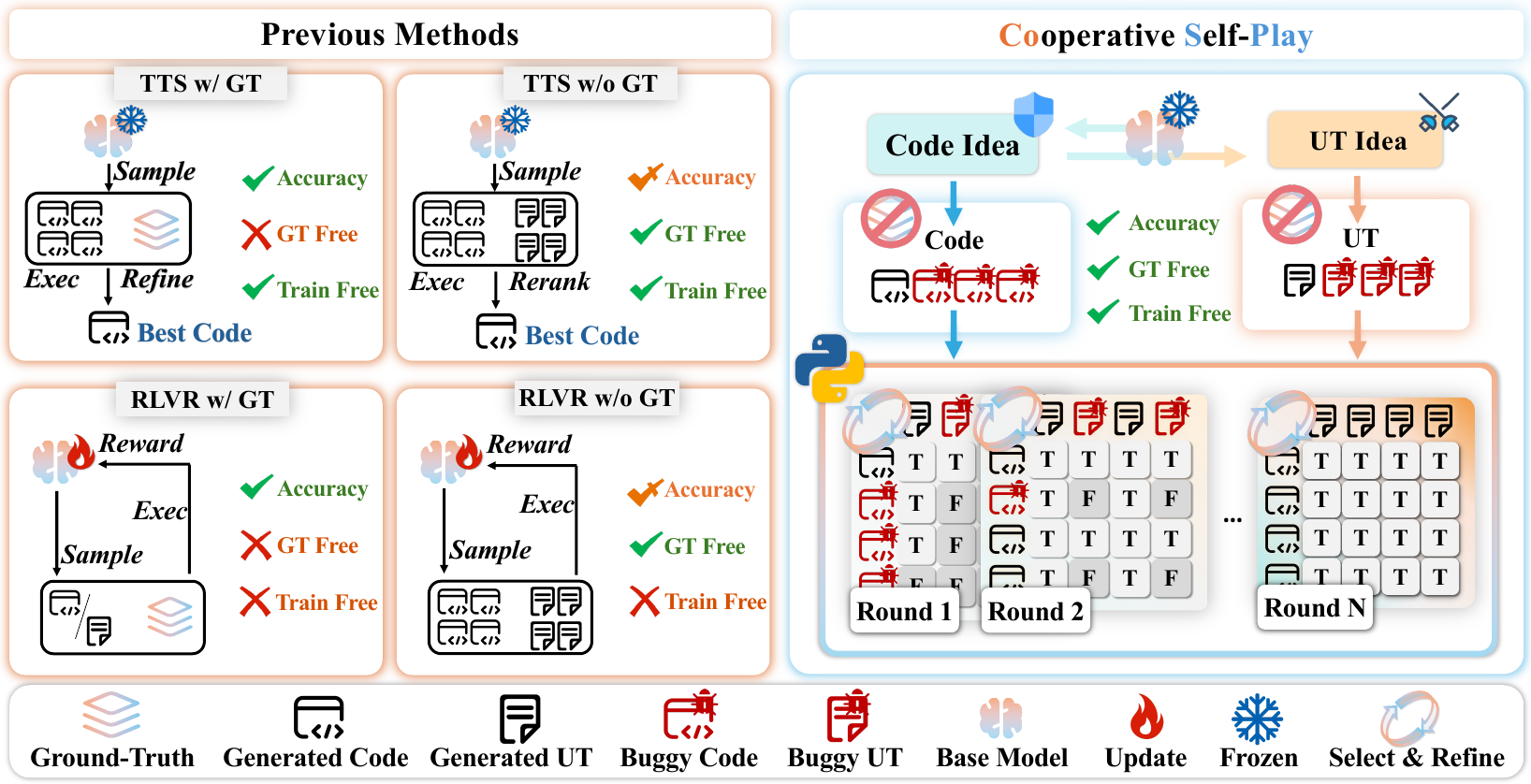}
    \caption{Our motivation: achieving high accuracy without any Ground-Truth and weight updating.}
    \label{fig:motivation}
    \vspace{-2em}
\end{figure*}
Recently, advanced by RLVR~\cite{deepseekr1, yu2025dapo, kimiteam2025kimik15} and TTS~\cite{brown2024largelanguagemonkeysscaling, Reward-guided-TTS:lightman2023letsverifystepstep, Reward-guided-TTS:liu2025inferencetimescalinggeneralistreward}, LLMs have achieved significant progress in coding capability. In code generation, UTs provide executable input-output checks that make functional correctness verifiable~\cite{wang2025cure, yu2025orps}. This has made them widely used for guiding recent code-generation methods. As shown in Figure~\ref{fig:motivation}, SOTA RLVR methods leverage large collections of GT UTs for training, i.e. test cases with guaranteed-correct input-output pairs, typically curated by human experts or official problem setters~\cite{code-r1, zhang2025focuseddpo, wang2025cure}. However, curating such GT UTs is too costly to scale~\cite{chen2022codet, liu2023iscorrect},
and GT-free RLVR methods like Absolute-Zero~\cite{zhao2025absolutezero} still remain less competitive in practice and require costly model-weight updates~\cite{yue2025doesreinforcementlearningreally, zhao2025echo}.

While TTS methods avoid weight updates, they still need to evaluate the code quality at inference time
for refinement, filtering, and selection. 
When GT UTs are available, these decisions can be grounded in trusted input-output checks~\cite{li2025s*, shinn2023reflexion, yu2025orps, ldb}.
In GT-free settings, however, it must be   constructed from self-generated UTs, typically by sampling codes and tests before heuristic or mathematical re-ranking~\cite{chen2022codet, huang2024mpsc}; we provide a broader discussion in Appendix~\ref{sec:related_work}.
Yet scaling this pipeline is fragile: generated UTs may have incorrect expected outputs, fail to distinguish among plausible candidates, or spuriously agree with wrong code~\cite{xie2025ctrl}.
The key question therefore shifts from merely sampling more code or UTs to a co-evolution perspective: \textit{\textbf{Can codes and UTs co-evolve at inference time to achieve competitive code generation without GT?}}

To this end, we present CoSPlay, a GT-free, training-free three-stage pipeline that creates a co-evolving loop between code candidates and UTs.
First, \textit{{Exploration-Attack-Guided Code-UT Idea Generation}} bootstraps the loop by exploring plausible solution strategies and deriving targeted UT ideas from their potential failure modes, so the initial UT pool is grounded in solution-specific assumptions and can better distinguish plausible code candidates.
Given this initialized pool, \textit{{Execution-Matrix-Driven Iterative Self-Play}} lets codes and UTs evaluate each other: high-support UTs provide reliable refinement signals, while high-support codes provide evidence for judging UTs. These pass-count signals guide code-side pruning/refinement and UT-side refreshing/replacement, progressively improving both pools.
Finally, when the evolved Code-UT pool still leaves BoN-tied codes, \textit{{Output-Consensus-Based Cluster Selection}} handles the remaining ambiguity among BoN-tied codes by comparing their outputs on random valid inputs and selecting the most reliable consensus cluster and code, since correct codes should agree while incorrect codes tend to diverge.

Our contributions are summarized as follows:
\begin{enumerate}[topsep=0pt, itemsep=1pt, parsep=1pt, leftmargin=*]
    \item We introduce \textbf{CoSPlay}, a GT-free, training-free framework that enables code candidates and self-generated UT to co-evolve at test-time, providing a scalable inference strategy for competitive code generation without any GT-data.

    \item Extensive experiments across four challenging coding benchmarks demonstrate the \textbf{effectiveness}, \textbf{generalizability}, and \textbf{scalability} of CoSPlay. It matches or surpasses strong RLVR models trained with massive GT UTs and improves various BoN across instruct, RLVR-tuned, and frontier-scale backbones. Under comparable output-token budgets, CoSPlay outperforms existing GT-free TTS methods, and its performance continues to rise as the budget is scaled up.
    
    \item We establish internal quality-evaluation principles for GT-free Code-UT co-evolution.  Without GT labels, UT pass counts estimate test reliability, code pass counts estimate code candidate quality, and output-consensus clusters estimate final-selection confidence.  Theory and before/after self-play studies show that these signals correlate with true quality and reliably support co-evolution, offering practical diagnostics for future GT-free code generation methods.
\end{enumerate}

\begin{figure*}[!t]
    \centering    \includegraphics[width=\textwidth]{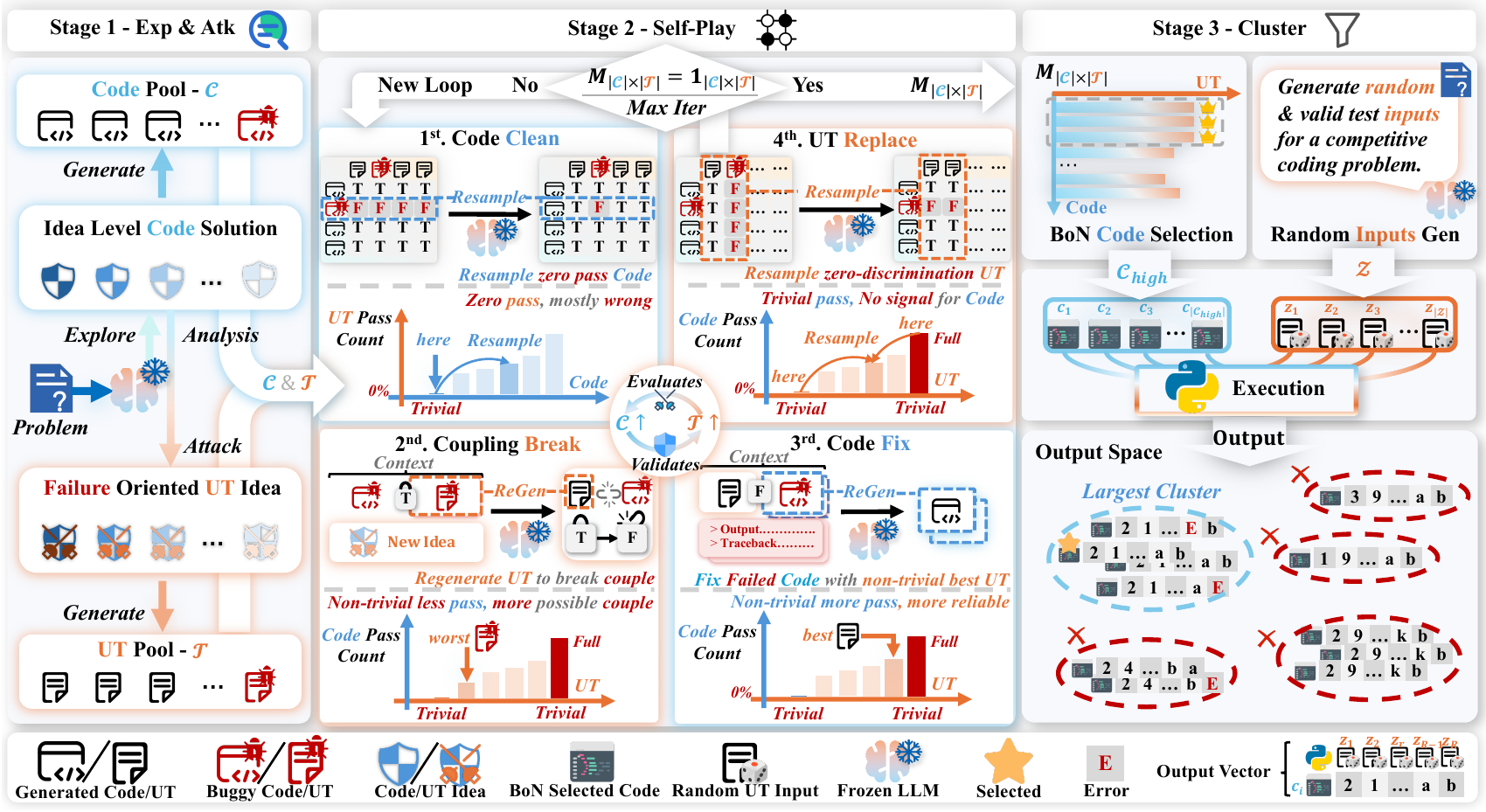}
    \caption{
Method Overview.
Given a coding problem, CoSPlay first explores solution-oriented code ideas and derives failure-oriented UT ideas from them, guiding the resulting UTs to target plausible failure modes and better distinguish among codes.
It then builds a Code-UT execution matrix, whose pass-count statistics measure how many UTs each code passes and how many codes each UT accepts. These bidirectional signals guide code cleaning, spurious Code-UT coupling breaking, code refinement, and UT refreshing, enabling Code-UT co-evolution.
Finally, when UTs cannot distinguish highest-scoring codes, CoSPlay evaluates them on random valid inputs and clusters their output signatures. Since correct codes should agree while incorrect codes tend to diverge, this output-consensus structure provides a reliable signal for final code selection.
}
    \vspace{-1.5em}
    \label{fig:methodology}
\end{figure*}

\section{Methodology}
\vspace{-0.5em}
\subsection{Exploration-Attack-Guided Code and UT Idea Generation}

\label{sec:idea_exploration}

Directly prompting the LLM to generate UTs often yields generic or superficial cases, because the tests are not grounded in the way code candidates may fail~\cite{bekmyradov2026llmstakingshortcutstest}.  Discriminative UTs instead target solution-specific assumptions: for example, duplicate elements can expose solutions that incorrectly assume all values are unique. Motivated by PlanSearch~\cite{wang2025plansearch}, we therefore explore plausible solution strategies in natural language and derive corresponding targeted UT ideas from their potential failure modes.
%
Let $\mathcal{P}$ denote the problem description. The exploration proceeds in two steps:
\textbf{(i) Solution exploration.}
We first prompt the LLM to produce a set of high-level algorithmic hints $\mathcal{H}=\{h_i\}_{i=1}^{n}$, covering likely strategies, data structures, and edge cases implied by $\mathcal{P}$. 
These hints are not intended to be complete solutions, but to span diverse solution directions. 
We then construct a subset space 
$\Omega=\{\omega\subseteq\mathcal{H}\mid |\omega|\in\{1,2\}\}$ and prompt the LLM to expand each hint subset $\omega\in\Omega$ into a set of detailed 
natural-language solution plans $\mathcal{S}_\omega$, yielding the solution-plan pool $\mathcal{S}=\bigcup_{\omega\in\Omega}\mathcal{S}_\omega$.
Limiting the subset size to 1 or 2 balances solution diversity against combinatorial explosion. \textbf{(ii) Failure-oriented UT attack idea exploration.}
For each plan $s\in\mathcal{S}$, we ask the LLM to \textit{identify potential failure modes, hidden edge cases, and implementation pitfalls.} 
This yields a UT attack idea set $\mathcal{A}_s$, and we define the final UT attack idea pool as $\mathcal{A}=\bigcup_{s\in\mathcal{S}}\mathcal{A}_s$.

\subsection{Code and UT Initialization}
\vspace{-0.5em}

\textbf{Pool Construction.}
Starting from the solution-plan pool $\mathcal{S}$ and the UT attack-idea pool $\mathcal{A}$, we build the initial code and UT pools for self-play.
For the code pool $\mathcal{C} = \{c_i\}_{i=1}^{N_c}$, we shuffle $\mathcal{S}$ and use each selected plan once to prompt the LLM for one candidate $c_i$.
For the UT pool $\mathcal{T} = \{t_j\}_{j=1}^{N_t}$, candidate inputs come from two sources.
The random source asks the LLM to produce a valid input directly from the problem statement.
The attack source shuffles $\mathcal{A}$ and asks the LLM to instantiate each selected idea as a concrete valid input.
For each input $x_j$, we ask the LLM to solve it multiple times; if the sampled outputs mostly agree, we take the agreed answer as $y_j$ and add $t_j=(x_j,y_j)$ to $\mathcal{T}$.

\textbf{Intuition of UT Sampling Strategy.}
To start self-play and co-evolve the two pools, we need an advantage signal that can judge the quality of generated samples. Specifically, correct code should be more likely to pass generated UTs than wrong code, and correct UTs should be more likely to be passed by generated code than wrong-output UTs. Appendix~\ref{app:theory:initial-matrix} derives the quality threshold required for each advantage to hold, with a looser code quality threshold for judging UTs and a stricter UT quality threshold for judging code.
To meet this stricter UT quality requirement, we apply output self-consistency: a candidate input is kept only when multiple sampled outputs agree on the same answer, which raises the correctness prior of the initial UT pool.
However, this filter is uneven across sources. Attack-derived inputs from $\mathcal{A}$ are highly discriminative because they target plausible failure modes, but their expected outputs are also harder for the LLM to infer~\cite{prasad2025utgen}, so self-consistency may discard some otherwise useful probes before self-play begins.
We therefore supplement the pool with random valid inputs sampled directly from the problem statement, which provide broader sanity checks under the same self-consistency criterion.

\begin{figure*}[!t] \centering \setlength{\tabcolsep}{0pt} \begin{subfigure}[t]{0.245\textwidth} \centering \includegraphics[width=\linewidth]{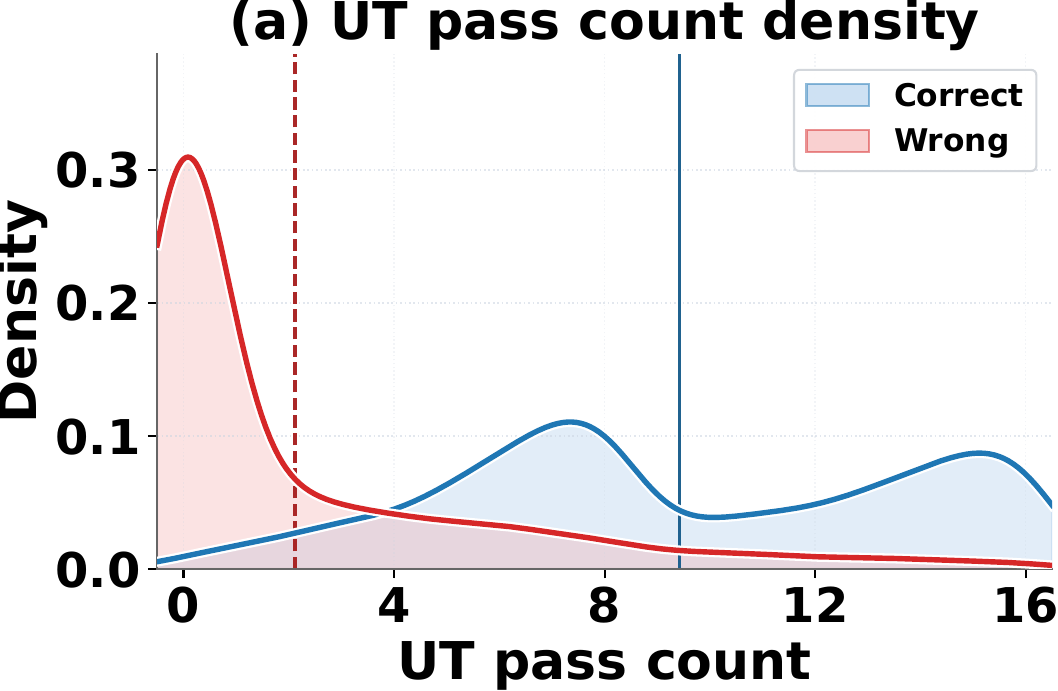} 
\label{fig:round00_panel_ii} 
\end{subfigure} \hfill 
\begin{subfigure}[t]{0.245\textwidth} \centering \includegraphics[width=\linewidth]{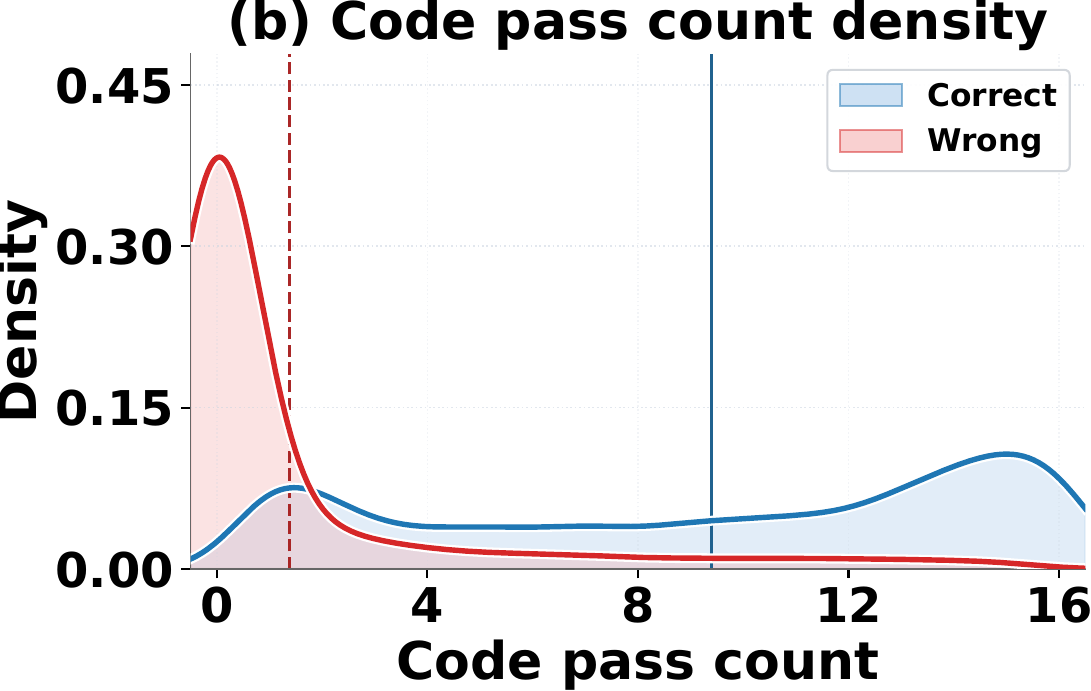} 
\label{fig:round00_panel_iii} \end{subfigure} \hfill \begin{subfigure}[t]{0.245\textwidth} \centering \includegraphics[width=\linewidth]{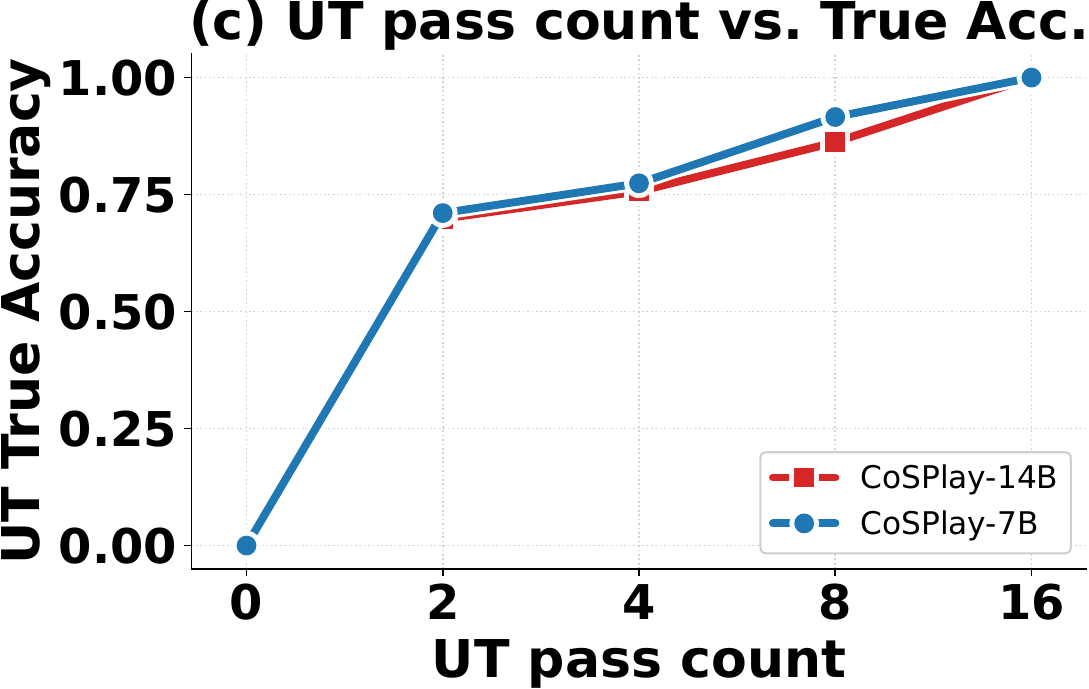} 
\label{fig:round00_panel_v} \end{subfigure} \hfill 
\begin{subfigure}[t]{0.245\textwidth} \centering \includegraphics[width=\linewidth]{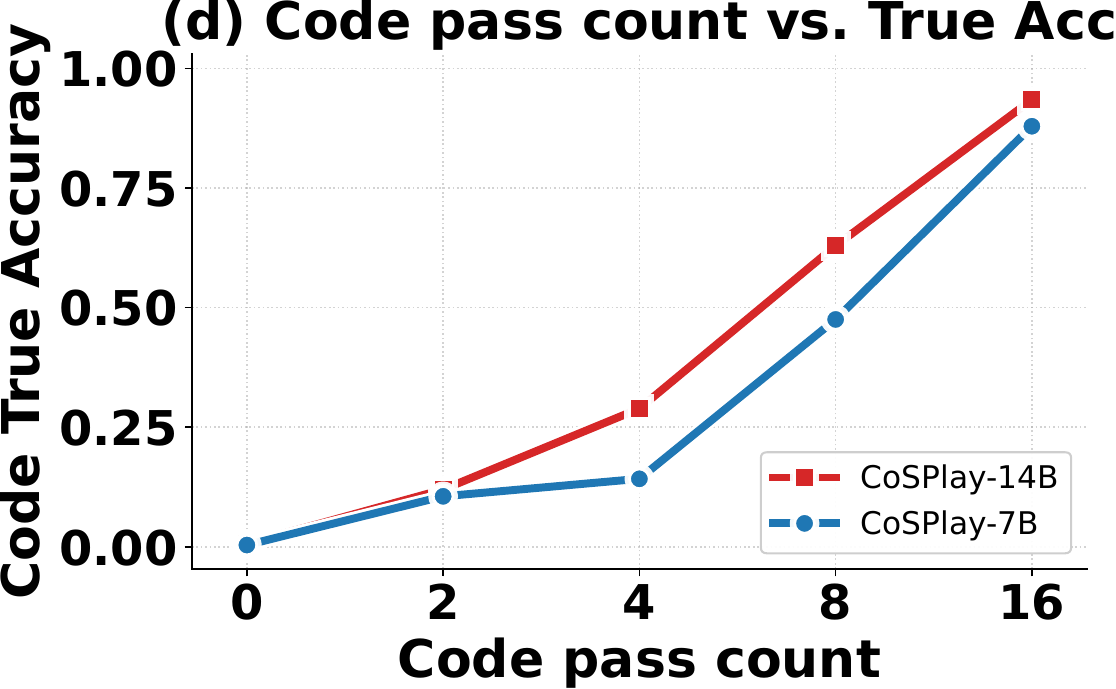} 
\label{fig:round00_panel_vi} \end{subfigure} 
\vspace{-1.5em}
\caption{ \textbf{Round-0 pass-count analysis.} Panels (a-b) show the density distributions of UT and code pass counts for correct and wrong candidates, while panels (c-d) show GT correctness as a function of pass count. 
} 
\label{fig:round00_original_panels_i_to_vi}
\vspace{-1.5em}
\end{figure*}

\vspace{-1.0em}
\subsection{Execution-Matrix-Driven Iterative Self-Play}
\label{sec:self_play_framework}
\vspace{-0.5em}


\textbf{Setup.} Given the execution matrix $\mathbf{M}\in\{0,1\}^{N_c\times N_t}$, where
$\mathbf{M}_{ij}=1$ indicates that code $c_i$ passes UT $t_j$, we define the UT pass
count $p^{\mathrm{UT}}_j$ as the number of code candidates that pass UT $t_j$, and the code pass count $p^{\mathrm{code}}_i$ as the number of UTs passed by code candidate $c_i$. We also define their normalized pass rates
$\alpha^{\mathrm{UT}}_j$ and $\alpha^{\mathrm{code}}_i$:
\begin{equation}
p^{\mathrm{UT}}_j = \sum_{i=1}^{N_c}\mathbf{M}_{ij}, 
\quad p^{\mathrm{code}}_i =\sum_{j=1}^{N_t}\mathbf{M}_{ij}, 
\quad \alpha^{\mathrm{UT}}_j = p^{\mathrm{UT}}_j/{N_c},
\quad \alpha^{\mathrm{code}}_i = p^{\mathrm{code}}_i/N_t,
\end{equation}

\textbf{Pass-count Signal.}
The initial codes/UTs are produced by independent sampling. For ease of analysis, we therefore follow prior work that analyzes pass-count-based signals under idealized i.i.d. assumptions~\cite{chen2022codet,huang2024mpsc,coderm,wang2025cure}. Theorem~\ref{thm:count-signal} shows that, when the advantage thresholds in Eq.~\ref{equ:Thredhold} hold, i.e., correct codes are more likely to pass generated UTs and correct UTs are more likely to be passed by generated codes, both sides have a positive posterior direction: \textit{higher code/UT pass counts imply a larger posterior probability of correctness}.
Our empirical results in Figure~\ref{fig:round00_original_panels_i_to_vi} further support this correlation in realistic Round-0 settings before self-play: correct codes/UTs concentrate more in high-pass-count regions, and their GT correctness increases with pass count. Thus, we use pass counts to guide execution-matrix diagnosis.

\textbf{Self-play Pipeline.} Each play iteration alternates over four steps: pruning all-failing codes, refreshing lowest-support non-trivial UTs to break spurious Code-UT coupling, refining codes with highest-support non-trivial UTs, and replacing zero-discrimination UTs. 
After each step, we re-execute all Code-UT pairs to refresh $\mathbf{M}$, since updating one pool can change the diagnostic status of the other. 
A step is skipped if no target code/UT exists, and the loop terminates when all entries of $\mathbf{M}$ become $1$ or the iteration budget $T_{\max}$ is reached.

\textbf{Step \ding{172}: Code Cleaning.}
All-failing codes and zero-support UTs carry different information.
\textit{\textbf{If each validated UT has probability $\rho_T>0$ of having the correct output, then a code with $\alpha^{\mathrm{code}}_i=0$ fails at least one correct UT with probability $1-(1-\rho_T)^{N_t}$, which approaches one exponentially in $N_t$.}}
Output self-consistency raises this UT prior, making the judgment more reliable.
Such code is therefore likely incorrect and provides little signal for improving the UT pool~\cite{coderepair}.
By contrast, a zero-support UT may still be a valid hard probe that exposes a shared bug in the current code pool, so deleting it may discard useful tests.
We therefore first remove all-failing codes and replace them with newly generated ones:
\begin{equation}
\mathcal{C} \leftarrow (\mathcal{C} \setminus \mathcal{C}_{\text{fail}}) \cup \mathcal{C}_{\text{new}},
\end{equation}
where $\mathcal{C}_{\text{fail}} = \{ c_i \in \mathcal{C} \mathrel{|} \alpha^{\mathrm{code}}_i = 0 \}$, and $\mathcal{C}_{\text{new}}$ denotes  new codes given a solution idea from $\mathcal{S}$, with $|\mathcal{C}_{\text{new}}| = |\mathcal{C}_{\text{fail}}|$.


\textbf{Step \ding{173}: Breaking Spurious Code-UT Coupling.}
After code cleaning, we turn to diagnose the UT pool before using it for refinement. 
A UT with a low but non-zero pass rate (i.e. non-trivial worst) is risky: unlike a zero-pass UT, which rejects all codes universally, \textit{\textbf{it is supported by only a small subset of codes and may therefore reflect spurious coupling with a specific wrong logic}}. 
If such a UT is chosen for code refinement, it can propagate the minority behavior to other codes. 
The goal is not to discard low-support UTs indiscriminately, but to \textbf{\textit{break the suspicious agreement between a minority code behavior and a potentially coupled UT.}}
We therefore refresh this UT by providing the LLM with the original UT and the code it spuriously passed, instructing it to generate a new UT given $\mathcal{A}$ while avoiding the same suspicious agreement pattern.
\begin{equation}
\mathcal{T} \leftarrow 
(\mathcal{T} \setminus \mathcal{T}_{j_{\mathrm{worst}}}) 
\cup 
\mathcal{T}'_{j_{\mathrm{worst}}},
\end{equation}
where $j_{\mathrm{worst}} = \arg\min_j \{ \alpha^{{\mathrm{UT}}}_j \mid 0 < \alpha^{{\mathrm{UT}}}_j < 1 \}$ denotes the selected UT index, $\mathcal{T}_{j_{\mathrm{worst}}}$ is the original UT, and $\mathcal{T}'_{j_{\mathrm{worst}}}$ is the regenerated UT given an attack idea from $\mathcal{A}$ with $\mathcal{T}_{j_{\mathrm{worst}}}$ and its passed codes. We do not assume the selected UT is necessarily wrong; it serves as a practical proxy for the UT most likely to be spuriously coupled.

\textbf{Step \ding{174}: Code Fixing.}
After code cleaning and replacing the most suspicious UT, we turn to \textbf{\textit{refine buggy codes using a trustworthy yet discriminative UT}}. 
A UT with $\alpha^{{\mathrm{UT}}}_j=1$ may be trustworthy enough but provides no refinement signal. 
We therefore \textit{\textbf{select the UT with the highest pass rate but strictly below $1$}}, i.e., the \textit{non-trivial best UT} which is reliable enough and can still distinguish buggy codes.
For each code failing this UT, we provide the LLM with the buggy code, the UT itself, and the actual output of the code to guide refinement:
\begin{equation}
\mathcal{C} 
\leftarrow 
(\mathcal{C} \setminus \mathcal{C}_{\mathrm{target}}) 
\cup \mathcal{C}_{\mathrm{fixed}},
\end{equation}
where
$j_{\mathrm{best}} = \arg\max_j \{ \alpha^{\mathrm{UT}}_j \mid 0 < \alpha^{\mathrm{UT}}_j < 1 \}$
denotes the index of selected high-support yet non-trivial UT.
$\mathcal{C}_{\mathrm{target}} 
= \{ c_i \in \mathcal{C} \mid \mathbf{M}_{i,j_{\mathrm{best}}} = 0 \}$
denotes the code candidates that fail the selected UT $t_{j_{\mathrm{best}}}$ and are therefore targeted for refinement.
$\mathcal{C}_{\mathrm{fixed}}$ denotes the refined versions of $\mathcal{C}_{\mathrm{target}}$, generated using $t_{j_{\mathrm{best}}}$ and the corresponding execution feedback.

\textbf{Step \ding{175}: Replacing Zero-Discrimination UTs.}
Because code refinement improves the code pool, \textit{\textbf{previously discriminative UTs may become trivial under the updated pool}}: they are either passed by all candidates ($\alpha^{{\mathrm{UT}}}_j = 1$) or by none ($\alpha^{{\mathrm{UT}}}_j = 0$). 
All-pass UTs are too easy for the current code pool, whereas a zero-pass-rate UT is more likely to be noisy or mismatched to the current code pool~\cite{ song2024effectivelargelanguagemodel,xie2025ctrl}.
Since such UTs no longer provide useful refinement or selection signals~\cite{repairr1, coderm}, we replace them with newly generated UTs from $\mathcal{A}$:
\begin{equation}
\mathcal{T} \leftarrow (\mathcal{T} \setminus \mathcal{T}_{\text{triv}}) \cup \mathcal{T}_{\text{new}},
\end{equation}
where $\mathcal{T}_{\text{triv}} = \{(x_j, y_j) \in \mathcal{T} \mid \alpha^{{\mathrm{UT}}}_j \in \{0,1\}\}$ denotes the set of zero-discrimination UTs, and $\mathcal{T}_{\text{new}}$ is the set of newly generated UTs. We maintain a constant UT pool size by enforcing $|\mathcal{T}_{\text{new}}| = |\mathcal{T}_{\text{triv}}|$.


\vspace{-1em}

\subsection{Output-Consensus-Based Cluster Selection}
\label{sec:consensus_selection}
\vspace{-0.5em}

After self-play, we score each code by its final pass count 
$p_i^{\mathrm{code}}$ and retain the BoN-tied top subset
$\mathcal{C}_{\mathrm{high}}
=
\left\{
c_i \in \mathcal{C}
\mid
p_i^{\mathrm{code}}=\max_{\ell} p_{\ell}^{\mathrm{code}}
\right\}.
$
Because discrete pass counts can leave multiple top-scoring candidates, BoN alone cannot break ties. We thus apply an \emph{unsupervised execution-consensus clustering}: we prompt the LLM to generate $R$ \textit{random and valid} inputs $\mathcal{Z}=\{z_r\}_{r=1}^{R}$ which do not require outputs, and each $c_i\in\mathcal{C}_{\mathrm{high}}$ is executed on $\mathcal{Z}$. With $o_i(z_r)$ denoting the observed output of $c_i$ on $z_r$, each code induces the output signature
$\sigma(c_i):=(o_i(z_1),\ldots,o_i(z_R))$.

\textbf{Theoretical Motivation.}
Correct code candidates should share the same output signature, while incorrect code candidates typically split across different behaviors. As the number of random inputs increases, matching any fixed wrong signature becomes exponentially harder, whereas all correct candidates keep sharing the true signature. Thus, the true signature becomes the dominant population mode, and the empirical largest-cluster rule converges to it as the independently sampled candidate pool grows. Proposition~\ref{prop:signature-mode-separation} formalizes this statement, with the proof given in Appendix~\ref{app:theory}.

\textbf{Clustering with Execution Errors.}
In practice, some executions may fail to produce a normal output, we mark these cases as $o_i(z_r)=\mathrm{ERR}$.
We first define a pairwise observed-compatibility relation between codes.
Two candidates $c_i,c_j\in\mathcal{C}_{\mathrm{high}}$ are observed-compatible, denoted $c_i\sim_{\mathrm{obs}}c_j$, if they
produce the same valid output on every random input:
\begin{equation}
c_i \sim_{\mathrm{obs}} c_j
\iff
o_i(z_r)=o_j(z_r),\quad
\forall r\in\{1,\ldots,R\}:\ o_i(z_r)\neq\mathrm{ERR} \wedge \ o_j(z_r)\neq\mathrm{ERR}.
\end{equation}
Execution errors are treated as missing evidence rather than output conflicts.
Then we build output-consensus clusters over $\mathcal{C}_{\mathrm{high}}$.
Let $\mathcal{G}=\{G_m\}_{m=1}^{M}$ denote the resulting collection, where $M$ is the number of clusters produced and each $G_m\subseteq\mathcal{C}_{\mathrm{high}}$ contains code candidates with no observed output conflict.
Because this compatibility relation may not be transitive when $\mathrm{ERR}$
are ignored, it does not directly induce a unique partition. We therefore construct conservative pairwise-compatible clusters by processing
candidates in a fixed deterministic order. Specifically, when processing a candidate $c_i$, we add it to an
existing cluster $G_m$ only if it is compatible with every member of that
cluster, i.e. $\forall c_j\in G_m,\ c_i\sim_{\mathrm{obs}}c_j.$
If no such cluster exists, $c_i$ starts a new cluster.

\textbf{Reliability Scoring and Selection.}
In the error-free case, selecting the largest cluster is sufficient. With runtime errors, however, candidates with many failed executions may appear compatible due to few valid comparisons. We therefore score clusters by ordered valid pairwise agreements. For $c_i,c_j\in G_m$ and input $z_k$, define
\setlength{\abovedisplayskip}{4pt}
\setlength{\belowdisplayskip}{4pt}
\begin{equation}
a_{ij}^{(k)}
=
\mathbf{1}\!\left\{
o_i(z_k)\neq\mathrm{ERR},\
o_j(z_k)\neq\mathrm{ERR},\
o_i(z_k)=o_j(z_k)
\right\}.
\end{equation}
We define the individual reliability score $S_{\mathrm{ind}}(c_i)$ and the cluster reliability score $S_{\mathrm{cls}}(G_m)$ as
\setlength{\abovedisplayskip}{4pt}
\setlength{\belowdisplayskip}{4pt}
\begin{equation}
S_{\mathrm{ind}}(c_i) = \sum_{c_j \in G_m \setminus \{c_i\}} \sum_{k=1}^{R} a_{ij}^{(k)}, 
\qquad
S_{\mathrm{cls}}(G_m)=\sum_{c_i\in G_m}S_{\mathrm{ind}}(c_i).
\end{equation}
Based on the cluster-level and individual-level scores, we first select the most reliable cluster
$G^*=\arg\max_{G_m\in\mathcal{G}}S_{\mathrm{cls}}(G_m)$,
and then choose the most reliable code $c^*$ within it:
$c^*=\arg\max_{c_i\in G^*}S_{\mathrm{ind}}(c_i)$.
When all executions are error-free, every within-cluster pair agrees on all
$R$ inputs, yielding
$S_{\mathrm{cls}}(G_m)=|G_m|(|G_m|-1)R$.
Thus, with fixed $R$, maximizing $S_{\mathrm{cls}}$ is equivalent to selecting the largest cluster. With runtime errors, the same score remains more reliable than raw cluster size because it rewards only ordered valid pairwise agreements.

\providecommand{\deltaqwenlabel}{\textcolor{black!65}{$\Delta$ vs.\ Qwen-Ins.}}
\providecommand{\deltacurelabel}{\textcolor{black!65}{$\Delta$ vs.\ CURE}}
\providecommand{\deltaval}[1]{\textcolor{green!45!black}{\textbf{#1}}}
\providecommand{\deltaneg}[1]{\textcolor{red!60!black}{\textbf{#1}}}
\providecommand{\deltadash}{\textcolor{black!35}{--}}

\begingroup
\setlength{\dbltextfloatsep}{4pt plus 1pt minus 1pt}
\setlength{\dblfloatsep}{4pt plus 1pt minus 1pt}

\begin{table*}[!t]
    \centering
    \scriptsize
    \setlength{\tabcolsep}{0.5pt}
    \renewcommand{\arraystretch}{1.15}
    \caption{Performance comparison between \textbf{CoSPlay} and \textbf{RLVR models}. 
\textbf{Base} is the backbone variant. 
\textbf{\#Data} indicates the number of GT data used for training. 
\textbf{Signal} (Sig.): BoN accuracy of generated UTs selecting code from Qwen2.5-7B-Ins., which measures the discrimination ability of UTs. 
\textbf{UT} \& \textbf{Code}: average accuracy of self-generated UTs and codes, respectively. 
\textbf{BoN}: Best-of-16 accuracy using self-generated codes and UTs. 
\textbf{Average} reports the dataset-weighted average across benchmarks. 
Colors (\textcolor{Red}{red}/\textcolor{Green}{blue}) indicate values above/below the specific column mean; shades use a rank-stepped scale within each column, so the top ranks are visually separated even when their values are close. 
$\Delta$: absolute gain of \textbf{CoSPlay w/ Cluster} over the corresponding base model.}
    \label{tab:overall}

    \resizebox{\textwidth}{!}{
    \begin{tabular}{l||c|c||cccc||cccc||cccc||cccc||cccc}
        \toprule
        \multirow{2}{*}{\textbf{Model}} &
        \multirow{2}{*}{\textbf{Base}} &
        \multirow{2}{*}{\textbf{\#Data}} &
        \multicolumn{4}{c||}{\textbf{LiveBench}} &
        \multicolumn{4}{c||}{\textbf{LiveCodeBench}} &
        \multicolumn{4}{c||}{\textbf{CodeContests}} &
        \multicolumn{4}{c||}{\textbf{CodeForces}} &
        \multicolumn{4}{c}{\textbf{Average}} \\

         & & & \textbf{Sig.} & \textbf{UT} & \textbf{Code} & \textbf{BoN}
         & \textbf{Sig.} & \textbf{UT} & \textbf{Code} & \textbf{BoN}
         & \textbf{Sig.} & \textbf{UT} & \textbf{Code} & \textbf{BoN}
         & \textbf{Sig.} & \textbf{UT} & \textbf{Code} & \textbf{BoN}
         & \textbf{Sig.} & \textbf{UT} & \textbf{Code} & \textbf{BoN} \\

        \hline
        \hline
        \rowcolor{gray!15} \multicolumn{23}{c}{\textbf{7B Base Models}} \\ \hline
        Qwen2.5-7B-Instruct & Ins. & - & \cellcolor{Green!92!white}36.2 & \cellcolor{Green!92!white}9.0 & \cellcolor{Green!62!white}33.1 & \cellcolor{Green!62!white}36.2 & \cellcolor{Green!92!white}31.2 & \cellcolor{Green!76!white}14.4 & \cellcolor{Green!62!white}27.0 & \cellcolor{Green!62!white}31.2 & \cellcolor{Green!76!white}24.1 & \cellcolor{Green!52!white}21.3 & \cellcolor{Red!52!white}21.5 & \cellcolor{Green!62!white}24.1 & \cellcolor{Green!44!white}7.1 & \cellcolor{Green!44!white}12.9 & \cellcolor{Green!76!white}5.0 & \cellcolor{Green!76!white}7.1 & \cellcolor{Green!92!white}22.1 & \cellcolor{Green!52!white}14.6 & \cellcolor{Green!62!white}19.0 & \cellcolor{Green!62!white}22.1 \\
        \quad + CodeT & Ins. & - & \cellcolor{Green!76!white}37.5 & \cellcolor{Green!92!white}9.0 & \cellcolor{Green!62!white}33.1 & \cellcolor{Green!52!white}37.5 & \cellcolor{Green!76!white}32.3 & \cellcolor{Green!76!white}14.4 & \cellcolor{Green!62!white}27.0 & \cellcolor{Green!52!white}32.3 & \cellcolor{Green!52!white}25.0 & \cellcolor{Green!52!white}21.3 & \cellcolor{Red!52!white}21.5 & \cellcolor{Green!44!white}25.0 & \cellcolor{Green!37!white}7.6 & \cellcolor{Green!44!white}12.9 & \cellcolor{Green!76!white}5.0 & \cellcolor{Green!62!white}7.6 & \cellcolor{Green!76!white}22.9 & \cellcolor{Green!52!white}14.6 & \cellcolor{Green!62!white}19.0 & \cellcolor{Green!52!white}22.9 \\
        \quad + Cluster & Ins. & - & \cellcolor{Green!52!white}42.7 & \cellcolor{Green!92!white}9.0 & \cellcolor{Green!62!white}33.1 & \cellcolor{Green!44!white}42.7 & \cellcolor{Green!37!white}35.2 & \cellcolor{Green!76!white}14.4 & \cellcolor{Green!62!white}27.0 & \cellcolor{Green!37!white}35.2 & \cellcolor{Green!44!white}25.1 & \cellcolor{Green!52!white}21.3 & \cellcolor{Red!52!white}21.5 & \cellcolor{Green!37!white}25.1 & \cellcolor{Green!44!white}7.1 & \cellcolor{Green!44!white}12.9 & \cellcolor{Green!76!white}5.0 & \cellcolor{Green!76!white}7.1 & \cellcolor{Green!37!white}24.4 & \cellcolor{Green!52!white}14.6 & \cellcolor{Green!62!white}19.0 & \cellcolor{Green!44!white}24.4 \\
        Qwen2.5-7B-Coder & Coder & - & \cellcolor{Green!62!white}40.6 & \cellcolor{Green!76!white}12.1 & \cellcolor{Green!92!white}3.8 & \cellcolor{Green!92!white}19.0 & \cellcolor{Green!62!white}33.4 & \cellcolor{Green!62!white}16.0 & \cellcolor{Green!92!white}4.1 & \cellcolor{Green!92!white}17.7 & \cellcolor{Green!37!white}25.2 & \cellcolor{Green!52!white}21.3 & \cellcolor{Green!92!white}2.6 & \cellcolor{Green!92!white}12.8 & \cellcolor{Green!62!white}6.8 & \cellcolor{Green!52!white}6.7 & \cellcolor{Green!92!white}1.8 & \cellcolor{Green!92!white}3.5 & \cellcolor{Green!62!white}23.4 & \cellcolor{Green!76!white}13.3 & \cellcolor{Green!92!white}3.0 & \cellcolor{Green!92!white}12.0 \\
        Qwen2.5-7B-Coder-Ins. & Coder-Ins. & - & \cellcolor{Green!52!white}42.7 & \cellcolor{Green!62!white}12.5 & \cellcolor{Red!37!white}36.3 & \cellcolor{Red!31!white}47.7 & \cellcolor{Green!52!white}33.7 & \cellcolor{Green!92!white}13.4 & \cellcolor{Red!31!white}28.8 & \cellcolor{Green!44!white}34.4 & \cellcolor{Green!92!white}24.0 & \cellcolor{Green!92!white}12.7 & \cellcolor{Green!37!white}20.9 & \cellcolor{Green!76!white}23.2 & \cellcolor{Green!52!white}6.9 & \cellcolor{Green!92!white}5.3 & \cellcolor{Green!44!white}5.4 & \cellcolor{Green!37!white}7.9 & \cellcolor{Green!52!white}23.5 & \cellcolor{Green!92!white}10.4 & \cellcolor{Red!31!white}20.0 & \cellcolor{Green!37!white}24.5 \\
        Seed-Coder-8B-Instruct & - & - & \cellcolor{Green!37!white}43.5 & \cellcolor{Green!31!white}26.9 & \cellcolor{Red!62!white}37.6 & \cellcolor{Red!21!white}45.1 & \cellcolor{Green!31!white}36.1 & \cellcolor{Green!37!white}30.0 & \cellcolor{Red!62!white}32.5 & \cellcolor{Red!37!white}40.8 & \cellcolor{Green!52!white}25.0 & \cellcolor{Green!44!white}31.8 & \cellcolor{Green!62!white}18.3 & \cellcolor{Green!31!white}26.9 & \cellcolor{Green!31!white}8.4 & \cellcolor{Green!31!white}17.0 & \cellcolor{Red!52!white}6.4 & \cellcolor{Green!26!white}10.5 & \cellcolor{Green!31!white}25.2 & \cellcolor{Green!37!white}25.5 & \cellcolor{Red!52!white}21.4 & \cellcolor{Red!37!white}28.2 \\

        \hline \rowcolor{gray!5} \multicolumn{23}{c}{RL Methods} \\ \hline

        AceCoder-7B-Rule & Coder-Ins. & 22k & \cellcolor{Green!37!white}43.5 & \cellcolor{Green!52!white}17.2 & \cellcolor{Red!31!white}36.1 & \cellcolor{Red!26!white}45.8 & \cellcolor{Green!44!white}34.1 & \cellcolor{Green!52!white}17.5 & \cellcolor{Red!44!white}30.5 & \cellcolor{Green!31!white}35.9 & \cellcolor{Green!62!white}24.4 & \cellcolor{Green!76!white}18.8 & \cellcolor{Green!52!white}20.1 & \cellcolor{Green!76!white}23.2 & \cellcolor{Green!92!white}6.4 & \cellcolor{Green!76!white}5.9 & \cellcolor{Green!44!white}5.4 & \cellcolor{Green!52!white}7.7 & \cellcolor{Green!44!white}23.6 & \cellcolor{Green!62!white}13.7 & \cellcolor{Red!37!white}20.5 & \cellcolor{Green!31!white}24.8 \\
        AceCoder-7B-RM & Coder-Ins. & 329k & \cellcolor{Green!44!white}43.0 & \cellcolor{Green!44!white}17.7 & \cellcolor{Red!52!white}37.5 & \cellcolor{Red!37!white}48.4 & \cellcolor{Green!52!white}33.7 & \cellcolor{Green!44!white}19.9 & \cellcolor{Red!37!white}30.3 & \cellcolor{Green!26!white}37.1 & \cellcolor{Green!52!white}25.0 & \cellcolor{Green!62!white}20.0 & \cellcolor{Green!44!white}20.6 & \cellcolor{Green!44!white}25.0 & \cellcolor{Green!76!white}6.6 & \cellcolor{Green!62!white}6.6 & \cellcolor{Green!62!white}5.2 & \cellcolor{Green!44!white}7.8 & \cellcolor{Green!44!white}23.6 & \cellcolor{Green!44!white}15.1 & \cellcolor{Red!44!white}20.6 & \cellcolor{Green!26!white}25.8 \\
        AZR-7B-Coder & Coder & \cellcolor{mycolor_1}\textbf{0} & \cellcolor{Green!31!white}44.5 & \cellcolor{Green!37!white}23.1 & \cellcolor{Green!76!white}12.5 & \cellcolor{Green!76!white}34.9 & \cellcolor{Green!26!white}37.1 & \cellcolor{Green!31!white}31.7 & \cellcolor{Green!76!white}12.4 & \cellcolor{Green!76!white}29.8 & \cellcolor{Red!37!white}29.1 & \cellcolor{Green!37!white}39.0 & \cellcolor{Green!76!white}13.7 & \cellcolor{Green!52!white}24.4 & \cellcolor{Green!37!white}7.6 & \cellcolor{Green!37!white}16.9 & \cellcolor{Green!52!white}5.3 & \cellcolor{Green!31!white}8.4 & \cellcolor{Green!26!white}26.1 & \cellcolor{Green!31!white}27.0 & \cellcolor{Green!76!white}10.2 & \cellcolor{Green!76!white}21.9 \\
        CURE-7B & Ins. & 4.5k & \cellcolor{Red!76!white}52.3 & \cellcolor{Red!62!white}54.6 & \cellcolor{Red!44!white}37.4 & \cellcolor{Red!62!white}51.6 & \cellcolor{Red!44!white}42.1 & \cellcolor{Red!62!white}58.4 & \cellcolor{Red!52!white}31.2 & \cellcolor{Red!52!white}42.7 & \cellcolor{Red!62!white}33.9 & \cellcolor{Red!62!white}63.7 & \cellcolor{Red!62!white}25.5 & \cellcolor{Red!52!white}34.5 & \cellcolor{Red!44!white}11.5 & \cellcolor{Red!62!white}46.6 & \cellcolor{Red!76!white}7.5 & \cellcolor{Red!62!white}16.1 & \cellcolor{Red!44!white}31.0 & \cellcolor{Red!76!white}54.9 & \cellcolor{Red!62!white}22.5 & \cellcolor{Red!52!white}32.9 \\


\hline \rowcolor{gray!5} \multicolumn{23}{c}{Our Method} \\ \hline
        \textbf{Qwen2.5-7B-Ins. + CoSPlay} & Ins. & \cellcolor{mycolor_1}\textbf{0} & \cellcolor{Red!44!white}47.4 & \cellcolor{Red!95!white}66.7 & \cellcolor{Red!76!white}38.8 & \cellcolor{Red!52!white}50.0 & \cellcolor{Red!52!white}43.2 & \cellcolor{Red!95!white}77.0 & \cellcolor{Red!76!white}33.2 & \cellcolor{Red!44!white}41.9 & \cellcolor{Red!44!white}32.1 & \cellcolor{Red!76!white}74.4 & \cellcolor{Red!76!white}26.0 & \cellcolor{Red!44!white}34.3 & \cellcolor{Red!52!white}13.3 & \cellcolor{Red!76!white}85.0 & \cellcolor{Red!62!white}6.8 & \cellcolor{Red!76!white}16.8 & \cellcolor{Red!52!white}31.3 & \cellcolor{Red!95!white}78.3 & \cellcolor{Red!76!white}23.3 & \cellcolor{Red!44!white}32.6 \\
        \quad w/ Cluster & Ins. & \cellcolor{mycolor_1}\textbf{0} & \cellcolor{Red!52!white}50.3 & \cellcolor{Red!95!white}66.7 & \cellcolor{Red!76!white}38.8 & \cellcolor{Red!44!white}49.7 & \cellcolor{Red!62!white}43.7 & \cellcolor{Red!95!white}77.0 & \cellcolor{Red!76!white}33.2 & \cellcolor{Red!62!white}43.4 & \cellcolor{Red!52!white}32.6 & \cellcolor{Red!76!white}74.4 & \cellcolor{Red!76!white}26.0 & \cellcolor{Red!62!white}34.6 & \cellcolor{Red!62!white}13.6 & \cellcolor{Red!76!white}85.0 & \cellcolor{Red!62!white}6.8 & \cellcolor{Red!76!white}16.8 & \cellcolor{Red!62!white}31.9 & \cellcolor{Red!95!white}78.3 & \cellcolor{Red!76!white}23.3 & \cellcolor{Red!62!white}33.2 \\

        \rowcolor{green!4}
        \textbf{\deltaqwenlabel} & \deltadash & \deltadash
                & \deltaval{+14.1} & {\deltaval{+57.7}} & {\deltaval{+5.7}} & {\deltaval{+13.5}}
                & \deltaval{+12.5} & \deltaval{+62.6} & \deltaval{+6.2} & \deltaval{+12.2}
                & \deltaval{+8.5} & \deltaval{+53.1} & \deltaval{+4.5} & \deltaval{+10.5}
                & \deltaval{+6.5} & \deltaval{+72.1} & \deltaval{+1.8} & \deltaval{+9.7}
                & \deltaval{+9.8} & \deltaval{+63.7} & \deltaval{+4.3} & \deltaval{+11.1} \\
        \addlinespace[1pt] \textbf{CURE-7B + CoSPlay} & Ins. & 4.5k & \cellcolor{Red!62!white}51.3 & \cellcolor{Red!76!white}66.2 & \cellcolor{Red!95!white}45.4 & \cellcolor{Red!76!white}52.1 & \cellcolor{Red!76!white}44.6 & \cellcolor{Red!76!white}75.6 & \cellcolor{Red!95!white}37.4 & \cellcolor{Red!76!white}48.5 & \cellcolor{Red!76!white}35.3 & \cellcolor{Red!95!white}76.6 & \cellcolor{Red!95!white}29.8 & \cellcolor{Red!76!white}39.6 & \cellcolor{Red!76!white}14.2 & \cellcolor{Red!95!white}85.6 & \cellcolor{Red!95!white}11.1 & \cellcolor{Red!95!white}23.1 & \cellcolor{Red!76!white}33.0 & \cellcolor{Red!95!white}78.3 & \cellcolor{Red!95!white}27.7 & \cellcolor{Red!76!white}38.4 \\
        \quad w/ Cluster & Ins. & 4.5k & \cellcolor{Red!95!white}53.6 & \cellcolor{Red!76!white}66.2 & \cellcolor{Red!95!white}45.4 & \cellcolor{Red!95!white}52.3 & \cellcolor{Red!95!white}44.7 & \cellcolor{Red!76!white}75.6 & \cellcolor{Red!95!white}37.4 & \cellcolor{Red!95!white}48.6 & \cellcolor{Red!95!white}35.6 & \cellcolor{Red!95!white}76.6 & \cellcolor{Red!95!white}29.8 & \cellcolor{Red!95!white}40.0 & \cellcolor{Red!95!white}14.4 & \cellcolor{Red!95!white}85.6 & \cellcolor{Red!95!white}11.1 & \cellcolor{Red!95!white}23.1 & \cellcolor{Red!95!white}33.4 & \cellcolor{Red!95!white}78.3 & \cellcolor{Red!95!white}27.7 & \cellcolor{Red!95!white}38.6 \\

                \rowcolor{green!4}
        \textbf{\deltacurelabel} & \deltadash & \deltadash
                & \deltaval{+1.3} & {\deltaval{+11.6}} & {\deltaval{+8.0}} & {\deltaval{+0.7}}
                & \deltaval{+2.6} & \deltaval{+17.2} & \deltaval{+6.2} & \deltaval{+5.9}
                & \deltaval{+1.7} & \deltaval{+12.9} & \deltaval{+4.3} & \deltaval{+5.5}
                & \deltaval{+2.9} & \deltaval{+39.0} & \deltaval{+3.6} & \deltaval{+7.0}
                & \deltaval{+2.4} & \deltaval{+23.4} & \deltaval{+5.2} & \deltaval{+5.7} \\

        \hline
        \hline
        \rowcolor{gray!15} \multicolumn{23}{c}{\textbf{14B Base Models}} \\ \hline
        Qwen2.5-14B-Instruct & Ins. & - & \cellcolor{Green!52!white}47.7 & \cellcolor{Green!62!white}26.4 & \cellcolor{Green!52!white}39.2 & \cellcolor{Green!52!white}51.0 & \cellcolor{Green!44!white}41.4 & \cellcolor{Green!62!white}38.0 & \cellcolor{Green!62!white}34.5 & \cellcolor{Green!52!white}46.2 & \cellcolor{Green!44!white}30.4 & \cellcolor{Green!52!white}43.2 & \cellcolor{Green!62!white}25.1 & \cellcolor{Green!52!white}34.6 & \cellcolor{Green!62!white}9.9 & \cellcolor{Green!52!white}22.6 & \cellcolor{Green!52!white}7.3 & \cellcolor{Green!62!white}13.3 & \cellcolor{Green!52!white}29.1 & \cellcolor{Green!52!white}32.4 & \cellcolor{Green!62!white}23.9 & \cellcolor{Green!44!white}33.2 \\
        \quad + CodeT & Ins. & - & \cellcolor{Green!37!white}49.5 & \cellcolor{Green!62!white}26.4 & \cellcolor{Green!52!white}39.2 & \cellcolor{Red!31!white}52.1 & \cellcolor{Green!76!white}42.2 & \cellcolor{Green!62!white}38.0 & \cellcolor{Green!62!white}34.5 & \cellcolor{Red!31!white}48.4 & \cellcolor{Green!52!white}29.8 & \cellcolor{Green!52!white}43.2 & \cellcolor{Green!62!white}25.1 & \cellcolor{Green!31!white}35.6 & \cellcolor{Green!44!white}10.5 & \cellcolor{Green!52!white}22.6 & \cellcolor{Green!52!white}7.3 & \cellcolor{Green!52!white}13.5 & \cellcolor{Green!44!white}29.7 & \cellcolor{Green!52!white}32.4 & \cellcolor{Green!62!white}23.9 & \cellcolor{Green!31!white}34.3 \\
         \quad + Cluster & Ins. & - & \cellcolor{Green!42!white}48.4 & \cellcolor{Green!62!white}26.4 & \cellcolor{Green!52!white}39.2 & \cellcolor{Red!31!white}52.1 & \cellcolor{Green!76!white}42.2 & \cellcolor{Green!62!white}38.0 & \cellcolor{Green!62!white}34.5 & \cellcolor{Red!26!white}47.4 & \cellcolor{Green!52!white}29.4 & \cellcolor{Green!52!white}43.2 & \cellcolor{Green!62!white}25.1 & \cellcolor{Green!44!white}34.7 & \cellcolor{Green!62!white}9.1 & \cellcolor{Green!52!white}22.6 & \cellcolor{Green!52!white}7.3 & \cellcolor{Green!44!white}13.6 & \cellcolor{Green!52!white}29.0 & \cellcolor{Green!52!white}32.4 & \cellcolor{Green!62!white}23.9 & \cellcolor{Green!37!white}33.8 \\
 
        Qwen2.5-14B-Coder & Coder & - & \cellcolor{Green!76!white}42.2 & \cellcolor{Green!92!white}13.9 & \cellcolor{Green!92!white}10.8 & \cellcolor{Green!92!white}30.2 & \cellcolor{Green!76!white}36.1 & \cellcolor{Green!92!white}20.7 & \cellcolor{Green!92!white}9.7 & \cellcolor{Green!92!white}26.7 & \cellcolor{Green!62!white}28.3 & \cellcolor{Green!92!white}25.6 & \cellcolor{Green!92!white}6.8 & \cellcolor{Green!92!white}21.3 & \cellcolor{Green!52!white}10.2 & \cellcolor{Green!92!white}16.7 & \cellcolor{Green!92!white}2.1 & \cellcolor{Green!92!white}9.3 & \cellcolor{Green!76!white}26.3 & \cellcolor{Green!92!white}19.5 & \cellcolor{Green!92!white}6.7 & \cellcolor{Green!92!white}20.1 \\
        Qwen2.5-14B-Coder-Ins. & Coder-Ins. & - & \cellcolor{Green!44!white}48.2 & \cellcolor{Green!44!white}37.8 & \cellcolor{Red!52!white}47.2 & \cellcolor{Red!37!white}53.9 & \cellcolor{Green!52!white}39.5 & \cellcolor{Green!52!white}40.6 & \cellcolor{Red!52!white}41.0 & \cellcolor{Red!37!white}49.4 & \cellcolor{Green!76!white}27.5 & \cellcolor{Green!76!white}34.8 & \cellcolor{Green!62!white}25.1 & \cellcolor{Green!62!white}31.4 & \cellcolor{Green!92!white}7.9 & \cellcolor{Green!76!white}18.1 & \cellcolor{Green!76!white}5.6 & \cellcolor{Green!76!white}9.6 & \cellcolor{Green!62!white}27.2 & \cellcolor{Green!62!white}31.5 & \cellcolor{Red!52!white}26.5 & \cellcolor{Green!52!white}32.8 \\
        DeepSeek-Coder-V2-16B & - & - & \cellcolor{Green!62!white}45.8 & \cellcolor{Green!52!white}31.0 & \cellcolor{Green!76!white}37.4 & \cellcolor{Green!76!white}38.3 & \cellcolor{Green!62!white}39.3 & \cellcolor{Green!44!white}42.1 & \cellcolor{Green!76!white}25.7 & \cellcolor{Green!76!white}36.7 & \cellcolor{Red!44!white}33.6 & \cellcolor{Green!44!white}51.6 & \cellcolor{Green!76!white}22.1 & \cellcolor{Green!37!white}35.1 & \cellcolor{Red!52!white}13.8 & \cellcolor{Green!44!white}37.2 & \cellcolor{Green!62!white}6.5 & \cellcolor{Green!31!white}16.8 & \cellcolor{Green!37!white}30.1 & \cellcolor{Green!44!white}41.0 & \cellcolor{Green!76!white}19.5 & \cellcolor{Green!76!white}29.6 \\

        \hline \rowcolor{gray!5} \multicolumn{23}{c}{RL Methods} \\ \hline

        AZR-14B-Coder & Coder & \cellcolor{mycolor_1}\textbf{0} & \cellcolor{Green!92!white}39.8 & \cellcolor{Green!76!white}22.2 & \cellcolor{Green!62!white}38.9 & \cellcolor{Green!62!white}43.5 & \cellcolor{Green!92!white}34.1 & \cellcolor{Green!76!white}29.7 & \cellcolor{Green!52!white}35.3 & \cellcolor{Green!62!white}39.8 & \cellcolor{Green!92!white}25.8 & \cellcolor{Green!62!white}35.3 & \cellcolor{Green!52!white}26.4 & \cellcolor{Green!76!white}30.4 & \cellcolor{Green!76!white}8.5 & \cellcolor{Green!62!white}18.9 & \cellcolor{Green!44!white}9.5 & \cellcolor{Green!37!white}16.2 & \cellcolor{Green!92!white}24.3 & \cellcolor{Green!76!white}26.2 & \cellcolor{Green!52!white}25.1 & \cellcolor{Green!62!white}30.3 \\
        CURE-14B & Ins. & 4.5k & \cellcolor{Red!62!white}53.6 & \cellcolor{Red!95!white}77.3 & \cellcolor{Red!62!white}47.7 & \cellcolor{Red!62!white}59.9 & \cellcolor{Red!44!white}45.1 & \cellcolor{Red!95!white}86.1 & \cellcolor{Red!62!white}41.2 & \cellcolor{Red!44!white}50.5 & \cellcolor{Red!95!white}34.9 & \cellcolor{Red!95!white}87.0 & \cellcolor{Red!76!white}32.5 & \cellcolor{Red!62!white}44.1 & \cellcolor{Red!76!white}14.9 & \cellcolor{Red!62!white}77.5 & \cellcolor{Red!76!white}12.0 & \cellcolor{Red!62!white}26.1 & \cellcolor{Red!52!white}33.6 & \cellcolor{Red!95!white}82.4 & \cellcolor{Red!62!white}30.1 & \cellcolor{Red!52!white}41.8 \\
        \hline \rowcolor{gray!5} \multicolumn{23}{c}{Our Method} \\ \hline
        \textbf{Qwen2.5-14B-Ins. + CoSPlay} & Ins. & \cellcolor{mycolor_1}\textbf{0} & \cellcolor{Red!76!white}54.9 & \cellcolor{Red!62!white}70.1 & \cellcolor{Red!76!white}47.9 & \cellcolor{Red!44!white}55.7 & \cellcolor{Red!76!white}46.0 & \cellcolor{Red!62!white}74.4 & \cellcolor{Red!76!white}43.5 & \cellcolor{Red!52!white}53.6 & \cellcolor{Red!62!white}34.4 & \cellcolor{Red!62!white}71.7 & \cellcolor{Red!62!white}31.5 & \cellcolor{Red!44!white}39.1 & \cellcolor{Red!62!white}14.3 & \cellcolor{Red!76!white}86.1 & \cellcolor{Red!62!white}11.8 & \cellcolor{Red!52!white}24.7 & \cellcolor{Red!62!white}33.8 & \cellcolor{Red!62!white}77.6 & \cellcolor{Red!76!white}30.8 & \cellcolor{Red!44!white}41.2 \\
        \quad w/ Cluster & Ins. & \cellcolor{mycolor_1}\textbf{0} & \cellcolor{Red!95!white}56.5 & \cellcolor{Red!62!white}70.1 & \cellcolor{Red!76!white}47.9 & \cellcolor{Red!52!white}57.6 & \cellcolor{Red!95!white}46.2 & \cellcolor{Red!62!white}74.4 & \cellcolor{Red!76!white}43.5 & \cellcolor{Red!62!white}54.8 & \cellcolor{Red!52!white}34.0 & \cellcolor{Red!62!white}71.7 & \cellcolor{Red!62!white}31.5 & \cellcolor{Red!52!white}40.0 & \cellcolor{Red!62!white}14.3 & \cellcolor{Red!76!white}86.1 & \cellcolor{Red!62!white}11.8 & \cellcolor{Red!44!white}24.6 & \cellcolor{Red!76!white}34.0 & \cellcolor{Red!62!white}77.6 & \cellcolor{Red!76!white}30.8 & \cellcolor{Red!62!white}42.0 \\
        
        \rowcolor{green!4}
        \textbf{\deltaqwenlabel} & \deltadash & \deltadash
                & \deltaval{+8.8} & {\deltaval{+43.7}} & {\deltaval{+8.7}} & {\deltaval{+6.6}}
                & \deltaval{+4.8} & \deltaval{+36.4} & \deltaval{+9.0} & \deltaval{+8.6}
                & \deltaval{+3.6} & \deltaval{+28.5} & \deltaval{+6.4} & \deltaval{+5.4}
                & \deltaval{+4.4} & \deltaval{+63.5} & \deltaval{+4.5} & \deltaval{+11.3}
                & \deltaval{+4.9} & \deltaval{+45.2} & \deltaval{+6.9} & \deltaval{+8.8} \\
        \addlinespace[1pt] \textbf{CURE-14B + CoSPlay} & Ins. & 4.5k & \cellcolor{Red!62!white}53.6 & \cellcolor{Red!76!white}71.5 & \cellcolor{Red!95!white}52.4 & \cellcolor{Red!95!white}61.2 & \cellcolor{Red!52!white}45.4 & \cellcolor{Red!76!white}76.2 & \cellcolor{Red!95!white}46.9 & \cellcolor{Red!76!white}54.9 & \cellcolor{Red!76!white}34.6 & \cellcolor{Red!76!white}75.1 & \cellcolor{Red!95!white}36.8 & \cellcolor{Red!76!white}44.9 & \cellcolor{Red!95!white}15.3 & \cellcolor{Red!95!white}89.4 & \cellcolor{Red!95!white}19.1 & \cellcolor{Red!76!white}32.5 & \cellcolor{Red!62!white}33.8 & \cellcolor{Red!76!white}80.1 & \cellcolor{Red!95!white}36.0 & \cellcolor{Red!76!white}45.9 \\
        \quad w/ Cluster & Ins. & 4.5k & \cellcolor{Red!76!white}54.9 & \cellcolor{Red!76!white}71.5 & \cellcolor{Red!95!white}52.4 & \cellcolor{Red!76!white}60.7 & \cellcolor{Red!62!white}45.5 & \cellcolor{Red!76!white}76.2 & \cellcolor{Red!95!white}46.9 & \cellcolor{Red!95!white}55.3 & \cellcolor{Red!95!white}34.9 & \cellcolor{Red!76!white}75.1 & \cellcolor{Red!95!white}36.8 & \cellcolor{Red!95!white}45.3 & \cellcolor{Red!95!white}15.3 & \cellcolor{Red!95!white}89.4 & \cellcolor{Red!95!white}19.1 & \cellcolor{Red!95!white}32.6 & \cellcolor{Red!95!white}34.1 & \cellcolor{Red!76!white}80.1 & \cellcolor{Red!95!white}36.0 & \cellcolor{Red!95!white}46.2 \\

        \rowcolor{green!4}
        \textbf{\deltacurelabel} & \deltadash & \deltadash
                & \deltaval{+1.3} & {\deltaneg{-5.8}} & {\deltaval{+4.7}} & {\deltaval{+0.8}}
                & \deltaval{+0.4} & \deltaneg{-9.9} & \deltaval{+5.7} & \deltaval{+4.8}
                & \deltaval{+0.0} & \deltaneg{-11.9} & \deltaval{+4.3} & \deltaval{+1.2}
                & \deltaval{+0.4} & \deltaval{+11.9} & \deltaval{+7.1} & \deltaval{+6.5}
                & \deltaval{+0.5} & \deltaneg{-2.3} & \deltaval{+5.9} & \deltaval{+4.4} \\

        \bottomrule
    \end{tabular}
    }
    \vspace{-0.6cm}
\end{table*}
\endgroup

\vspace{-0.5em}
\section{Experiments}
\vspace{-0.5em}
\subsection{Settings}
\vspace{-0.5em}
\textbf{Datasets.}
For evaluation, we adopt the coding benchmarks suite curated by CURE~\cite{wang2025cure}, which includes four widely used and challenging datasets: LiveBench~\cite{white2024livebench}, LiveCodeBench (v2)~\cite{jain2024livecodebench}, CodeContests~\cite{li2022codecontests}, and CodeForces~\cite{penedo2025codeforces}. These datasets contain 128, 511, 239, and 467 problems, respectively. Detailed information can be viewed in Appendix~\ref{app:dataset}. To reduce computational cost, all experiments other than the main full-benchmark results in Table~\ref{tab:overall} are conducted on a 200-problem benchmark. 
For each random seed, we sample 50 problems from each dataset, resulting in 200 problems in total. We use three different random seeds to construct three independent 200-benchmarks and report the averaged results across them. Unless otherwise specified, all results beyond Table~\ref{tab:overall} are based on this 200-benchmark setting.

\textbf{Baselines.}
We use Qwen2.5-7B-Instruct and Qwen2.5-14B-Instruct~\cite{bai2023qwentechnicalreport} as the standard base models for CoSPlay and other TTS baselines: S*~\cite{li2025s*}, SFS~\cite{sfs}, MPSC~\cite{huang2024mpsc}, CodeTree~\cite{codertee}, ThinkCoder~\cite{thinkbeforecoding} and PowerSampling~\cite{karan2025reasoningsamplingbasemodel}. We further benchmark against strong instruct and RLVR-trained models at 14B and 7B scales. At the 14B scale, baselines include Qwen2.5-14B-Coder, Qwen2.5-14B-Coder-Instruct~\cite{hui2024qwen25codertechnicalreport}, Absolute-Zero-14B~\cite{zhao2025absolutezero}, and CURE-14B~\cite{wang2025cure}. At the 7B scale, we include Qwen2.5-7B-Coder, Qwen2.5-7B-Coder-Instruct~\cite{hui2024qwen25codertechnicalreport}, Seed-Coder-8B-Instruct~\cite{seed2025seedcoder}, AceCoder-7B-RM, AceCoder-7B-Rule~\cite{zeng2025acecoder}, CURE-7B~\cite{wang2025cure}, and Absolute-Zero-Coder-7B~\cite{zhao2025absolutezero}. More detailed baseline settings are provided in Appendix~\ref{app:detailed_baseline_setting}.  

\textbf{Implementation Details}
All experiments are conducted on two NVIDIA H100 (80GB) GPUs using vLLM~\cite{vllm}, with a temperature of 0.8, top-p of 0.95 and top-k of 40. In all experiments, we initialize half of the UT inputs using random valid inputs and the other half using inputs from $\mathcal{A}$. For each generated UT input, we sample 4 expected outputs and keep the UT only if at least 3 outputs agree. Besides, for CoSPlay, we set the code pool size $N_c$ to 16, the UT pool size $N_t$ to 16, and the maximum iteration budget $T_{\max}$ to 5.  

\vspace{-1.1em}
\subsection{Results}
\vspace{-0.5em}

\textbf{CoSPlay substantially improves both UT and code quality.}
Table~\ref{tab:overall} shows that CoSPlay consistently improves base models across most metrics without GT data or additional training. Starting from Qwen2.5-7B-Instruct, CoSPlay w/ Cluster improves average Signal, UT accuracy, code accuracy, and BoN from 22.1\%, 14.6\%, 19.0\%, and 22.1\% to 31.9\%, 78.3\%, 23.3\%, and 33.2\%, respectively. It slightly surpasses CURE-7B in average BoN (33.2\% vs.\ 32.9\%) while achieving much higher UT accuracy (78.3\% vs.\ 54.9\%). Similar trends hold at the 14B scale. When applied to CURE, CoSPlay further improves BoN from 32.9\% to 38.6\% on CURE-7B and from 41.8\% to 46.2\% on CURE-14B, demonstrating complementarity with RLVR. The consistent gains from the w/ Cluster rows further show that execution consensus provides more useful selection beyond BoN. Compared with CodeT~\cite{chen2022codet}, which also relies on execution-output agreement for selection, our Cluster-only variant achieves comparable performance, while both are  outperformed by the full CoSPlay framework.

\textbf{CoSPlay surpasses scaled TTS baselines.}
Figure~\ref{fig:combined_tts_generalization_density}(a) compares CoSPlay with TTS methods under their largest scaled configurations.
CoSPlay forms the outer envelope of the radar plot on most axes, achieving higher Pass@1 across both 7B and 14B models on most benchmarks.
Moreover, CoSPlay w/ Cluster further expands this envelope, showing that execution consensus can better select BoN-tied codes.One exception is Qwen2.5-14B-Instruct on LiveBench, where scaling BoN to $N=256$ surpasses CoSPlay-14B, suggesting that simple sampling can still work well on easier benchmarks; however, this scaling alone is insufficient on harder benchmarks.
These results indicate that allocating test-time compute to execution-guided verification, code refinement, and selection for Code-UT co-evolution is useful.
Detailed results are reported in Table~\ref{tab:tts_performance}.

\textbf{Ablation Study.} Figure~\ref{fig:combined_tts_generalization_density}(b) shows that the full CoSPlay framework achieves the best BoN accuracy by the final self-play round. The \textit{w/o exp-atk} variant consistently performs worst, confirming the importance of idea-level exploration for constructing a strong initial Code-UT pool. Removing individual self-play components, including Step~2, Step~3, the non-trivial-best heuristic, or self-consistency, also degrades performance, indicating that each module contributes to stable long-term improvement. Figure~\ref{fig:combined_tts_generalization_density}(c) further shows that the Signal of the full model increases consistently across rounds, suggesting that self-play progressively strengthens the discriminative power of the UT pool. In contrast, ablated variants exhibit weaker or unstable trajectories, especially \textit{w/o exp-atk}, whose Signal declines over time. Detailed results are provided in Appendix~\ref{app:Detailed ablation results}.

    

\begin{figure*}[t]
    \centering
    
    \includegraphics[
        width=\textwidth,
        trim=0 0 0 0,
        clip
    ]{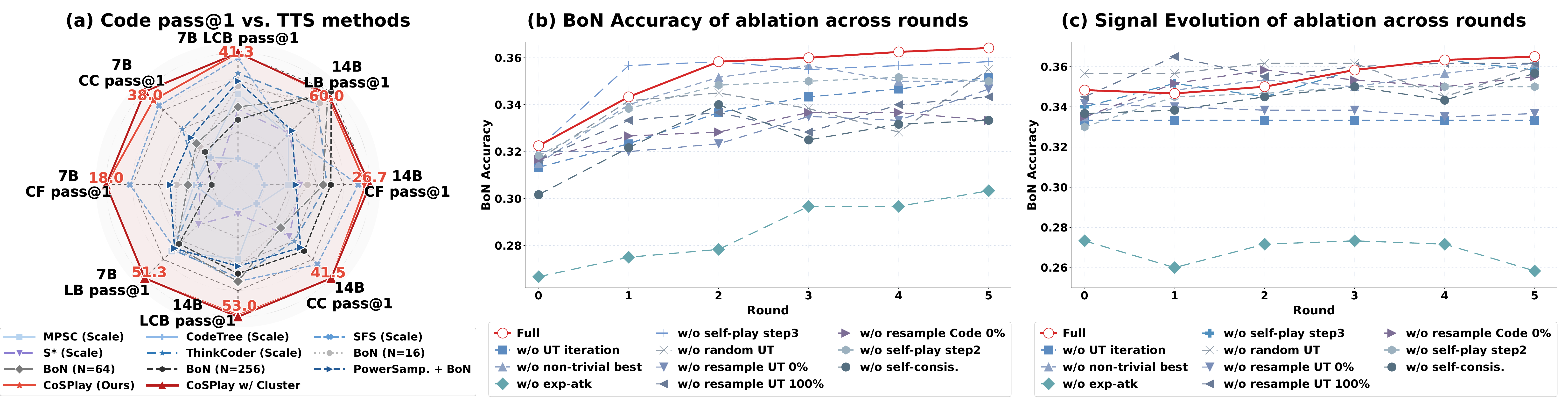}
    
    
    
    \caption{
    (a) The code Pass@1 vs other TTS methods. (b) \textbf{BoN} Accuracy of ablation across rounds. (c) \textbf{Signal} Evolution of ablation across rounds
    }
    \vspace{-1.9em}
    \label{fig:combined_tts_generalization_density}
\end{figure*}

\vspace{-0.8em}
\section{Discussion}
\vspace{-0.8em}
\subsection{Generalizability of CoSPlay}
\label{generalization}
\vspace{-0.3cm}
\textbf{Generalization across model families.}
We evaluate CoSPlay on instruct, RLVR-tuned, and frontier-scale models to test the generalization.
As shown in Figure~\ref{fig:generalization_exp_atk_cluster_score_row}(a), CoSPlay w/ Cluster improves the average BoN over the simple Best-of-16 baseline from 37.8\% to 45.3\%, yielding a +7.5\% absolute gain. Table~\ref{tab:generalization_combined} in Appendix~\ref{app:detailed generalization} further shows that these BoN gains come from jointly improving both code and UT accuracy across backbones.
The gains are large across various instruct and RLVR-tuned models: Seed-Coder-8B-Instruct improves from 32.3\% to 42.3\%, DeepSeek-Coder-V2-Lite-Instruct-16B from 30.8\% to 41.3\%, AceCoder-Rule-7B from 29.0\% to 39.5\%, and AceCoder-RM-7B from 29.3\% to 40.5\%. These results suggest that CoSPlay is not specialized to one model family, scale, or post-training recipe, but can serve as a general inference-time complement to different backbones.
We further evaluate CoSPlay on frontier-scale models.
CoSPlay w/ Cluster raises the average BoN of DeepSeek-V3.2-685B from 65.7\% to 68.2\% and Gemini-2.0-Flash from 53.8\% to 55.3\%.
The gains are more pronounced on CodeForces, the hardest benchmark in our suite, where CoSPlay w/ Cluster improves DeepSeek-V3.2-685B from 39.3\% to 50.0\% and Gemini-2.0-Flash from 36.7\% to 38.7\%.
This suggests that CoSPlay still mines test-time gains from frontier-scale models, especially facing challenging problems.

\textbf{When does self-play help?}
The gains are not uniform across different models and benchmarks, as shown in Table~\ref{tab:generalization_combined}.
This is tied to how CoSPlay obtains verification signal: both code quality and UT reliability are inferred from the backbone's own generated Code-UT executions rather than GT.
This signal become useful for self-play only when their pass-count statistics can distinguish higher-quality objects from lower-quality ones.
This happens when support is aligned with correctness: correct codes tend to pass more generated UTs than incorrect codes, and reliable UTs tend to be passed by more candidate codes than unreliable UTs.
When this alignment holds, row and column pass counts provide GT-free internal signals for ranking, refreshing, and repairing the Code-UT pool; the empirical correlation is shown in Figure~\ref{fig:generalization_density_related_curves}, and the posterior direction is formalized in Appendix~\ref{app:theory:pass-counts}.
This gives a three-regime view of self-play. 
At the lower boundary, if the backbone is too weak for a benchmark, the initial UT correctness prior can fall below the code-side threshold $\rho_T^\star$ analyzed in Appendix~\ref{app:theory:initial-matrix}. In this regime, passing more generated UTs no longer reliably indicates a more correct code candidate, so noisy or wrong-oracle tests and spurious Code-UT coupling can dominate the execution matrix and prevent stable improvement. 
At the upper boundary, if the backbone is already very strong, the Code-UT pool can approach saturation: low-support items selected for refresh may have similar expected quality to newly sampled or repaired items. Refreshing then provides little expected improvement and mainly introduces sampling variance, which can lead to fluctuations rather than monotonic gains. 
Between them, the backbone is strong enough for pass counts to remain correctness-aligned, while the current pool still contains low-support samples whose expected quality can be improved through regeneration or repair. Here, pass-count-guided updates can improve these weak samples and move the pool toward a more reliable state for final selection.

\begin{figure*}[!t]
    \centering

    \begin{minipage}[t]{0.45\textwidth}
        \vspace{0pt}
        \centering
        \includegraphics[width=\linewidth]{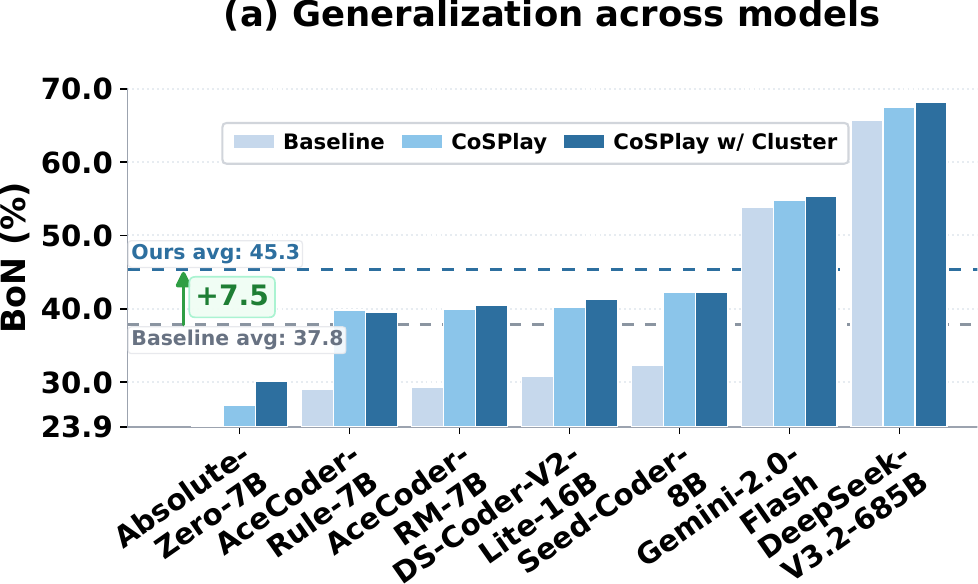}
    \end{minipage}%
    \hfill
    \begin{minipage}[t]{0.27\textwidth}
        \vspace{0pt}
        \centering
        \includegraphics[width=\linewidth]{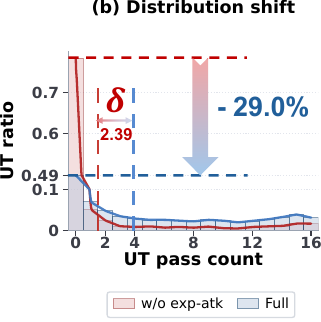}
    \end{minipage}%
    \hfill
    \begin{minipage}[t]{0.27\textwidth}
        \vspace{0pt}
        \centering
        \includegraphics[width=\linewidth]{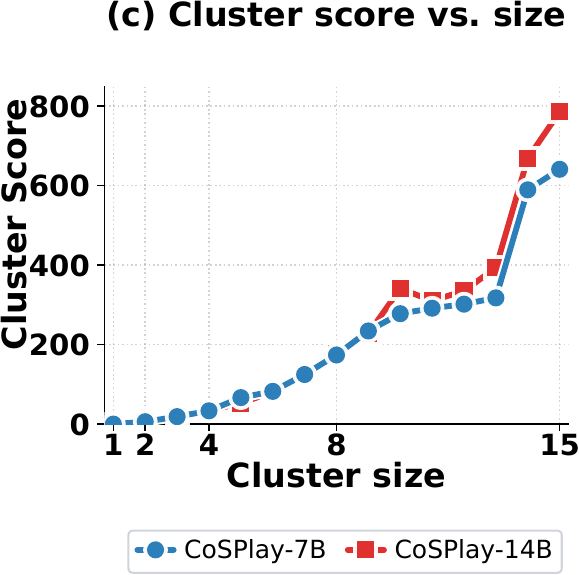}
    \end{minipage}%
    \vspace{-0.3em}
    \caption{
    (a) shows the generalization of CoSPlay across diverse base and RL models.
    (b) compares UT pass-count distributions between direct sampling and exp-atk based generation, where $\delta$ denotes the gap between their average UT pass counts.
    (c) shows the positive relationship between cluster size and average cluster score.
    }
    \vspace{-1.2em}
    \label{fig:generalization_exp_atk_cluster_score_row}
\end{figure*}
\begin{figure*}[!t]
    \centering

    \begin{minipage}[t]{0.29\textwidth}
        \centering
        \includegraphics[width=\linewidth]{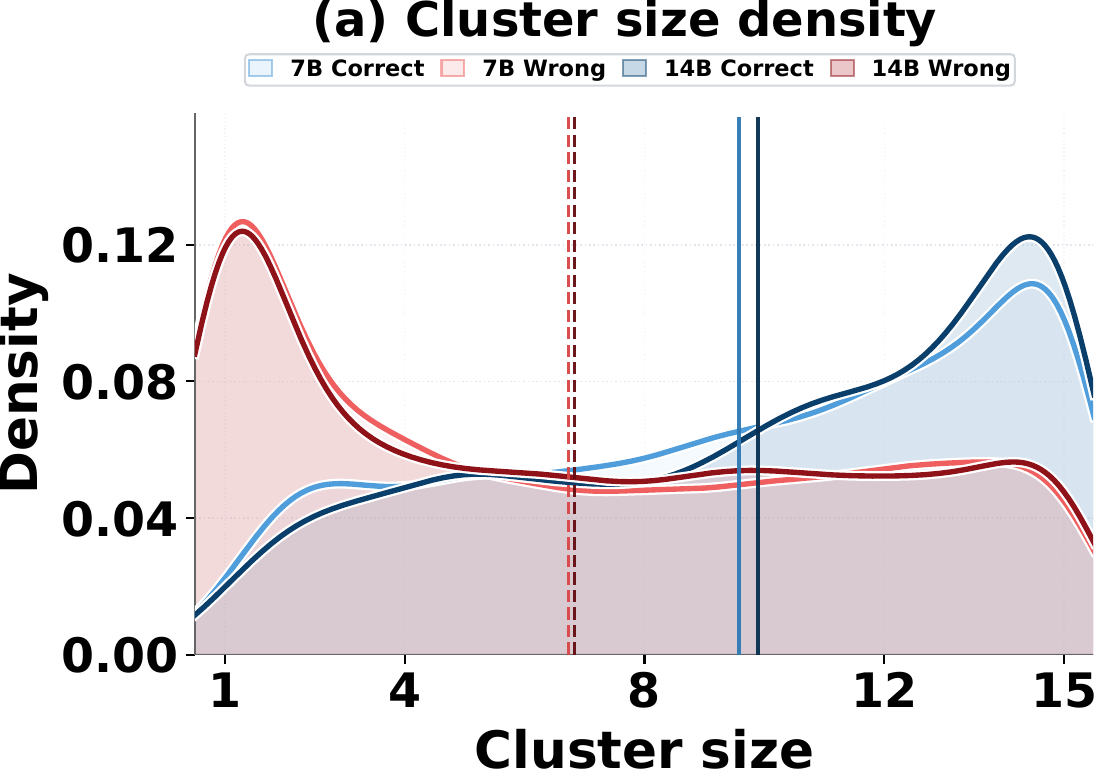}
    \end{minipage}
    \hfill
    \begin{minipage}[t]{0.29\textwidth}
        \centering
        \includegraphics[width=\linewidth]{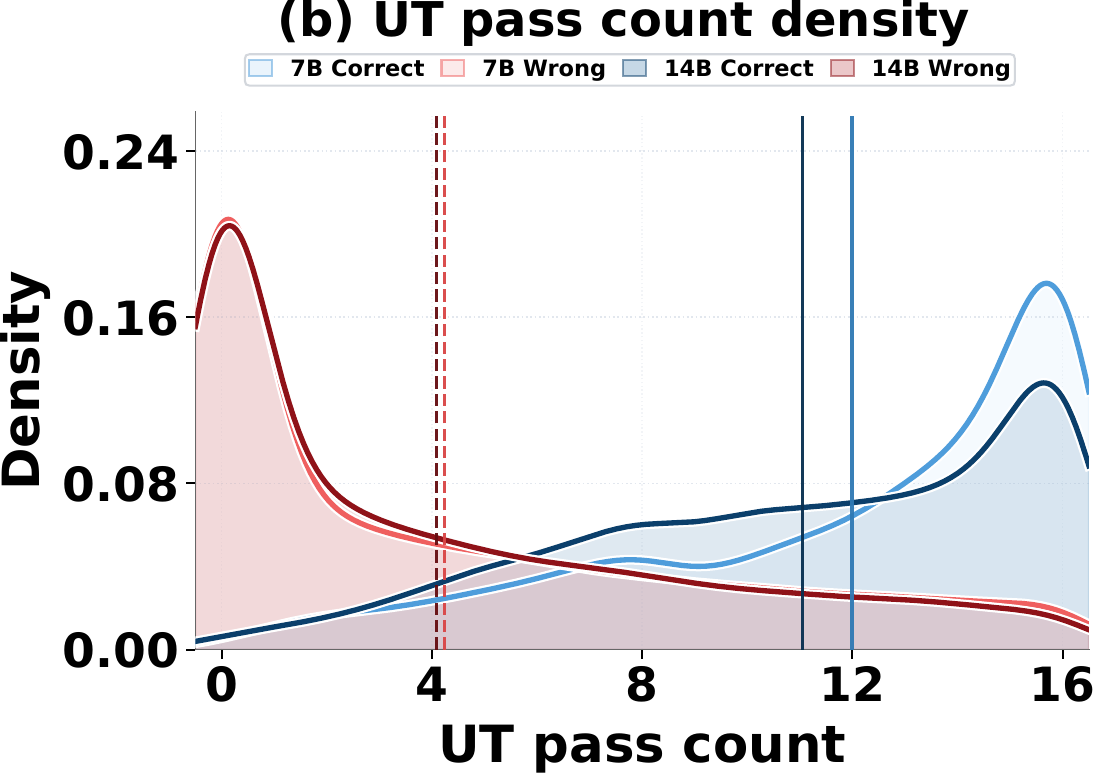}
    \end{minipage}
    \hfill
    \begin{minipage}[t]{0.29\textwidth}
        \centering
        \includegraphics[width=\linewidth]{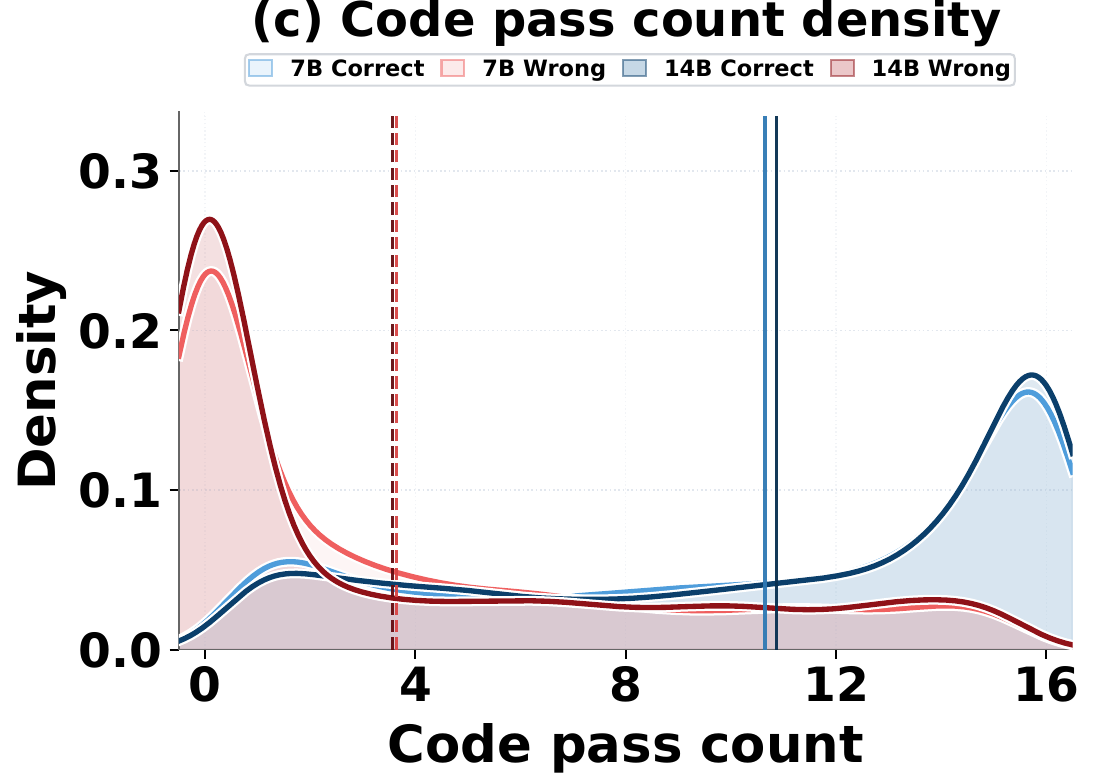}
    \end{minipage}


    \begin{minipage}[t]{0.29\textwidth}
        \centering
        \includegraphics[width=\linewidth]{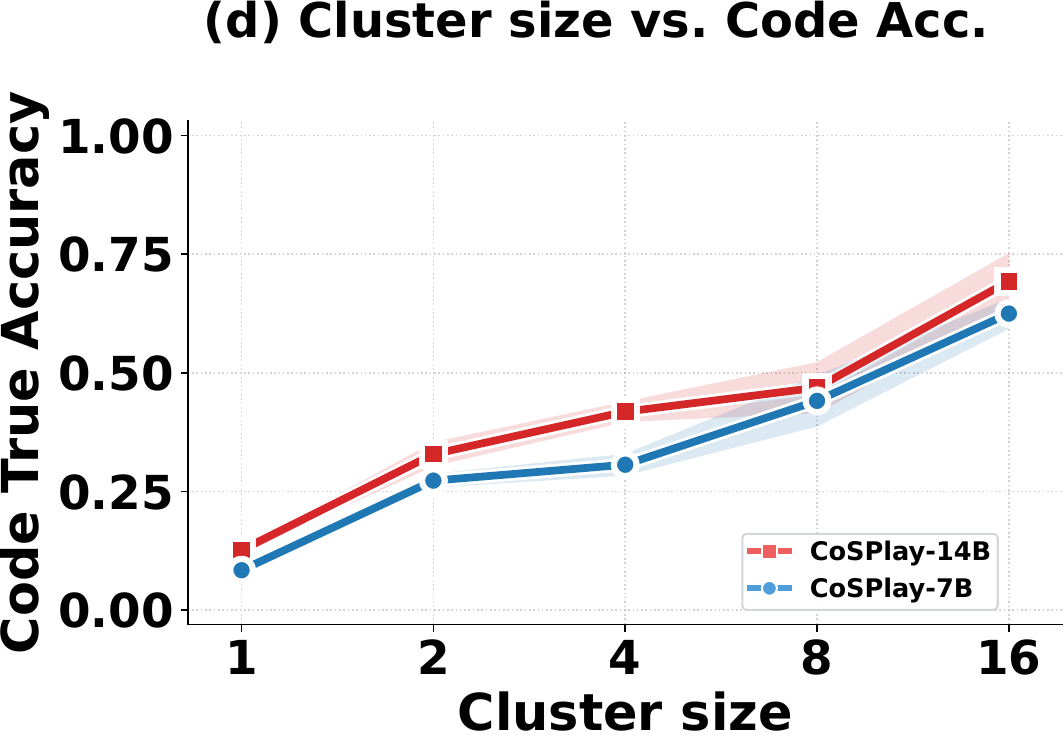}
    \end{minipage}
    \hfill
    \begin{minipage}[t]{0.29\textwidth}
        \centering
        \includegraphics[width=\linewidth]{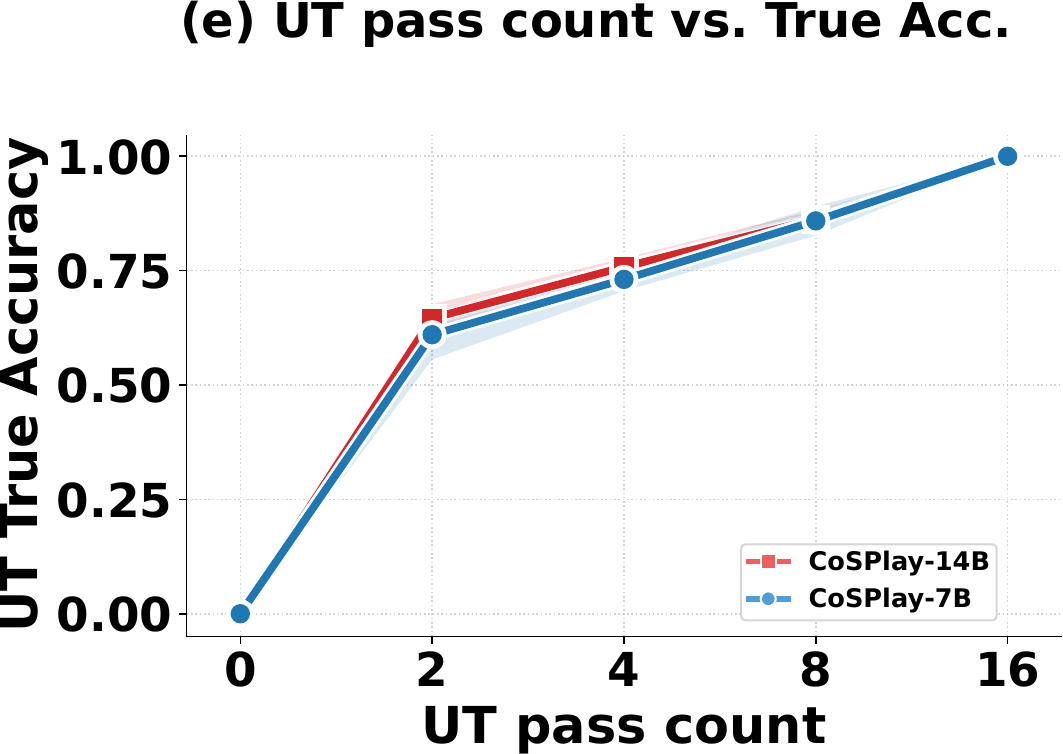}
    \end{minipage}
    \hfill
    \begin{minipage}[t]{0.29\textwidth}
        \centering
        \includegraphics[width=\linewidth]{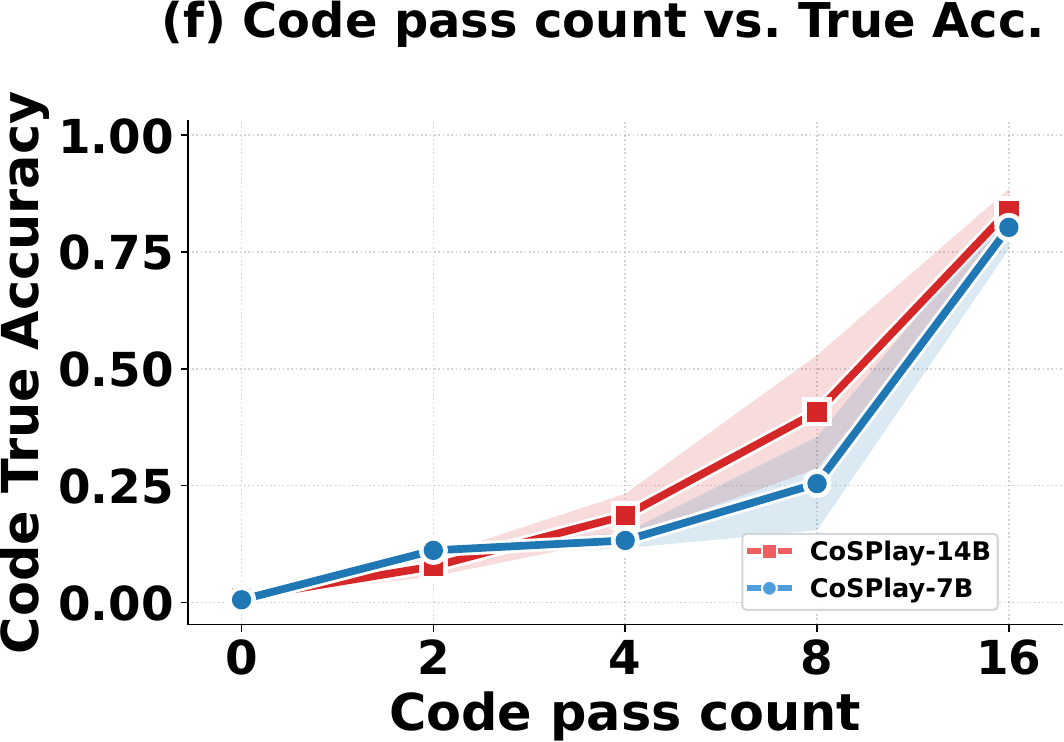}
    \end{minipage}

    \vspace{-0.5em}
    \caption{
    \textbf{Execution-consensus and pass-count analysis.}
    Panels (a-c) show the density distributions of cluster size, UT pass count, and code pass count for correct and wrong candidates, where vertical lines indicate the corresponding mean values.
    Panels (d-f) show that GT correctness increases with larger cluster sizes and higher pass counts.
    These results support the use of execution-consensus clusters and execution-matrix pass counts to estimate code quality and UT reliability.
    }
    \vspace{-1.7em}
    \label{fig:generalization_density_related_curves}
\end{figure*}

\vspace{-0.3cm}
\subsection{Exploration-Attack Guided Idea Generation Improves Initial UT Quality}
\vspace{-0.2cm}
Our exploration-attack (exp-atk) stage improves the accuracy and discriminativeness of initial UTs.
As shown in Table~\ref{tab:ablation_minibench} in Appendix~\ref{app:Detailed ablation results}, even without self-play or self-consistency, exploration alone raises $UT_{\text{Acc}}$ from 12.5\% to 37.2\% and Signal from 26.5\% to 33.7\%.
Figure~\ref{fig:generalization_exp_atk_cluster_score_row}(b) further explains this improvement from the pass-count distribution.
Without exp-atk, most UTs concentrate near zero pass count, suggesting that many initial UTs are invalid, overly difficult, or tied to spurious code behavior.
With exp-atk, the distribution shifts toward higher pass counts, increasing the average UT pass count by $\delta=2.39$ and reducing extreme low-support UTs by 29.0\%.
This indicates that exp-atk produces more reliable UTs that are supported by multiple candidate codes, while still preserving enough variation for selection.
Together, these results show that exp-atk improves both UT correctness and discriminative power.

\vspace{-0.2cm}
\subsection{Pass Count Reflects True Accuracy}
\vspace{-0.2cm}
To validate that pass count reflects actual quality, we study the relationship between pass count and GT correctness for both code and UTs. As shown in Figure~\ref{fig:generalization_density_related_curves}(e-f), both UTs and codes exhibit a strong positive correlation: UTs passed by more codes are more likely to be correct, while codes passing more UTs are also more likely to be functionally correct. Consistently, Figure~\ref{fig:generalization_density_related_curves}(b-c) shows that correct candidates concentrate more heavily in high pass-count regions than wrong candidates. These results indicate that pass counts from the execution matrix can serve as an effective GT-free proxy for candidate quality. 

\vspace{-0.2cm}
\subsection{Cluster-Based Selection Reflects Functional Consistency}
\vspace{-0.2cm}
Our cluster-based selection is built on the assumption that larger execution-consensus clusters are more likely to contain correct code. 
\textbf{We first verify that cluster size is positively correlated with true code correctness.}
Figure~\ref{fig:generalization_density_related_curves}(a) shows that correct codes concentrate in large clusters, while wrong codes mostly form clusters of size 1 or 2; Figure~\ref{fig:generalization_density_related_curves}(d) further shows that average true code accuracy increases steadily with cluster size for both 7B and 14B models. 
This supports the intuition that correct codes tend to share consistent functional behavior, whereas wrong codes fail in diverse ways. 
\textbf{We further verify that cluster size is positively correlated with our aggregate cluster score.}
As shown in Figure~\ref{fig:generalization_exp_atk_cluster_score_row}(c), larger clusters generally obtain higher scores for both CoSPlay-7B and CoSPlay-14B, indicating that our scoring rule preserves the largest-cluster intuition while adapting it to runtime-error settings. 
Since raw cluster size can be unreliable when candidates produce invalid executions, the cluster score aggregates valid pairwise agreements and provides a more robust estimate of cluster reliability.

\begin{figure*}[!t]
    \centering
    \includegraphics[width=\textwidth]{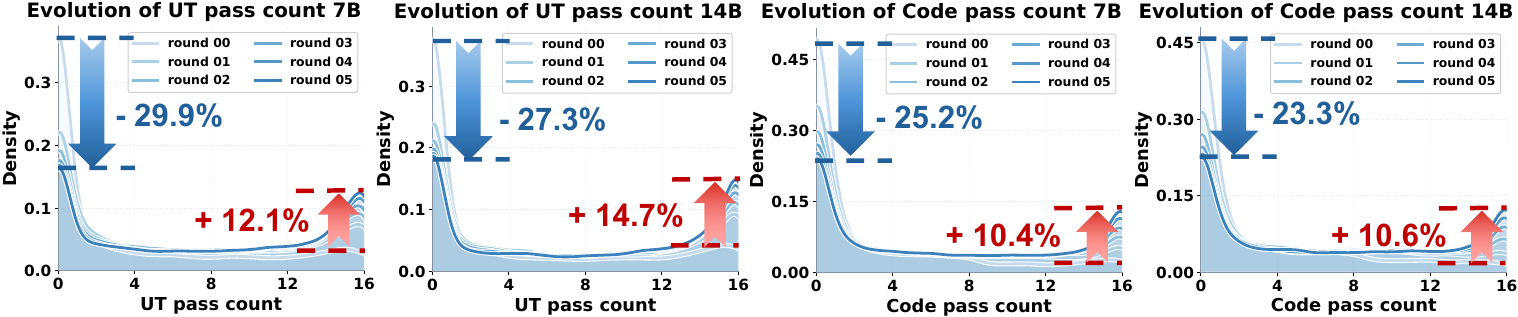}
    \caption{
    \textbf{Evolution of pass-count distributions during self-play.}
    Both UT and code pass-count distributions progressively shift toward higher-support regions across self-play rounds, suggesting that execution-matrix-driven self-play gradually concentrates support on more reliable UTs and stronger code candidates.
    }
    \vspace{-1.2em}
    \label{fig:evolution_of_pass_count}
\end{figure*}

\begin{figure*}[!t]
    \centering
    \includegraphics[width=\textwidth]{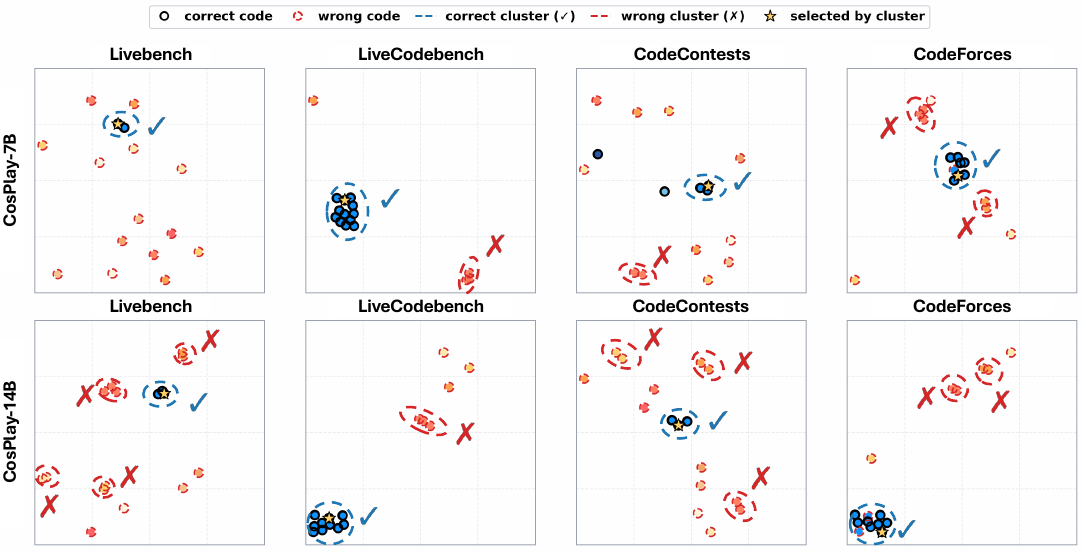}
    \vspace{-0.5em}
    \caption{
    \textbf{t-SNE visualization of clusters.}
    Across four datasets, correct codes tend to form compact high-density clusters, whereas incorrect codes are more scattered, supporting execution-consensus clustering as effective GT-free selection signal.
    }
    \vspace{-1.6em}
    \label{fig:cluster_analysis_tsne}
\end{figure*}

\vspace{-0.4cm}
\subsection{Self-play progressively shifts pass-count distributions to the right}
\vspace{-0.2cm}
Figure~\ref{fig:evolution_of_pass_count} shows the average evolution of UT and code pass-count distributions from round 0 to round 5. For both 7B and 14B models, self-play consistently reduces the density in low pass-count regions and increases the density near the maximum pass count. This trend appears for both UTs and codes: fewer UTs remain rarely passed by the code pool, and more code candidates pass a larger number of UTs after iterative refinement. Quantitatively, the low-pass-density regions decrease by 29.9\%/27.3\% for UTs and 25.2\%/23.3\% for codes at the 7B/14B scales, while the high-pass-density regions increase accordingly. 
These results suggest that execution-matrix-driven self-play progressively removes low-support UTs and codes, improving UT reliability and code quality.

\vspace{-0.4cm}
\subsection{t-SNE Visualization of Clusters}
\vspace{-0.2cm}
Figure~\ref{fig:cluster_analysis_tsne} provides t-SNE visualizations of code candidates from CoSPlay-7B and CoSPlay-14B using their execution-output similarity. By projecting output-similarity patterns into a low-dimensional space, we observe that correct codes (blue dots) tend to concentrate in large execution-consensus clusters, often including the largest cluster indicated by dashed blue circles. This observation is consistent with the intuition behind our strategy: functionally equivalent correct codes yield identical outputs for the same input, thereby exhibiting high consistency. In contrast, many incorrect codes (red dots) appear more scattered or form smaller isolated clusters, suggesting that their failure modes are often less consistent across random inputs. Together with the quantitative cluster-size analysis, this visualization provides qualitative evidence that cluster size can serve as a useful empirical proxy for code correctness. Larger clusters generally indicate stronger output consistency and are empirically associated with a higher likelihood of correctness. This provides additional support for using output-consensus clustering to select more functionally consistent candidates from the code pool.

\vspace{-2em}
\subsection{Scalability of CoSPlay with Increasing Candidate-Pool Size}
\vspace{-0.5em}
As shown in Figure~\ref{fig:combined_core_and_rank_analysis}(a), both Qwen2.5-7B-Instruct and CURE-7B obtain higher BoN as $k$ increases with CoSPlay, indicating that our method can effectively exploit larger inference-time candidate pools.
Notably, Qwen2.5-7B-Instruct + CoSPlay becomes comparable to CURE-7B when $k \geq 16$, despite using no GT data or training.
At $k=64$, CoSPlay improves Qwen2.5-7B-Instruct by 13.2\%, while CURE-7B + CoSPlay continues to benefit from larger budgets, showing that CoSPlay can further enhance already RLVR-tuned models. 
Moreover, both CoSPlay settings show no clear saturation at $k=64$, suggesting
that performance may continue to improve with larger candidate pools; however,
we leave this further scaling study to future work due to computational
resource constraints.
\afterpage{%
\begin{figure*}[!t]


    \begin{subfigure}[t]{0.49\textwidth}
        \centering
        \includegraphics[width=\linewidth]{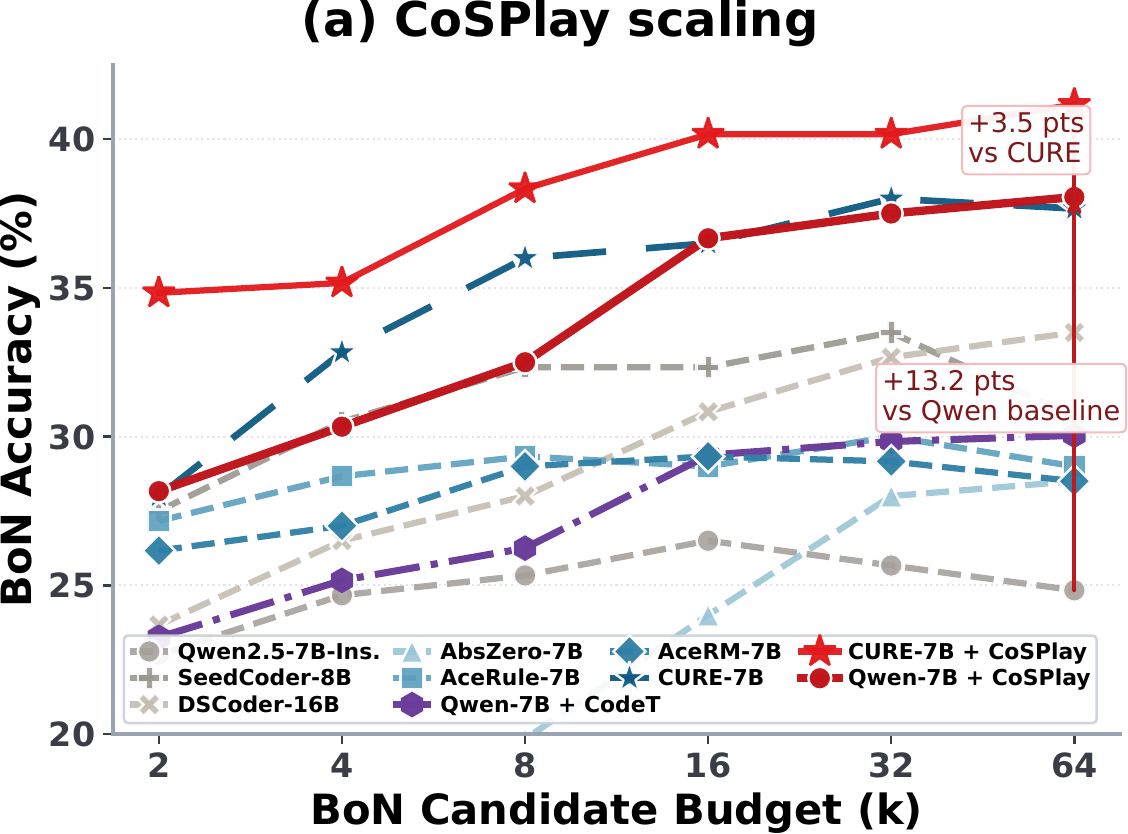}
        \label{fig:core_f_scaling}
    \end{subfigure}
    \hfill
    \begin{subfigure}[t]{0.49\textwidth}
        \centering
        \includegraphics[width=\linewidth]{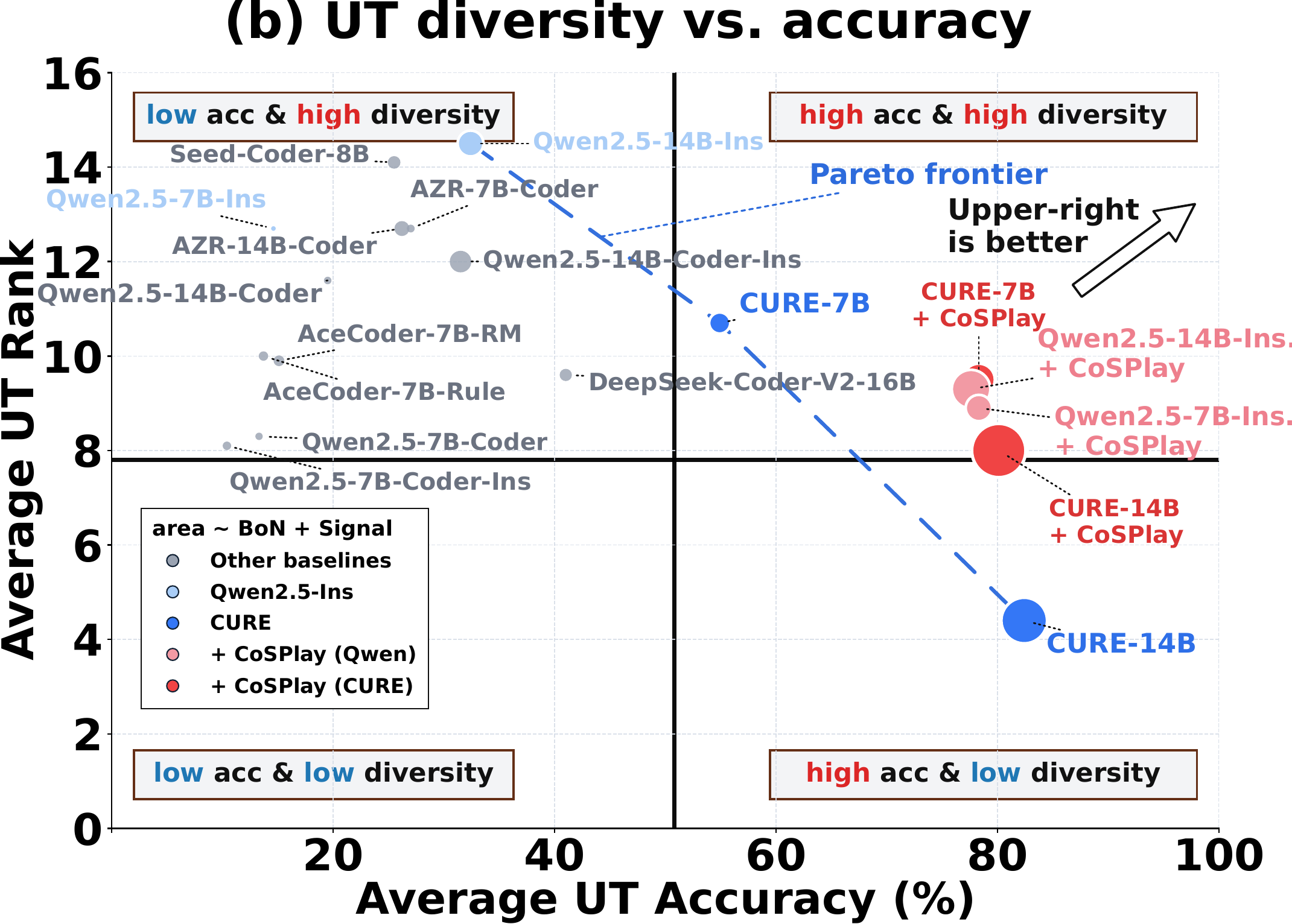}
        \label{fig:core_g_rank_tradeoff}
    \end{subfigure}

    \vspace{-0.5em}
    \caption{
    (a) shows the scalability of CoSPlay with candidate-pool size.
    (b) shows the trade-off between UT diversity and UT accuracy for CoSPlay, instruct-tuned models, and RLVR models.
    }
    \vspace{-1.2em}
    \label{fig:combined_core_and_rank_analysis}
\end{figure*}
  \begin{figure*}[!t]
    \centering
    \includegraphics[width=\textwidth]{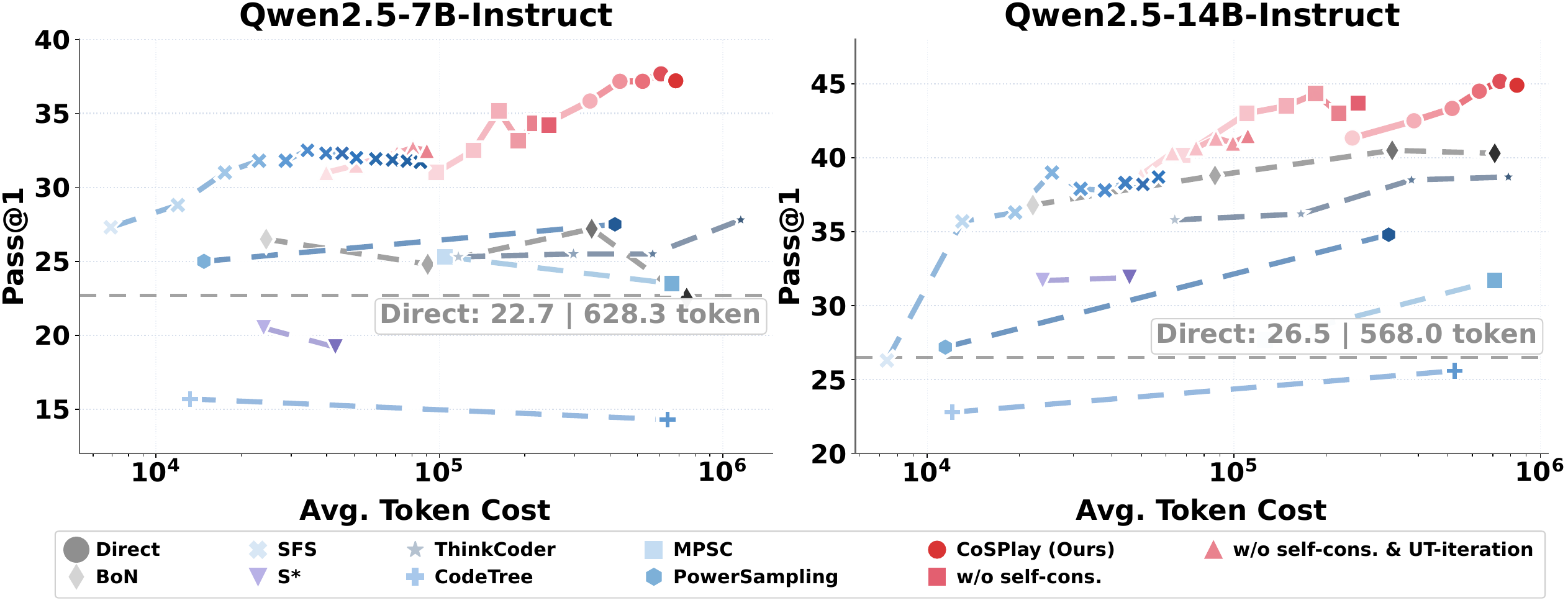}

    \caption{
Token cost versus Pass@1 of TTS methods and CoSPlay on Qwen2.5-Instruct models.
For each baseline method, darker markers indicate its scaled variant with a larger inference budget.
For CoSPlay, markers in each setting show the trajectory from self-play round 0 to round 5, illustrating how performance improves as additional self-play compute is used.
Detailed scaling configurations and data are provided in Appendix~\ref{app:tts_setting} and Appendix~\ref{app:tts_analysis}.
}
    \label{fig:tts_cost_analysis}
    \vspace{-1.5em}
\end{figure*}

    
}

\vspace{-0.5em}
\subsection{UT Accuracy-Diversity Trade-off of CoSPlay}
\vspace{-0.5em}
To quantify UT diversity, we define the \textit{UT rank} $\mathcal{T}_{\mathrm{rank}} = \left| \{x_j \mid (x_j, y_j) \in \mathcal{T}\} \right|$ as the number of unique test inputs in the pool $\mathcal{T}$.
While this metric does not fully capture semantic diversity, it provides a lightweight estimate of UT diversity. 
Figure~\ref{fig:combined_core_and_rank_analysis}(b) shows that CoSPlay-enhanced models achieve a better UT accuracy-diversity trade-off than the original baselines.
Without CoSPlay, the Pareto frontier is mainly formed by models that favor either high diversity with low accuracy or high accuracy with low diversity.
After applying CoSPlay, the corresponding points move closer to the upper-right region, indicating improved balance between UT correctness and diversity. CURE-14B is a representative case: it achieves high UT accuracy but relatively low UT rank, suggesting that RL-based training may improve correctness while reducing diversity~\cite{zhao2025echo}.
CoSPlay alleviates this issue by promoting UT diversity through idea-level exploration, while preserving UT accuracy by filtering low-quality UTs with pass-count signals from the execution matrix.
Qwen-based models and CURE-7B mainly gain in UT accuracy with slightly decreased yet competitive rank.
The larger marker sizes further indicate higher BoN and Signal scores, showing that this improved trade-off translates into stronger final selection capability.

\vspace{-0.5em}
\subsection{Efficiency Analysis of CoSPlay}
\vspace{-0.5em}
Figure~\ref{fig:tts_cost_analysis} compares CoSPlay with TTS methods under different token budgets.
Although full CoSPlay uses additional computation for Code-UT generation, execution, and self-play, existing TTS baselines still achieve lower Pass@1 when scaled to comparable or even larger token costs.
Moreover, as the inference budget increases, CoSPlay improves faster than competing TTS baselines, indicating stronger budget scalability rather than a simple constant gain.
This shows that CoSPlay attains a higher performance ceiling rather than merely benefiting from larger inference budgets.
The ablation variants further show that even after removing costly components such as self-consistency and UT iteration, CoSPlay substantially reduces token consumption while still outperforming most baselines, and in some cases even surpasses baselines scaled to much larger budgets.
These results suggest that CoSPlay provides a favorable cost-performance trade-off by allocating test-time compute to verifiable Code-UT co-evolution rather than undirected candidate sampling.

\section{Conclusion and Future Directions}
\label{sec: limitation and future work}

We introduce \textbf{CoSPlay}, a \textbf{GT-free} and \textbf{training-free} test-time scaling framework that jointly optimizes code candidates and unit tests via execution-driven co-evolution. CoSPlay first explores diverse high-level solution strategies together with their potential failure modes to derive discriminative UT ideas for better UT generation. It then leverages execution-matrix pass-count signals to iteratively refine both the code pool and the UT pool through cooperative self-play. Finally, for BoN-tied codes, CoSPlay further performs output-consensus clustering on random valid inputs. Since correct codes tend to produce consistent outputs while incorrect ones often diverge in different ways, it first selects the most reliable consensus cluster and then chooses the highest-reliability code within that cluster as the final solution.
It consistently outperforms GT-free TTS baselines and matches or surpasses RLVR methods, improving both UT and code quality. These results suggest a scalable route toward GT-free test-time improvement for verifiable reasoning.

\textit{Limitation and Future Directions.}
CoSPlay relies on an executable environment to obtain Code-UT feedback, making direct extension to non-executable reasoning tasks more challenging.
CoSPlay also consumes more tokens due to iterative Code-UT evolution. However, under comparable test-time compute budgets, it achieves the best performance among evaluated methods and substantially outperforms existing GT-free TTS methods. Improving token efficiency and extending to broader reasoning tasks remain promising directions.


\newpage

\bibliographystyle{plain}
\bibliography{cosplay}

\newpage
\appendix
\onecolumn
\appendix
\startcontents[appendix]
\section*{Appendix}

\setcounter{tocdepth}{2}
\printcontents[appendix]{l}{1}{\setcounter{tocdepth}{2}}

\newpage

\section{Theory Analysis}
\label{app:theory}

\subsection{Analysis of Pass-Count Signals}
This section analyzes an i.i.d. model for the standard BoN baseline where codes and UTs are sampled independently. The goal is to explain why pass counts can serve as statistical signals for both code quality and UT quality. 
The first part introduces the execution-matrix setup and the thresholds under which pass counts allow codes and UTs to judge each other's reliability. The second part derives the posterior direction and convergence rate induced by binomial support counts.

\subsubsection{When Pass Counts Indicate Correctness}
\label{app:theory:initial-matrix}
\label{app:theory:setup}
Consider one independent sampling round under the standard BoN baseline: code
candidates and a UT pool are generated without self-play updates. Fix a coding
problem, and let
\[
    \mathcal C=\{c_i\}_{i=1}^{N_c},
    \qquad
    \mathcal T=\{t_j\}_{j=1}^{N_t},
    \qquad
    t_j=(x_j,y_j).
\]
The execution matrix are
\[
    M_{ij}:=\mathbf 1\{\operatorname{Exec}(c_i,x_j)=y_j\},
\]
the UT-side and code-side pass counts, and their
corresponding pass rates represented as:
\[
    p_j^{\mathrm{UT}}=\sum_{i=1}^{N_c}M_{ij},
    \qquad
    p_i^{\mathrm{code}}=\sum_{j=1}^{N_t}M_{ij},
    \qquad
    \alpha_j^{\mathrm{UT}}=\frac{p_j^{\mathrm{UT}}}{N_c},
    \qquad
    \alpha_i^{\mathrm{code}}=\frac{p_i^{\mathrm{code}}}{N_t}.
\]

Here \(p_j^{\mathrm{UT}}\) and \(\alpha_j^{\mathrm{UT}}\) are UT-side support
statistics, while \(p_i^{\mathrm{code}}\) and \(\alpha_i^{\mathrm{code}}\) are
code-side support statistics. Runtime errors are treated as failures.

Let \(C^+\) and \(U^+\) denote the events that a randomly sampled code and a
randomly sampled test are correct, respectively. Let \(C^-\) and \(U^-\) denote
their complements, and write
\[
    \rho_C=\mathbb P(C^+),
    \qquad
    \rho_T=\mathbb P(U^+),
    \qquad
    0<\rho_C,\rho_T<1.
\]
Codes and tests are sampled independently. After conditioning on the correctness
state of the evaluated object, pass events from different evaluators are
conditionally independent. A correct code always passes a correct test and
always fails an incorrect test; the only randomness comes from incorrect codes.
Let \(\varepsilon_1,\varepsilon_2\in[0,1]\) be the pass rates of an incorrect
code on a correct test and on an incorrect test, respectively:
\[
\begin{alignedat}{2}
    \mathbb P(M=1\mid C^+,U^+) &= 1,
    \qquad&
    \mathbb P(M=1\mid C^+,U^-) &= 0, \\
    \mathbb P(M=1\mid C^-,U^+) &= \varepsilon_1,
    &
    \mathbb P(M=1\mid C^-,U^-) &= \varepsilon_2 .
\end{alignedat}
\]
By the law of total probability and independence, the marginal pass
probabilities are
\begin{equation}
\label{eq:theta-phi-def}
\begin{alignedat}{2}
    \theta_1 &:= \mathbb P(M=1\mid U^+)
        = \rho_C+(1-\rho_C)\varepsilon_1,
    \\
    \theta_0 &:= \mathbb P(M=1\mid U^-)
        = (1-\rho_C)\varepsilon_2, 
    \\
    \varphi_1 &:= \mathbb P(M=1\mid C^+)
        = \rho_T,
    \\
    \varphi_0 &:= \mathbb P(M=1\mid C^-)
        = \varepsilon_1\rho_T+\varepsilon_2(1-\rho_T).
\end{alignedat}
\end{equation}
\(\theta_1,\theta_0\) are the probabilities that a correct or incorrect test is
passed by a random code. \(\varphi_1,\varphi_0\) are the probabilities that a
correct or incorrect code passes a random test. The former captures the fact
that correct tests should be easier to pass, while incorrect tests should be
harder to pass. The latter captures the fact that correct codes should pass
tests more often, whereas the pass noise of incorrect codes is controlled by
\(\varepsilon_1\) and \(\varepsilon_2\). The gaps between the positive and
negative classes quantify the attraction advantage of positive objects over
negative objects on an execution-matrix entry \(M\):
\begin{equation}
\label{equ:Thredhold}
\begin{aligned}
    \Delta_U
    &:=
    \theta_1-\theta_0
    =
    \rho_C+(1-\rho_C)(\varepsilon_1-\varepsilon_2),
    &
    \Delta_U>0
    &\Longleftrightarrow
    \rho_C>\rho_C^\star:=
    \frac{\varepsilon_2-\varepsilon_1}{1+\varepsilon_2-\varepsilon_1},
    \\
    \Delta_C
    &:=
    \varphi_1-\varphi_0
    =
    \rho_T(1-\varepsilon_1)-\varepsilon_2(1-\rho_T),
    &
    \Delta_C>0
    &\Longleftrightarrow
    \rho_T>\rho_T^\star:=
    \frac{\varepsilon_2}{1-\varepsilon_1+\varepsilon_2}.
\end{aligned}
\end{equation}

The UT-side advantage $\Delta_U$ is easier to obtain.  Any correct code passes correct UTs, and partially correct code can still pass UTs that match its correct logic. In contrast, wrong code can pass a wrong-output UT only when it matches that specific wrong logic. Since correct solutions share the same answer logic while wrong answers are usually dispersed, a usable base model is unlikely to make wrong code pass wrong-output UTs more often than correct UTs, namely
$\varepsilon_1\geq\varepsilon_2$. This makes the UT-side threshold easy to satisfy.
The code-side advantage $\Delta_C$ is also controlled by
the incorrect-test noise $\varepsilon_2$. When both $\varepsilon_1$ and
$\varepsilon_2$ are small, $\rho_T^\star$ is also small. However, the code
side still requires the fraction of correct tests $\rho_T$ in the UT pool to exceed this
noise level $\rho_T^\star$. Hence its threshold is relatively higher than the UT-side
threshold, and it needs a sufficiently high probability of correctness to keep discriminative.

\subsubsection{Support-Count Posterior}
\label{app:theory:pass-counts}

Each object to be judged is evaluated by $m$ sources. The $i$-th source
returns a binary support result $X_i\in\{0,1\}$, and the support count is
$S=\sum_{i=1}^m X_i$. If these results are independent and identically
distributed under each hypothesis, then $S$ follows the corresponding
binomial distribution. In the notation of the main CoSPlay paper, the UT side
takes $m=N_c$ and $S=p_j^{\mathrm{UT}}$, while the code side takes
$m=N_t$ and $\S=p_i^{\mathrm{code}}$.

\begin{theorem}[Posterior direction of support counts]
\label{thm:count-signal}
Suppose the support count \(S\) satisfies
\[
    S\mid H^+\sim\operatorname{Binomial}(m,q_1),
    \qquad
    S\mid H^-\sim\operatorname{Binomial}(m,q_0),
\]
where $H^+$ is the hypothesis that the object is correct, $H^-$ is the
hypothesis that the object is incorrect, $q_1$ is the probability that a
correct object is supported by one evaluator, $q_0$ is the probability that an
incorrect object is supported by one evaluator, and $0<q_0,q_1<1$. If
$q_1>q_0$, then a larger support count implies a larger posterior probability
of correctness. If $q_1=q_0$, the support count does not change the posterior.
If $q_1<q_0$, then a larger support count implies a smaller posterior
probability of correctness. The posterior odds ratio changes exponentially with
the support count.
\end{theorem}

\begin{proof}
Let the prior be \(\rho=\mathbb P(H^+)\in(0,1)\). By Bayes' rule and the
binomial probability mass function,
\begin{equation}
\label{eq:count-odds}
\begin{aligned}
    \operatorname{OR}(s)
    &:=
    \frac{\mathbb P(H^+\mid S=s)}
         {\mathbb P(H^-\mid S=s)}
    =
    \frac{\rho}{1-\rho}
    \frac{\binom ms q_1^s(1-q_1)^{m-s}}
         {\binom ms q_0^s(1-q_0)^{m-s}}
    =
    \frac{\rho}{1-\rho}
    \left(\frac{1-q_1}{1-q_0}\right)^m
    r^s,
    r:=
    \frac{q_1(1-q_0)}{q_0(1-q_1)} .
\end{aligned}
\end{equation}
The posterior probability and the odds ratio satisfy
\begin{equation}\label{eq:posterior-odds}
    \mathbb P(H^+\mid S=s)
    =
    \frac{\operatorname{OR}(s)}{1+\operatorname{OR}(s)} .
\end{equation}
The function \(x\mapsto x/(1+x)\) is strictly increasing on \(x>0\), and
\[
    r>1
    \Longleftrightarrow
    q_1(1-q_0)>q_0(1-q_1)
    \Longleftrightarrow
    q_1>q_0 .
\]
Thus the sign of \(q_1-q_0\) determines the posterior direction:
\begin{enumerate}
    \item If \(q_1>q_0\), then \(r>1\). Both the posterior odds ratio and the
    posterior probability of correctness are strictly increasing in \(s\).
    \item If \(q_1=q_0\), then \(r=1\). The support count changes neither the
    posterior odds ratio nor the posterior probability of correctness.
    \item If \(q_1<q_0\), then \(0<r<1\). Both the posterior odds ratio and the
    posterior probability of correctness are strictly decreasing in \(s\).
\end{enumerate}
Moreover, Eq.~\eqref{eq:count-odds} shows that the posterior odds ratio depends
on \(s\) through the exponential factor \(r^s\).
\end{proof}

\begin{corollary}[Posterior convergence at a fixed support ratio]
\label{cor:critical-ratio}
Under the binomial model in Theorem~\ref{thm:count-signal}, let \(m\) be the
number of independent evaluators, and let \(s=\lfloor\eta m\rfloor\), where
\(\eta\in[0,1]\) is fixed. If \(q_1\ne q_0\), then there exists a unique
threshold \(\eta^\star\) between \(q_0\) and \(q_1\). As \(m\to\infty\), any
fixed support ratio on the side biased toward \(q_1\) makes the posterior
probability of correctness converge exponentially fast to \(1\), while any
fixed support ratio on the other side makes it converge exponentially fast to
\(0\). If \(q_1=q_0\) or \(\eta=\eta^\star\), no exponential decision is
obtained.
\end{corollary}

\begin{proof}
Let the prior be \(\rho=\mathbb P(H^+)\in(0,1)\), and write
\[
    r
    :=
    \frac{q_1(1-q_0)}{q_0(1-q_1)} .
\]
Let \(\delta_m=s-\eta m\). Since \(s=\lfloor\eta m\rfloor\), we have
\(|\delta_m|\le 1\). By Eq.~\eqref{eq:count-odds},
\begin{align*}
    \log\operatorname{OR}(s)
    &=
    \log\frac{\rho}{1-\rho}
    + m\log\frac{1-q_1}{1-q_0}
    + s\log r \\
    &=
    \log\frac{\rho}{1-\rho}
    + m\left[
        \log\frac{1-q_1}{1-q_0}
        +\eta\log r
    \right]
    +\delta_m\log r \\
    &=
    \log\frac{\rho}{1-\rho}
    + m\left[
        \log\frac{1-q_1}{1-q_0}
        +\eta\log\frac{q_1(1-q_0)}{q_0(1-q_1)}
    \right]
    +O(1) \\
    &=
    \log\frac{\rho}{1-\rho}
    + m\left[
        \eta\log\frac{q_1}{q_0}
        +(1-\eta)\log\frac{1-q_1}{1-q_0}
    \right]
    +O(1) \\
    &=
    \log\frac{\rho}{1-\rho}
    +mD(\eta)+O(1).
\end{align*}
Here
\begin{equation}
\label{eq:fixed-ratio-rate}
    D(\eta)
    :=
    \eta\log\frac{q_1}{q_0}
    +(1-\eta)\log\frac{1-q_1}{1-q_0}.
\end{equation}

If \(q_1=q_0\), then \(D(\eta)\equiv0\). We now assume \(q_1\ne q_0\). By
Eq.~\eqref{eq:fixed-ratio-rate},
\[
    D(q_0)=
    -\operatorname{KL}\big(\operatorname{Bern}(q_0)\Vert\operatorname{Bern}(q_1)\big)
    < 0, \qquad
    D(q_1)=
    \operatorname{KL}\big(\operatorname{Bern}(q_1)\Vert\operatorname{Bern}(q_0)\big)
    > 0.
\]
Thus \(D\) is a nonconstant linear function, and its unique zero lies between
\(q_0\) and \(q_1\). Solving \(D(\eta^\star)=0\) in
Eq.~\eqref{eq:fixed-ratio-rate} gives
\[
    \eta^\star
    :=
    \frac{\log\{(1-q_0)/(1-q_1)\}}{\log r}.
\]
Finally, by Eq.~\eqref{eq:posterior-odds},
\[
    \mathbb P(H^+\mid S=s)
    =
    \frac{1}{1+\exp\{-\log\operatorname{OR}(s)\}}
    =
    \frac{1}{1+\exp\{-mD(\eta)+O(1)\}}.
\]
Moreover, \(D'(\eta)=\log r\). We distinguish three cases:
\begin{enumerate}
    \item If \(q_1>q_0\), then \(D\) is strictly increasing. Hence, when
    \(\eta>\eta^\star\), \(D(\eta)>0\), and the posterior probability of
    correctness converges exponentially fast to \(1\). When
    \(\eta<\eta^\star\), \(D(\eta)<0\), and the posterior probability of
    correctness converges exponentially fast to \(0\).
    \item If \(q_1<q_0\), then \(D\) is strictly decreasing, so the sign
    direction is reversed. Hence, when \(\eta<\eta^\star\), the posterior
    probability of correctness converges exponentially fast to \(1\), and when
    \(\eta>\eta^\star\), it converges exponentially fast to \(0\).
    \item If \(\eta=\eta^\star\), then \(D(\eta)=0\), and the posterior odds
    remain only at constant order. No exponential decision is obtained.
\end{enumerate}
This proves the claim.
\end{proof}

\subsection{Analysis of Execution-Consensus Clustering}
Execution-Consensus Clustering is applied after BoN filtering. Let
\(\mathcal C_{\mathrm{high}}\) be the high-scoring set retained by BoN, and draw
a code \(c\) from this set. All probabilities below are conditional on
\(c\in\mathcal C_{\mathrm{high}}\) and are taken over both code and input
randomness. Draw \(R\) independent valid random inputs \(z_1,\ldots,z_R\), and
write
\[
    \sigma(c):=\big(c(z_1),\ldots,c(z_R)\big),\qquad
    \sigma^+=(\sigma^+_1,\ldots,\sigma^+_R),
\]
where \(\boldsymbol{\sigma}(c)\) is the output signature of candidate \(c\), and
\(\sigma^+\) is the correct output signature.

\begin{proposition}[Separation of the correct output signature]
\label{prop:signature-mode-separation}
Suppose \(\alpha:=\mathbb P(C^+)>0\), and there exists
\(\beta\in[0,1)\) such that every incorrect code \(c\), output value \(y\),
and \(r\in[1,R]\) satisfy \(\mathbb P(c(z_r)=y\mid c)\le\beta\). Then, for
sufficiently large \(R\), \(\sigma^+\) is the unique maximum-probability
signature. Moreover, for any such \(R\), if the signature space is finite and
candidates are sampled independently, then the largest empirical cluster
converges almost surely to the correct output-consensus cluster.
\end{proposition}

\begin{proof}
A correct code has output signature \(\sigma^+\), so
\(\mathbb P(\sigma(c)=\sigma^+\mid C^+)=1\). By the law of total probability,
\begin{equation}
\label{eq:true-signature-lower}
\begin{aligned}
    \mathbb P\big(\sigma(c)=\sigma^+\big)
    &=
    \mathbb P(C^+)\mathbb P\big(\sigma(c)=\sigma^+\mid C^+\big)
    +\mathbb P(C^-)\mathbb P\big(\sigma(c)=\sigma^+\mid C^-\big)
    \\
    &\ge
    \mathbb P(C^+)\mathbb P\big(\sigma(c)=\sigma^+\mid C^+\big)
    =
    \alpha .
\end{aligned}
\end{equation}
Let \(\mathcal S_R\) be the length-\(R\) output-signature space, write
\(\mathcal S_R^-:=\mathcal S_R\setminus\{\sigma^+\}\), and take any incorrect
signature \(\sigma^-=(\sigma^-_1,\ldots,\sigma^-_R)\in\mathcal S_R^-\).
A correct code never produces an incorrect signature, so
\[
    \mathbb P\big(\sigma(c)=\sigma^-\mid C^+\big)=0.
\]

The anti-concentration condition and the independence of \(z_1,\ldots,z_R\)
give
\[
    \mathbb P\big(\sigma(c)=\sigma^-\mid C^-\big)
    =
    \mathbb E_{c\mid C^-}
    \left[
        \mathbb P\big(\sigma(c)=\sigma^-\mid c\big)
    \right]
    =
    \mathbb E_{c\mid C^-}
    \left[
        \prod_{r=1}^R
        \mathbb P\big(c(z_r)=\sigma^-_r\mid c\big)
    \right]
    \le
    \beta^R .
\]
Again by the law of total probability,
\begin{equation}
\label{eq:wrong-signature-upper}
\begin{aligned}
    \mathbb P\big(\sigma(c)=\sigma^-\big)
    &=
    \mathbb P(C^+)\mathbb P\big(\sigma(c)=\sigma^-\mid C^+\big)
    +\mathbb P(C^-)\mathbb P\big(\sigma(c)=\sigma^-\mid C^-\big)
    \\
    &=
    (1-\alpha)\mathbb P\big(\sigma(c)=\sigma^-\mid C^-\big)\le
    (1-\alpha)\beta^R .
\end{aligned}
\end{equation}
Combining Eqs.~\eqref{eq:true-signature-lower} and
\eqref{eq:wrong-signature-upper}, and taking the supremum over incorrect
signatures, gives
\[
    \mathbb P\big(\sigma(c)=\sigma^+\big)
    -
    \sup_{\sigma^-\in\mathcal S_R^-}
    \mathbb P\big(\sigma(c)=\sigma^-\big)
    \ge
    \alpha-(1-\alpha)\beta^R
    \xrightarrow{R\to\infty}
    \alpha > 0.
\]
Therefore, for sufficiently large \(R\), \(\sigma^+\) has strictly larger
probability than any incorrect signature and is the unique population mode.

Take such an \(R\), and let
\(\delta_R:=\mathbb P(\sigma(c)=\sigma^+)
-\sup_{\sigma^-\in\mathcal S_R^-}
\mathbb P(\sigma(c)=\sigma^-)>0\). If candidates \(c_1,\ldots,c_n\) are sampled
independently from \(\mathcal C_{\mathrm{high}}\), then for any incorrect
signature \(\sigma^-\), the strong law of large numbers gives
\[
    \frac1n
    \sum_{i=1}^n
    \left(
    \mathbf 1\{\sigma(c_i)=\sigma^+\}
    -
    \mathbf 1\{\sigma(c_i)=\sigma^-\}
    \right)
    \xrightarrow{\mathrm{a.s.}}
    \mathbb P\big(\sigma(c)=\sigma^+\big)
    -
    \mathbb P\big(\sigma(c)=\sigma^-\big)
    \ge
    \delta_R .
\]
Since \(\mathcal S_R\) is finite, this convergence holds for all incorrect
signatures simultaneously. Hence, almost surely, there exists \(n_0\) such that
for every \(n\ge n_0\), the empirical cluster size of \(\sigma^+\) is strictly
larger than that of every incorrect signature. Therefore the largest empirical
cluster is eventually the correct output-consensus cluster.
\end{proof}

\clearpage
\section{Detailed Algorithm Design}

\algnewcommand\algorithmicinput{\textbf{Input:}}
\algnewcommand\algorithmicoutput{\textbf{Output:}}
\algnewcommand\Input{\item[\algorithmicinput]}
\algnewcommand\Output{\item[\algorithmicoutput]}

\begin{algorithm}[H]
\caption{The CoSPlay Framework}
\label{alg:cosplay}
\begin{algorithmic}[1]
\setlength{\itemsep}{0.12em}

\Input Problem $\mathcal{P}$, pool sizes $N_c,N_t$, clustering input size $R$, max self-play iterations $T_{\max}$
\Output The best generated code candidate $c^*$

\Statex \textcolor{gray}{\textbf{\textit{// Stage 1: Code and UT Idea Exploration}}}
\State $\mathcal{H} \gets \text{Generate high-level solution hints from } \mathcal{P} \text{ via LLM}$
\State $\Omega \gets \{\omega \subseteq \mathcal{H} \mid |\omega| \in \{1,2\}\}$
\State $\mathcal{S} \gets \bigcup_{\omega \in \Omega} \text{Generate detailed plans from }(\mathcal{P},\omega)\text{ via LLM}$
\State $\mathcal{A} \gets \bigcup_{s \in \mathcal{S}} \text{Generate failure-oriented UT ideas from }s$

\Statex \textcolor{gray}{\textbf{\textit{// Stage 2: Execution-Guided Iterative Self-Play}}}
\State $\mathcal{C} \gets \text{Generate } N_c \text{ code candidates conditioned on plans in } \mathcal{S}$
\State $\mathcal{T} \gets \text{Initialize } N_t \text{ UTs (half random valid, half from } \mathcal{A}\text{) via self-consistency}$
\Statex \textcolor{gray}{\textit{// \textsc{Refresh}$(\mathcal{C},\mathcal{T})$: compute $\mathbf{M}_{ij}=\mathbb{I}(\mathrm{Exec}(c_i,x_j)=y_j)$, $p_j^{\mathrm{UT}}=\sum_i\mathbf{M}_{ij}$, $p_i^{\mathrm{code}}=\sum_j\mathbf{M}_{ij}$, $\alpha_j^{\mathrm{UT}}=p_j^{\mathrm{UT}}/N_c$, and $\alpha_i^{\mathrm{code}}=p_i^{\mathrm{code}}/N_t$.}}
\Statex \textcolor{gray}{\textit{// After each step in self-play, \textsc{Refresh} is called to refresh M.}}

\For{$t = 1$ \textbf{to} $T_{\max}$}
    \If{all $\mathbf{M}_{ij}=1$}
        \State \textbf{break}
    \EndIf

    \Statex \textcolor{gray}{\textit{// Step 1: Code Cleaning}}
    \State $\mathcal{C}_{\mathrm{fail}} \gets \{c_i \in \mathcal{C} \mid \alpha_i^{\mathrm{code}}=0\}$
    \State $\mathcal{C} \gets (\mathcal{C}\setminus\mathcal{C}_{\mathrm{fail}})\cup \{\text{Generate }|\mathcal{C}_{\mathrm{fail}}|\text{ new codes from }\mathcal{S}\}$

    \Statex \textcolor{gray}{\textit{// Step 2: Breaking Spurious Code--UT Coupling}}
    \If{there exists $j$ such that $0<\alpha_j^{\mathrm{UT}}<1$}
        \State $j_{\mathrm{worst}} \gets \arg\min_j \{\alpha_j^{\mathrm{UT}} \mid 0<\alpha_j^{\mathrm{UT}}<1\}$
        \State $\mathcal{T} \gets (\mathcal{T}\setminus\{\mathcal{T}_{j_{\mathrm{worst}}}\}) \cup \{\text{Regenerate UT from }\mathcal{A}\text{ using }\mathcal{T}_{j_{\mathrm{worst}}}\text{ and its passed codes}\}$
    \EndIf

    \Statex \textcolor{gray}{\textit{// Step 3: Code Fixing}}
    \If{there exists $j$ such that $0<\alpha_j^{\mathrm{UT}}<1$}
        \State $j_{\mathrm{best}} \gets \arg\max_j \{\alpha_j^{\mathrm{UT}} \mid 0<\alpha_j^{\mathrm{UT}}<1\}$
        \State $\mathcal{C}_{\mathrm{target}} \gets \{c_i \in \mathcal{C} \mid \mathbf{M}_{i,j_{\mathrm{best}}}=0\}$
        \State $\mathcal{C} \gets (\mathcal{C}\setminus\mathcal{C}_{\mathrm{target}})\cup \{\text{Repair } \mathcal{C}_{\mathrm{target}} \text{ using }t_{j_{\mathrm{best}}}\text{ and execution feedback}\}$
    \EndIf

    \Statex \textcolor{gray}{\textit{// Step 4: Replacing Zero-Discrimination UTs}}
    \State $\mathcal{T}_{\mathrm{triv}} \gets \{t_j\in\mathcal{T} \mid \alpha_j^{\mathrm{UT}}\in\{0,1\}\}$
    \If{$\mathcal{T}_{\mathrm{triv}}\neq\emptyset$}
        \State $\mathcal{T} \gets (\mathcal{T}\setminus\mathcal{T}_{\mathrm{triv}})\cup \{\text{Generate }|\mathcal{T}_{\mathrm{triv}}|\text{ new UTs from }\mathcal{A}\text{ via self-consistency}\}$
    \EndIf
\EndFor

\Statex \textcolor{gray}{\textbf{\textit{// Stage 3: Execution-Consensus-Based Cluster Selection}}}
\State $\mathcal{C}_{\mathrm{high}} \gets \{c_i \in \mathcal{C} \mid p_i^{\mathrm{code}}=\max_{\ell}p_{\ell}^{\mathrm{code}}\}$
\State $\mathcal{Z} \gets \text{Generate }R\text{ random valid inputs }\{z_r\}_{r=1}^{R}\text{ via }\mathcal{M}$
\State Let $o_{i}(z_r)=\mathrm{Exec}(c_i,z_r)$, with runtime errors mapped to $\mathrm{ERR}$
\State Partition $\mathcal{C}_{\mathrm{high}}$ into $\mathcal{G}=\{G_1,\dots,G_M\}$ by runtime-error-aware compatibility $c_i\sim c_j$

\For{each cluster $G_m \in \mathcal{G}$}
    \For{each candidate $c_i \in G_m$}
        \State $S_{\mathrm{ind}}(c_i) \gets \sum_{c_j \in G_m \setminus \{c_i\}} \sum_{r=1}^{R} \mathbb{I}\!\left[o_{i}(z_r)=o_{j}(z_r)\neq \mathrm{ERR}\right]$
    \EndFor
    \State $S_{\mathrm{cls}}(G_m) \gets \sum_{c_i \in G_m} S_{\mathrm{ind}}(c_i)$
\EndFor

\State $G^* \gets \arg\max_{G_m \in \mathcal{G}} S_{\mathrm{cls}}(G_m)$
\State $c^* \gets \arg\max_{c_i \in G^*} S_{\mathrm{ind}}(c_i)$ \Comment{Break ties}
\State \Return $c^*$
\end{algorithmic}
\end{algorithm}

\section{Related Work}
\label{sec:related_work}
\subsection{Fine-tuning Methods for Code Generation}

For code generation, most existing training-based methods rely heavily on external ground-truth data, approaches like $\mu$Code~\cite{jain2025mucode} and AceCoder~\cite{zeng2025acecoder} employ Supervised Fine-Tuning using ground-truth or synthetic datasets, while Focused DPO~\cite{zhang2025focuseddpo} and CTRL~\cite{xie2025ctrl} utilize RLVR by optimizing rewards design. Similarly, O1-Coder~\cite{zhang2024o1coder} and UTGen~\cite{prasad2025utgen} fine-tune models specifically for UT generation by leveraging ground-truth code. Some recent approaches aim to co-evolve these capabilities using RLVR but often retain critical dependencies: UTRL~\cite{lee2025utrl} requires ground-truth code, and CURE~\cite{wang2025cure} relies on ground-truth UT. While Absolute-Zero~\cite{zhao2025absolutezero} achieves purely self-supervised learning without external data, it demands prohibitive computational resources for RLVR training.

\subsection{TTS for LLM Code Generation}

Test-Time Scaling (TTS) strategies typically begin with repeated sampling, such as Best-of-N (BoN), but they often struggle to identify the best candidate without GT data. To improve selection, consensus-based methods such as CodeT~\cite{chen2022codet} and MPSC~\cite{huang2024mpsc} use self-generated UTs for voting, although CodeT is limited by the potential unreliability of generated tests and MPSC incurs substantial overhead from extensive sampling. More broadly, sampling-based inference-time methods such as PowerSampling improve reasoning by drawing from sharpened sequence-level distributions without additional training, suggesting that stronger performance can sometimes be elicited purely through better test-time sampling~\cite{karan2025reasoningsamplingbasemodel}. Beyond simple resampling, a growing line of work treats code generation as a structured search problem. SFS formulates code generation as black-box optimization in code space and improves exploration diversity to avoid local optima~\cite{sfs}, while CodeTree organizes planning, generation, and debugging within a unified tree structure guided by both execution feedback and agent feedback~\cite{codertee}. ThinkCoder further decomposes test-time computation into a thorough exploration stage followed by an optimal refinement stage, improving efficiency by first broadening the solution space and then sharpening promising candidates~\cite{thinkbeforecoding}. Finally, execution-guided approaches leverage runtime feedback for iterative improvement. Methods such as ORPS~\cite{yu2025orps}, S*~\cite{li2025s*}, LDB~\cite{ldb}, and Reflexion~\cite{shinn2023reflexion} use execution signals for process supervision or refinement. Another paper~\cite{revisit} further shows that self-generated tests can support iterative debugging, while also highlighting the bias and reliability issues that arise when such tests are used naively.

\subsection{Self-Play in LLM}

The self-play paradigm, popularized by AlphaZero ~\cite{silver2017alphazero} and AlphaGo~\cite{silver2016alphago}, has recently been adapted to optimize LLMs via RL. However, most existing approaches in the coding domain remain dependent on external supervision and heavy training. For instance, CURE ~\cite{wang2025cure} relies on GT UT to guide the co-evolution of generation capability of code and UT, while UTRL ~\cite{lee2025utrl} necessitates ground-truth code for improving model performance. Although Absolute-Zero ~\cite{zhao2025absolutezero} achieves a fully self-supervised setup without human data, it yields limited performance gains despite the prohibitive training costs of RLVR training. Crucially, these methods are exclusively training-time optimizations. In contrast, CoSPlay explores inference-time self-play for GT-free code generation, leveraging the pass count in execution metrix as signals  to iteratively refine and filter candidates without parameter updates.

\subsection{Search in Natural Language Space}

To enhance the reasoning capabilities of LLMs, recent studies have proposed Chain-of-Thought~\cite{wei2023cot} strategies into structured search processes within the natural language space. Frameworks such as Tree of Thoughts~\cite{yao2023treeofthought} and Reasoning via Planning~\cite{hao2023rap} treat reasoning steps as a search tree, exploring various intermediate thoughts to reach a solution. In the specific domain of code generation, methods like PlanSearch~\cite{wang2025plansearch} demonstrate that generating high-level natural language plans before implementation significantly outperforms direct code generation. The primary benefit of such natural language search lies in its ability to introduce structural diversity, preventing the model from converging to a narrow set of repetitive solutions—a common issue in direct code resampling. However, most existing methods for code generation remain static, generating plans without leveraging dynamic execution feedback or iterative self-play to refine.


 
\section{Detailed baseline setting}
\label{app:tts_setting}

\textbf{RL Baselines.} Detailed data usage and base models are listed in Table~\ref{tab:RL_baselines}. We follow the original prompts and LLM parameters as provided in each method's code or paper.

\textbf{TTS Baselines.} For a fair comparison, we control the total code and UT generation budgets to be comparable across methods, and all UTs are generated by the model itself without external supervision. Specific settings are as follows:

\textbf{(1) SFS}~\cite{sfs}: num\_seed = 10, max\_iters = 40 (total budget of 50) and we scale to 80, generated UT = 6; we use the \texttt{Role} seed theme, the best-performing variant reported in the original paper. The sampling temperature is $0.2$.

\textbf{(2) CodeTree}~\cite{codertee}: full \texttt{agent} setting with max budget = 50, generated UT = 6, max depth = 5, search width = 10. The scaled-up setting with max budget = 1000, generated UT = 16, max depth = 20, search width = 100. The sampling temperature is $0.0$.

\textbf{(3) MPSC}~\cite{huang2024mpsc}: solution = 60, specification = 50, test case = 100; we use \texttt{MPSC-label} mode, the best-performing mode for code generation reported in the original paper. The scaled-up setting with solution = 300, specification = 250, test case = 600. The sampling temperature is $0.8$.  

\textbf{(4) PowerSampling}~\cite{karan2025reasoningsamplingbasemodel}: n\_samples = 16, mcmc\_steps = 10, code sampling and UT generation temperature = 0.6.

\textbf{(5) ThinkCoder}~\cite{thinkbeforecoding}: $k=20$ per iteration, each candidate prompted to generate 3 UTs,  temperature is 1.0. budget $n$ to 2, and we scale to 5/10/20

\textbf{(6) S*}~\cite{li2025s*}: we use the local S* pipeline with generated UT = 16, $N=16$ code samples per round, $\texttt{num\_round}=3$ and we scale to 6. The sampling temperature is $0.7$. 

\textbf{(7) BoN}: we set N = 16/64/128/256 for both code and UT. 
\begin{table}[htbp]
\centering
\caption{ Training Data Source and Base Model for baselines.}
\label{tab:RL_baselines}
\setlength{\tabcolsep}{2pt} 
\renewcommand{\arraystretch}{1.2}
\begin{tabular}{llll}
\toprule
\textbf{Model} & \textbf{Base model} & \textbf{Data} & \textbf{training}\\
\midrule
    AceCoder-7B-Coder-Rule & Qwen2.5-7B-Coder-Instruct & 22k code data & Yes\\
    AceCoder-7B-Coder-RM & Qwen2.5-7B-Coder-Instruct & 22k code data + RM(307k data) & Yes\\
    Absolute-Zero-Coder-7B & Qwen-2.5-7B-Coder & \textbf{0} & Yes\\
    CURE-7B & Qwen2.5-7B-Instruct & 4.5k CodeContests data & Yes\\
    CoSPlay-7B & Qwen2.5-7B-Instruct & \textbf{0} & \textbf{No}\\
    
    \midrule 
    
    Absolute-Zero-Coder-14B & Qwen-2.5-14B-Coder & \textbf{0} & Yes\\
    CURE-14B & Qwen2.5-14B-Instruct & 4.5k CodeContests data & Yes\\
    CoSPlay-14B & Qwen2.5-14B-Instruct & \textbf{0} & \textbf{No}\\
    
    \midrule

    AceCoder-RM-7B & Qwen2.5-Coder-7B-Instruct & 307k preference pairs & Yes\\
    AceCoder-RM-32B & Qwen2.5-Coder-32B-Instruct & 307k preference pairs & Yes\\

\bottomrule
\end{tabular}
\end{table}

\begin{table*}[!htbp]
    \centering
    \tiny
    \renewcommand{\arraystretch}{0.9}
    \caption{\textbf{Detailed TTS computational cost breakdown on Qwen2.5-7B-Instruct.} The columns are defined as follows: \textbf{Method} indicates the evaluated baseline or our proposed method; \textbf{Code} and \textbf{UT} denote the average numbers of generated code candidates and unit tests per problem, respectively; \textbf{Avg. Call} and \textbf{Avg. Token} measure the average number of LLM calls and generated tokens across all evaluated problems; \textbf{Succ. Call} and \textbf{Succ. Token} report the corresponding averages over successfully solved problems. Finally, \textbf{Pass@1} denotes the final pass@1 accuracy of each method; for CoSPlay variants, it is reported after applying the cluster step. \textbf{Bold} means the best performance.}
    \label{tab:TTS_cost_7B}
    
    \resizebox{\textwidth}{!}{%
    \begin{tabular}{@{} l *{7}{c} @{}}
        \toprule
        \multirow{2}{*}{\textbf{Method}} & \multirow{2}{*}{\textbf{Code}} & \multirow{2}{*}{\textbf{UT}} & \textbf{Avg.} & \textbf{Avg.} & \textbf{Succ.} & \textbf{Succ.} & \multirow{2}{*}{\textbf{Pass@1}} \\
        & & & \textbf{Call} & \textbf{Token} & \textbf{Call} & \textbf{Token} & \\
        \midrule
        
        \rowcolor[gray]{0.92} \multicolumn{8}{c}{\textbf{Qwen2.5-7B-Instruct}} \\
        \midrule
        Direct              & 1.0   & 0.0   & 1.0     & 628.3      & 1.0     & 515.6      & 22.7 \\
        BoN (N=16)          & 16.0  & 16.0  & 32.0    & 24,478.1   & 32.0    & 20,160.3   & 26.5 \\
        BoN (N=64)          & 64.0  & 64.0  & 128.0   & 90,952.9   & 128.0   & 72,583.1   & 24.8 \\
        BoN (N=128)         & 128.0 & 128.0 & 256.0   & 344,674.5  & 256.0   & 260,191.6  & 27.2 \\
        BoN (N=256)         & 256.0 & 256.0 & 512.0   & 745,385.6  & 512.0   & 555,398.1  & 22.5 \\
        SFS (Round 20)      & 24.1  & 6.0   & 43.7    & 17,501.6   & 33.2    & 11,634.4   & 31.0 \\
        SFS (Round 40)      & 37.3  & 6.0   & 69.0    & 28,677.3   & 43.4    & 15,823.4   & 31.8 \\
        SFS    (Round 80)     & 55.2 & 6.0   & 122.0  & 50,951.9   & 100.8 & 45,871.3 & 32.0 \\
        S* (Round 3) & 48.0 & 16.0 & 70.7 & 23,973.3 & 67.8 & 15,707.0 & 20.5 \\
        S* (Round 6) & 96.0 & 16.0 & 121.7 & 42,945.7 & 118.8 & 29,509.3 & 19.2 \\
        ThinkCoder (Round 2)         & 39.4 & 118.2 & 40.4 & 116,646.3 & 39.2 & 96,088.0 & 25.3 \\
        ThinkCoder (Round 5)        & 97.1  & 291.3  & 99.6    & 297,219.4  & 95.6    & 253,272.6  & 25.5 \\
        ThinkCoder (Round 10)         & 189.6 & 568.7 & 192.7 & 565,558.7 & 176.0 & 476,425.6 & 25.5 \\
        ThinkCoder (Round 20)         & 382.9 & 1148.6 & 386.9 & 1,156,593.2 & 361.8 & 971,276.7 & 27.8 \\
        CodeTree            & 15.6  & 5.8   & 38.2    & 13,173.5   & 35.3    & 9,068.0    & 15.7 \\
        CodeTree (Scale)  & 907.3 & 15.5  & 1,849.5 & 638,268.8  & 1,868.7 & 402,823.6  & 14.3 \\
        MPSC                & 60.0  & 99.5  & 6.8     & 104,709.6  & 6.9     & 97,773.8   & 25.3 \\
        MPSC (Scale) & 300.0 & 478.4 & 18.8    & 661,763.8  & 18.9    & 612,221.5  & 23.5 \\
        PowerSampling       & 21.2  & 0.0   & 49.3    & 14,762.1   & 41.2    & 9,957.4    & 25.0 \\
        PowerSampling + BoN & 338.7 & 338.7 & 1,477.2 & 416,238.5  & 1,264.3 & 303,952.7  & 27.5 \\
        \hline
        \textbf{CoSPlay (Ours)} & 38.5 & 206.0 & 776.1 & 682,292.9 & 603.4 & 479,437.9 & \textbf{37.2} \\
        \quad \textbf{w/o self-cons.} & 41.0 & 74.9 & 395.3 & 243,800.4 & 374.4 & 208,419.8 & 34.2 \\
        \quad \textbf{w/o self-cons. \& UT-iteration} & 41.4 & 24.0 & 198.4 & 90,291.4 & 183.8 & 68,641.2 & 32.5 \\
        
        \bottomrule
    \end{tabular}
    }
\end{table*}
\vspace{-1em}
\begin{table*}[!htbp]
    \centering
    \tiny
    \renewcommand{\arraystretch}{0.85}
    \caption{\textbf{Detailed TTS cost breakdown on Qwen2.5-14B-Instruct.} The column definitions follow Table~\ref{tab:TTS_cost_7B}.}
    \label{tab:TTS_cost_14B}
    
    \resizebox{\textwidth}{!}{%
    \begin{tabular}{@{} l *{7}{c} @{}}
    \toprule
    \multirow{2}{*}{\textbf{Method}} 
    & \multirow{2}{*}{\textbf{Code}} 
    & \multirow{2}{*}{\textbf{UT}} 
    & \textbf{Avg.} 
    & \textbf{Avg.} 
    & \textbf{Succ.} 
    & \textbf{Succ.} 
    & \multirow{2}{*}{\textbf{Pass@1}} \\
    & & & \textbf{Call} & \textbf{Token} & \textbf{Call} & \textbf{Token} & \\
    \midrule
    \rowcolor[gray]{0.92} \multicolumn{8}{c}{\textbf{Qwen2.5-14B-Instruct}} \\
    \midrule
        Direct              & 1.0   & 0.0   & 1.0    & 568.0      & 1.0    & 472.8      & 26.5 \\
        BoN (N=16)          & 16.0  & 16.0  & 32.0   & 22,102.9   & 32.0   & 19,505.9   & 36.8 \\
        BoN (N=64)          & 64.0  & 64.0  & 128.0  & 86,830.2   & 128.0  & 77,246.0   & 38.8 \\
        BoN (N=128)         & 128.0 & 128.0 & 256.0  & 328,727.6  & 256.0  & 293,841.5  & 40.5 \\
        BoN (N=256)         & 256.0 & 256.0 & 512.0  & 710,314.1  & 512.0  & 635,587.2  & 40.3 \\
        SFS (Round 20) & 26.8 & 6.0 & 50.6 & 19,341.6 & 47.1 & 15,658.6 & 36.3 \\
        SFS (Round 40) & 42.8 & 6.0 & 76.3 & 31,629.9 & 68.1 & 24,228.1 & 37.9 \\
        SFS (Round 80) & 60.7 & 6.0 & 131.9 & 56,802.1 & 157.7 & 59,382.7 & 38.7 \\
        S* (Round 3) & 48.0 & 16.0 & 51.3 & 23,787.8 & 51.1 & 18,855.8 & 31.7 \\
        S* (Round 6) & 96.0 & 16.0 & 118.3 & 45,672.6 & 114.6 & 35,259.8 & 31.9 \\
        ThinkCoder (Round 2)         & 39.3 & 117.8 & 39.3 & 64,189.4 & 39.3 & 57,505.0 & 35.8 \\
        ThinkCoder (Round 5)          & 96.8 & 290.5 & 96.8 & 165,598.4 & 93.8 & 148,618.3 & 36.2 \\
        ThinkCoder (Round 10)         & 190.8 & 572.5 & 190.8 & 379,948.4 & 184.1 & 370,315.7 & 38.5 \\
        ThinkCoder (Round 20)         & 375.9 & 1127.6 & 375.9 & 786,704.7 & 361.6 & 798,329.0 & 38.7 \\
        CodeTree            & 14.4  & 5.9   & 33.9   & 12,056.6   & 23.4   & 7,078.5    & 22.8 \\
        CodeTree (Scale) & 842.0 & 15.9 & 1711.4 & 525,068.9 & 1552.2 & 402,579.9 & 25.6 \\
        MPSC                & 60.0  & 93.8  & 6.5    & 126,328.7  & 6.8    & 142,208.8  & 27.7 \\
        MPSC (Scale) & 300.0 & 412.7 & 18.4   & 710,164.0  & 18.9   & 647,139.1  & 31.7 \\
        PowerSampling       & 18.2  & 0.0   & 43.1   & 11,451.3   & 37.2   & 8,488.4    & 27.2 \\
        PowerSampling + BoN & 282.6 & 282.6 & 1273.2 & 320,064.1  & 1,084.4 & 242,852.4 & 34.8 \\
        \hline
        \textbf{CoSPlay (Ours)} & 36.2 & 171.3 & 763.7 & 838,745.7 & 664.8 & 748,581.2 & \textbf{45.0} \\
        \quad \textbf{w/o self-cons.} & 39.0 & 83.7 & 328.4 & 254,158.5 & 314.8 & 232,620.7 & 43.7 \\
        \quad \textbf{w/o self-cons. \& UT-iteration} & 41.8 & 24.0 & 202.9 & 111,134.0 & 184.8 & 89,815.0 & 41.5 \\
        
        \bottomrule
    \end{tabular}%
    }
    
    
\end{table*}

\label{app:detailed_baseline_setting}

\section{Analysis of the cost and performance balance for TTS methods}
\label{app:tts_analysis}

\begin{table*}[t]
\centering
\scriptsize
\renewcommand{\arraystretch}{1.19}
\setlength{\tabcolsep}{2pt}
\caption{\textbf{Performance comparison of TTS methods}. \textbf{pass@1} reports the accuracy of the final selected solution. \textbf{Code} and \textbf{UT} denote the average numbers of sampled code candidates and sampled unit tests, respectively. \textbf{Avg.} represents the average performance across all four datasets. \textbf{Bold} means the best performance.}
\label{tab:tts_performance}

\resizebox{\textwidth}{!}{%
\begin{tabular}{lc|ccc|ccc|ccc|ccc|ccc}
\toprule
\textbf{Method} & \textbf{Venue} & \multicolumn{3}{c|}{\textbf{LB}} & \multicolumn{3}{c|}{\textbf{LCB}} & \multicolumn{3}{c|}{\textbf{CC}} & \multicolumn{3}{c|}{\textbf{CF}} & \multicolumn{3}{c}{\textbf{Avg.}} \\
& & pass@1 & \#Code & \#UT & pass@1 & \#Code & \#UT & pass@1 & \#Code & \#UT & pass@1 & \#Code & \#UT & pass@1 & \#Code & \#UT \\
\midrule

\rowcolor{gray!15}
\multicolumn{17}{l}{\textbf{Qwen2.5-7B-Instruct}} \\
\midrule

Direct & - & 36.0 & 1.0 & 0.0 & 27.1 & 1.0 & 0.0 & 21.3 & 1.0 & 0.0 & 6.5 & 1.0 & 0.0 & 22.7 & 1.0 & 0.0 \\
BoN (N=16) & - & 38.0 & 16.0 & 16.0 & 32.7 & 16.0 & 16.0 & 25.3 & 16.0 & 16.0 & 10.0 & 16.0 & 16.0 & 26.5 & 16.0 & 16.0 \\
BoN (N=64) & - & 37.3 & 64.0 & 64.0 & 27.3 & 64.0 & 64.0 & 26.0 & 64.0 & 64.0 & 8.7 & 64.0 & 64.0 & 24.8 & 64.0 & 64.0 \\
BoN (N=128) & - & 44.0 & 128.0 & 128.0 & 30.0 & 128.0 & 128.0 & 24.7 & 128.0 & 128.0 & 10.0 & 128.0 & 128.0 & 27.2 & 128.0 & 128.0 \\
BoN (N=256) & - & 36.0 & 256.0 & 256.0 & 24.0 & 256.0 & 256.0 & 24.0 & 256.0 & 256.0 & 6.0 & 256.0 & 256.0 & 22.5 & 256.0 & 256.0 \\
MPSC & ACL24 Long & 39.3 & 60.0 & 99.2 & 30.7 & 60.0 & 99.2 & 23.3 & 60.0 & 99.6 & 8.0 & 60.0 & 100.0 & 25.3 & 60.0 & 99.5 \\
MPSC (Scale) & ACL24 Long & 41.3 & 300.0 & 392.7 & 29.3 & 300.0 & 406.5 & 18.7 & 300.0 & 546.6 & 4.7 & 300.0 & 567.6 & 23.5 & 300.0 & 478.4 \\
CodeTree & NAACL25 Findings & 18.0 & 15.0 & 5.9 & 14.0 & 16.3 & 5.9 & 22.7 & 16.1 & 5.8 & 8.0 & 15.0 & 5.8 & 15.7 & 15.6 & 5.8 \\
CodeTree (Scale) & NAACL25 Findings & 15.0 & 973.8 & 15.7 & 13.3 & 933.6 & 15.7 & 20.7 & 924.6 & 15.4 & 8.1 & 797.2 & 15.3 & 14.3 & 907.3 & 15.5 \\
SFS (Round 20) & ICLR25 Poster & 38.0 & 21.4 & 6.0 & 38.0 & 19.7 & 6.0 & 33.3 & 25.8 & 6.0 & 14.7 & 29.4 & 6.0 & 31.0 & 24.1 & 6.0 \\
SFS (Round 40) & ICLR25 Poster & 38.0 & 31.9 & 6.0 & 40.0 & 27.6 & 6.0 & 34.7 & 41.0 & 6.0 & 14.7 & 48.6 & 6.0 & 31.8 & 37.3 & 6.0 \\
SFS (Round 80) & ICLR25 Poster & 38.0 & 33.5 & 6.0 & 40.0 & 32.0 & 6.0 & 34.7 & 69.8 & 6.0 & 15.3 & 85.5 & 6.0 & 32.0 & 55.2 & 6.0 \\
S* (Round 3) & EMNLP 2025 Findings & 27.3 & 48.0 & 16.0 & 26.7 & 48.0 & 16.0 & 20.7 & 48.0 & 16.0 & 7.3 & 48.0 & 16.0 & 20.5 & 48.0 & 16.0 \\
S* (Round 6) & EMNLP 2025 Findings & 26.0 & 96.0 & 16.0 & 26.0 & 96.0 & 16.0 & 20.0 & 96.0 & 16.0 & 4.7 & 96.0 & 16.0 & 19.2 & 96.0 & 16.0 \\
ThinkCoder (Round 2) & ACL25 Findings & 36.7 & 39.2 & 117.6 & 32.7 & 38.9 & 116.8 & 27.3 & 39.5 & 118.4 & 4.7 & 40.0 & 120.0 & 25.3 & 39.4 & 118.2 \\
ThinkCoder (Round 5) & ACL25 Findings & 36.7 & 96.4 & 289.2 & 33.3 & 94.4 & 283.2 & 26.7 & 97.6 & 292.8 & 5.3 & 100.0 & 300.0 & 25.5 & 97.1 & 291.3 \\
ThinkCoder (Round 10) & ACL25 Findings & 38.0 & 191.1 & 573.2 & 34.0 & 185.7 & 557.2 & 24.7 & 181.5 & 544.4 & 5.3 & 200.0 & 600.0 & 25.5 & 189.6 & 568.7 \\
ThinkCoder (Round 20) & ACL25 Findings & 38.7 & 378.9 & 1136.8 & 36.0 & 366.9 & 1100.8 & 29.3 & 387.6 & 1162.8 & 7.3 & 398.0 & 1194.0 & 27.8 & 382.9 & 1148.6 \\
PowerSampling & ICLR26 Oral & 37.3 & 20.6 & 0.0 & 31.3 & 23.6 & 0.0 & 22.0 & 20.3 & 0.0 & 9.3 & 20.3 & 0.0 & 25.0 & 21.2 & 0.0 \\
\quad + BoN & ICLR26 Oral & 38.0 & 308.0 & 308.0 & 34.0 & 332.9 & 332.9 & 27.3 & 345.2 & 345.2 & 10.7 & 368.6 & 368.6 & 27.5 & 338.7 & 338.7 \\

\midrule

\textbf{CoSPlay} & - & \textbf{51.3} & 28.9 & 194.2 & \textbf{41.3} & 29.6 & 193.6 & 36.0 & 36.7 & 208.7 & \textbf{18.0} & 58.9 & 227.7 & 36.7 & 38.5 & 206.0 \\
\quad + Cluster & - & \textbf{51.3} & 28.9 & 194.2 & \textbf{41.3} & 29.6 & 193.6 & \textbf{38.0} & 36.7 & 208.7 & \textbf{18.0} & 58.9 & 227.7 & \textbf{37.2} & 38.5 & 206.0 \\
\hline
\textbf{w/o self cons.} & - & 48.0 & 30.5 & 70.8 & 38.0 & 30.1 & 67.2 & 30.7 & 46.2 & 74.6 & 16.7 & 57.2 & 86.9 & 33.3 & 41.0 & 74.9 \\
\quad + Cluster & - & 50.0 & 30.5 & 70.8 & 39.3 & 30.1 & 67.2 & 30.7 & 46.2 & 74.6 & 16.7 & 57.2 & 86.9 & 34.2 & 41.0 & 74.9 \\
\hline
\textbf{w/o  self cons. \& UT-iteration} & - & 43.3 & 30.9 & 24.0 & 34.0 & 29.8 & 24.0 & 33.3 & 41.5 & 24.0 & 16.7 & 63.4 & 24.0 & 31.8 & 41.4 & 24.0 \\
\quad + Cluster & - & 44.7 & 30.9 & 24.0 & 34.7 & 29.8 & 24.0 & 34.0 & 41.5 & 24.0 & 16.7 & 63.4 & 24.0 & 32.5 & 41.4 & 24.0 \\

\midrule

\rowcolor{gray!15}
\multicolumn{17}{l}{\textbf{Qwen2.5-14B-Instruct}} \\
\midrule

Direct & - & 41.0 & 1.0 & 0.0 & 34.2 & 1.0 & 0.0 & 22.9 & 1.0 & 0.0 & 7.9 & 1.0 & 0.0 & 26.5 & 1.0 & 0.0 \\
BoN (N=16) & - & 56.0 & 16.0 & 16.0 & 44.0 & 16.0 & 16.0 & 31.3 & 16.0 & 16.0 & 16.0 & 16.0 & 16.0 & 36.8 & 16.0 & 16.0 \\
BoN (N=64) & - & 58.7 & 64.0 & 64.0 & 46.7 & 64.0 & 64.0 & 31.3 & 64.0 & 64.0 & 18.7 & 64.0 & 64.0 & 38.8 & 64.0 & 64.0 \\
BoN (N=128) & - & 60.7 & 128.0 & 128.0 & 45.3 & 128.0 & 128.0 & 34.7 & 128.0 & 128.0 & 21.3 & 128.0 & 128.0 & 40.5 & 128.0 & 128.0 \\
BoN (N=256) & - & 60.0 & 256.0 & 256.0 & 45.3 & 256.0 & 256.0 & 36.0 & 256.0 & 256.0 & 20.0 & 256.0 & 256.0 & 40.3 & 256.0 & 256.0 \\
MPSC & ACL24 Long & 42.0 & 60.0 & 99.2 & 33.3 & 60.0 & 88.6 & 23.3 & 60.0 & 99.3 & 12.0 & 60.0 & 88.0 & 27.7 & 60.0 & 93.8 \\
MPSC (Scale) & ACL24 Long & 44.7 & 300.0 & 389.8 & 42.7 & 300.0 & 347.0 & 26.7 & 300.0 & 451.7 & 12.7 & 300.0 & 462.4 & 31.7 & 300.0 & 412.7 \\
CodeTree & NAACL25 Findings & 28.7 & 10.7 & 6.0 & 31.3 & 11.7 & 6.0 & 24.7 & 14.0 & 5.9 & 6.7 & 21.4 & 5.9 & 22.8 & 14.4 & 5.9 \\
CodeTree (Scale) & NAACL25 Findings & 33.3 & 839.7 & 15.8 & 34.2 & 846.4 & 15.9 & 26.4 & 829.2 & 15.9 & 8.6 & 852.9 & 15.9 & 25.6 & 842.0 & 15.9 \\
SFS (Round 20) & ICLR25 Poster & 40.0 & 26.5 & 6.0 & 44.7 & 23.3 & 6.0 & 38.0 & 28.2 & 6.0 & 22.7 & 29.2 & 6.0 & 36.3 & 26.8 & 6.0 \\
SFS (Round 40) & ICLR25 Poster & 42.7 & 42.5 & 6.0 & 46.0 & 35.0 & 6.0 & 38.7 & 45.5 & 6.0 & 24.0 & 48.0 & 6.0 & 37.9 & 42.8 & 6.0 \\
SFS (Round 80) & ICLR25 Poster & 44.7 & 43.9 & 6.0 & 46.7 & 37.7 & 6.0 & 38.7 & 75.5 & 6.0 & 24.7 & 85.8 & 6.0 & 38.7 & 60.7 & 6.0 \\

S* (Round 3) & EMNLP 2025 Findings & 44.7 & 48.0 & 16.0 & 34.7 & 48.0 & 16.0 & 32.7 & 48.0 & 16.0 & 14.7 & 48.0 & 16.0 & 31.7 & 48.0 & 16.0 \\
S* (Round 6) & EMNLP 2025 Findings & 45.0 & 96.0 & 16.0 & 34.7 & 96.0 & 16.0 & 33.0 & 96.0 & 16.0 & 14.7 & 96.0 & 16.0 & 31.9 & 96.0 & 16.0 \\
ThinkCoder (Round 2) & ACL25 Findings & 53.3 & 38.8 & 116.4 & 46.0 & 39.5 & 118.4 & 29.3 & 39.2 & 117.6 & 14.7 & 39.6 & 118.8 & 35.8 & 39.3 & 117.8 \\
ThinkCoder (Round 5) & ACL25 Findings & 52.0 & 94.3 & 282.8 & 46.0 & 97.9 & 293.6 & 30.7 & 96.8 & 290.4 & 16.0 & 98.4 & 295.2 & 36.2 & 96.8 & 290.5 \\
ThinkCoder (Round 10) & ACL25 Findings & 55.3 & 186.3 & 558.8 & 46.7 & 189.3 & 568.0 & 34.7 & 191.3 & 574.0 & 17.3 & 196.4 & 589.2 & 38.5 & 190.8 & 572.5 \\
ThinkCoder (Round 20) & ACL25 Findings & 55.3 & 366.9 & 1100.8 & 46.0 & 369.2 & 1107.6 & 34.0 & 377.1 & 1131.2 & 19.3 & 390.3 & 1170.8 & 38.7 & 375.9 & 1127.6 \\
PowerSampling & ICLR26 Oral & 41.3 & 15.3 & 0.0 & 36.7 & 16.1 & 0.0 & 21.3 & 22.0 & 0.0 & 9.3 & 19.2 & 0.0 & 27.2 & 18.2 & 0.0 \\
PowerSampling + BoN & ICLR26 Oral & 46.0 & 252.2 & 252.2 & 44.0 & 263.9 & 263.9 & 35.3 & 316.1 & 316.1 & 14.0 & 298.4 & 298.4 & 34.8 & 282.6 & 282.6 \\

\midrule

\textbf{CoSPlay} & - & 58.7 & 25.7 & 158.0 & 52.7 & 27.7 & 156.8 & \textbf{41.5} & 35.5 & 171.0 & 26.0 & 56.0 & 199.5 & 44.7 & 36.2 & 171.3 \\
\quad + Cluster & - & \textbf{59.0} & 25.7 & 158.0 & \textbf{53.0} & 27.7 & 156.8 & \textbf{41.5} & 35.5 & 171.0 & \textbf{26.7} & 56.0 & 199.5 & \textbf{45.0} & 36.2 & 171.3 \\
\hline
\textbf{w/o self-cons.} & - & 57.3 & 28.2 & 82.8 & 48.7 & 28.7 & 79.7 & 42.0 & 39.3 & 83.1 & 25.7 & 59.6 & 89.1 & 43.4 & 39.0 & 83.7 \\
\quad + Cluster & - & 58.0 & 28.2 & 82.8 & 48.7 & 28.7 & 79.7 & 42.0 & 39.3 & 83.1 & 26.0 & 59.6 & 89.1 & 43.7 & 39.0 & 83.7 \\
\hline
\textbf{w/o self-cons. \& UT-iteration} & - & 56.0 & 29.4 & 24.0 & 49.3 & 32.3 & 24.0 & 32.7 & 40.8 & 24.0 & 26.0 & 64.5 & 24.0 & 41.1 & 41.8 & 24.0 \\
\quad + Cluster & - & 56.7 & 29.4 & 24.0 & 50.0 & 32.3 & 24.0 & 33.3 & 40.8 & 24.0 & 26.0 & 64.5 & 24.0 & 41.5 & 41.8 & 24.0 \\

\bottomrule
\end{tabular}%
}
\vspace{-10pt}
\end{table*}

Table~\ref{tab:tts_performance} reports the per-dataset pass@1 results of each TTS method, while Table~\ref{tab:TTS_cost_7B} and Table~\ref{tab:TTS_cost_14B} summarize their computational costs on Qwen2.5-7B-Instruct and Qwen2.5-14B-Instruct, respectively. 
Overall, CoSPlay achieves the best pass@1 accuracy on both model scales, improving the strongest competing TTS baseline from 32.0\% to 37.2\% on 7B and from 40.5\% to 45.0\% on 14B. 
More importantly, this gain is not obtained by simply sampling more programs. 
On 7B, several scaled baselines use comparable or even larger token budgets than CoSPlay but achieve much lower accuracy: BoN (N=256) uses 745K tokens and obtains 22.5\%, MPSC (Scale) uses 662K tokens and obtains 23.5\%, CodeTree (Scale) uses 638K tokens and obtains 14.3\%, and ThinkCoder (Round 20) uses 1.16M tokens but only reaches 27.8\%. 
In contrast, CoSPlay uses 682K tokens and reaches 37.2\%, showing that its advantage comes from more effective use of test-time computation rather than brute-force scaling.

The ablation results further show a favorable efficiency--performance trade-off. 
The full CoSPlay pipeline includes two token-consuming components: self-consistency validation for UT generation and iterative UT updating during self-play. 
Removing self-consistency reduces the 7B token cost from 682K to 244K while still achieving 34.2\% pass@1, which remains higher than all competing TTS baselines. 
Further removing UT-iteration reduces the cost to 90K tokens and 198 calls, yet the method still achieves 32.5\% pass@1, slightly outperforming the strongest baseline SFS (Round 80) at 32.0\%. 
A similar trend holds on 14B: the lightweight variant reduces the token cost from 839K to 111K and the number of calls from 763.7 to 202.9, while still reaching 41.5\% pass@1, above the strongest competing baseline. 
These results suggest that even when the most expensive modules are removed, the core execution-guided design of CoSPlay remains more efficient than existing TTS strategies under comparable budgets.

This comparison is especially important under our GT-free setting. 
Several prior TTS methods were originally designed to benefit from reliable public tests or external feedback, but in our setting they must rely only on self-generated UTs. 
Without guaranteed-correct tests, their search or refinement signals can become noisy, and scaling them does not necessarily improve performance. 
For example, on 7B, S* drops from 20.5\% at Round 3 to 19.2\% at Round 6, CodeTree and CodeTree (Scale) remain below direct sampling, and scaling MPSC from 60 code candidates to 300 reduces pass@1 from 25.3\% to 23.5\%. 
These results indicate that naively increasing the number of generated codes, UTs, or search steps can amplify noisy self-generated feedback rather than improve selection.

CoSPlay addresses this issue by spending additional computation on improving the quality of both sides of the execution matrix. 
During UT generation, exploration-attack ideas produce more targeted tests for potential failure cases. 
During self-play, low-value codes are pruned, suspicious low-support UTs are refreshed, buggy codes are refined using high-support non-trivial UTs, and zero-discrimination UTs are replaced. 
Thus, the extra computation is used to strengthen the verification signal itself, rather than merely increasing the sample count. 
This explains why the full method outperforms scaled TTS baselines at similar or larger budgets, while its lightweight variants still preserve strong accuracy with substantially lower token usage.

Finally, all competing baselines are evaluated under generous budgets. 
Methods with built-in stopping criteria, such as SFS~\cite{sfs} and CodeTree~\cite{codertee}, are allowed to run until convergence. 
For methods without automatic stopping, we assign large sampling or search budgets to avoid under-exploration. 
Therefore, the comparisons in Table~\ref{tab:tts_performance}, Table~\ref{tab:TTS_cost_7B}, and Table~\ref{tab:TTS_cost_14B} should be understood as comparisons against strong scaled configurations of existing TTS baselines. 
Detailed hyperparameter and budget configurations are provided in Appendix~\ref{app:tts_setting}.

\section{Detailed data about generalization of our method}
\label{app:detailed generalization}
As shown in Table~\ref{tab:generalization_combined}, we provide a detailed breakdown of CoSPlay on different instruct and RLVR-tuned backbones. The evaluated models span from 7B-scale models to DeepSeek-V3.2-685B, allowing us to examine whether the proposed Code-UT co-evolution mechanism transfers across model scales and training paradigms. Averaged over all evaluated backbones, CoSPlay w/ Cluster improves BoN from 37.8\% to 45.3\%, corresponding to a +7.5\% absolute gain.

The table shows three main trends. First, CoSPlay brings large gains on most mid-sized models. For instance, Seed-Coder-8B-Instruct improves from 32.3\% to 42.3\%, and DeepSeek-Coder-V2-Lite-Instruct-16B improves from 30.8\% to 41.3\%. Similar improvements are observed on RLVR-tuned 7B models: Absolute-Zero-7B increases from 24.0\% to 30.2\%, AceCoder-Rule-7B from 29.0\% to 39.5\%, and AceCoder-RM-7B from 29.3\% to 40.5\%. These results indicate that CoSPlay is not limited to one specific backbone type, but can serve as an inference-time complement to both instruction tuning and RLVR-style post-training.

Second, CoSPlay remains useful when strong models still face difficult benchmarks. DeepSeek-V3.2-685B already has a strong average baseline BoN of 65.7\%, but CoSPlay w/ Cluster further improves it to 68.2\%. The improvement is most visible on CodeForces, where BoN increases from 39.3\% to 50.0\%. This suggests that even very strong models can benefit from additional self-generated UTs when the task remains hard and the candidate pool contains many plausible but incorrect solutions.

Third, the table also reveals two boundary cases. On easier benchmarks that are already within the backbone's capability, extra Code-UT self-play may disturb a near-optimal solution and hurt performance. For example, on LiveBench, Gemini-2.0-Flash decreases from 68.0\% to 64.7\% after CoSPlay, and CoSPlay w/ Cluster recovers only to 66.0\%, still below the baseline. DeepSeek-V3.2-685B also drops from 85.1\% to 78.4\% on LiveBench after CoSPlay. In the opposite case, when the backbone is too weak to cover the data patterns of a difficult benchmark, the initial code and UT generations may be too noisy to provide reliable self-play signals. This can be seen from Absolute-Zero-7B on CodeForces, where CoSPlay w/ Cluster remains below the 11.3\% baseline BoN.

Overall, these detailed results suggest that CoSPlay is most effective when the initial Code-UT signal is useful but not saturated. In such settings, execution feedback can guide code and UTs to improve each other. When the baseline is already near-optimal on an easier benchmark, the extra self-play process may introduce unnecessary perturbations; when the initial model is too weak on a difficult benchmark, the signal may be too noisy to bootstrap reliable co-evolution. Detailed analysis are provided in Section~\ref{generalization}

\begin{table*}[t]
    \centering
    \scriptsize
    \setlength{\tabcolsep}{2pt} 
    \renewcommand{\arraystretch}{1.2} 
    
    \caption{Generalization capability of CoSPlay across different \textbf{base models} and \textbf{RL-tuned models}. \textbf{UT} and \textbf{Code} represent the average accuracy of generated unit tests and code solutions, respectively. \textbf{Signal} denotes the quality of the generated UTs as selection signals. \textbf{BoN} reports the accuracy of the final selected code. The \textbf{baseline} evaluated here represents the model's standard \textbf{BoN} performance. \textbf{Bold} means the best performance.}
    \label{tab:generalization_combined}
    \resizebox{\textwidth}{!}{%
    \begin{tabular}{l|cccc|cccc|cccc|cccc|cccc}
        \toprule
        \multirow{2}{*}{\textbf{Method}} &
        \multicolumn{4}{c|}{\textbf{LiveBench}} &
        \multicolumn{4}{c|}{\textbf{LiveCodeBench}} &
        \multicolumn{4}{c|}{\textbf{CodeContests}} &
        \multicolumn{4}{c|}{\textbf{CodeForces}} &
        \multicolumn{4}{c}{\textbf{Average}} \\ 
         & \textbf{Signal} & \textbf{UT} & \textbf{Code} & \textbf{BoN}
         & \textbf{Signal} & \textbf{UT} & \textbf{Code} & \textbf{BoN}
         & \textbf{Signal} & \textbf{UT} & \textbf{Code} & \textbf{BoN}
         & \textbf{Signal} & \textbf{UT} & \textbf{Code} & \textbf{BoN}
         & \textbf{Signal} & \textbf{UT} & \textbf{Code} & \textbf{BoN} \\
        \midrule
        \midrule
        \multicolumn{21}{c}{\textbf{Base Models}} \\
        \hline
        
        \multicolumn{21}{l}{\textit{\textbf{Seed-Coder-8B-Instruct}}} \\
        Baseline & {48.0} & 26.8 & 40.7 & 50.7 & {32.0} & 26.5 & 32.4 & 38.7 & {28.7} & 21.4 & 16.6 & 26.7 & {10.0} & 13.5 & 6.7 & 13.3 & {29.7} & 22.1 & 24.1 & 32.3 \\
        \textbf{CoSPlay (Ours)} & {50.7} & \textbf{63.3} & \textbf{47.5} & \textbf{56.0} & \textbf{42.7} & \textbf{75.9} & \textbf{42.2} & \textbf{52.7} & \textbf{34.0} & \textbf{77.4} & \textbf{29.5} & \textbf{35.3} & \textbf{17.3} & \textbf{93.1} & \textbf{13.5} & \textbf{25.3} & {36.2} & \textbf{77.4} & \textbf{33.2} & \textbf{42.3} \\
        \quad \textbf{+ Cluster} & \textbf{52.0} & \textbf{63.3} & \textbf{47.5} & \textbf{56.0} & \textbf{42.7} & \textbf{75.9} & \textbf{42.2} & \textbf{52.7} & \textbf{34.0} & \textbf{77.4} & \textbf{29.5} & \textbf{35.3} & \textbf{17.3} & \textbf{93.1} & \textbf{13.5} & \textbf{25.3} & \textbf{36.5} & \textbf{77.4} & \textbf{33.2} & \textbf{42.3} \\
        
        \midrule 
        \multicolumn{21}{l}{\textit{\textbf{DeepSeek-Coder-V2-Lite-Instruct-16B}}} \\
        Baseline & {48.7} & 27.4 & 30.8 & 38.7 & {38.0} & 39.9 & 28.4 & 34.7 & \textbf{34.0} & 51.9 & 22.8 & 34.0 & {12.7} & 27.3 & 7.9 & 16.0 & {33.3} & 36.6 & 22.5 & 30.8 \\
        \textbf{CoSPlay (Ours)} & {52.7} & \textbf{68.7} & \textbf{42.0} & 48.0 & \textbf{43.3} & \textbf{86.6} & \textbf{43.5} & 48.0 & {32.0} & \textbf{82.8} & \textbf{35.8} & 40.7 & \textbf{17.3} & \textbf{92.9} & \textbf{15.4} & \textbf{24.0} & \textbf{36.3} & \textbf{82.7} & \textbf{34.2} & 40.2 \\
        \quad \textbf{+ Cluster} & \textbf{53.3} & \textbf{68.7} & \textbf{42.0} & \textbf{51.3} & \textbf{43.3} & \textbf{86.6} & \textbf{43.5} & \textbf{48.7} & {31.3} & \textbf{82.8} & \textbf{35.8} & \textbf{41.3} & {16.7} & \textbf{92.9} & \textbf{15.4} & \textbf{24.0} & {36.2} & \textbf{82.7} & \textbf{34.2} & \textbf{41.3} \\
        
        \midrule
        
        \multicolumn{21}{l}{\textit{\textbf{DeepSeek-V3.2-685B}}} \\
        Baseline & \textbf{58.7} & \textbf{82.1} & 68.5 & \textbf{77.3} & {45.3} & \textbf{85.1} & 75.7 & \textbf{82.7} & \textbf{38.0} & 76.3 & 58.8 & 63.3 & \textbf{18.0} & 52.4 & 23.2 & 39.3 & \textbf{40.0} & 74.0 & 56.6 & 65.7 \\
        \textbf{CoSPlay (Ours)} & {57.3} & 74.0 & \textbf{76.9} & 76.7 & {46.0} & 83.2 & \textbf{78.4} & 79.3 & {35.3} & \textbf{76.9} & \textbf{60.2} & 66.0 & {16.7} & \textbf{83.7} & \textbf{39.4} & 48.0 & {38.8} & \textbf{79.5} & \textbf{63.8} & 67.5 \\
        \quad \textbf{+ Cluster} & {57.3} & 74.0 & \textbf{76.9} & 76.0 & \textbf{46.7} & 83.2 & \textbf{78.4} & 80.0 & {36.0} & \textbf{76.9} & \textbf{60.2} & \textbf{66.7} & {16.7} & \textbf{83.7} & \textbf{39.4} & \textbf{50.0} & {39.2} & \textbf{79.5} & \textbf{63.8} & \textbf{68.2} \\

        \midrule
        
        \multicolumn{21}{l}{\textit{\textbf{Gemini-2.0-Flash}}} \\
        Baseline & \textbf{56.7} & 50.3 & 59.8 & \textbf{68.0} & {42.0} & 50.2 & 56.6 & 62.0 & {30.7} & 36.2 & 44.9 & 48.7 & \textbf{14.7} & 26.4 & 33.0 & 36.7 & \textbf{36.0} & 40.8 & 48.6 & 53.8 \\
        \textbf{CoSPlay (Ours)} & {51.3} & \textbf{76.2} & \textbf{63.2} & 64.7 & {42.7} & \textbf{83.5} & \textbf{61.5} & 63.3 & \textbf{32.7} & \textbf{67.0} & \textbf{48.7} & 50.7 & {13.3} & \textbf{69.3} & \textbf{35.3} & \textbf{40.7} & {35.0} & \textbf{74.0} & \textbf{52.2} & 54.8 \\
        \quad \textbf{+ Cluster} & {53.3} & \textbf{76.2} & \textbf{63.2} & 66.0 & \textbf{43.3} & \textbf{83.5} & \textbf{61.5} & \textbf{65.3} & \textbf{32.7} & \textbf{67.0} & \textbf{48.7} & \textbf{51.3} & {14.0} & \textbf{69.3} & \textbf{35.3} & 38.7 & {35.8} & \textbf{74.0} & \textbf{52.2} & \textbf{55.3} \\
        
        \midrule
        \midrule
        
        \multicolumn{21}{c}{\textbf{RL-tuned Models}} \\
        \hline

        \multicolumn{21}{l}{\textit{\textbf{Absolute-Zero-7B}}} \\
        Baseline & {44.0} & 19.3 & 11.6 & 36.7 & {32.7} & 31.0 & 11.6 & 24.7 & {28.0} & 28.5 & 8.9 & 23.3 & \textbf{11.3} & 15.7 & 4.8 & \textbf{11.3} & {29.0} & 23.6 & 9.2 & 24.0 \\
        \textbf{CoSPlay (Ours)} & {42.7} & \textbf{80.4} & \textbf{33.0} & 40.7 & {34.7} & \textbf{84.9} & \textbf{28.1} & 33.3 & \textbf{30.7} & \textbf{88.9} & \textbf{21.7} & 23.3 & {10.0} & \textbf{84.9} & \textbf{6.8} & 10.0 & {29.5} & \textbf{84.8} & \textbf{22.4} & 26.8 \\
        \quad \textbf{+ Cluster} & \textbf{52.0} & \textbf{80.4} & \textbf{33.0} & \textbf{47.3} & \textbf{36.0} & \textbf{84.9} & \textbf{28.1} & \textbf{36.7} & {29.3} & \textbf{88.9} & \textbf{21.7} & \textbf{26.7} & {9.3} & \textbf{84.9} & \textbf{6.8} & 10.0 & \textbf{31.7} & \textbf{84.8} & \textbf{22.4} & \textbf{30.2} \\


        \midrule
                
        \multicolumn{21}{l}{\textit{\textbf{AceCoder-Rule-7B}}} \\
        Baseline & {44.0} & 13.9 & 40.7 & 46.0 & {29.3} & 15.9 & 34.0 & 36.0 & {27.3} & 5.3 & 23.3 & 25.3 & {10.0} & 3.0 & 6.9 & 8.7 & {27.7} & 9.5 & 26.2 & 29.0 \\
        \textbf{CoSPlay (Ours)} & {51.3} & \textbf{64.6} & \textbf{45.2} & \textbf{54.7} & {43.3} & \textbf{69.7} & \textbf{40.0} & \textbf{51.3} & {33.3} & \textbf{72.7} & \textbf{25.5} & \textbf{34.0} & \textbf{17.3} & \textbf{87.4} & \textbf{9.6} & \textbf{19.3} & {36.3} & \textbf{73.6} & \textbf{30.1} & \textbf{39.8} \\
        \quad \textbf{+ Cluster} & \textbf{52.7} & \textbf{64.6} & \textbf{45.2} & \textbf{54.7} & \textbf{44.7} & \textbf{69.7} & \textbf{40.0} & \textbf{51.3} & \textbf{34.0} & \textbf{72.7} & \textbf{25.5} & 33.3 & \textbf{17.3} & \textbf{87.4} & \textbf{9.6} & 18.7 & \textbf{37.2} & \textbf{73.6} & \textbf{30.1} & 39.5 \\
        
        \midrule
        
        \multicolumn{21}{l}{\textit{\textbf{AceCoder-RM-7B}}} \\
        Baseline & {43.3} & 15.1 & 40.9 & 48.7 & {29.3} & 17.2 & 34.0 & 38.0 & {27.3} & 8.5 & 22.4 & 25.3 & {8.0} & 7.1 & 5.8 & 5.3 & {27.0} & 12.0 & 25.8 & 29.3 \\
        \textbf{CoSPlay (Ours)} & {54.0} & \textbf{61.3} & \textbf{47.0} & \textbf{56.7} & {43.3} & \textbf{70.7} & \textbf{39.8} & 47.3 & {31.3} & \textbf{73.5} & \textbf{25.8} & \textbf{34.7} & {16.0} & \textbf{85.2} & \textbf{10.4} & \textbf{21.3} & {36.2} & \textbf{72.7} & \textbf{30.8} & 40.0 \\
        \quad \textbf{+ Cluster} & \textbf{56.7} & \textbf{61.3} & \textbf{47.0} & 56.0 & \textbf{44.0} & \textbf{70.7} & \textbf{39.8} & \textbf{50.0} & \textbf{32.6} & \textbf{73.5} & \textbf{25.8} & \textbf{34.7} & \textbf{17.3} & \textbf{85.2} & \textbf{10.4} & \textbf{21.3} & \textbf{37.7} & \textbf{72.7} & \textbf{30.8} & \textbf{40.5} \\

        \midrule     
        \bottomrule
    \end{tabular}%
    }
\end{table*}

\section{Dataset Overview}
\label{app:dataset}
\textbf{LiveBench.}
LiveBench is a dynamic and contamination-limited benchmark designed for broad LLM evaluation, with tasks spanning multiple domains such as math, coding, reasoning, language, instruction following, and data analysis~\cite{white2024livebench}. A key feature of LiveBench is that it is updated frequently and emphasizes objective, automatically checkable evaluation, making it well suited for reducing benchmark contamination and maintaining difficulty over time. In our work, we use its coding subset (128 problems) as a challenging evaluation setting for measuring code reasoning and functional correctness under a more up-to-date benchmark distribution.

\textbf{LiveCodeBench.}
We use LiveCodeBench v2~\cite{jain2024livecodebench}, a contamination-controlled benchmark for evaluating code generation and related code abilities of LLMs. It contains 511 problems released between May 2023 and May 2024, collected from recent programming contests on LeetCode, AtCoder, and CodeForces. Since each problem is tagged with its original release time, the benchmark supports temporally controlled evaluation and helps reduce train-test leakage. The problems are drawn from real contests and have been validated by large numbers of participants, which helps ensure high data quality. 

\textbf{CodeContests.}
CodeContests is a competitive-programming dataset introduced with AlphaCode~\cite{li2022codecontests}. It is constructed from Codeforces, Description2Code, and APPS. In our experiments, we use a subset of 239 problems sampled from the original dataset provided by CURE~\cite{wang2025cure}.

\textbf{CodeForces.}
We use the \texttt{open-r1/codeforces} dataset~\cite{penedo2025codeforces}, which is built from CodeForces, one of the most popular competitive-programming platforms. The full dataset contains more than 10k unique problems collected from contests spanning the earliest rounds to 2025. In our experiments, we use a randomly sampled subset of 467 problems with difficulty levels below 2 provided by CURE~\cite{wang2025cure}. A key feature of this dataset is its emphasis on verifiability: for problems with multiple valid outputs, it provides custom checker programs, and it also includes additional challenging test cases beyond the truncated examples typically visible on the contest platform. These properties make it well suited for evaluating code reasoning and functional correctness on realistic competitive-programming tasks.

    
    
\begin{table*}[t]
    \centering
    \scriptsize
    \caption{Ablation study of CoSPlay-7B.}
    \label{tab:ablation_minibench}
    \renewcommand{\arraystretch}{1.15}
    \setlength{\tabcolsep}{1.3pt}
    \resizebox{\textwidth}{!}{%
    \begin{tabular}{l|cccc|cccc|cccc|cccc|cccc}
        \toprule
        \multirow{2}{*}{\textbf{Model Variant}} &
        \multicolumn{4}{c|}{\textbf{LiveBench}} &
        \multicolumn{4}{c|}{\textbf{LiveCodeBench}} &
        \multicolumn{4}{c|}{\textbf{CodeContests}} &
        \multicolumn{4}{c|}{\textbf{CodeForces}} &
        \multicolumn{4}{c}{\textbf{Average}} \\
        & \textbf{Signal} & \textbf{UT} & \textbf{Code} & \textbf{BoN}
        & \textbf{Signal} & \textbf{UT} & \textbf{Code} & \textbf{BoN}
        & \textbf{Signal} & \textbf{UT} & \textbf{Code} & \textbf{BoN}
        & \textbf{Signal} & \textbf{UT} & \textbf{Code} & \textbf{BoN}
        & \textbf{Signal} & \textbf{UT} & \textbf{Code} & \textbf{BoN} \\
        \midrule
        \textit{\textbf{Full}} & {54.7} & 71.9 & 41.3 & 51.3 & {40.7} & 72.7 & 31.6 & 41.3 & {34.0} & 77.4 & 23.7 & 36.0 & {16.7} & 92.8 & 9.4 & 18.0 & {36.5} & 78.7 & 26.5 & 36.7 \\
        w/o exp-atk & {42.0} & 29.4 & 41.7 & 48.0 & {26.7} & 32.5 & 30.4 & 32.7 & {27.3} & 42.1 & 25.7 & 30.7 & {7.3} & 26.2 & 7.6 & 10.0 & {25.8} & 32.5 & 26.3 & 30.3 \\
        w/o self-play & {49.3} & 53.0 & 30.9 & 50.0 & {36.7} & 60.4 & 24.3 & 40.7 & {32.7} & 59.4 & 18.8 & 29.3 & {17.3} & 50.0 & 5.6 & 12.0 & {34.0} & 55.7 & 19.9 & 33.0 \\
        w/o random UT & {51.3}& 65.3 & 41.8 & 52.7 & {40.7} & 74.3 & 31.9 & 40.0 & {34.7}  & 76.4 & 24.5 & 31.3 & {17.3} & 88.0 & 8.2 & 18.0 & {36.0} & 76.0 & 26.6 & 35.5 \\
        w/o non-trivial best & {53.3} & 71.7 & 42.1 & 50.7 & {42.0} & 72.7 & 32.3 & 37.3 & {32.7} & 77.7 & 25.4 & 32.7 & {16.7} & 91.6 & 10.4 & 19.3 & {36.2} & 78.4 & 27.6 & 35.0 \\
        w/o self-consis. & {51.3} & 55.4 & 41.9 & 48.0 & {40.0} & 64.4 & 30.9 & 38.0 & {35.3} & 62.6 & 23.0 & 30.7 & {16.0} & 72.2 & 8.2 & 16.7 & {35.7} & 63.7 & 26.0 & 33.3 \\
        w/o self-play step2 & {50.7} & 63.1 & 42.1 & 49.3 & {38.0} & 72.4 & 31.7 & 37.3 & {34.7} & 74.7 & 24.9 & 37.3 & {16.7} & 86.8 & 8.0 & 16.0 & {35.0} & 74.2 & 26.7 & 35.0 \\
        w/o self-play step3 & {52.0} & 63.0 & 39.7 & 52.7 & {41.3} & 76.0 & 29.8 & 40.0 & {32.7} & 70.9 & 22.1 & 32.0 & {17.3} & 85.8 & 8.6 & 18.7 & {35.8} & 73.9 & 25.1 & 35.8 \\
        w/o resample-UT 0\% & {48.7} & 33.6 & 42.2 & 46.7 & {38.0} & 48.1 & 31.9 & 38.0 & {32.7} & 43.5 & 24.4 & 28.0 & {14.7} & 30.7 & 9.0 & 16.7 & {33.5} & 38.9 & 26.9 & 32.3 \\
        w/o resample-UT 100\% &{52.7} & 77.6 & 41.3 & 47.3 & {42.0} & 88.0 & 32.0 & 40.0 & {33.3} & 83.8 & 24.3 & 33.3 & {16.7} & 86.4 & 8.7 & 16.7 & {36.2} & 83.9 & 26.6 & 34.3 \\
        w/o resample-Code 0\% & {51.3} & 77.2 & 33.8 & 48.0 & {42.7} & 80.7 & 25.5 & 38.0 & {31.3} & 73.4 & 22.4 & 32.7 & {16.7} & 91.0 & 7.0 & 14.7 & {35.5} & 80.6 & 22.2 & 33.3 \\
        w/o UT-iteration & {49.3} & 50.3 & 40.4 & 51.3 & {35.3} & 61.5 & 29.2 & 38.0 & {32.0} & 53.6 & 23.1 & 32.7 & {16.7} & 47.7 & 9.0 & 18.7 & {33.3} & 53.3 & 25.4 & 35.2 \\
        w/o self-play \& self-consis. & {49.3}& 38.3 & 33.3 & 46.7 & {38.7} & 49.6 & 21.3 & 36.0 & {32.0} & 30.5 & 18.7 & 24.7 & {14.7} & 30.3 & 5.3 & 13.3 & {33.7} & 37.2 & 19.7 & 30.2 \\
        baseline(Direct resample) & 38.0 & 8.2 & 36.0 & 38.0 & 32.7 & 13.9 & 27.1 & 32.7 & 25.3 & 15.1 & 21.3 & 25.3 & 10.0 & 12.6 & 6.5 & 10.0 & 26.5 & 12.5 & 22.7 & 26.5 \\
        \bottomrule
    \end{tabular}%
    }
\end{table*}

\section{Detailed Ablation Results}
\label{app:Detailed ablation results}

\subsection{Ablation Variant Descriptions}
We evaluate the contribution of each component through the following ablation variants.
\textbf{w/o exp-atk} removes Stage~1 entirely, generating codes and UTs directly without solution-aware idea exploration.
\textbf{w/o self-play} removes Stage~2 entirely, skipping all iterative refinement and selecting directly from the initial pool.
\textbf{w/o random UT} initializes the UT pool using only failure-oriented inputs, removing the random-valid-input half of the hybrid initialization strategy.
\textbf{w/o non-trivial best} replaces the targeted non-trivial Best UT selection in Step3 with a randomly chosen UT for code fixing.
\textbf{w/o self-consis.} disables self-consistency during UT output validation, using a single-sample output as the expected value instead of majority voting.
\textbf{w/o self-play step2} removes Step~2 (breaking spurious Code-UT coupling) while retaining all other self-play steps.
\textbf{w/o self-play step3} removes Step~3 (code fixing) while retaining all other self-play steps.
\textbf{w/o resample-UT 0\%} disables replacement of zero-pass-rate UTs in Step~4, retaining UTs that all candidates fail.
\textbf{w/o resample-UT 100\%} disables replacement of full-pass-rate UTs in Step~4, retaining UTs that all candidates pass.
\textbf{w/o resample-Code 0\%} disables code resampling in Step~1, retaining all-failing code candidates instead of replacing them.
\textbf{w/o UT-iteration} removes all UT-side update steps within self-play (Steps~2 and~4), keeping the UT pool fixed throughout all rounds while still performing code cleaning and fixing.
\textbf{w/o self-play \& self-consis.} removes both Stage~2 and self-consistency, serving as a lower-bound baseline.
\textbf{Baseline (Direct resample)} applies no CoSPlay mechanisms, directly sampling a large pool of codes and selecting via BoN with self-generated UTs.

\subsection{Analysis}
As shown in Table~\ref{tab:ablation_minibench}, removing the first stage leads to lower BoN performance even though raw code accuracy may increase, indicating that idea exploration mainly improves the discriminative quality of generated UTs rather than naive code quality. Ablations on the self-play stage show a consistent pattern: CoSPlay works because it jointly improves \emph{UT quality} and \emph{code quality}. On the UT side, removing random UT initialization, self-consistency, or UT regeneration all hurts final BoN. In particular, keeping 0\%-pass UTs causes severe degradation, while keeping 100\%-pass UTs also reduces BoN as such tests are trivial and non-discriminative. On the code side, removing the refinement step or not resampling 0\%-pass codes degrades both code accuracy and BoN. Finally, replacing the non-trivial best UT with a random UT weakens refinement, confirming that targeted UT selection provides a stronger refinement signal. Overall, these results support our design that effective self-play requires both progressively refined UTs and targeted code refinement.




\section{Random UT Initialization Enhances UT Diversity}
    

    


\begin{figure*}[!t]
    \centering
    
        \includegraphics[width=0.5\textwidth]{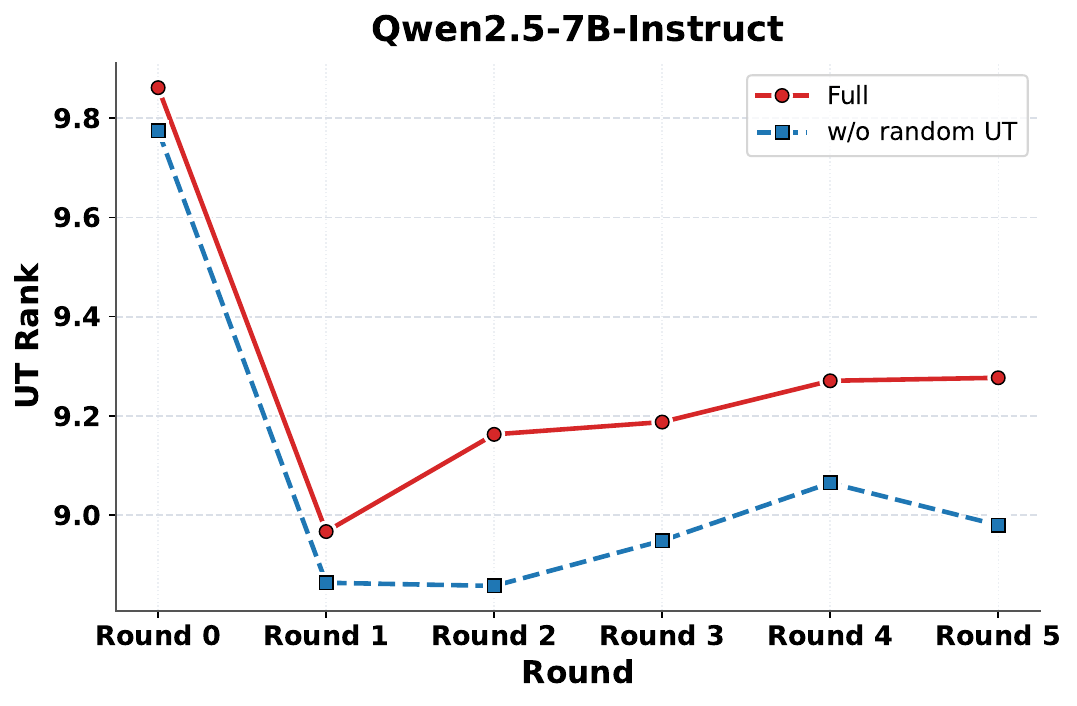}

    \caption{Comparison of the evolution of UT rank over self-play rounds between the ablation w/o random UT initialization and full CoSPlay.}
    \label{fig:random_init_and_cluster_scaling}
\end{figure*}

As shown in Figure~\ref{fig:random_init_and_cluster_scaling}, incorporating randomly generated valid UTs during initialization consistently yields a higher UT rank across all self-play rounds compared to the variant without random initialization. 
Although both variants drop after the first update, the full method recovers better in later rounds, suggesting that random UTs help preserve broader UT diversity. 
This complements attack-oriented UTs and prevents the pool from collapsing into a narrow set of similar failure-focused cases.




\section{Detailed data of UT pass count distribution about Idea-Level Exploration vs. direct sample}
\begin{figure*}[!t]
    \centering
    
    \includegraphics[width=0.99\textwidth]{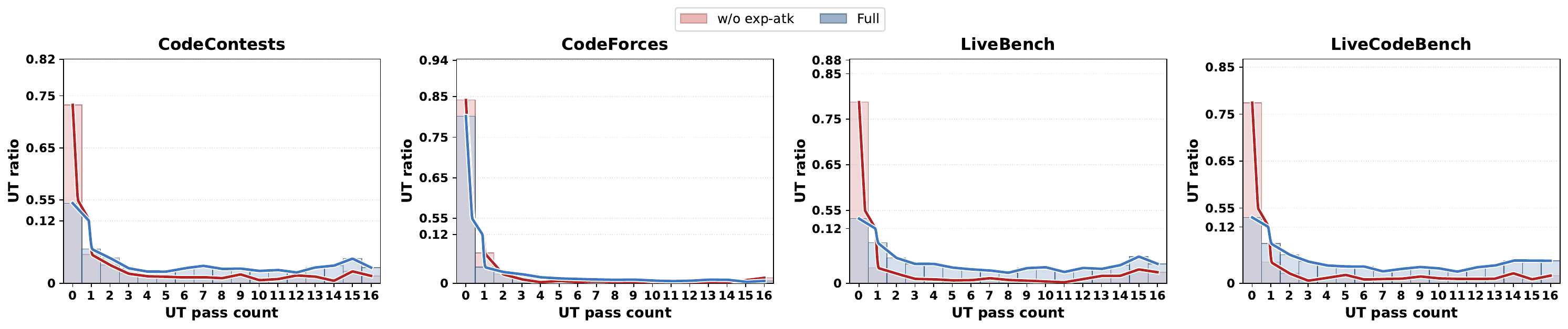}
    
    
    
    \caption{UT pass count (number of code candidates passing each UT) distributions at the UT initialization stage, comparing the full method (blue) against the variant w/o exp-atk (red).}
    \label{fig:attack_exploration_init}
\end{figure*}

Figure~\ref{fig:attack_exploration_init} presents a detailed comparison of UT pass count distributions between the full method and the variant without Exploration-Attack stage (w/o exp-atk), evaluated at the UT initialization stage before self-play stage. In all of the UT pass count plots, the w/o exp-atk variant shows a pronounced concentration of zero-pass UTs across all four datasets, indicating that a large proportion of directly sampled UTs are not passed by any code candidate--these are highly likely to be incorrect or noisy. In contrast, the full method substantially reduces this spike and shifts the distribution toward moderate pass counts (roughly 2-10). This more uniform spread suggests that UTs generated via our proposed Exploration-Attack stage are not only more accurate (fewer zero-pass outliers) but also more discriminative: instead of trivially failing or trivially passing all candidates, they meaningfully differentiate among code solutions. These distributions confirm that this exploration improves the quality of initial UT pool in a complementary manner: it produces UTs that are more accurate and more evenly discriminative, which in turn provides a stronger foundation for the subsequent self-play stage.

\section{Scalability Analysis of Random Input Size for Clustering}
\label{app:cluster_input_scaling}
\begin{figure*}[!t]
    \centering
    \includegraphics[
        width=0.24\textwidth,
        trim=30 0 30 0,
        clip
    ]{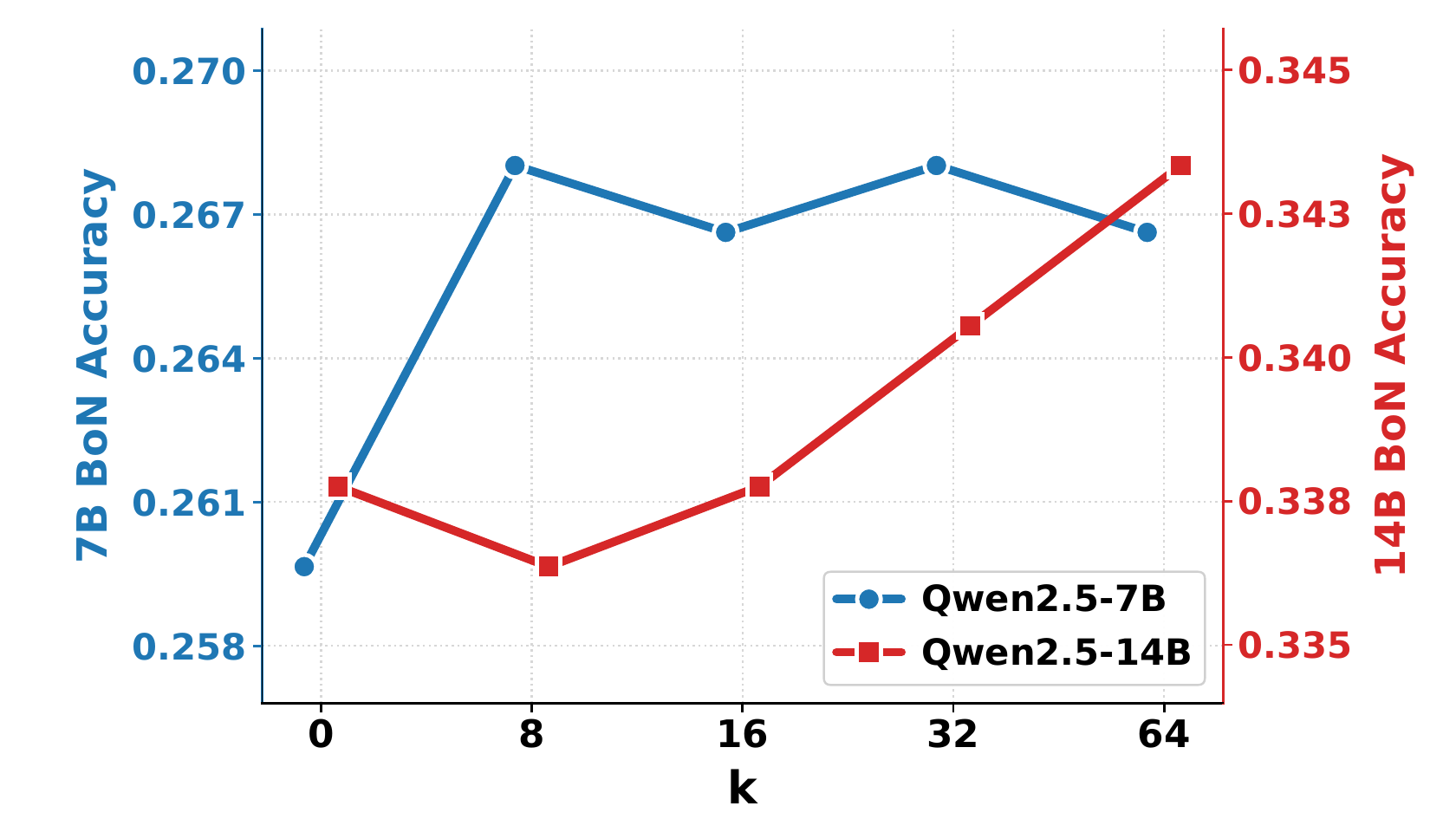}\hfill
    \includegraphics[
        width=0.24\textwidth,
        trim=30 0 30 0,
        clip
    ]{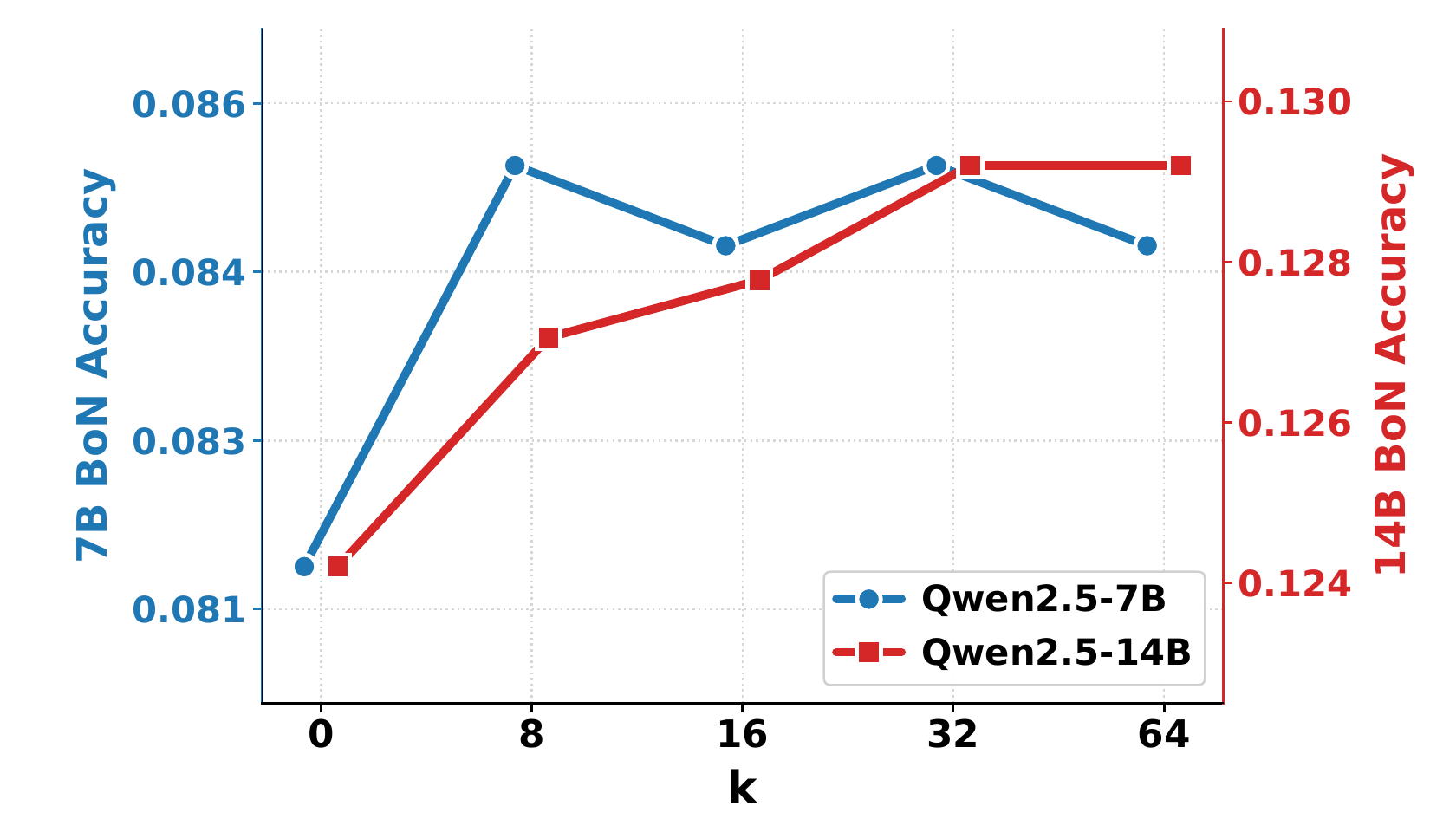}\hfill
    \includegraphics[
        width=0.24\textwidth,
        trim=30 0 30 0,
        clip
    ]{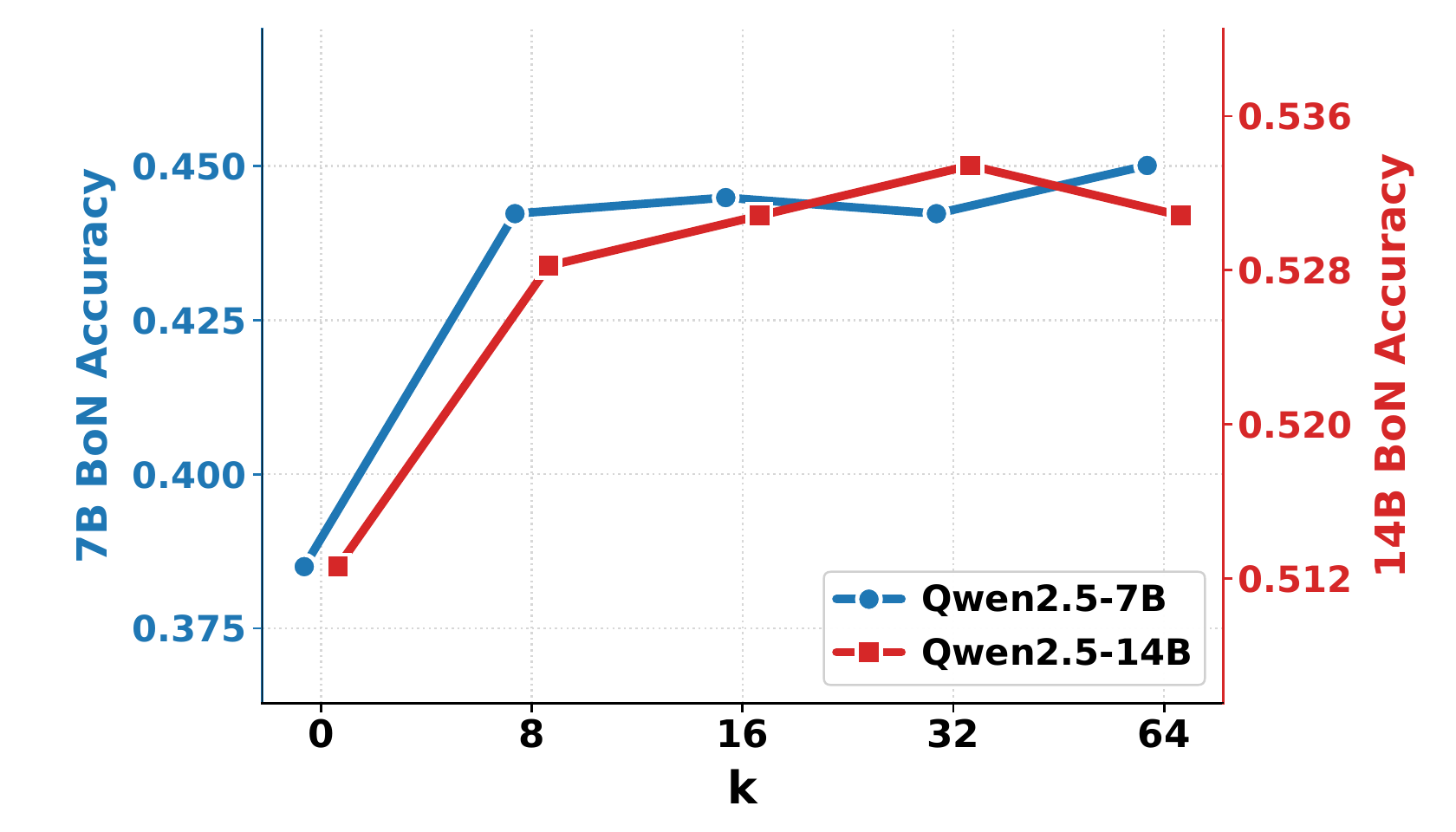}\hfill
    \includegraphics[
        width=0.24\textwidth,
        trim=30 0 30 0,
        clip
    ]{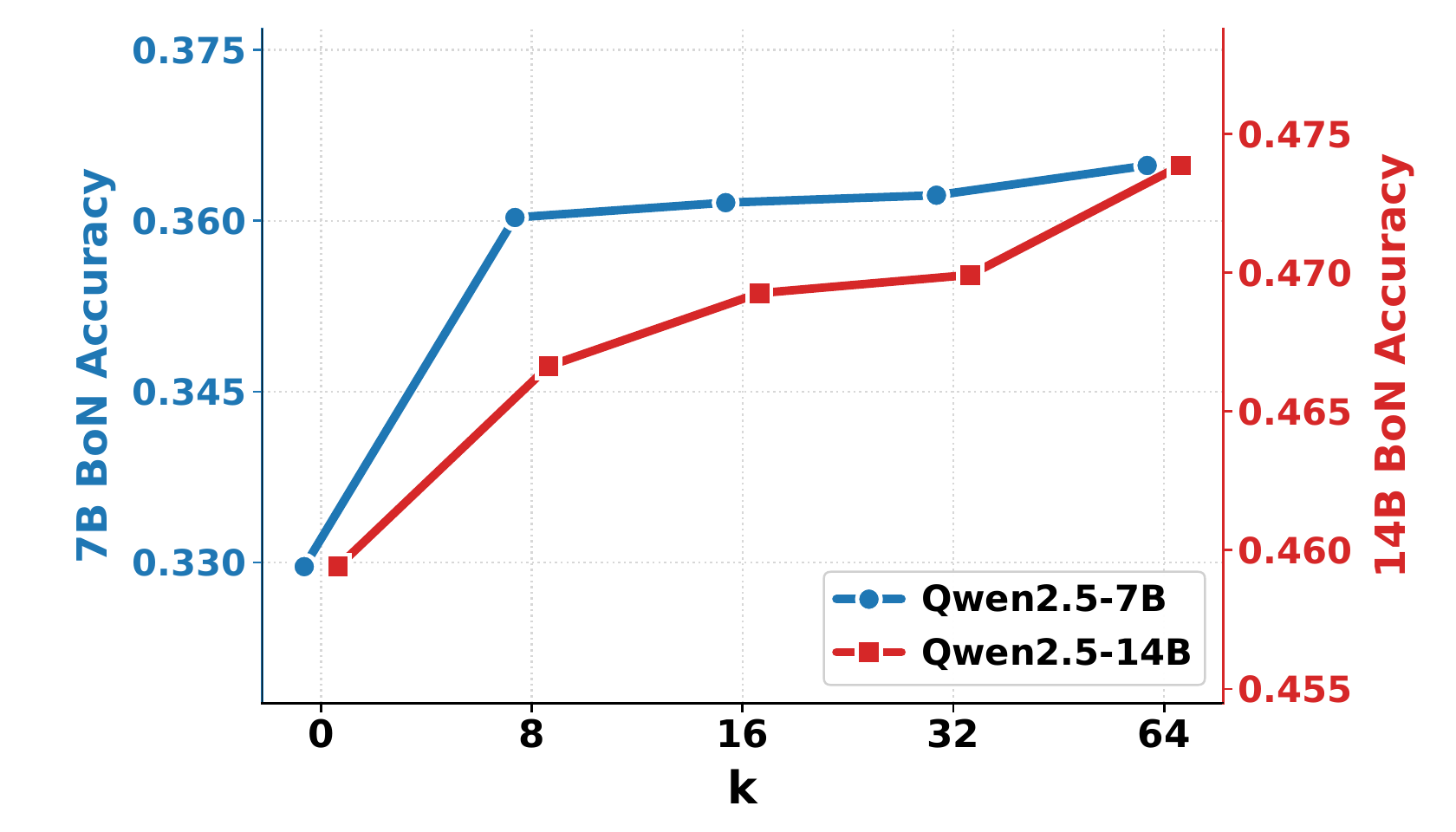}

    \vspace{-0.5em}
    \caption{
    Effect of the number of random valid inputs used for execution-consensus clustering.
    We vary the number of clustering inputs from 0 to 64 and report the results on CodeContests, CodeForces, LiveBench, and LiveCodeBench.
    Increasing the number of inputs provides a more reliable execution signature for grouping functionally equivalent candidates.
    }
    \vspace{-1em}
    \label{fig:cluster_scale_qwen}
\end{figure*}



To further evaluate the robustness of our execution-consensus-based cluster selection strategy, we analyze how the number of random valid inputs used for clustering affects the final performance. 
Specifically, we use Qwen2.5-7B-Instruct with BoN $(N=16)$ and vary the number of clustering inputs $k$ from 0 to 64.

As shown in Figure~\ref{fig:cluster_scale_qwen}, increasing $k$ generally improves or stabilizes performance across the four datasets. 
This is because more random inputs provide a more reliable execution signature, making it easier to group functionally equivalent candidates and separate spurious agreements. 
The gains are especially visible on LiveBench and LiveCodeBench, suggesting that execution-consensus selection benefits more when candidate behaviors are diverse and difficult to distinguish with only a few inputs.

The improvement is not strictly monotonic, since the randomly generated inputs may vary in discriminative strength. 
Nevertheless, the overall trend shows that using more clustering inputs leads to more stable selection, while performance does not collapse under larger $k$. 
This confirms that the cluster-selection stage scales favorably with the number of random inputs and can effectively exploit additional test-time execution budget.
 

\section{Variance analysis}
\label{app:error_analysis}
Table~\ref{tab:big-table-std} reports the mean performance and sample standard deviation over three independent runs. Each cell shows the averaged result, with the corresponding standard deviation reported in parentheses. The consistently small standard deviations across most datasets and metrics indicate that the observed improvements are stable rather than caused by random sampling variation. 

\newpage

\begin{table}[H]
    \centering
    \scriptsize
    \setlength{\tabcolsep}{2.2pt}
    \renewcommand{\arraystretch}{1.4}
    \caption{Mean accuracy and population standard deviation over three runs. Each cell reports the mean accuracy and the corresponding standard deviation, both expressed as percentages. The standard deviation is computed with normalization by the number of runs.}
    \label{tab:big-table-std}
    \resizebox{\textwidth}{!}{
    \begin{tabular}{l|cccc|cccc|cccc|cccc|cccc}
        \toprule
        \multirow{2}{*}{\textbf{Model}} & \multicolumn{4}{c|}{\textbf{LiveBench}} & \multicolumn{4}{c|}{\textbf{LiveCodeBench}} & \multicolumn{4}{c|}{\textbf{CodeContests}} & \multicolumn{4}{c|}{\textbf{CodeForces}} & \multicolumn{4}{c}{\textbf{Average}} \\
         & \textbf{Sig.} & \textbf{UT} & \textbf{Code} & \textbf{BoN} & \textbf{Sig.} & \textbf{UT} & \textbf{Code} & \textbf{BoN} & \textbf{Sig.} & \textbf{UT} & \textbf{Code} & \textbf{BoN} & \textbf{Sig.} & \textbf{UT} & \textbf{Code} & \textbf{BoN} & \textbf{Sig.} & \textbf{UT} & \textbf{Code} & \textbf{BoN} \\
        \midrule
        \rowcolor{gray!15} \multicolumn{21}{c}{\textbf{7B Models}} \\
        \midrule
        Qwen2.5-7B-Instruct & \shortstack{36.2\\(2.9)} & \shortstack{9.0\\(0.6)} & \shortstack{33.1\\(0.8)} & \shortstack{36.2\\(2.9)} & \shortstack{31.2\\(1.0)} & \shortstack{14.4\\(0.2)} & \shortstack{27.0\\(0.1)} & \shortstack{31.2\\(1.0)} & \shortstack{24.1\\(1.1)} & \shortstack{21.3\\(0.9)} & \shortstack{21.5\\(0.6)} & \shortstack{24.1\\(1.1)} & \shortstack{7.1\\(0.2)} & \shortstack{12.9\\(0.7)} & \shortstack{5.0\\(0.0)} & \shortstack{7.1\\(0.2)} & \shortstack{22.1\\(0.6)} & \shortstack{14.6\\(0.3)} & \shortstack{19.0\\(0.1)} & \shortstack{22.1\\(0.6)} \\
        
        
        \quad + CodeT & \shortstack{37.5\\(2.8)} & \shortstack{9.0\\(0.6)} & \shortstack{33.1\\(0.8)} & \shortstack{37.5\\(2.8)} & \shortstack{32.3\\(0.8)} & \shortstack{14.4\\(0.2)} & \shortstack{27.0\\(0.1)} & \shortstack{32.3\\(0.8)} & \shortstack{25.0\\(0.8)} & \shortstack{21.3\\(0.9)} & \shortstack{21.5\\(0.6)} & \shortstack{25.0\\(0.8)} & \shortstack{7.6\\(0.4)} & \shortstack{12.9\\(0.7)} & \shortstack{5.0\\(0.0)} & \shortstack{7.6\\(0.4)} & \shortstack{22.9\\(0.5)} & \shortstack{14.6\\(0.3)} & \shortstack{19.0\\(0.1)} & \shortstack{22.9\\(0.5)} \\

        \midrule
        
        Qwen2.5-7B-Coder & \shortstack{40.6\\(1.3)} & \shortstack{12.1\\(1.9)} & \shortstack{3.8\\(0.1)} & \shortstack{19.0\\(2.7)} & \shortstack{33.4\\(1.6)} & \shortstack{16.0\\(0.6)} & \shortstack{4.1\\(0.3)} & \shortstack{17.7\\(1.0)} & \shortstack{25.2\\(0.9)} & \shortstack{21.3\\(1.8)} & \shortstack{2.6\\(0.2)} & \shortstack{12.8\\(1.4)} & \shortstack{6.8\\(0.4)} & \shortstack{6.7\\(1.8)} & \shortstack{1.8\\(0.8)} & \shortstack{3.5\\(1.2)} & \shortstack{23.4\\(0.5)} & \shortstack{13.3\\(0.5)} & \shortstack{3.0\\(0.1)} & \shortstack{12.0\\(0.7)} \\

        \midrule
        
        Qwen2.5-7B-Coder-Ins. & \shortstack{42.7\\(1.6)} & \shortstack{12.5\\(0.3)} & \shortstack{36.3\\(0.3)} & \shortstack{47.7\\(1.1)} & \shortstack{33.7\\(0.4)} & \shortstack{13.4\\(1.3)} & \shortstack{28.8\\(0.7)} & \shortstack{34.4\\(2.9)} & \shortstack{24.0\\(0.7)} & \shortstack{12.7\\(0.6)} & \shortstack{20.9\\(0.2)} & \shortstack{23.2\\(0.4)} & \shortstack{6.9\\(0.4)} & \shortstack{5.3\\(0.8)} & \shortstack{5.4\\(0.1)} & \shortstack{7.9\\(0.6)} & \shortstack{23.5\\(0.5)} & \shortstack{10.4\\(0.3)} & \shortstack{20.0\\(0.2)} & \shortstack{24.5\\(1.0)} \\

        \midrule
        
        Seed-Coder-8B-Instruct & \shortstack{43.5\\(1.3)} & \shortstack{26.9\\(1.3)} & \shortstack{37.6\\(0.7)} & \shortstack{45.1\\(1.9)} & \shortstack{36.1\\(0.2)} & \shortstack{30.0\\(0.7)} & \shortstack{32.5\\(0.1)} & \shortstack{40.8\\(0.5)} & \shortstack{25.0\\(0.9)} & \shortstack{31.8\\(0.0)} & \shortstack{18.3\\(0.2)} & \shortstack{26.9\\(1.0)} & \shortstack{8.4\\(1.0)} & \shortstack{17.0\\(1.1)} & \shortstack{6.4\\(0.2)} & \shortstack{10.5\\(0.5)} & \shortstack{25.2\\(0.6)} & \shortstack{25.5\\(0.4)} & \shortstack{21.4\\(0.0)} & \shortstack{28.2\\(0.3)} \\

        \midrule
        
        AceCoder-7B-Rule & \shortstack{43.5\\(0.4)} & \shortstack{17.2\\(0.8)} & \shortstack{36.1\\(0.2)} & \shortstack{45.8\\(1.0)} & \shortstack{34.1\\(0.2)} & \shortstack{17.5\\(0.7)} & \shortstack{30.5\\(0.6)} & \shortstack{35.9\\(0.6)} & \shortstack{24.4\\(2.6)} & \shortstack{18.8\\(1.5)} & \shortstack{20.1\\(0.6)} & \shortstack{23.2\\(0.9)} & \shortstack{6.4\\(0.6)} & \shortstack{5.9\\(0.3)} & \shortstack{5.4\\(0.1)} & \shortstack{7.7\\(0.6)} & \shortstack{23.6\\(0.3)} & \shortstack{13.7\\(0.2)} & \shortstack{20.5\\(0.3)} & \shortstack{24.8\\(0.4)} \\

        \midrule
        
        AceCoder-7B-RM & \shortstack{43.0\\(1.7)} & \shortstack{17.7\\(0.6)} & \shortstack{37.5\\(0.5)} & \shortstack{48.4\\(1.7)} & \shortstack{33.7\\(1.1)} & \shortstack{19.9\\(0.7)} & \shortstack{30.3\\(0.4)} & \shortstack{37.1\\(0.2)} & \shortstack{25.0\\(0.5)} & \shortstack{20.0\\(0.9)} & \shortstack{20.6\\(0.4)} & \shortstack{25.0\\(0.5)} & \shortstack{6.6\\(0.3)} & \shortstack{6.6\\(1.0)} & \shortstack{5.2\\(0.1)} & \shortstack{7.8\\(0.4)} & \shortstack{23.6\\(0.5)} & \shortstack{15.1\\(0.7)} & \shortstack{20.6\\(0.1)} & \shortstack{25.8\\(0.4)} \\

        \midrule
        
        AZR-7B-Coder & \shortstack{44.5\\(1.1)} & \shortstack{23.1\\(0.7)} & \shortstack{12.5\\(0.4)} & \shortstack{34.9\\(2.2)} & \shortstack{37.1\\(0.6)} & \shortstack{31.7\\(1.1)} & \shortstack{12.4\\(0.3)} & \shortstack{29.8\\(1.0)} & \shortstack{29.1\\(0.7)} & \shortstack{39.0\\(1.2)} & \shortstack{13.7\\(6.3)} & \shortstack{24.4\\(2.9)} & \shortstack{7.6\\(0.7)} & \shortstack{16.9\\(0.5)} & \shortstack{5.3\\(0.7)} & \shortstack{8.4\\(0.3)} & \shortstack{26.1\\(0.4)} & \shortstack{27.0\\(0.5)} & \shortstack{10.2\\(1.3)} & \shortstack{21.9\\(0.7)} \\

        \midrule
        
        CURE-7B & \shortstack{52.3\\(1.7)} & \shortstack{54.6\\(1.3)} & \shortstack{37.4\\(0.1)} & \shortstack{51.6\\(2.3)} & \shortstack{42.1\\(0.3)} & \shortstack{58.4\\(0.9)} & \shortstack{31.2\\(0.3)} & \shortstack{42.7\\(0.7)} & \shortstack{33.9\\(0.7)} & \shortstack{63.7\\(0.9)} & \shortstack{25.5\\(0.3)} & \shortstack{34.5\\(1.1)} & \shortstack{11.5\\(0.4)} & \shortstack{46.6\\(1.9)} & \shortstack{7.5\\(0.2)} & \shortstack{16.1\\(1.0)} & \shortstack{31.0\\(0.4)} & \shortstack{54.9\\(0.5)} & \shortstack{22.5\\(0.1)} & \shortstack{32.9\\(0.3)} \\

        \midrule

        \textbf{Qwen2.5-7B-Ins. + CoSPlay} & \shortstack{47.4\\(1.6)} & \shortstack{66.7\\(1.3)} & \shortstack{38.8\\(0.6)} & \shortstack{50.0\\(1.3)} & \shortstack{43.2\\(1.0)} & \shortstack{77.0\\(1.3)} & \shortstack{33.2\\(0.1)} & \shortstack{41.9\\(1.5)} & \shortstack{32.1\\(0.9)} & \shortstack{74.4\\(2.4)} & \shortstack{26.0\\(0.4)} & \shortstack{34.3\\(2.0)} & \shortstack{13.3\\(0.2)} & \shortstack{85.0\\(0.6)} & \shortstack{6.8\\(0.5)} & \shortstack{16.8\\(1.3)} & \shortstack{31.3\\(0.5)} & \shortstack{78.3\\(0.9)} & \shortstack{23.3\\(0.2)} & \shortstack{32.6\\(0.4)} \\
        

        \quad w/ Cluster & \shortstack{50.3\\(1.0)} & \shortstack{66.7\\(1.3)} & \shortstack{38.8\\(0.6)} & \shortstack{49.7\\(1.6)} & \shortstack{43.7\\(0.8)} & \shortstack{77.0\\(1.3)} & \shortstack{33.2\\(0.1)} & \shortstack{43.4\\(1.4)} & \shortstack{32.6\\(0.7)} & \shortstack{74.4\\(2.4)} & \shortstack{26.0\\(0.4)} & \shortstack{34.6\\(1.8)} & \shortstack{13.6\\(0.3)} & \shortstack{85.0\\(0.6)} & \shortstack{6.8\\(0.5)} & \shortstack{16.8\\(1.3)} & \shortstack{31.9\\(0.3)} & \shortstack{78.3\\(0.9)} & \shortstack{23.3\\(0.2)} & \shortstack{33.2\\(0.2)} \\
        \rowcolor{green!4}
        
        \midrule

        \textbf{CURE-7B + CoSPlay} & \shortstack{51.3\\(1.9)} & \shortstack{66.2\\(2.6)} & \shortstack{45.4\\(0.9)} & \shortstack{52.1\\(0.4)} & \shortstack{44.6\\(0.2)} & \shortstack{75.6\\(0.9)} & \shortstack{37.4\\(0.3)} & \shortstack{48.5\\(1.3)} & \shortstack{35.3\\(0.4)} & \shortstack{76.6\\(1.5)} & \shortstack{29.8\\(0.1)} & \shortstack{39.6\\(1.5)} & \shortstack{14.2\\(0.2)} & \shortstack{85.6\\(0.2)} & \shortstack{11.1\\(0.3)} & \shortstack{23.1\\(0.5)} & \shortstack{33.0\\(0.2)} & \shortstack{78.3\\(0.4)} & \shortstack{27.7\\(0.3)} & \shortstack{38.4\\(0.7)} \\


        \quad w/ Cluster & \shortstack{53.6\\(1.9)} & \shortstack{66.2\\(2.6)} & \shortstack{45.4\\(0.9)} & \shortstack{52.3\\(0.6)} & \shortstack{44.7\\(0.2)} & \shortstack{75.6\\(0.9)} & \shortstack{37.4\\(0.3)} & \shortstack{48.6\\(1.1)} & \shortstack{35.6\\(0.7)} & \shortstack{76.6\\(1.5)} & \shortstack{29.8\\(0.1)} & \shortstack{40.0\\(0.9)} & \shortstack{14.4\\(0.2)} & \shortstack{85.6\\(0.2)} & \shortstack{11.1\\(0.3)} & \shortstack{23.1\\(0.5)} & \shortstack{33.4\\(0.2)} & \shortstack{78.3\\(0.4)} & \shortstack{27.7\\(0.3)} & \shortstack{38.6\\(0.5)} \\
        \midrule
        \rowcolor{gray!15} \multicolumn{21}{c}{\textbf{14B Models}} \\
        \midrule
        Qwen2.5-14B-Instruct & \shortstack{47.7\\(2.3)} & \shortstack{26.4\\(1.1)} & \shortstack{39.2\\(1.1)} & \shortstack{51.0\\(0.7)} & \shortstack{41.4\\(0.8)} & \shortstack{38.0\\(1.1)} & \shortstack{34.5\\(0.3)} & \shortstack{46.2\\(0.6)} & \shortstack{30.4\\(0.2)} & \shortstack{43.2\\(0.6)} & \shortstack{25.1\\(0.3)} & \shortstack{34.6\\(1.3)} & \shortstack{9.9\\(1.0)} & \shortstack{22.6\\(1.3)} & \shortstack{7.3\\(0.2)} & \shortstack{13.3\\(1.5)} & \shortstack{29.1\\(0.3)} & \shortstack{32.4\\(0.9)} & \shortstack{23.9\\(0.1)} & \shortstack{33.2\\(0.5)} \\


        \quad + CodeT & \shortstack{49.5\\(1.5)} & \shortstack{26.4\\(1.1)} & \shortstack{39.2\\(1.1)} & \shortstack{52.1\\(0.7)} & \shortstack{42.2\\(1.0)} & \shortstack{38.0\\(1.1)} & \shortstack{34.5\\(0.3)} & \shortstack{48.4\\(0.1)} & \shortstack{29.8\\(0.7)} & \shortstack{43.2\\(0.6)} & \shortstack{25.1\\(0.3)} & \shortstack{35.6\\(0.6)} & \shortstack{10.5\\(0.7)} & \shortstack{22.6\\(1.3)} & \shortstack{7.3\\(0.2)} & \shortstack{13.5\\(0.9)} & \shortstack{29.7\\(0.3)} & \shortstack{32.4\\(0.9)} & \shortstack{23.9\\(0.1)} & \shortstack{34.3\\(0.4)} \\

        \midrule

        Qwen2.5-14B-Coder & \shortstack{42.2\\(1.3)} & \shortstack{13.9\\(0.9)} & \shortstack{10.8\\(0.4)} & \shortstack{30.2\\(3.0)} & \shortstack{36.1\\(0.2)} & \shortstack{20.7\\(0.5)} & \shortstack{9.7\\(0.2)} & \shortstack{26.7\\(0.2)} & \shortstack{28.3\\(0.5)} & \shortstack{25.6\\(2.7)} & \shortstack{6.8\\(0.3)} & \shortstack{21.3\\(1.2)} & \shortstack{10.2\\(0.3)} & \shortstack{16.7\\(0.8)} & \shortstack{2.1\\(0.2)} & \shortstack{9.3\\(0.5)} & \shortstack{26.3\\(0.1)} & \shortstack{19.5\\(0.3)} & \shortstack{6.7\\(0.2)} & \shortstack{20.1\\(0.6)} \\
        \midrule

        Qwen2.5-14B-Coder-Ins. & \shortstack{48.2\\(0.4)} & \shortstack{37.8\\(0.3)} & \shortstack{47.2\\(0.4)} & \shortstack{53.9\\(0.6)} & \shortstack{39.5\\(0.6)} & \shortstack{40.6\\(1.0)} & \shortstack{41.0\\(0.2)} & \shortstack{49.4\\(0.6)} & \shortstack{27.5\\(1.9)} & \shortstack{34.8\\(1.1)} & \shortstack{25.1\\(0.2)} & \shortstack{31.4\\(0.7)} & \shortstack{7.9\\(0.2)} & \shortstack{18.1\\(0.7)} & \shortstack{5.6\\(0.0)} & \shortstack{9.6\\(0.6)} & \shortstack{27.2\\(0.1)} & \shortstack{31.5\\(0.6)} & \shortstack{26.5\\(0.1)} & \shortstack{32.8\\(0.4)} \\

        \midrule
        
        DeepSeek-Coder-V2-16B & \shortstack{45.8\\(1.3)} & \shortstack{31.0\\(0.8)} & \shortstack{37.4\\(0.3)} & \shortstack{38.3\\(1.1)} & \shortstack{39.3\\(0.7)} & \shortstack{42.1\\(3.6)} & \shortstack{25.7\\(6.9)} & \shortstack{36.7\\(2.6)} & \shortstack{33.6\\(1.4)} & \shortstack{51.6\\(6.7)} & \shortstack{22.1\\(6.8)} & \shortstack{35.1\\(2.6)} & \shortstack{13.8\\(1.4)} & \shortstack{37.2\\(6.1)} & \shortstack{6.5\\(0.8)} & \shortstack{16.8\\(1.1)} & \shortstack{30.1\\(1.1)} & \shortstack{41.0\\(4.1)} & \shortstack{19.5\\(3.7)} & \shortstack{29.6\\(0.9)} \\

        \midrule
        
        AZR-14B-Coder & \shortstack{39.8\\(2.3)} & \shortstack{22.2\\(0.4)} & \shortstack{38.9\\(0.5)} & \shortstack{43.5\\(0.4)} & \shortstack{34.1\\(3.2)} & \shortstack{29.7\\(0.6)} & \shortstack{35.3\\(0.2)} & \shortstack{39.8\\(0.3)} & \shortstack{25.8\\(1.7)} & \shortstack{35.3\\(1.5)} & \shortstack{26.4\\(0.5)} & \shortstack{30.4\\(1.5)} & \shortstack{8.5\\(1.0)} & \shortstack{18.9\\(0.5)} & \shortstack{9.5\\(0.1)} & \shortstack{16.2\\(1.0)} & \shortstack{24.3\\(2.0)} & \shortstack{26.2\\(0.1)} & \shortstack{25.1\\(0.1)} & \shortstack{30.3\\(0.3)} \\

        \midrule
        
        CURE-14B & \shortstack{53.6\\(0.7)} & \shortstack{77.3\\(0.5)} & \shortstack{47.7\\(0.2)} & \shortstack{59.9\\(1.8)} & \shortstack{45.1\\(0.2)} & \shortstack{86.1\\(0.8)} & \shortstack{41.2\\(0.2)} & \shortstack{50.5\\(0.6)} & \shortstack{34.9\\(0.7)} & \shortstack{87.0\\(1.1)} & \shortstack{32.5\\(0.3)} & \shortstack{44.1\\(0.5)} & \shortstack{14.9\\(0.1)} & \shortstack{77.5\\(1.1)} & \shortstack{12.0\\(0.2)} & \shortstack{26.1\\(0.5)} & \shortstack{33.6\\(0.2)} & \shortstack{82.4\\(0.4)} & \shortstack{30.1\\(0.1)} & \shortstack{41.8\\(0.6)} \\

        \midrule
        
        \textbf{Qwen2.5-14B-Ins. + CoSPlay} & \shortstack{54.9\\(0.4)} & \shortstack{70.1\\(2.5)} & \shortstack{47.9\\(0.2)} & \shortstack{55.7\\(1.9)} & \shortstack{46.0\\(1.4)} & \shortstack{74.4\\(1.0)} & \shortstack{43.5\\(0.3)} & \shortstack{53.6\\(0.4)} & \shortstack{34.4\\(1.2)} & \shortstack{71.7\\(0.7)} & \shortstack{31.5\\(0.6)} & \shortstack{39.1\\(1.4)} & \shortstack{14.3\\(0.2)} & \shortstack{86.1\\(0.7)} & \shortstack{11.8\\(0.2)} & \shortstack{24.7\\(1.4)} & \shortstack{33.8\\(0.5)} & \shortstack{77.6\\(0.5)} & \shortstack{30.8\\(0.2)} & \shortstack{41.2\\(0.8)} \\
        \quad w/ Cluster & \shortstack{56.5\\(1.3)} & \shortstack{70.1\\(2.5)} & \shortstack{47.9\\(0.2)} & \shortstack{57.6\\(1.0)} & \shortstack{46.2\\(1.3)} & \shortstack{74.4\\(1.0)} & \shortstack{43.5\\(0.3)} & \shortstack{54.8\\(1.2)} & \shortstack{34.0\\(0.9)} & \shortstack{71.7\\(0.7)} & \shortstack{31.5\\(0.6)} & \shortstack{40.0\\(1.3)} & \shortstack{14.3\\(0.2)} & \shortstack{86.1\\(0.7)} & \shortstack{11.8\\(0.2)} & \shortstack{24.6\\(1.2)} & \shortstack{34.0\\(0.4)} & \shortstack{77.6\\(0.5)} & \shortstack{30.8\\(0.2)} & \shortstack{42.0\\(0.9)} \\
        \rowcolor{green!4}

        \midrule

        \textbf{CURE-14B + CoSPlay} & \shortstack{53.6\\(1.5)} & \shortstack{71.5\\(2.6)} & \shortstack{52.4\\(0.6)} & \shortstack{61.2\\(1.3)} & \shortstack{45.4\\(0.9)} & \shortstack{76.2\\(2.1)} & \shortstack{46.9\\(0.2)} & \shortstack{54.9\\(0.2)} & \shortstack{34.6\\(0.4)} & \shortstack{75.1\\(2.9)} & \shortstack{36.8\\(0.3)} & \shortstack{44.9\\(2.6)} & \shortstack{15.3\\(0.1)} & \shortstack{89.4\\(1.9)} & \shortstack{19.1\\(0.3)} & \shortstack{32.5\\(1.4)} & \shortstack{33.8\\(0.2)} & \shortstack{80.1\\(2.0)} & \shortstack{36.0\\(0.2)} & \shortstack{45.9\\(0.9)} \\

        
        \quad w/ Cluster & \shortstack{54.9\\(1.3)} & \shortstack{71.5\\(2.6)} & \shortstack{52.4\\(0.6)} & \shortstack{60.7\\(0.7)} & \shortstack{45.5\\(0.9)} & \shortstack{76.2\\(2.1)} & \shortstack{46.9\\(0.2)} & \shortstack{55.3\\(0.7)} & \shortstack{34.9\\(0.8)} & \shortstack{75.1\\(2.9)} & \shortstack{36.8\\(0.3)} & \shortstack{45.3\\(2.3)} & \shortstack{15.3\\(0.1)} & \shortstack{89.4\\(1.9)} & \shortstack{19.1\\(0.3)} & \shortstack{32.6\\(1.5)} & \shortstack{34.1\\(0.2)} & \shortstack{80.1\\(2.0)} & \shortstack{36.0\\(0.2)} & \shortstack{46.2\\(0.7)} \\

        \bottomrule
    \end{tabular}
    }
\end{table}

\newpage

\section{Detailed Evolution of Cluster-Size Density for Correct and Incorrect Codes}
\label{app:cluster_density_visualization}

Figures~\ref{fig:cluster_density_7b_codecontests}--\ref{fig:cluster_density_14b_livecodebench} show the evolution of cluster-size distributions for correct and wrong code candidates during self-play.
Overall, correct codes tend to appear in larger execution-signature clusters, while wrong codes are more frequently distributed over smaller clusters.
This supports the use of cluster size as a GT-free proxy for functional correctness.

The separation is most evident on medium-difficulty datasets such as CodeContests and LiveCodeBench.
On these benchmarks, correct solutions often concentrate in large clusters, while wrong solutions remain relatively fragmented.
For harder datasets such as CodeForces, the distributions overlap more, indicating that candidate behaviors are harder to distinguish by cluster size alone.
These results suggest that execution-consensus clustering is especifically effective when the model can generate a non-trivial number of correct or near-correct candidates, allowing correct behaviors to form stable clusters.


\captionsetup[figure]{font=footnotesize,skip=1pt}

\begin{figure}[H]
    \centering

    \includegraphics[width=0.28\textwidth]{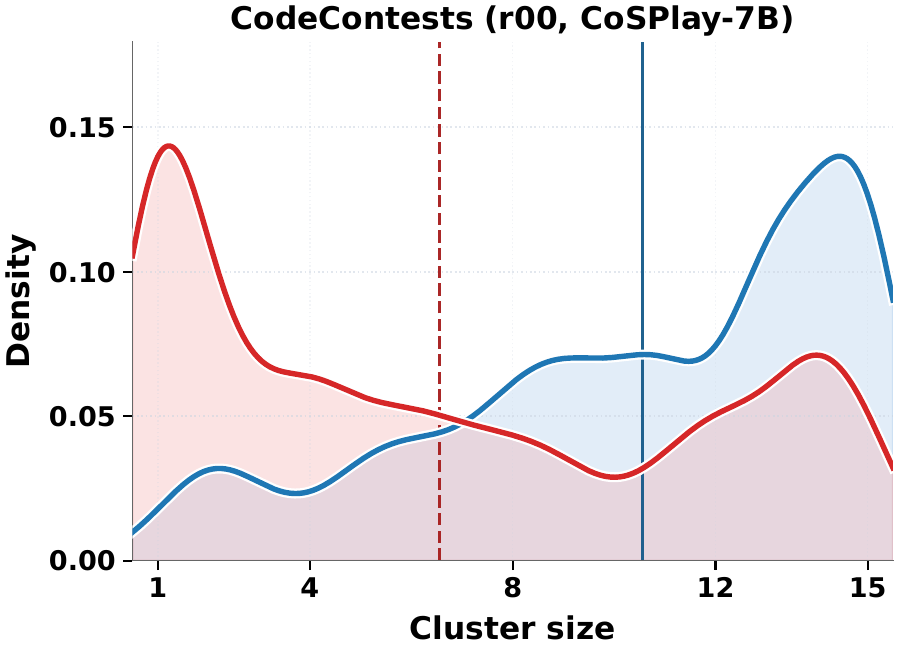}\hfill
    \includegraphics[width=0.28\textwidth]{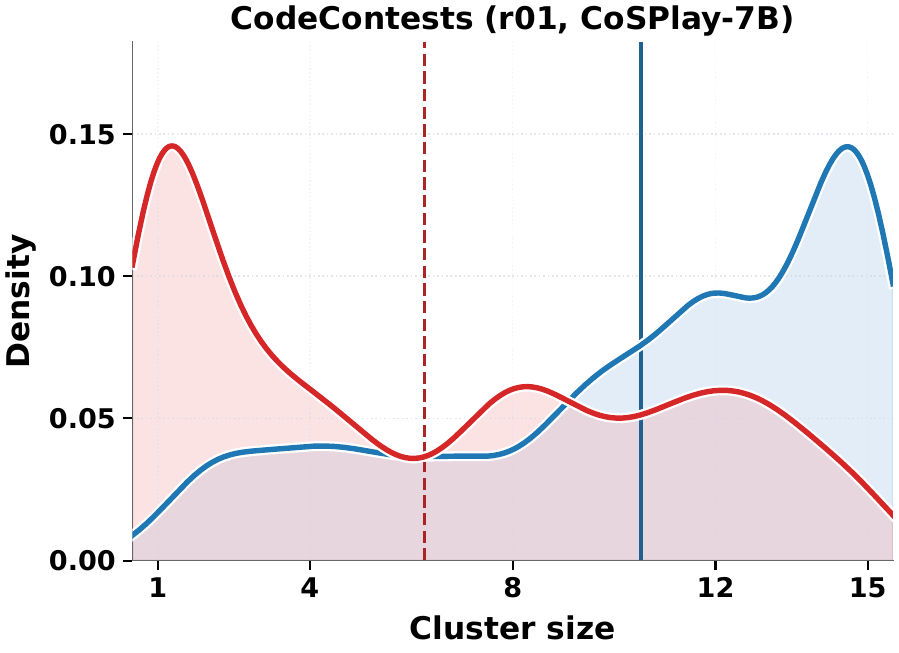}\hfill
    \includegraphics[width=0.28\textwidth]{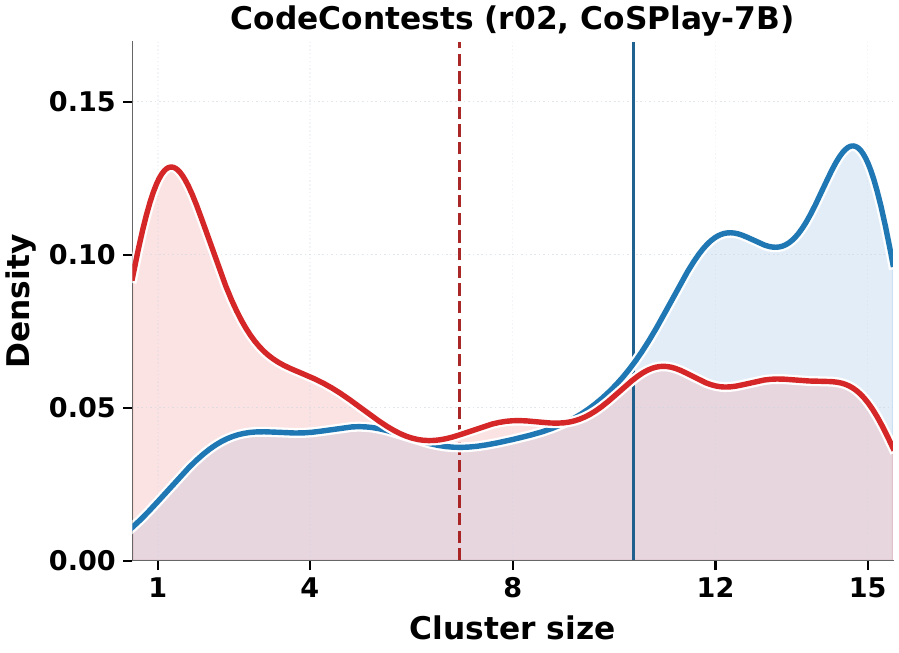}\\[-1.0em]

    \includegraphics[width=0.28\textwidth]{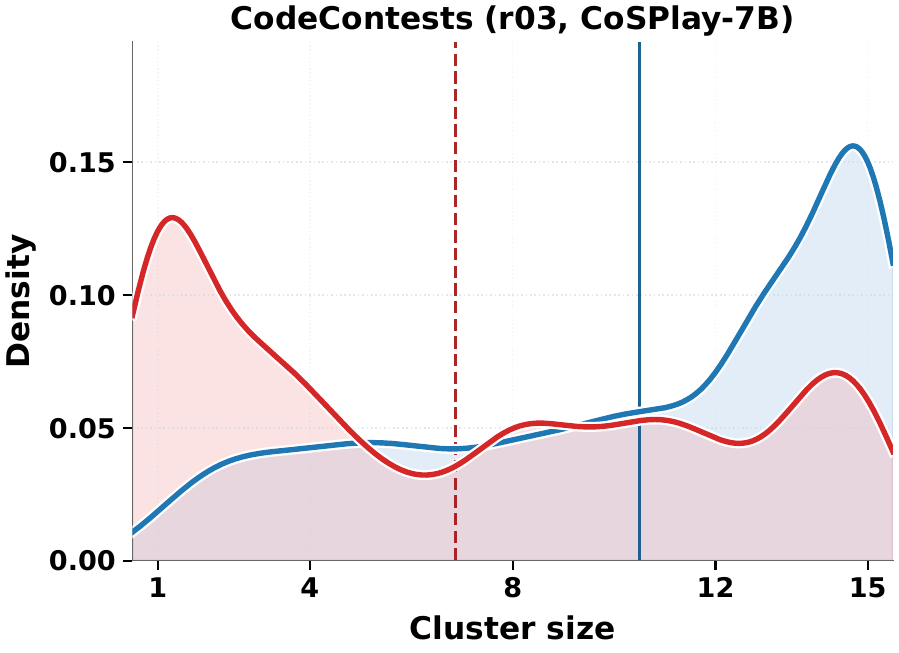}\hfill
    \includegraphics[width=0.28\textwidth]{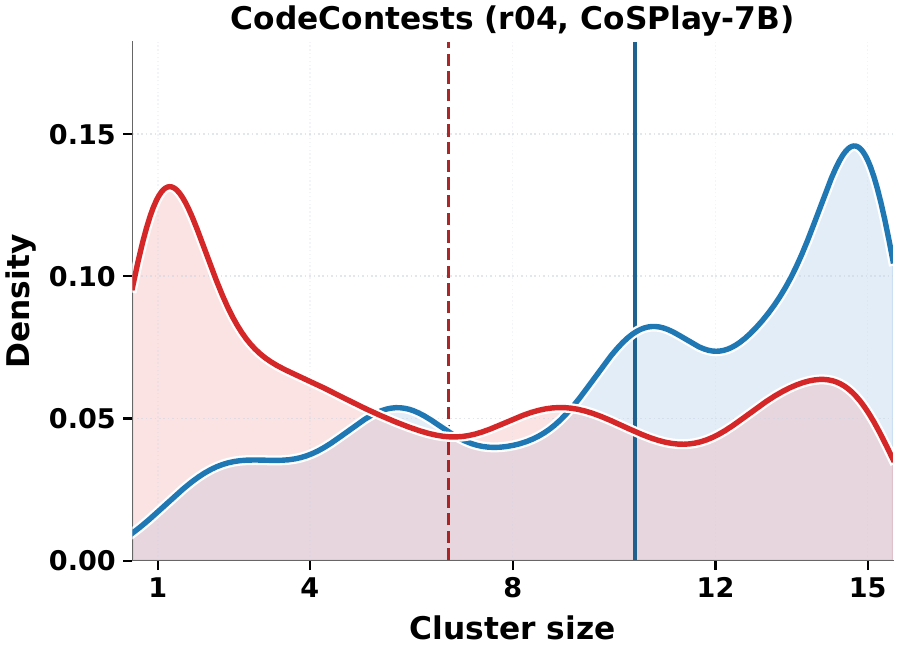}\hfill
    \includegraphics[width=0.28\textwidth]{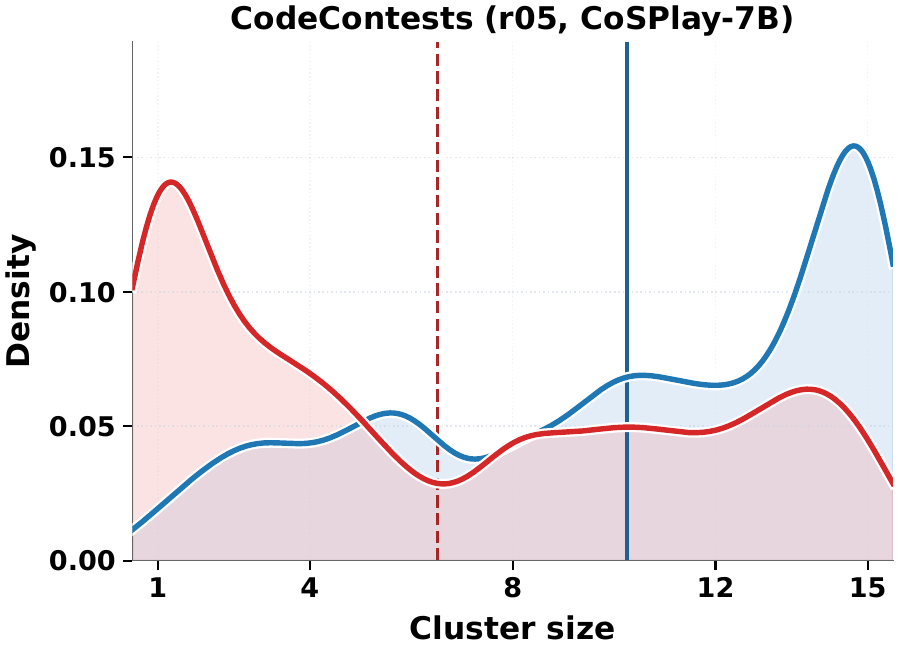}

    \vspace{-0.2em}
    \caption{Density distributions of cluster sizes for correct (blue) and wrong (red) code candidates during self-play on \textbf{CodeContests} with the \textbf{7B model}. The top row shows Round 0--2, and the bottom row shows Round 3--5. Vertical lines denote average values.}
    \label{fig:cluster_density_7b_codecontests}
\end{figure}

\begin{figure}[H]
    \centering

    \includegraphics[width=0.28\textwidth]{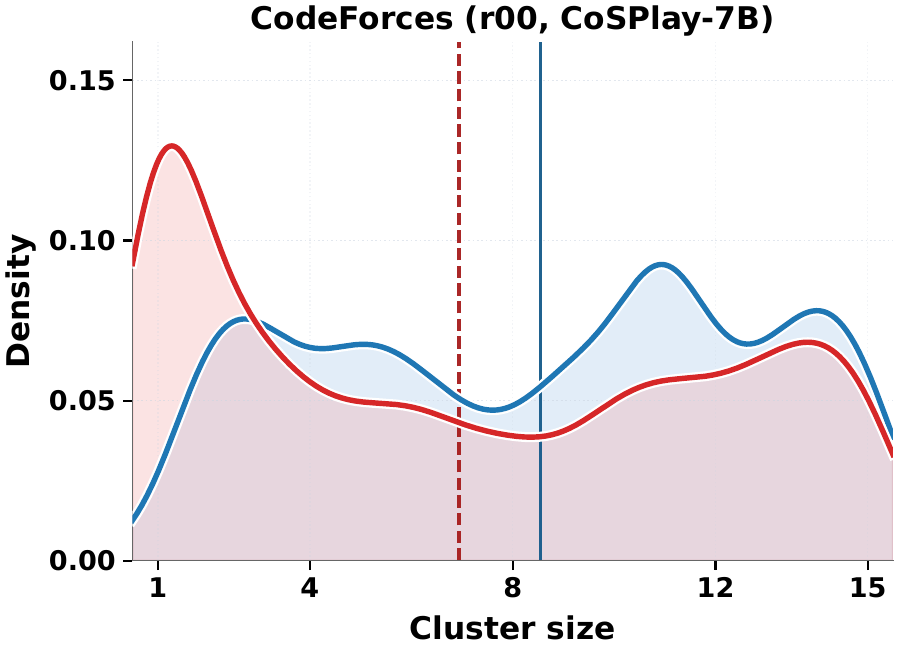}\hfill
    \includegraphics[width=0.28\textwidth]{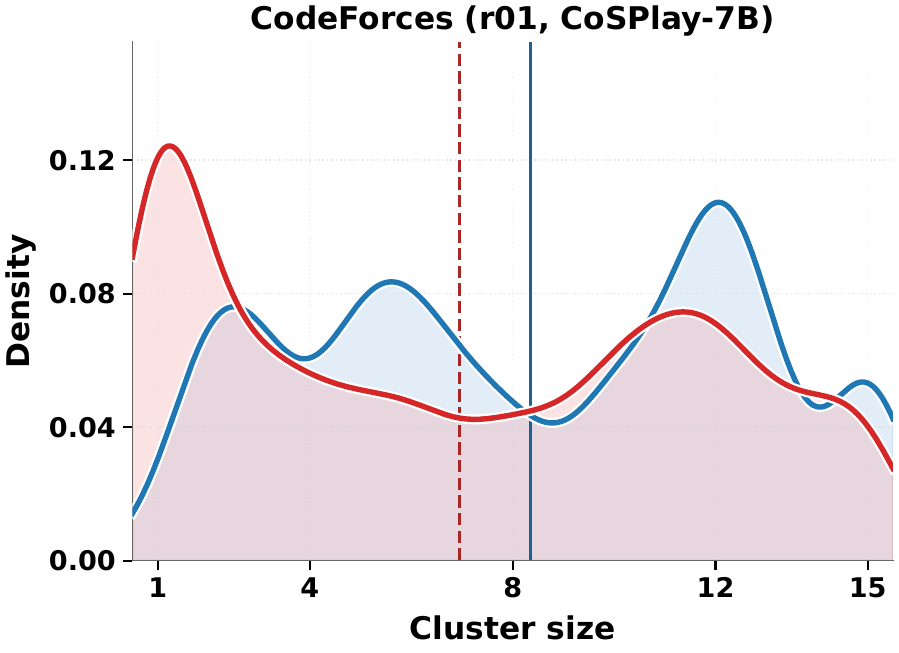}\hfill
    \includegraphics[width=0.28\textwidth]{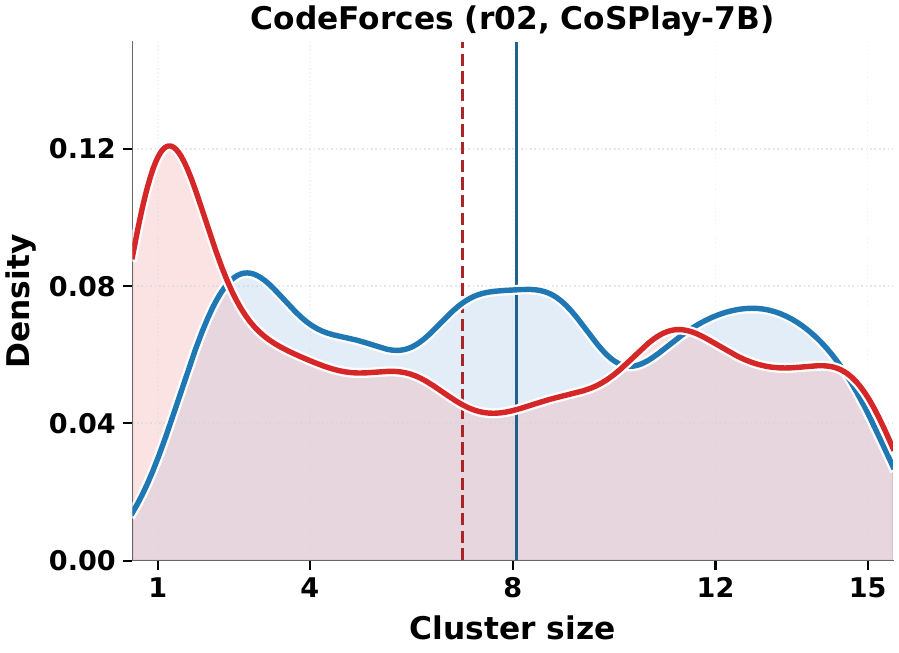}\\[-1.0em]

    \includegraphics[width=0.28\textwidth]{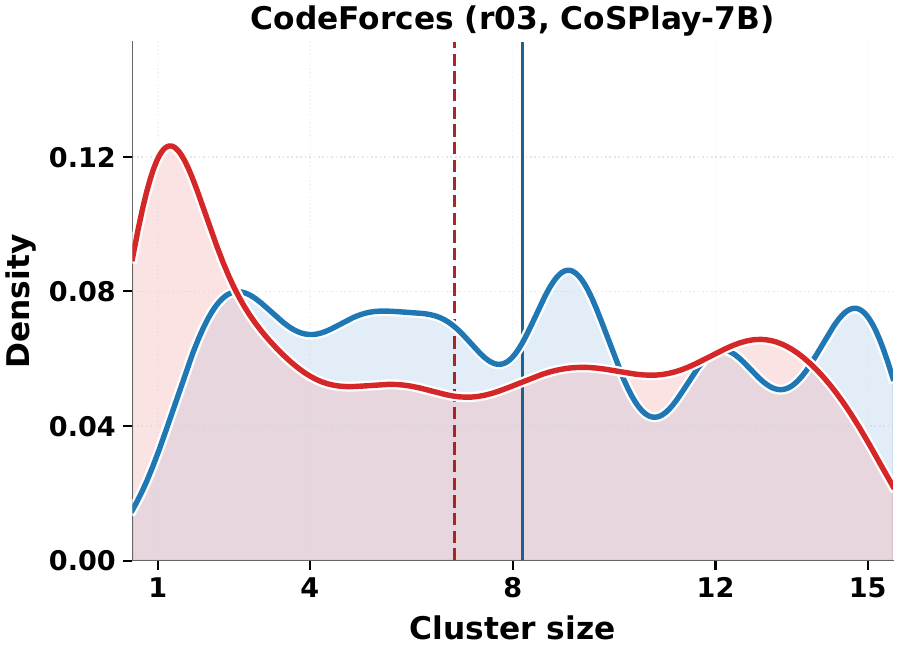}\hfill
    \includegraphics[width=0.28\textwidth]{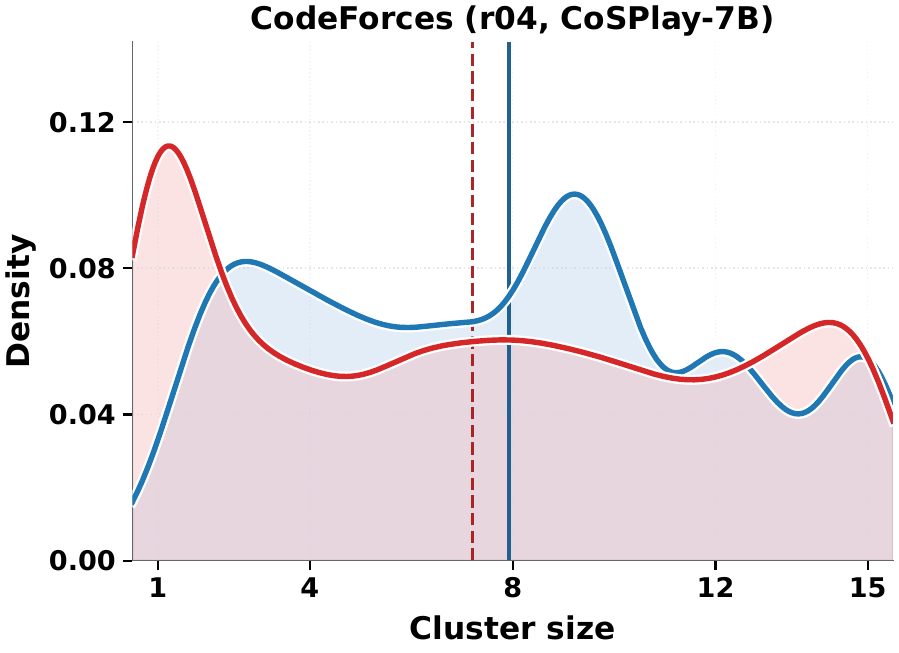}\hfill
    \includegraphics[width=0.28\textwidth]{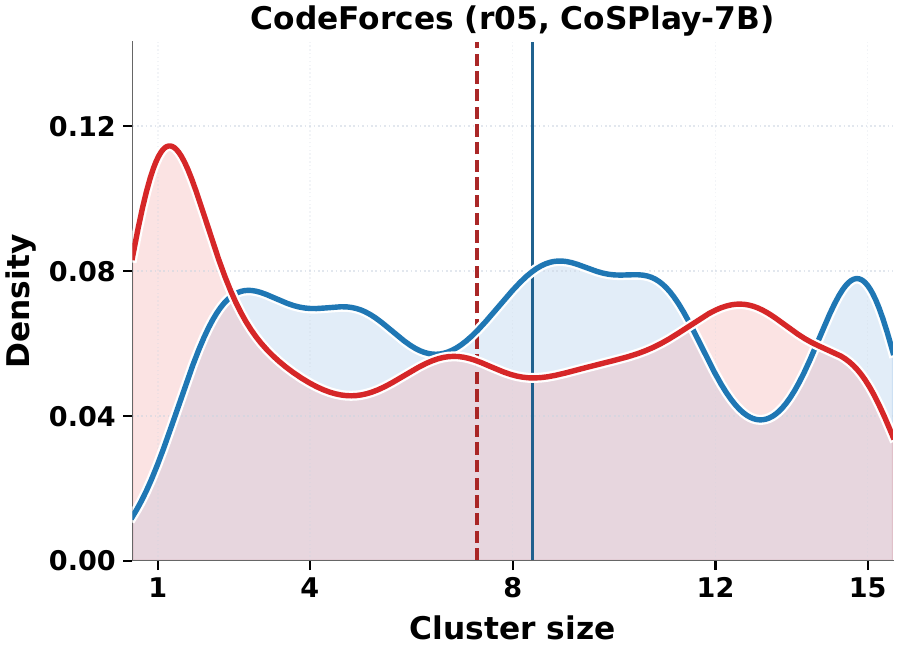}

    \vspace{-0.2em}
    \caption{Density distributions of cluster sizes for correct (blue) and wrong (red) code candidates during self-play on \textbf{CodeForces} with the \textbf{7B model}. The top row shows Round 0--2, and the bottom row shows Round 3--5. Vertical lines denote average values.}
    \label{fig:cluster_density_7b_codeforces}
\end{figure}

\par\noindent
\begin{minipage}{\textwidth}
    \centering
    \includegraphics[width=0.28\textwidth]{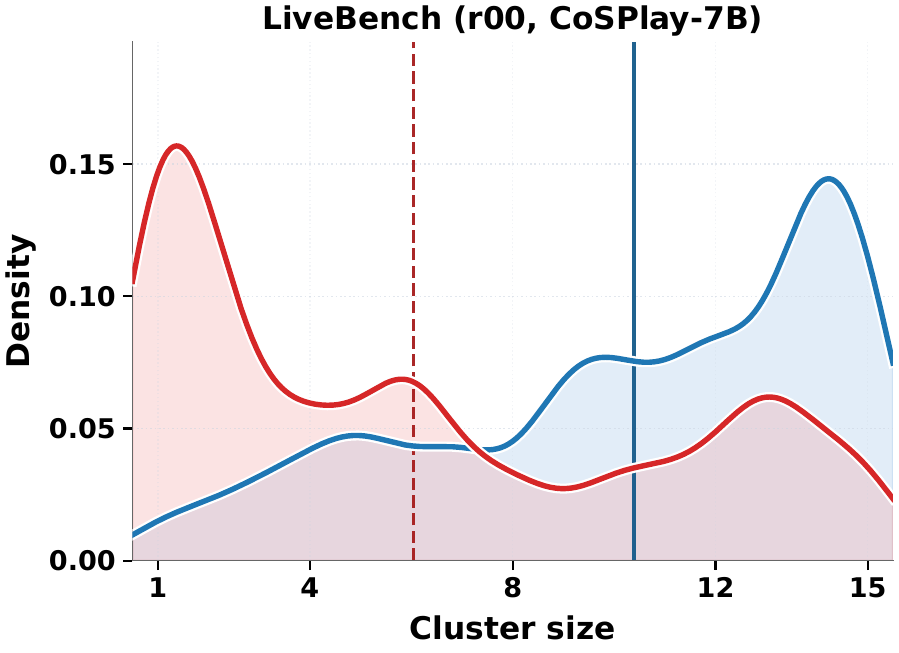}\hfill
    \includegraphics[width=0.28\textwidth]{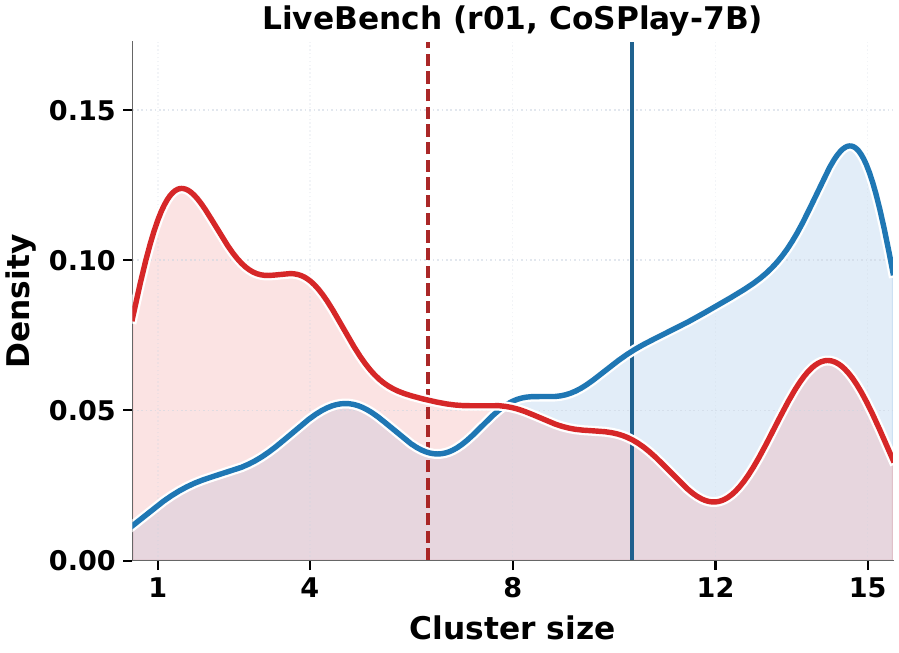}\hfill
    \includegraphics[width=0.28\textwidth]{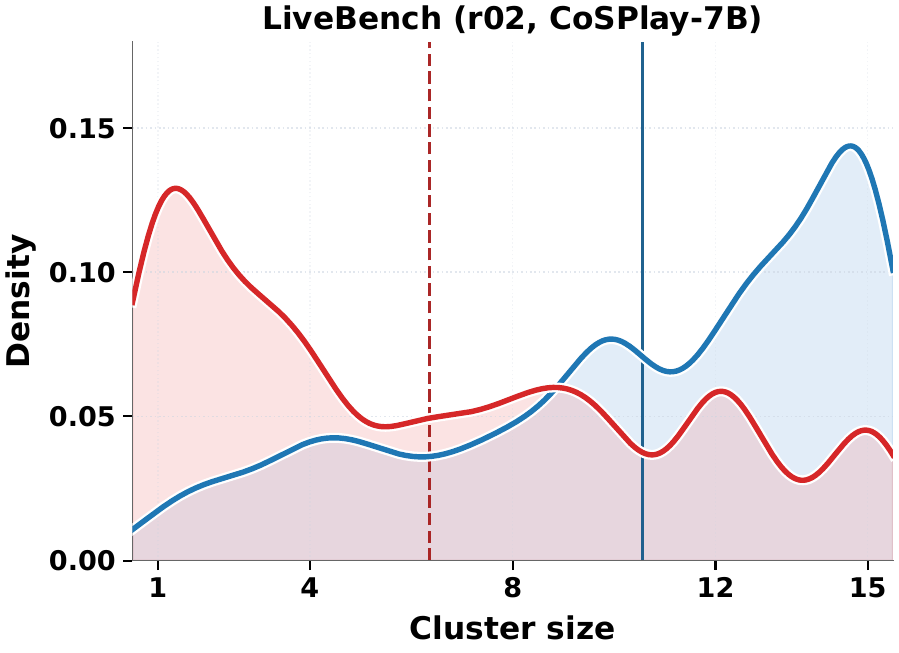}\\[-1.0em]
    \includegraphics[width=0.28\textwidth]{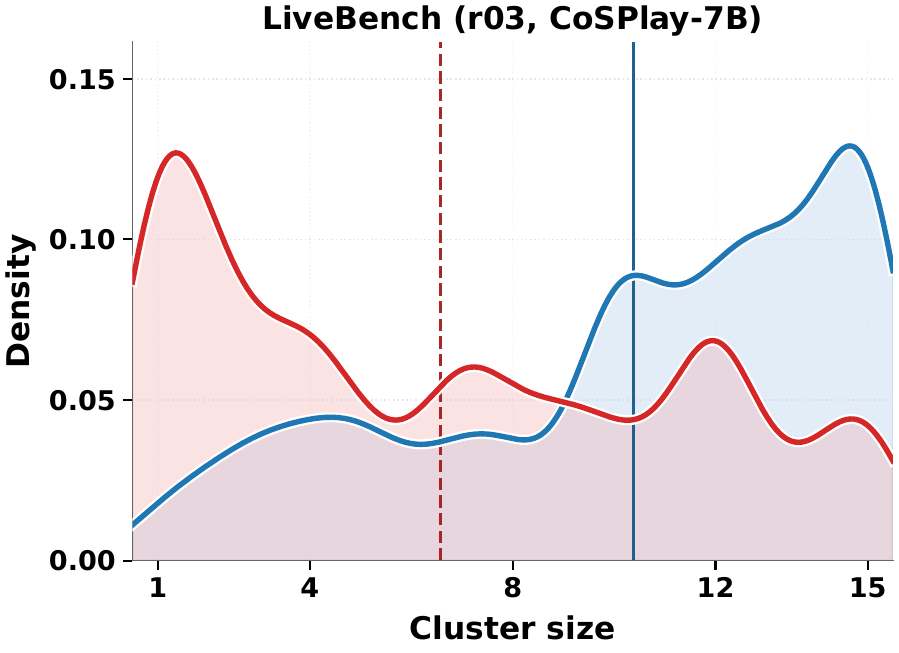}\hfill
    \includegraphics[width=0.28\textwidth]{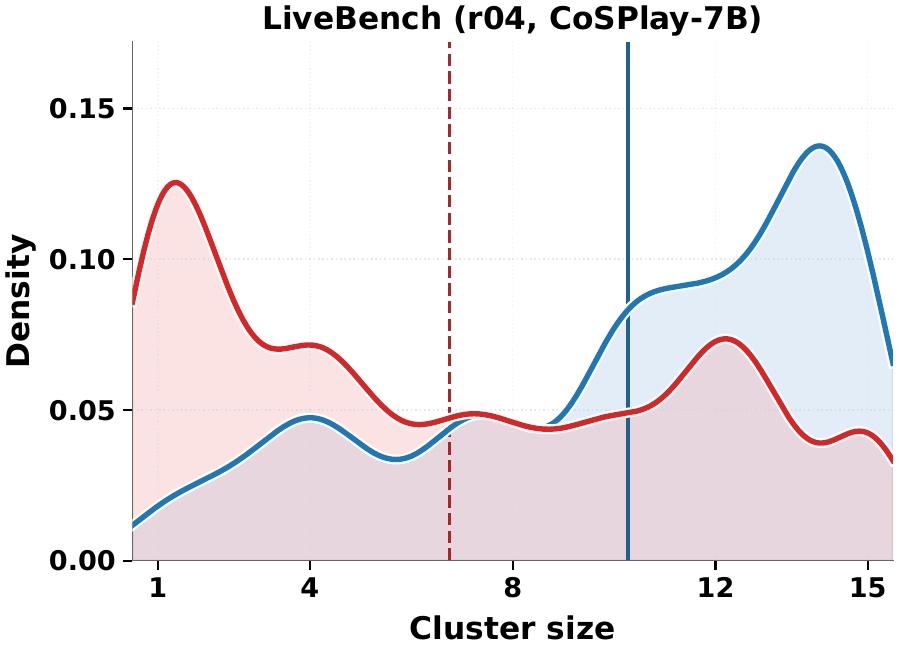}\hfill
    \includegraphics[width=0.28\textwidth]{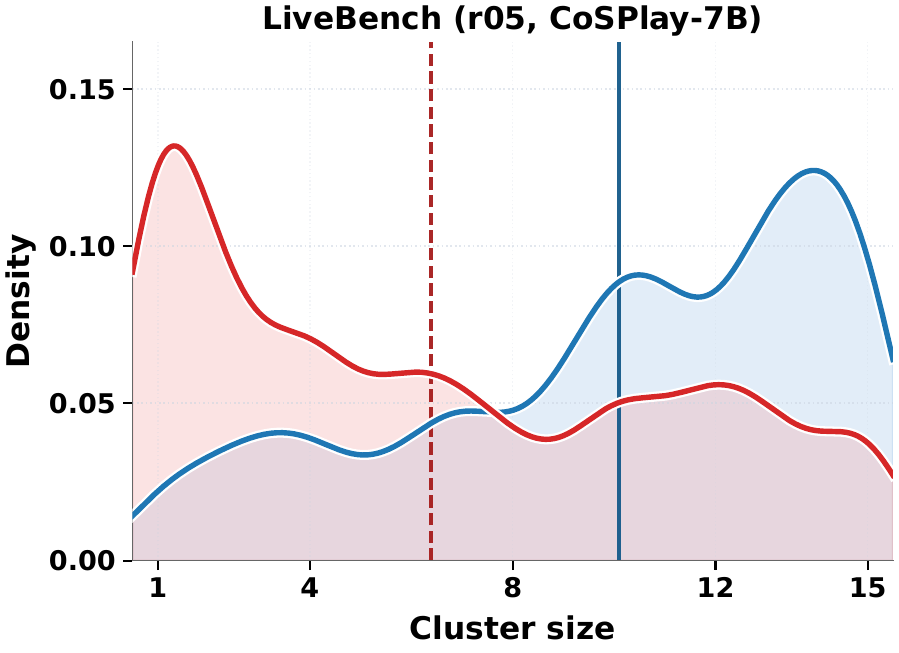}

    \vspace{-0.0em}
    \captionof{figure}{Density distributions of cluster sizes for correct (blue) and wrong (red) code candidates during self-play on \textbf{LiveBench} with the \textbf{7B model}. The top row shows Round 0--2, and the bottom row shows Round 3--5. Vertical lines denote average values.}
    \label{fig:cluster_density_7b_livebench}
\end{minipage}


\par\noindent
\begin{minipage}{\textwidth}
    \centering
    \includegraphics[width=0.28\textwidth]{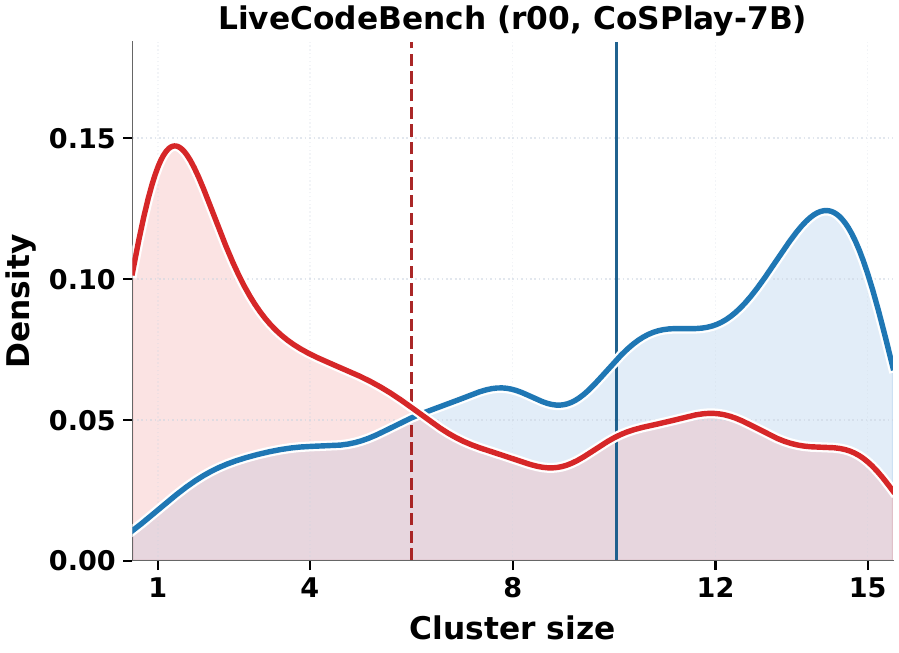}\hfill
    \includegraphics[width=0.28\textwidth]{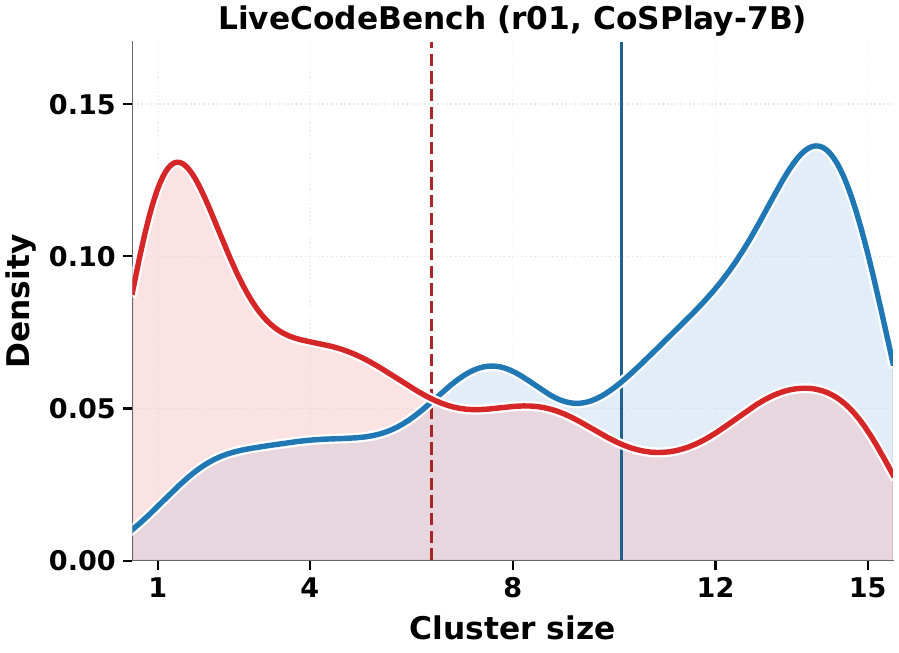}\hfill
    \includegraphics[width=0.28\textwidth]{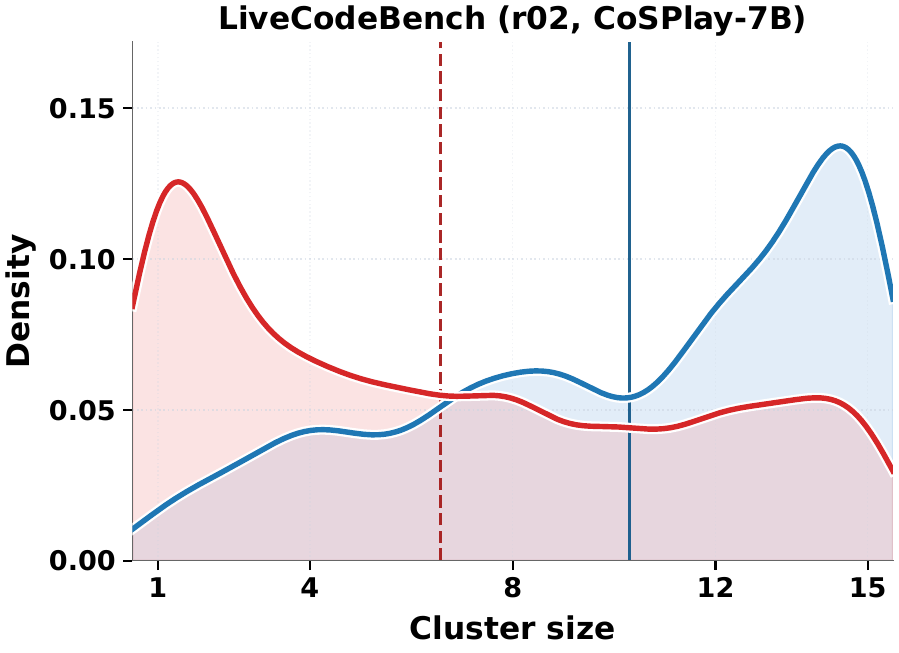}\\[-1.0em]
    \includegraphics[width=0.28\textwidth]{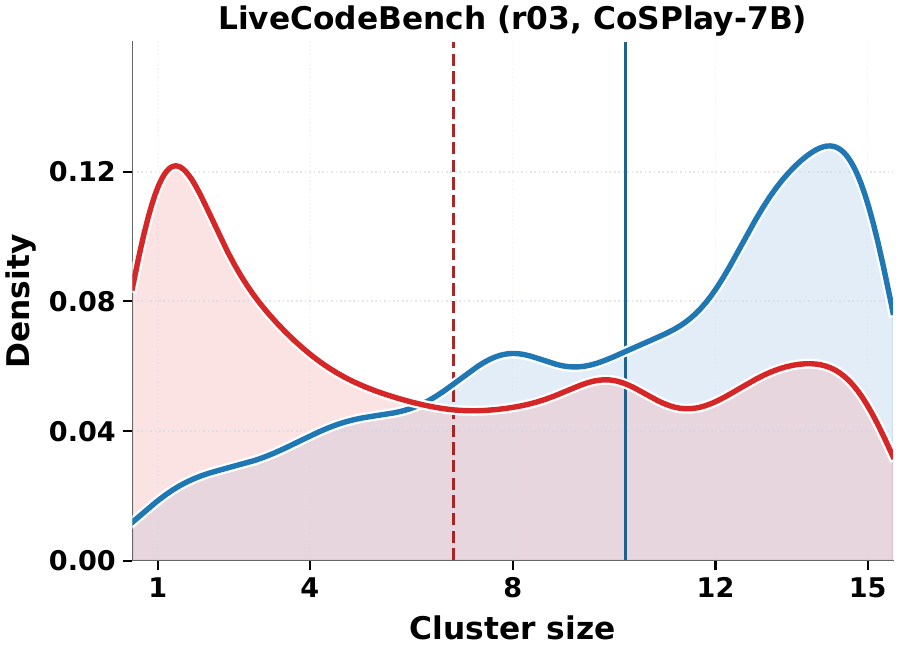}\hfill
    \includegraphics[width=0.28\textwidth]{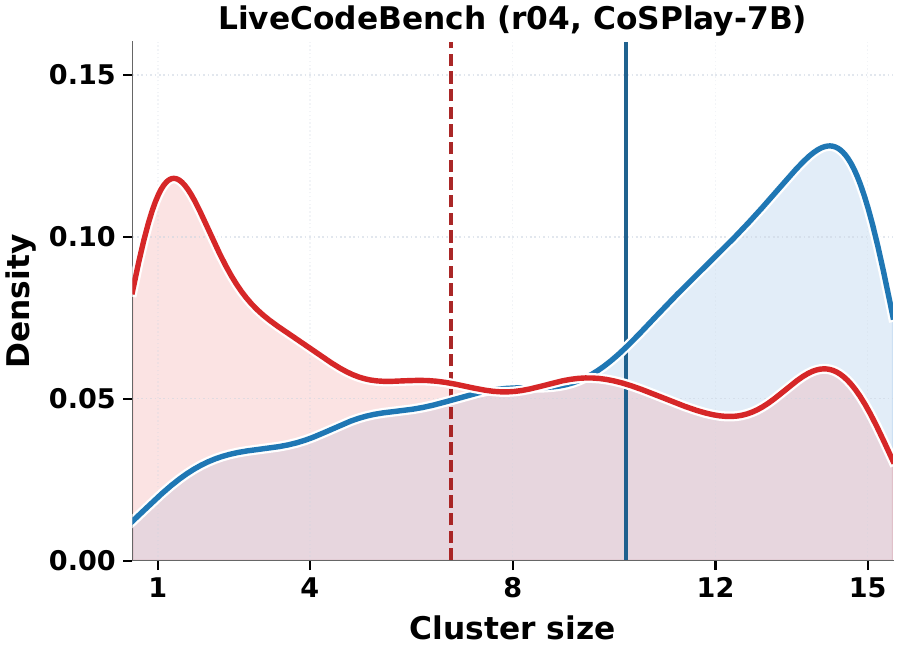}\hfill
    \includegraphics[width=0.28\textwidth]{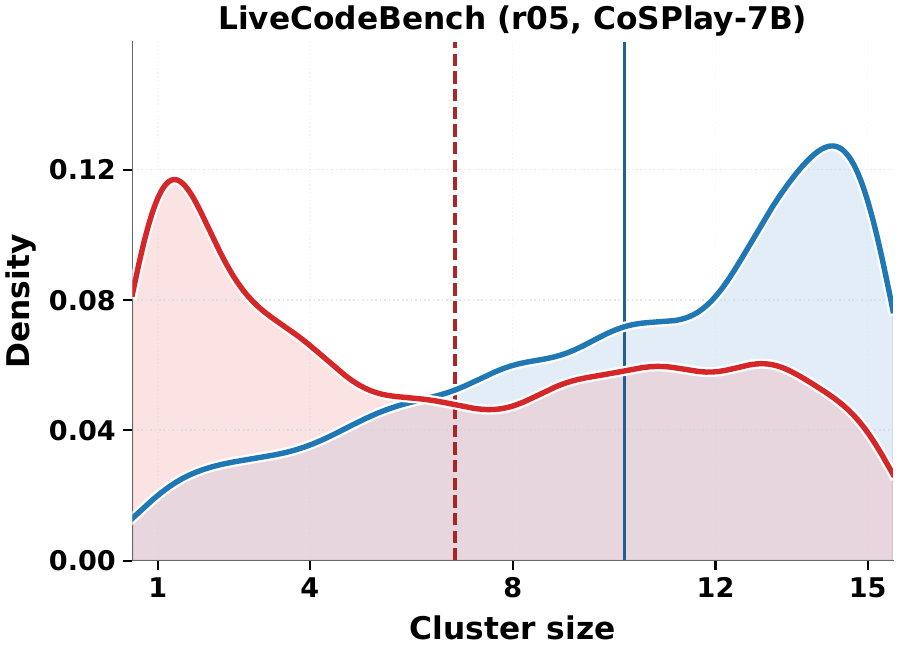}

    \vspace{-0.0em}
    \captionof{figure}{Density distributions of cluster sizes for correct (blue) and wrong (red) code candidates during self-play on \textbf{LiveCodeBench} with the \textbf{7B model}. The top row shows Round 0--2, and the bottom row shows Round 3--5. Vertical lines denote average values.}
    \label{fig:cluster_density_7b_livecodebench}
\end{minipage}

\par\noindent
\begin{minipage}{\textwidth}
    \centering
    \includegraphics[width=0.28\textwidth]{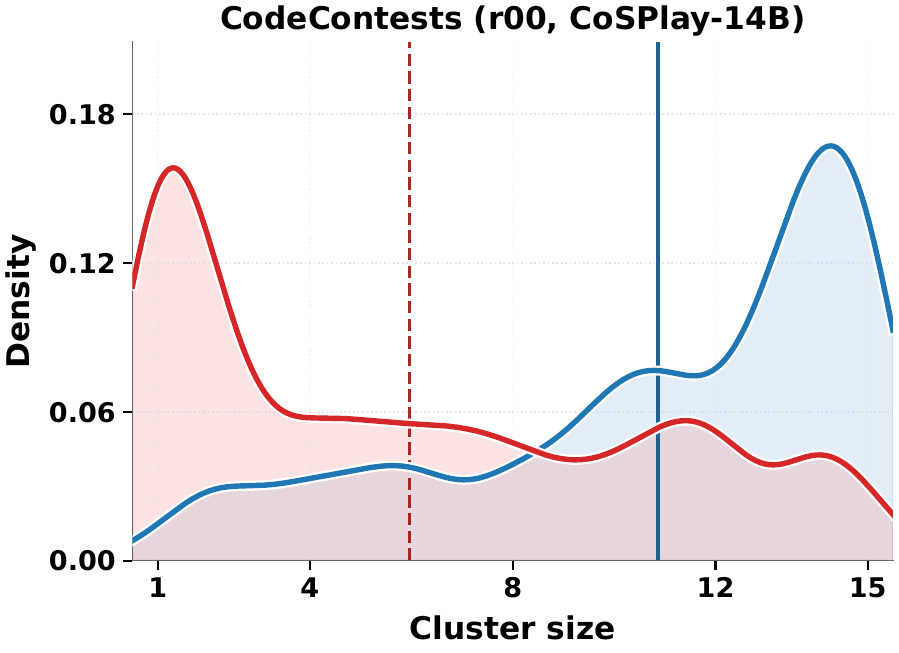}\hfill
    \includegraphics[width=0.28\textwidth]{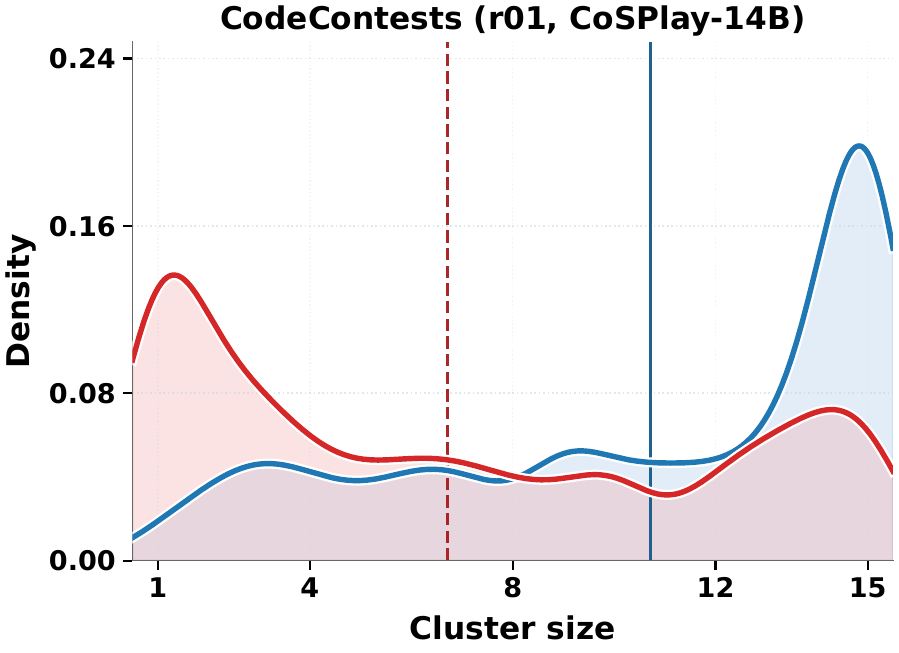}\hfill
    \includegraphics[width=0.28\textwidth]{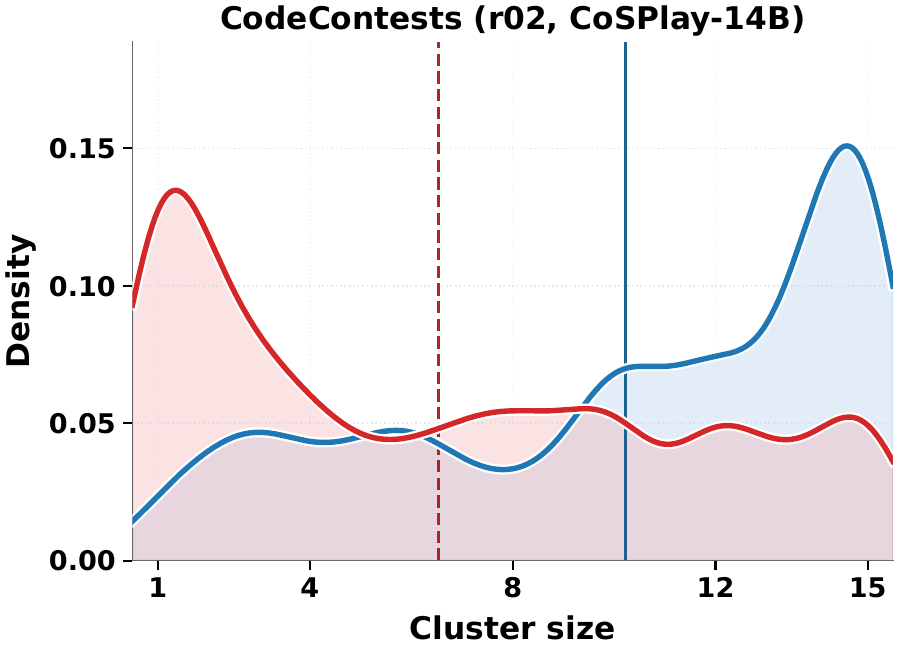}\\[-1.0em]
    \includegraphics[width=0.28\textwidth]{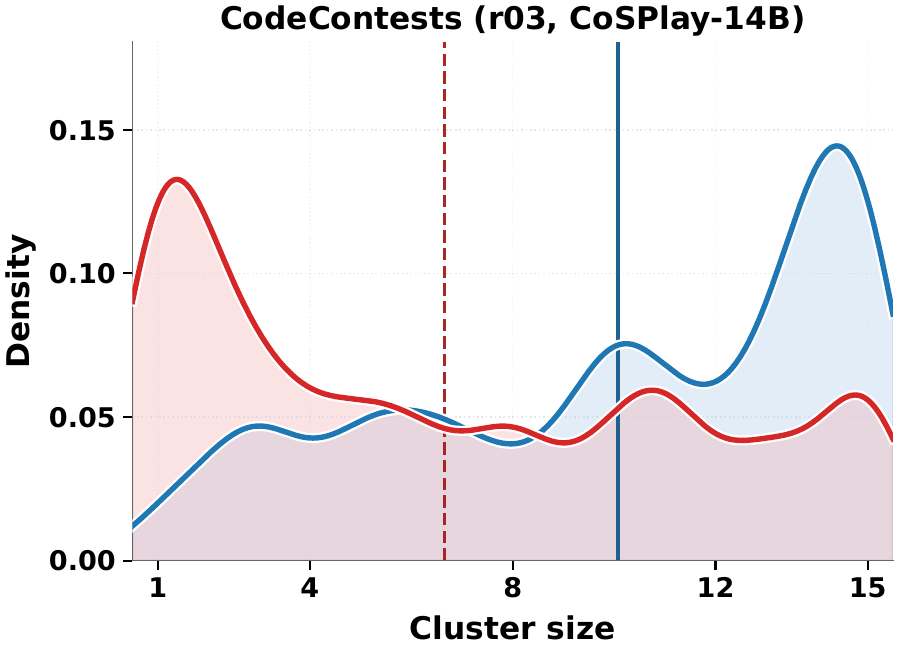}\hfill
    \includegraphics[width=0.28\textwidth]{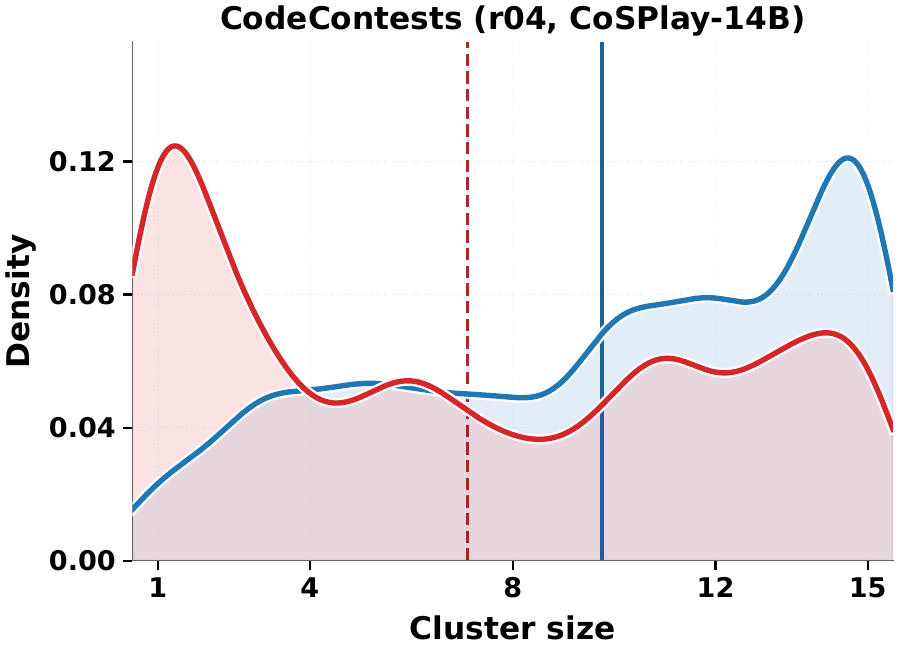}\hfill
    \includegraphics[width=0.28\textwidth]{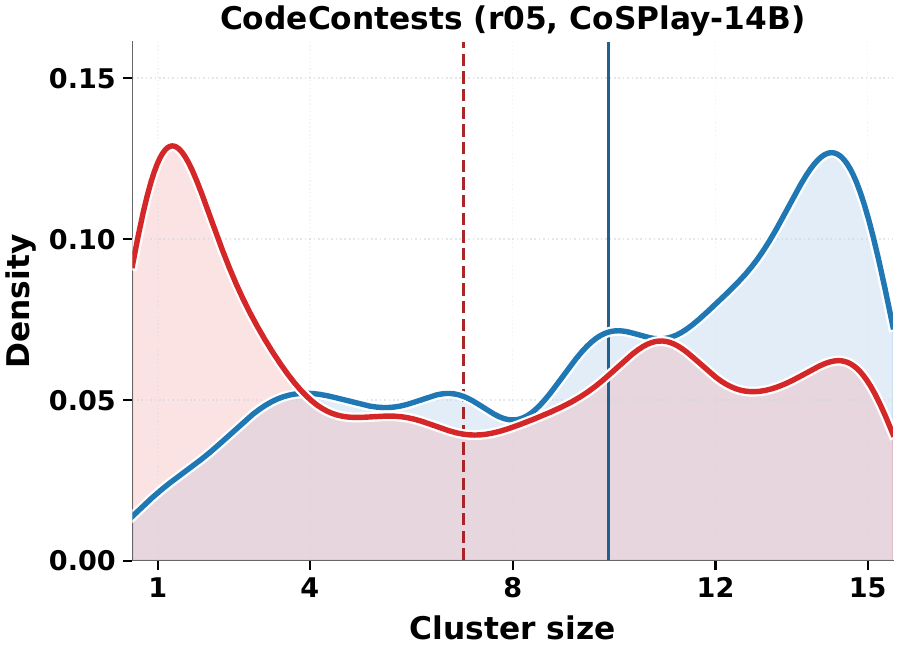}

    \vspace{-0.0em}
    \captionof{figure}{Density distributions of cluster sizes for correct (blue) and wrong (red) code candidates during self-play on \textbf{CodeContests} with the \textbf{14B model}. The top row shows Round 0--2, and the bottom row shows Round 3--5. Vertical lines denote average values.}
    \label{fig:cluster_density_14b_codecontests}
\end{minipage}

\par\noindent
\begin{minipage}{\textwidth}
    \centering
    \includegraphics[width=0.28\textwidth]{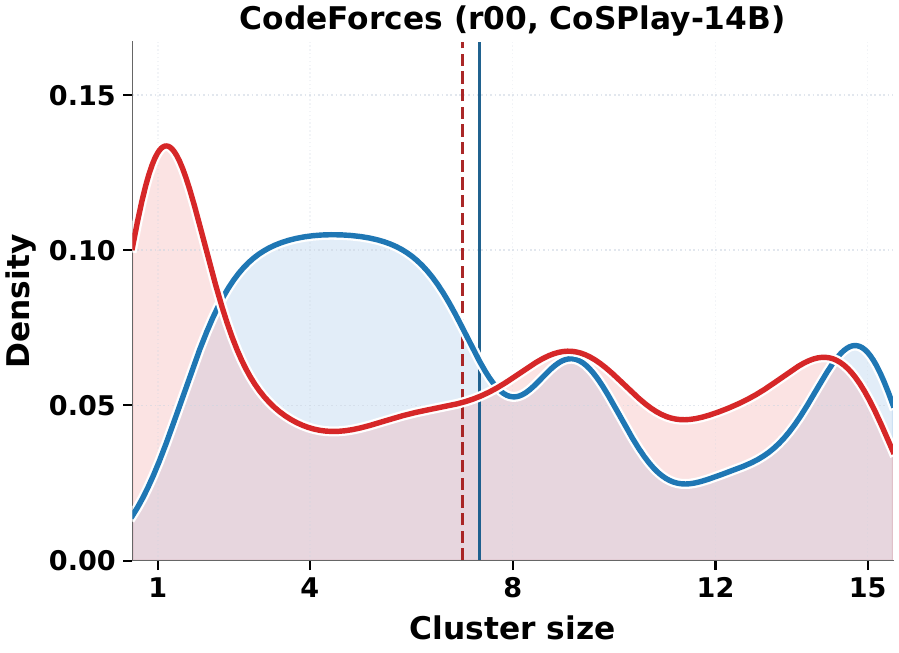}\hfill
    \includegraphics[width=0.28\textwidth]{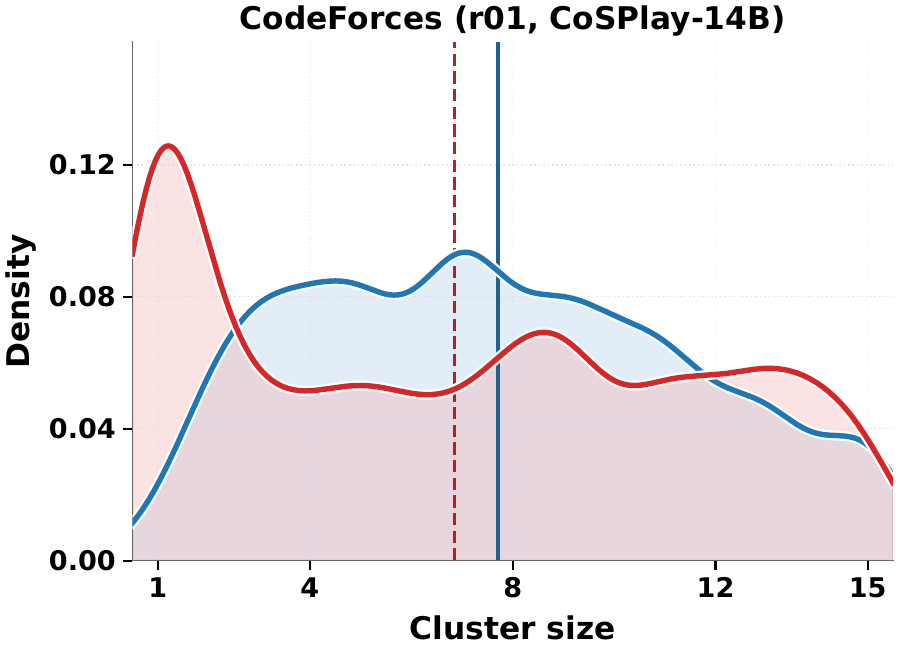}\hfill
    \includegraphics[width=0.28\textwidth]{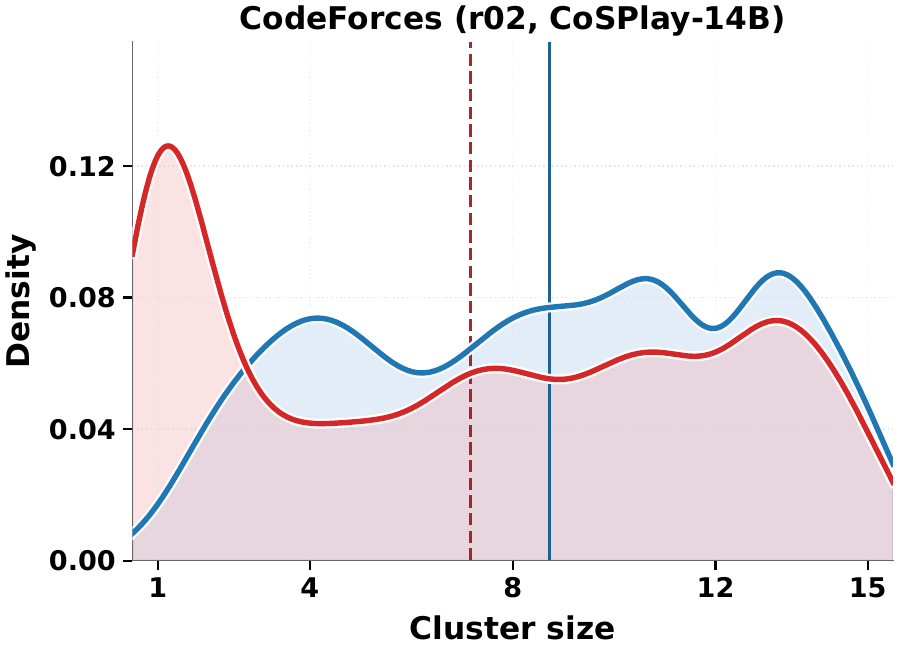}\\[-1.0em]
    \includegraphics[width=0.28\textwidth]{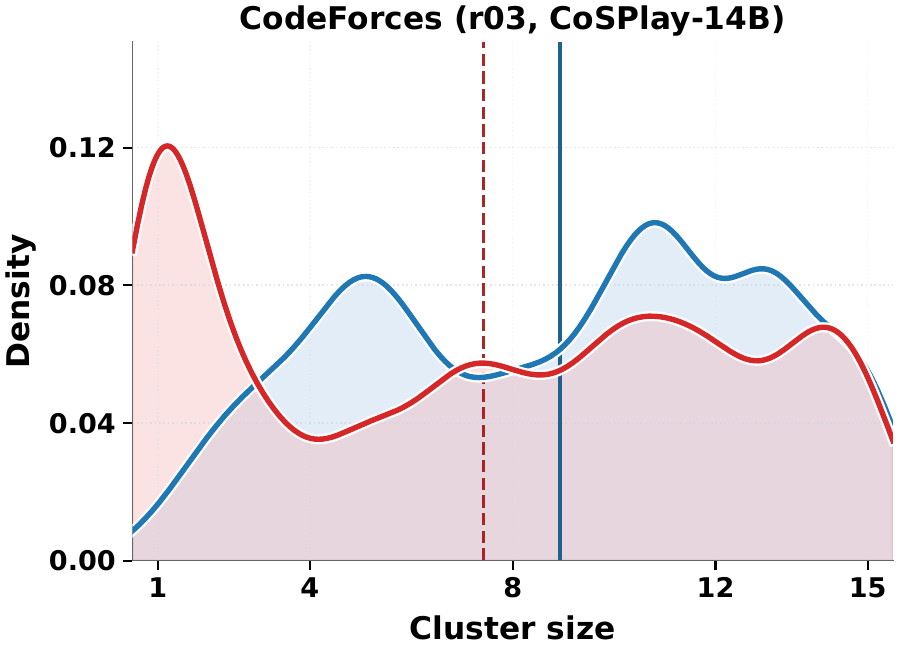}\hfill
    \includegraphics[width=0.28\textwidth]{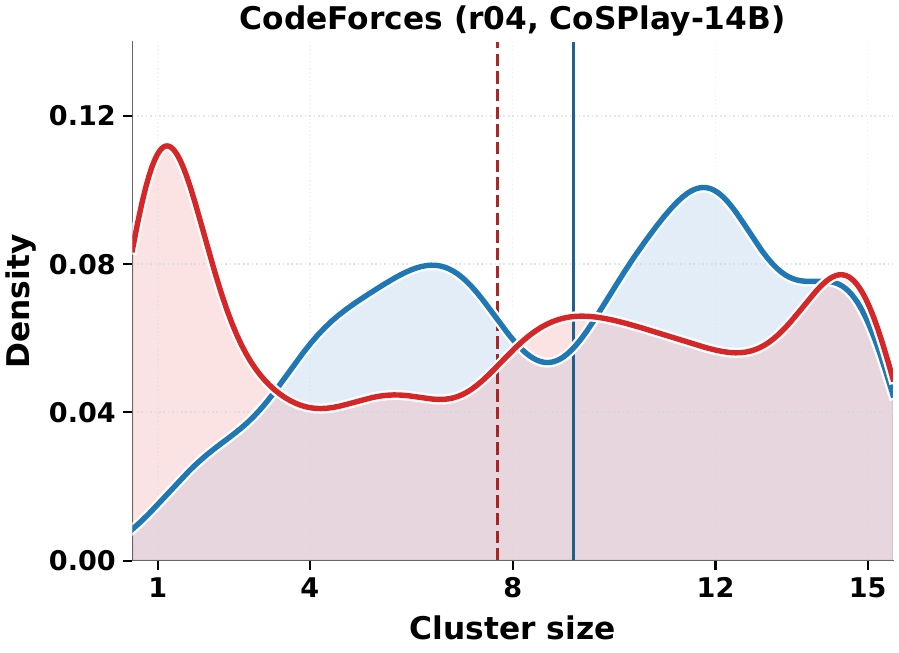}\hfill
    \includegraphics[width=0.28\textwidth]{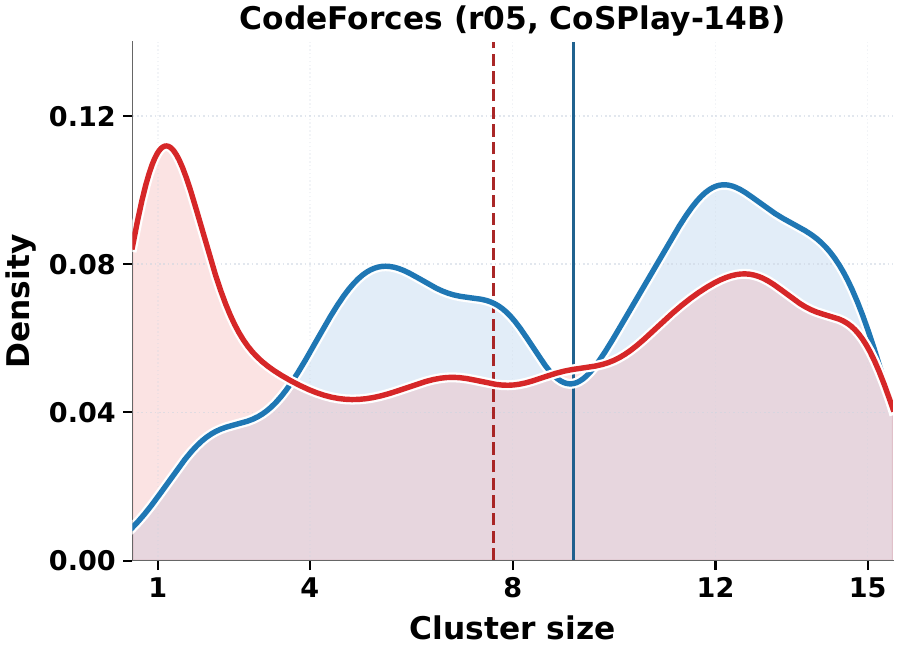}

    \vspace{-0.0em}
    \captionof{figure}{Density distributions of cluster sizes for correct (blue) and wrong (red) code candidates during self-play on \textbf{CodeForces} with the \textbf{14B model}. The top row shows Round 0--2, and the bottom row shows Round 3--5. Vertical lines denote average values.}
    \label{fig:cluster_density_14b_codeforces}
\end{minipage}


\par\noindent
\begin{minipage}{\textwidth}
    \centering
    \includegraphics[width=0.28\textwidth]{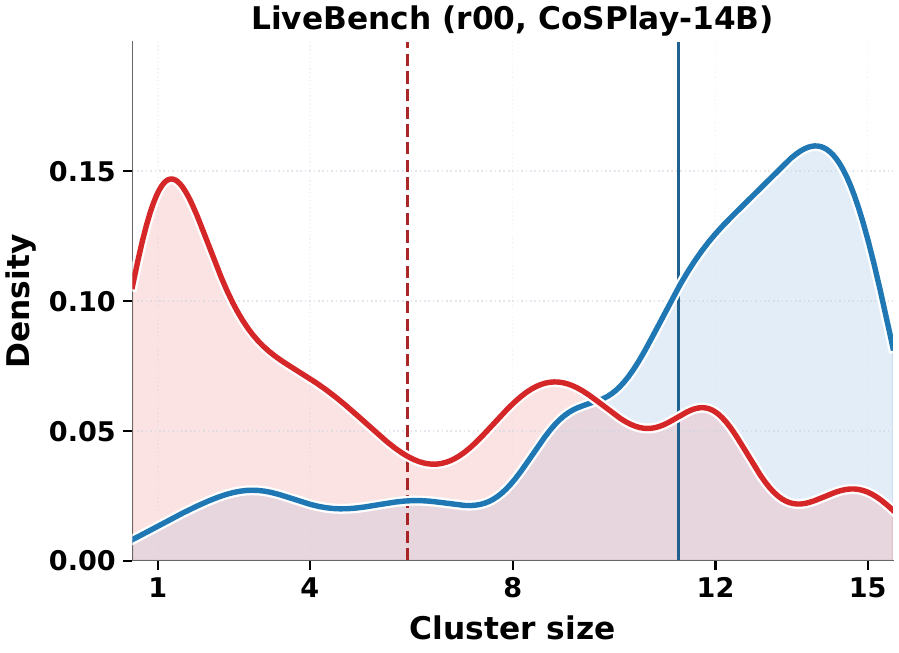}\hfill
    \includegraphics[width=0.28\textwidth]{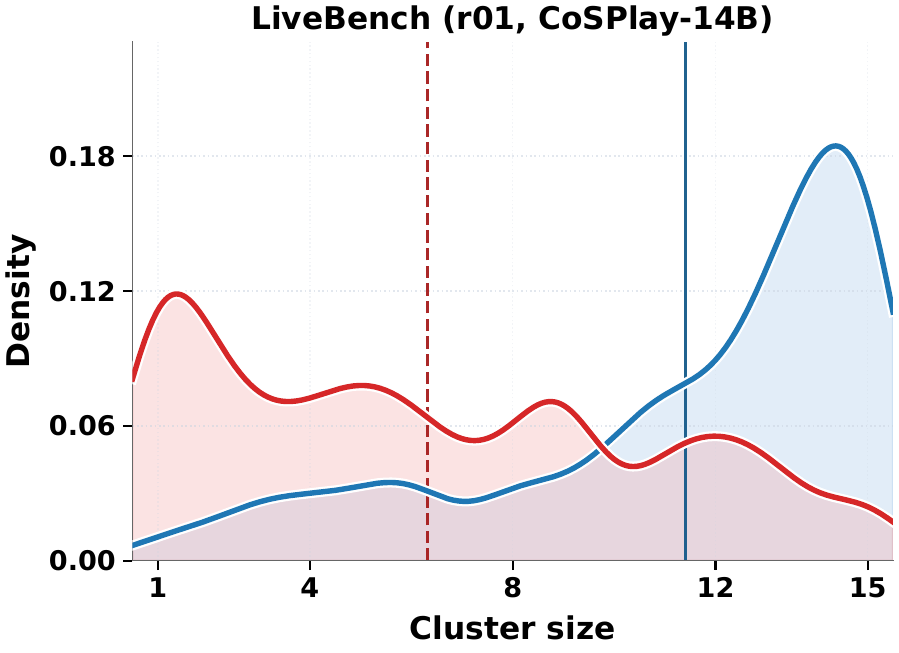}\hfill
    \includegraphics[width=0.28\textwidth]{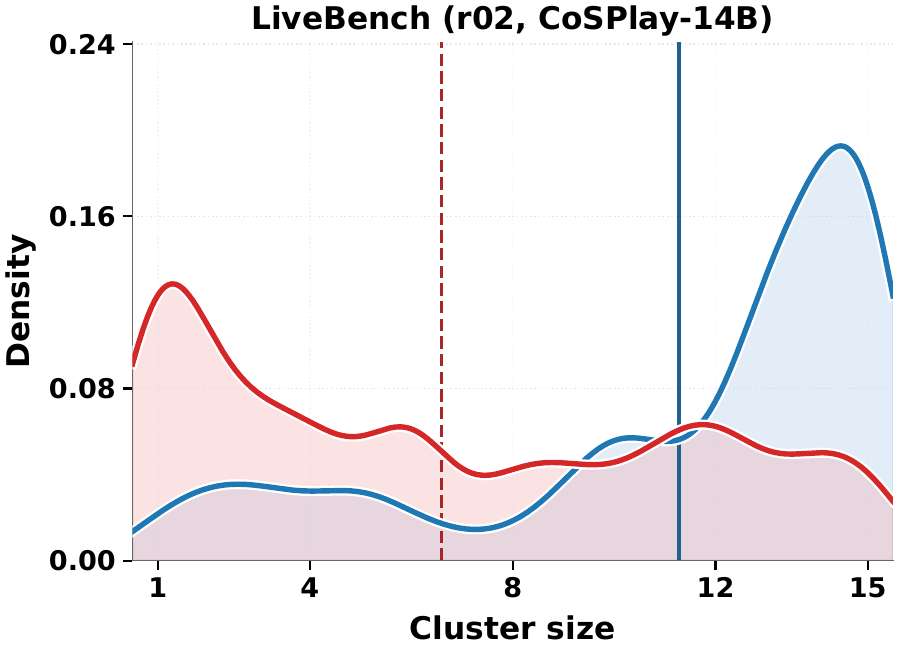}\\[-1.0em]
    \includegraphics[width=0.28\textwidth]{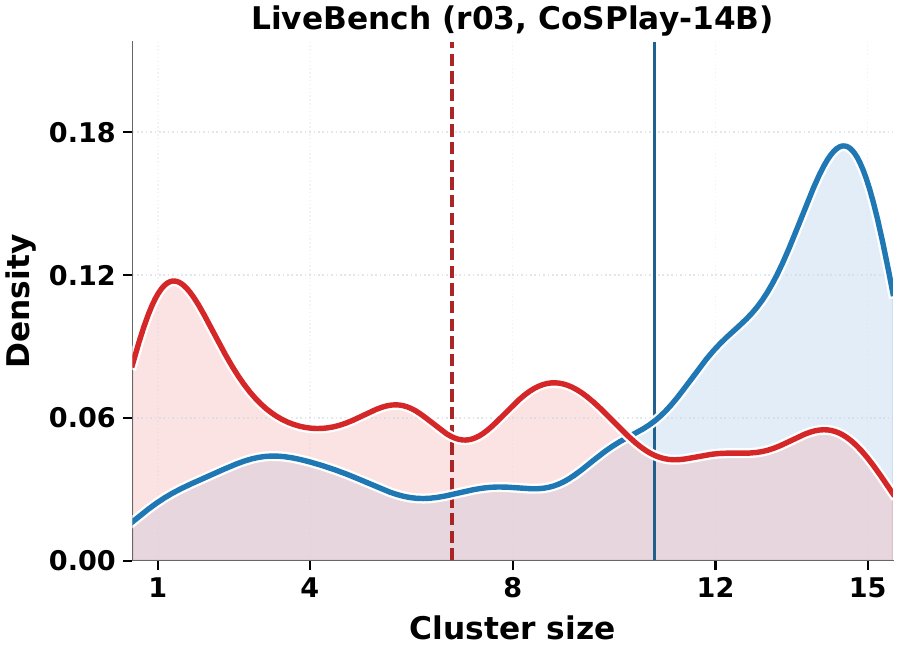}\hfill
    \includegraphics[width=0.28\textwidth]{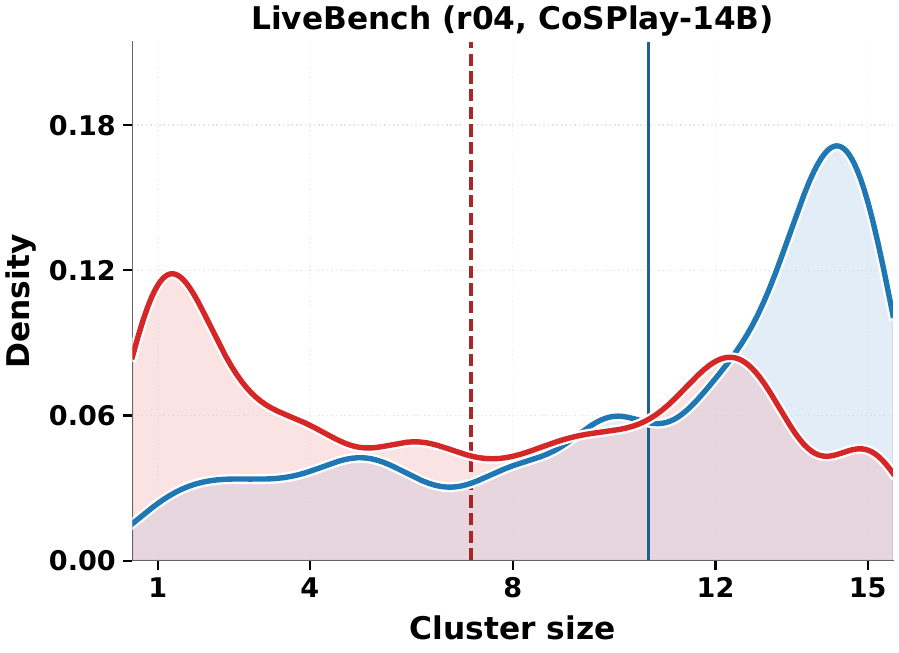}\hfill
    \includegraphics[width=0.28\textwidth]{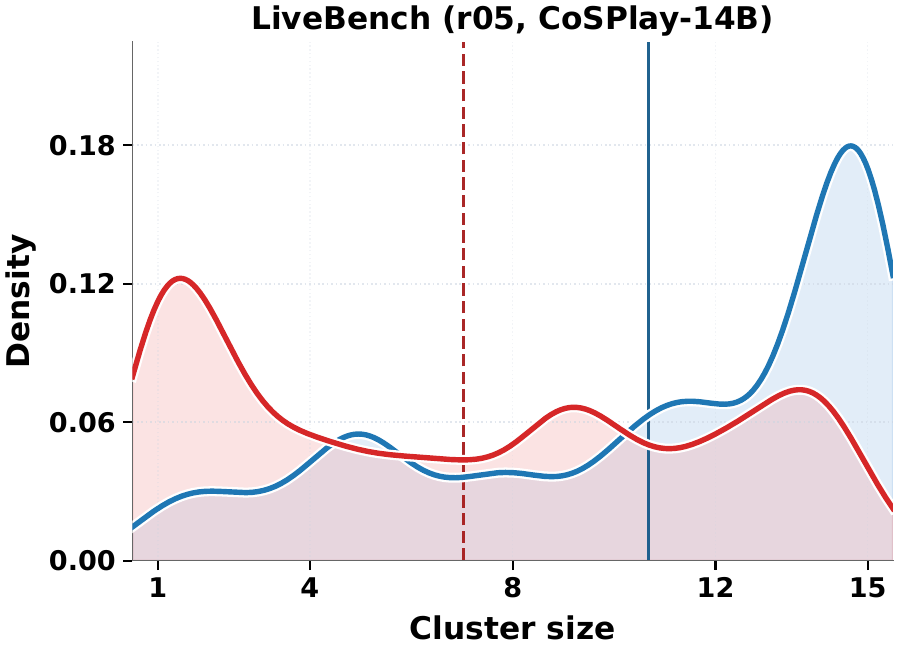}

    \vspace{-0.0em}
    \captionof{figure}{Density distributions of cluster sizes for correct (blue) and wrong (red) code candidates during self-play on \textbf{LiveBench} with the \textbf{14B model}. The top row shows Round 0--2, and the bottom row shows Round 3--5. Vertical lines denote average values.}
    \label{fig:cluster_density_14b_livebench}
\end{minipage}

\par\noindent
\begin{minipage}{\textwidth}
    \centering
    \includegraphics[width=0.28\textwidth]{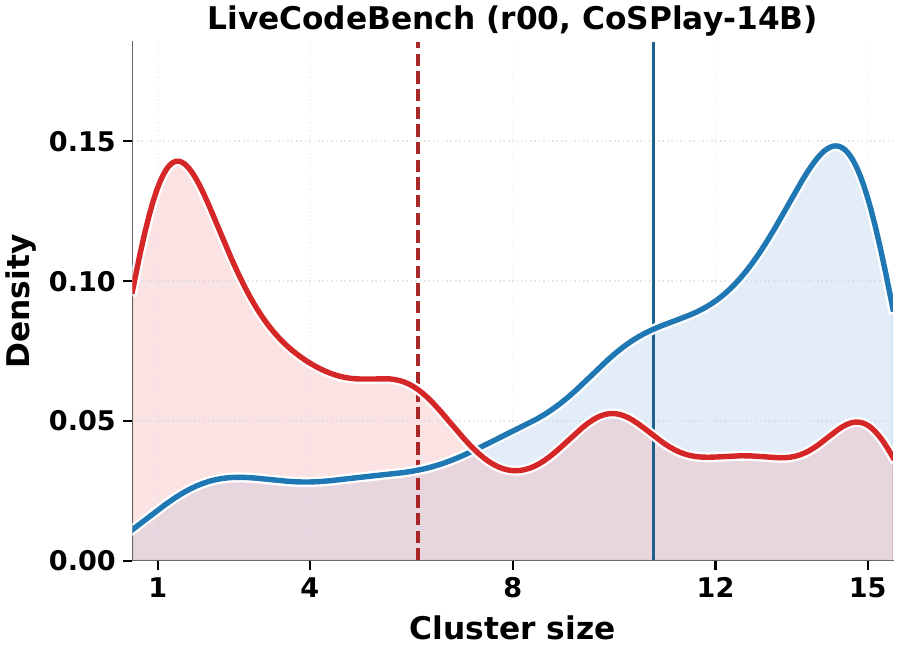}\hfill
    \includegraphics[width=0.28\textwidth]{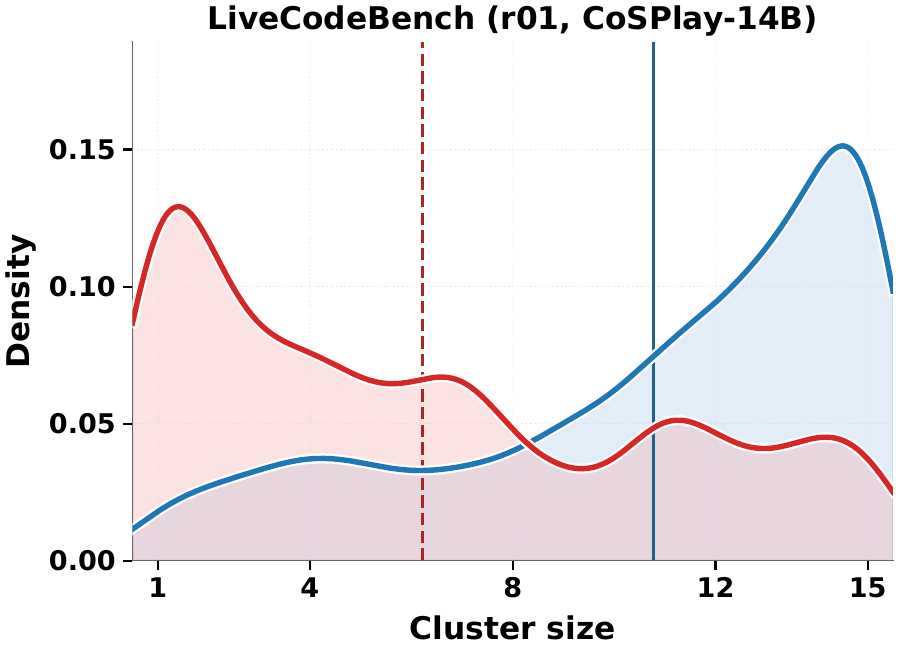}\hfill
    \includegraphics[width=0.28\textwidth]{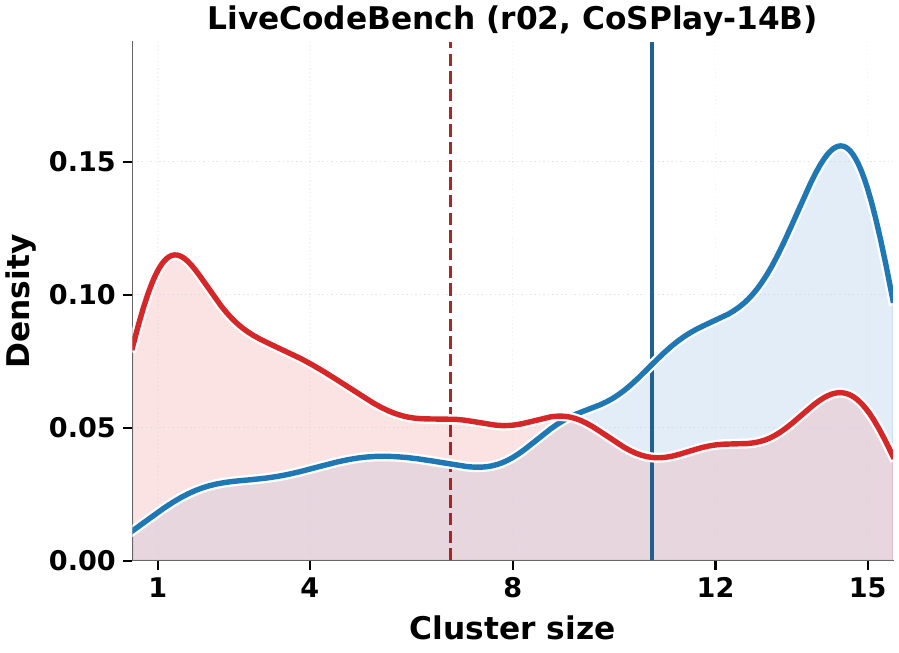}\\[-1.0em]
    \includegraphics[width=0.28\textwidth]{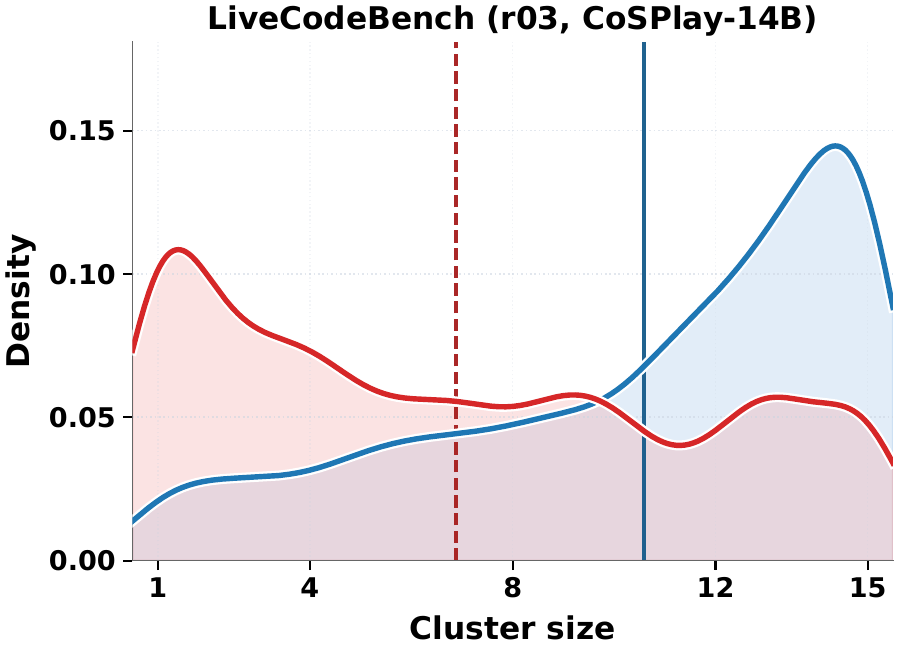}\hfill
    \includegraphics[width=0.28\textwidth]{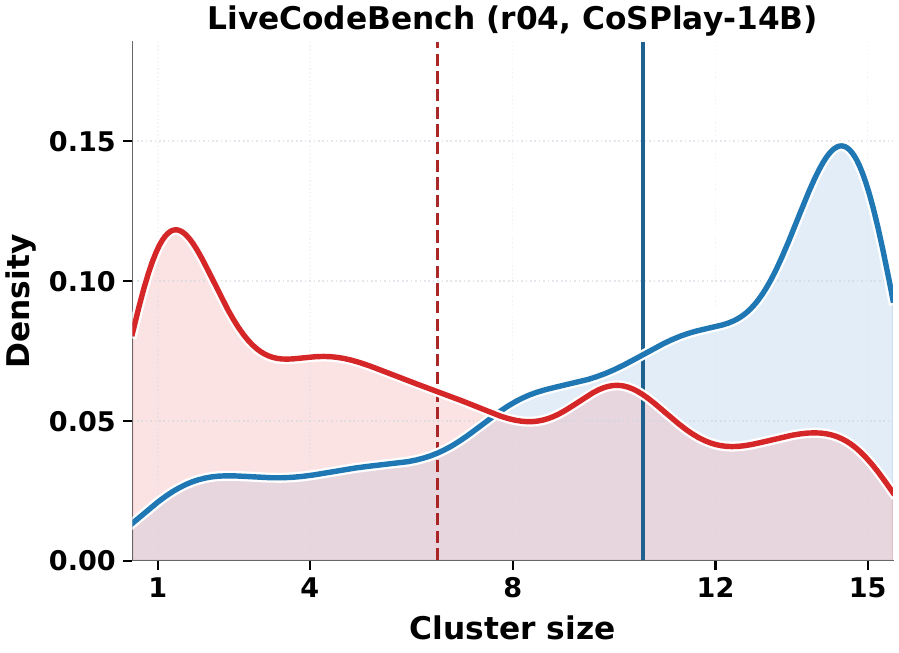}\hfill
    \includegraphics[width=0.28\textwidth]{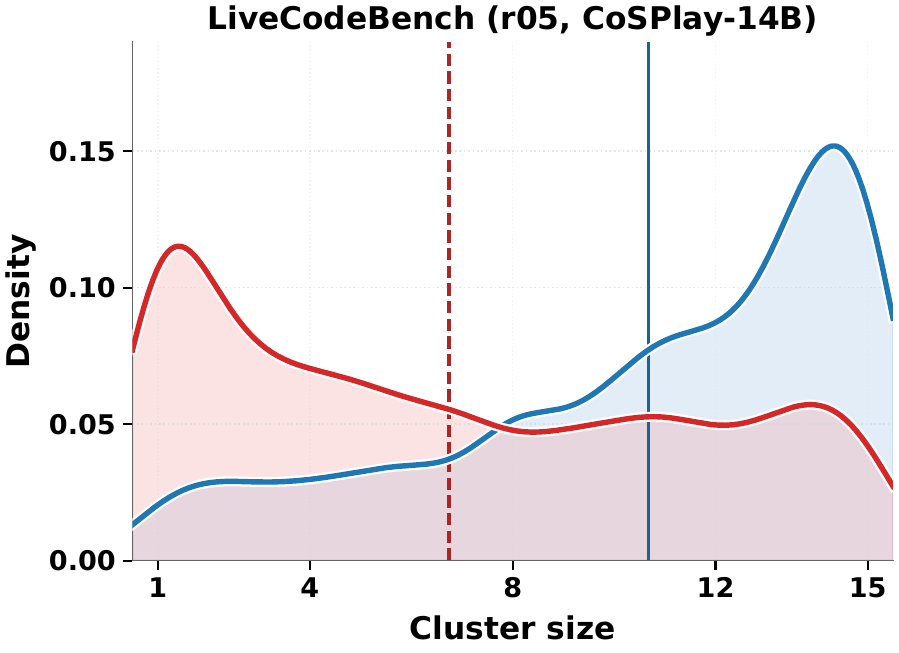}

    \vspace{-0.0em}
    \captionof{figure}{Density distributions of cluster sizes for correct (blue) and wrong (red) code candidates during self-play on \textbf{LiveCodeBench} with the \textbf{14B model}. The top row shows Round 0--2, and the bottom row shows Round 3--5. Vertical lines denote average values.}
    \label{fig:cluster_density_14b_livecodebench}
\end{minipage}

\clearpage
\section{Detailed Evolution of UT Pass-Count Density}
\label{app:ut pass count density}

Figures~\ref{fig:ut_pass_count_density_7b_codecontests}-\ref{fig:ut_pass_count_density_14b_livecodebench} provide per-dataset and per-round visualizations of the UT pass-count density for both 7B and 14B models. The UT pass count measures how many code candidates pass a generated UT. Across datasets, the density mass gradually shifts from extremely low pass-count regions toward higher pass-count regions as self-play proceeds. This indicates that CoSPlay progressively reduces overly hard, noisy, or weakly validated UTs that are passed by few candidates, while increasing the proportion of UTs that are semantically more consistent with the code pool.

Since UT pass count is positively correlated with true UT correctness, the rightward shift suggests that the UT pool becomes more reliable over self-play rounds. Meanwhile, the continued replacement of zero-discrimination UTs prevents the pool from collapsing into trivial tests, allowing the evolved UTs to remain useful for both refinement and selection.

\vspace{-1.0em}
\begin{figure}[H]
    \centering
    \includegraphics[width=0.28\textwidth]{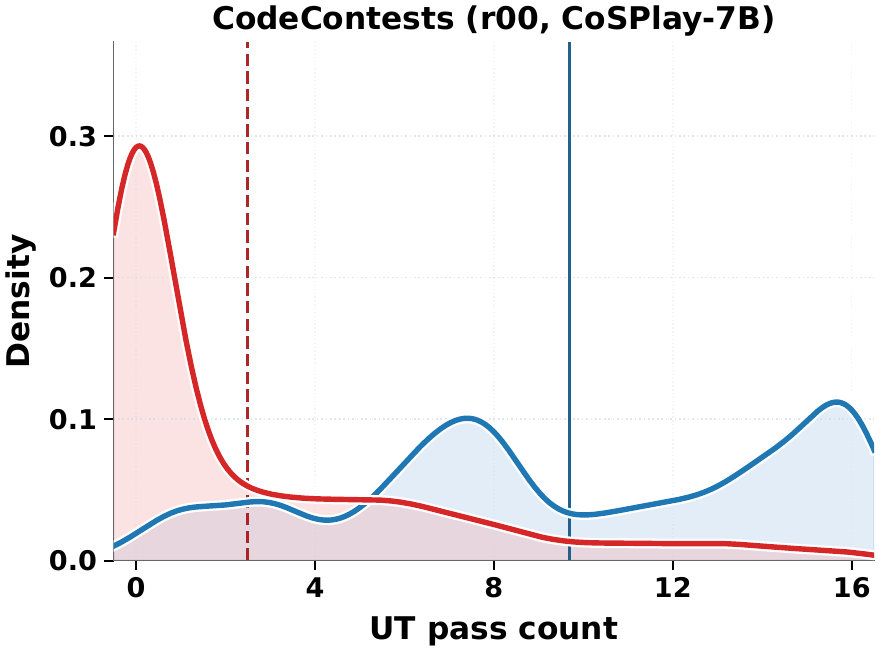}\hfill
    \includegraphics[width=0.28\textwidth]{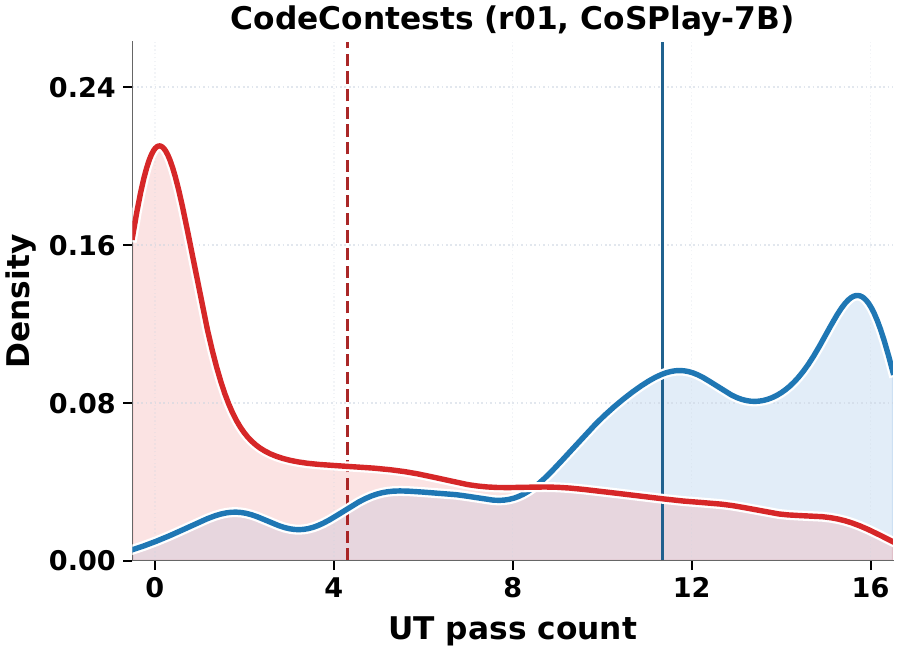}\hfill
    \includegraphics[width=0.28\textwidth]{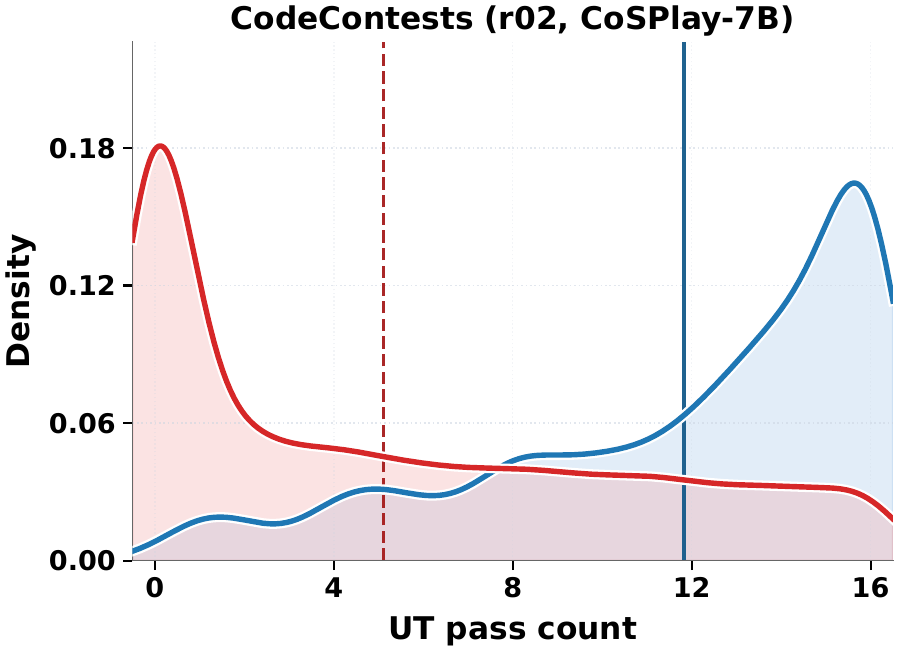}
    \vspace{-1.0em}

    \includegraphics[width=0.28\textwidth]{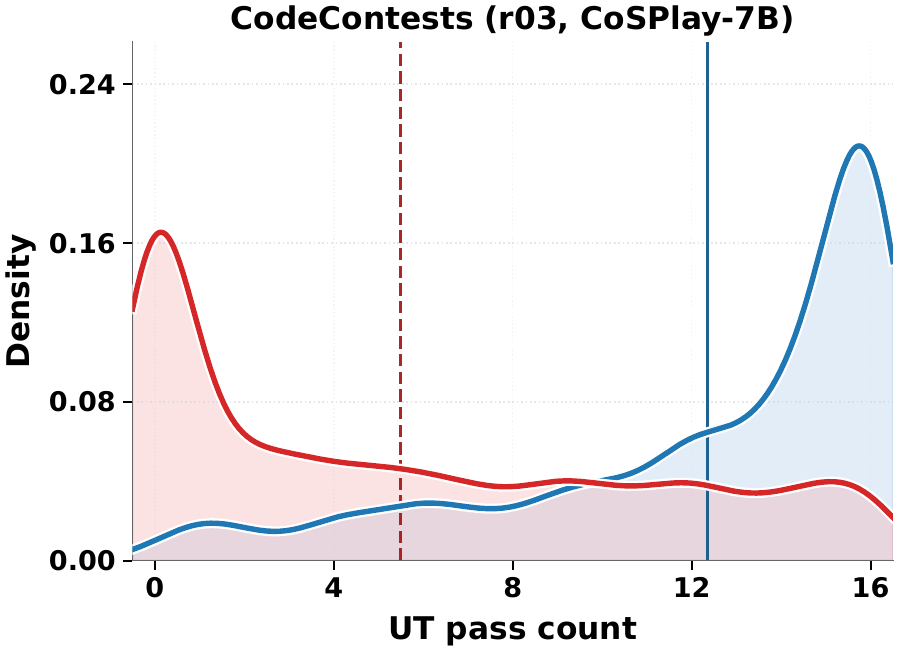}\hfill
    \includegraphics[width=0.28\textwidth]{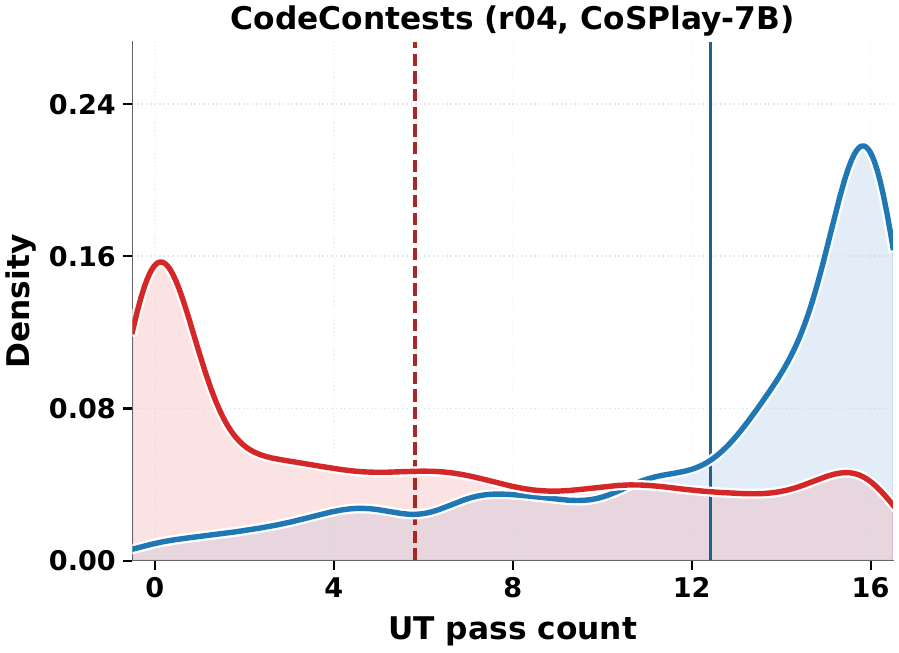}\hfill
    \includegraphics[width=0.28\textwidth]{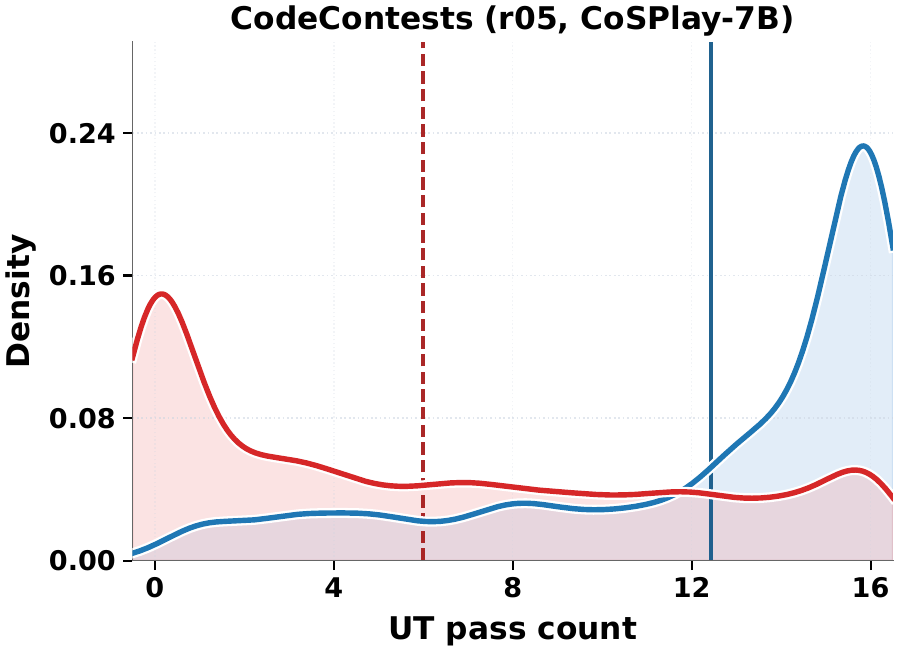}

    \caption{Density distributions of UT pass counts for correct (blue) and wrong (red)  UT candidates during self-play on \textbf{CodeContests} with the \textbf{7B model}. The top row shows Round 0-2, and the bottom row shows Round 3-5. Vertical lines mean the average values.}
    \label{fig:ut_pass_count_density_7b_codecontests}
\end{figure}


\begin{figure}[H]
    \centering
    \includegraphics[width=0.28\textwidth]{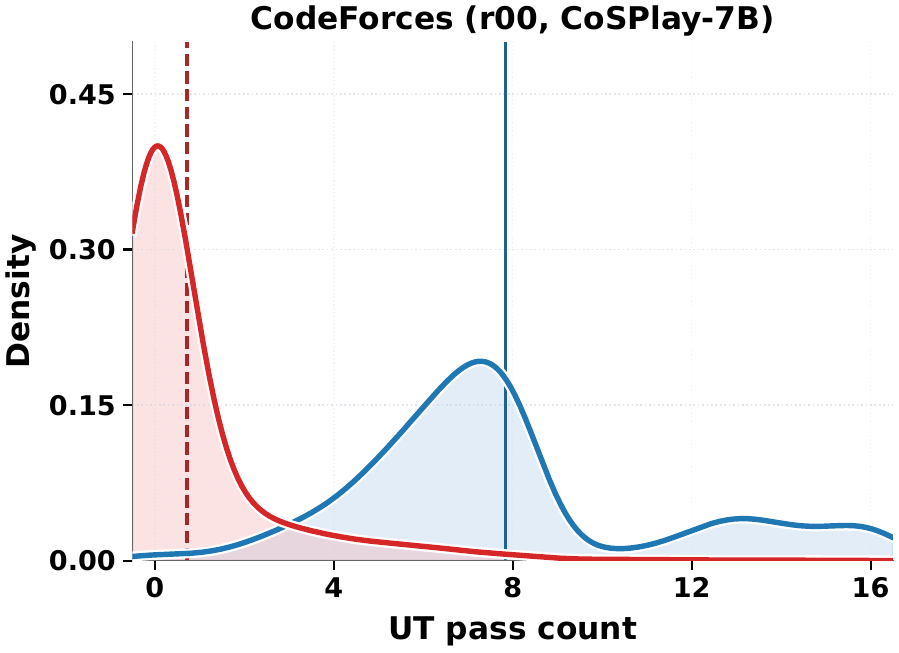}\hfill
    \includegraphics[width=0.28\textwidth]{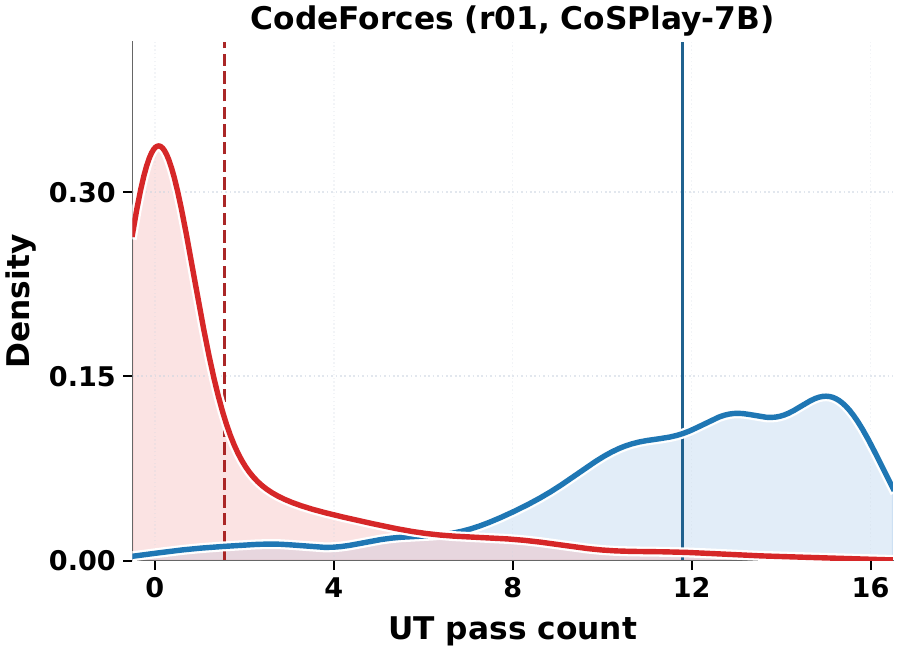}\hfill
    \includegraphics[width=0.28\textwidth]{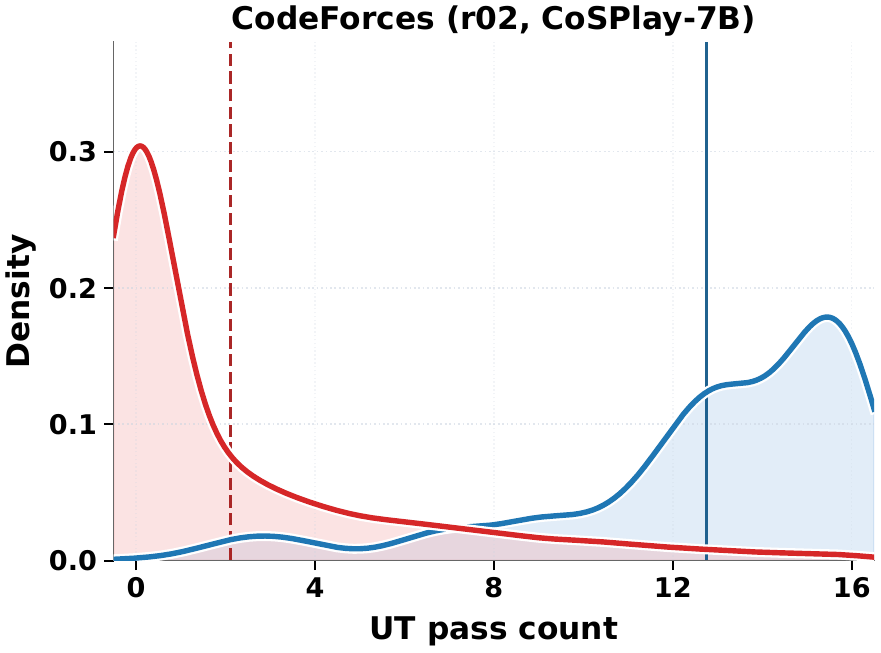}
    \vspace{-1.0em}

    \includegraphics[width=0.28\textwidth]{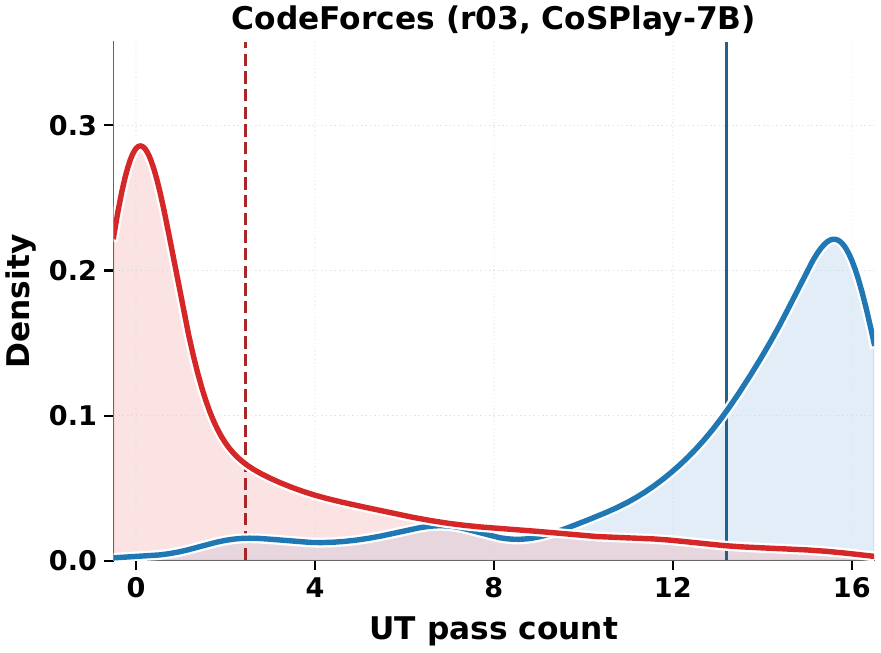}\hfill
    \includegraphics[width=0.28\textwidth]{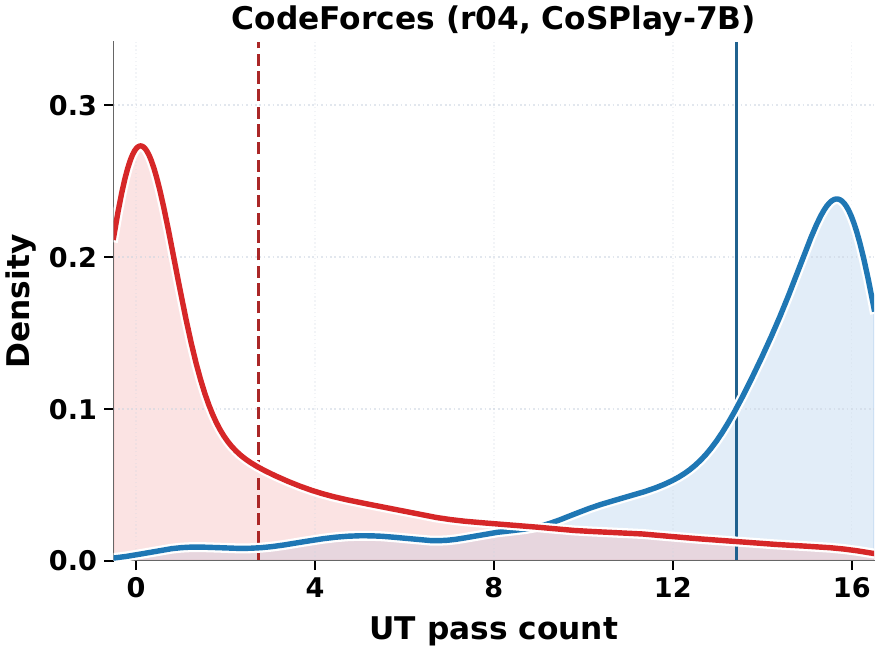}\hfill
    \includegraphics[width=0.28\textwidth]{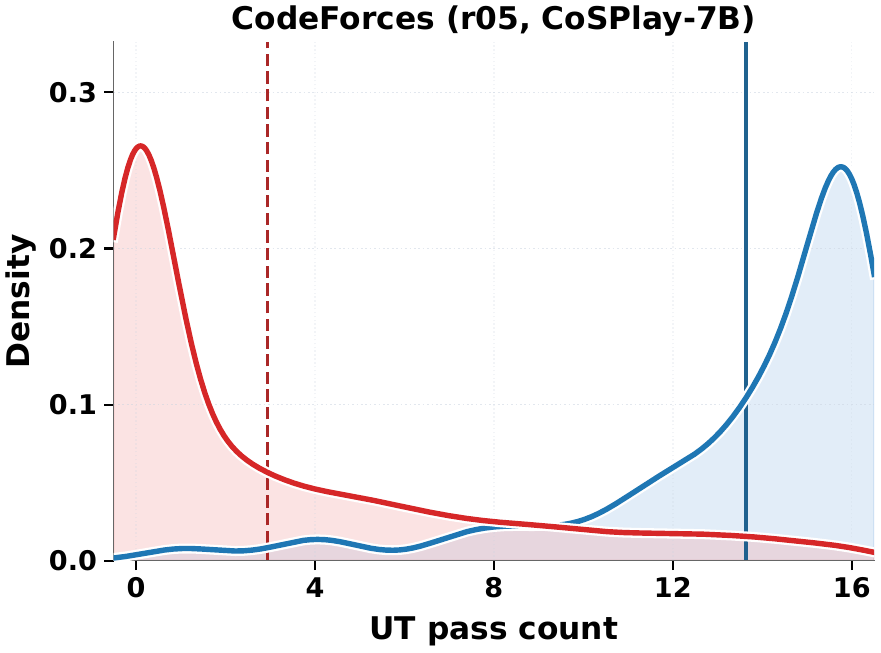}

    \caption{Density distributions of UT pass counts for correct (blue) and wrong (red)  UT candidates during self-play on \textbf{CodeForces} with the \textbf{7B model}. The top row shows Round 0-2, and the bottom row shows Round 3-5. Vertical lines mean the average values.}
    \label{fig:ut_pass_count_density_7b_codeforces}
\end{figure}

\begin{figure}[!t]
    \centering
    \includegraphics[width=0.28\textwidth]{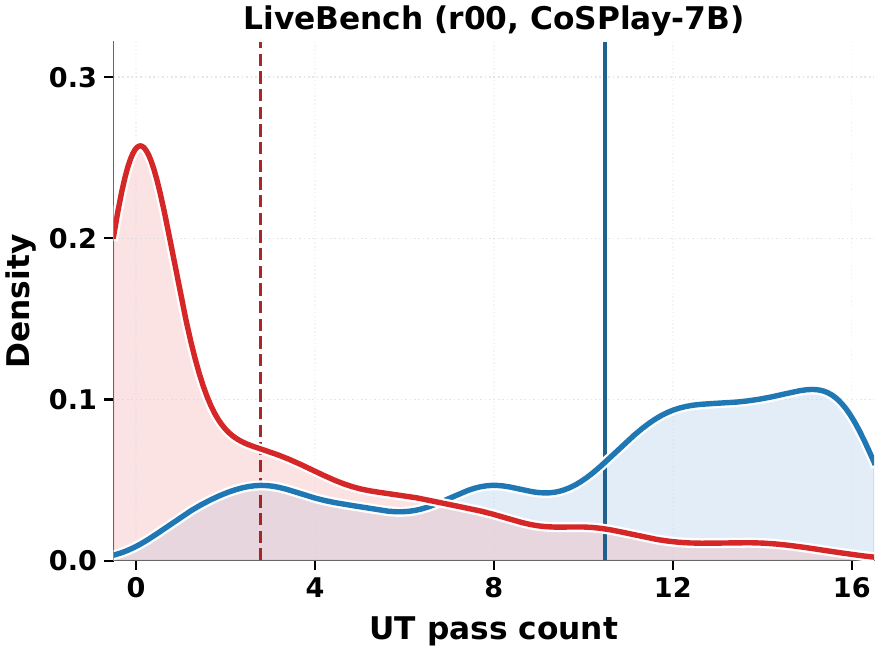}\hfill
    \includegraphics[width=0.28\textwidth]{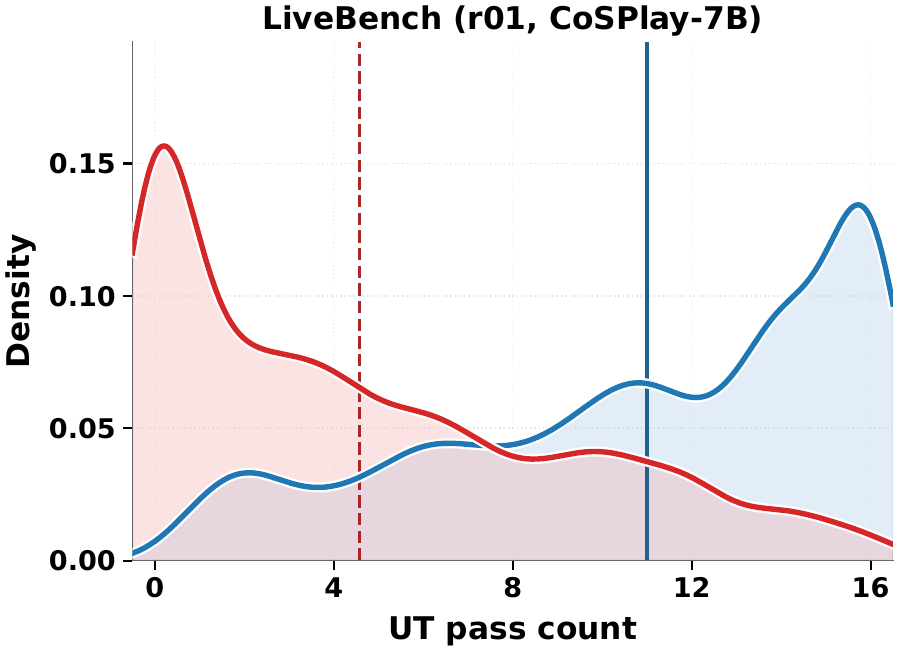}\hfill
    \includegraphics[width=0.28\textwidth]{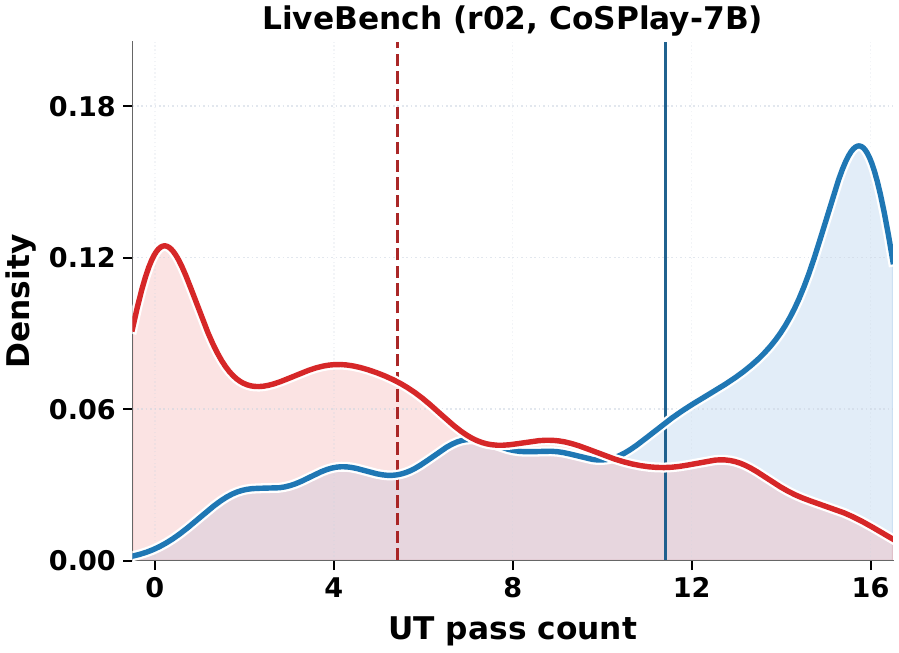}
    \vspace{-1.0em}

    \includegraphics[width=0.28\textwidth]{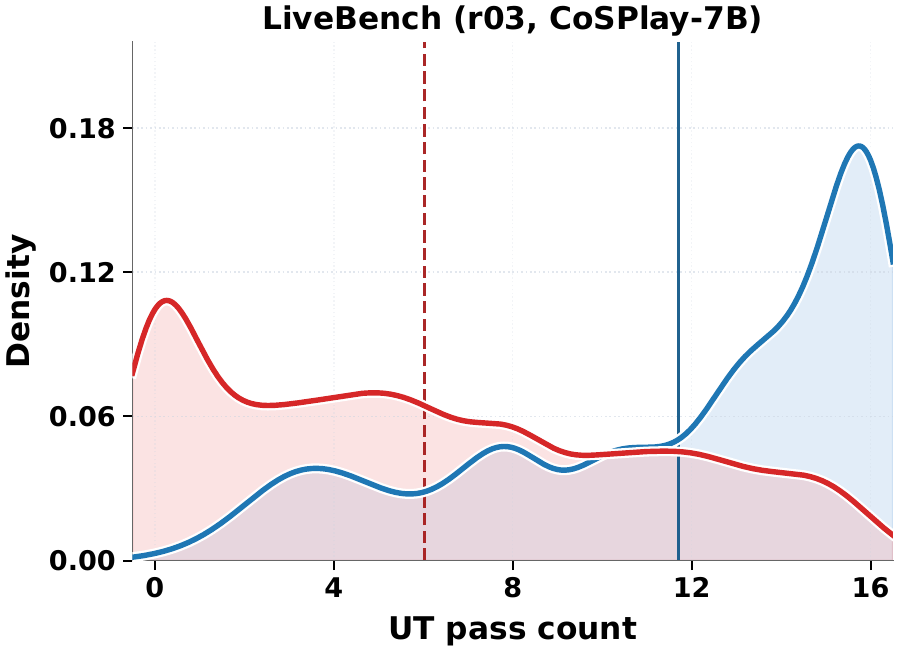}\hfill
    \includegraphics[width=0.28\textwidth]{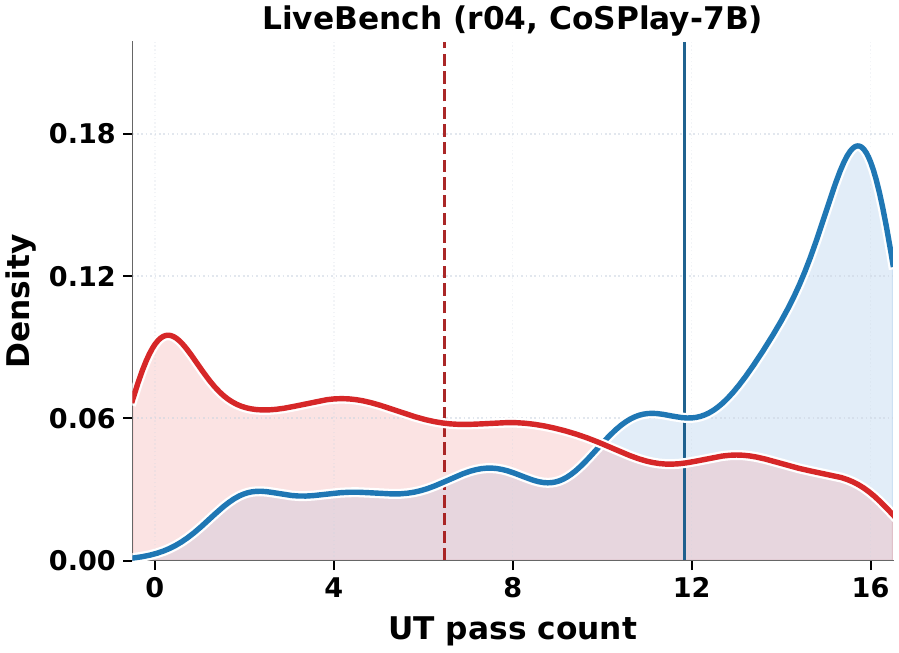}\hfill
    \includegraphics[width=0.28\textwidth]{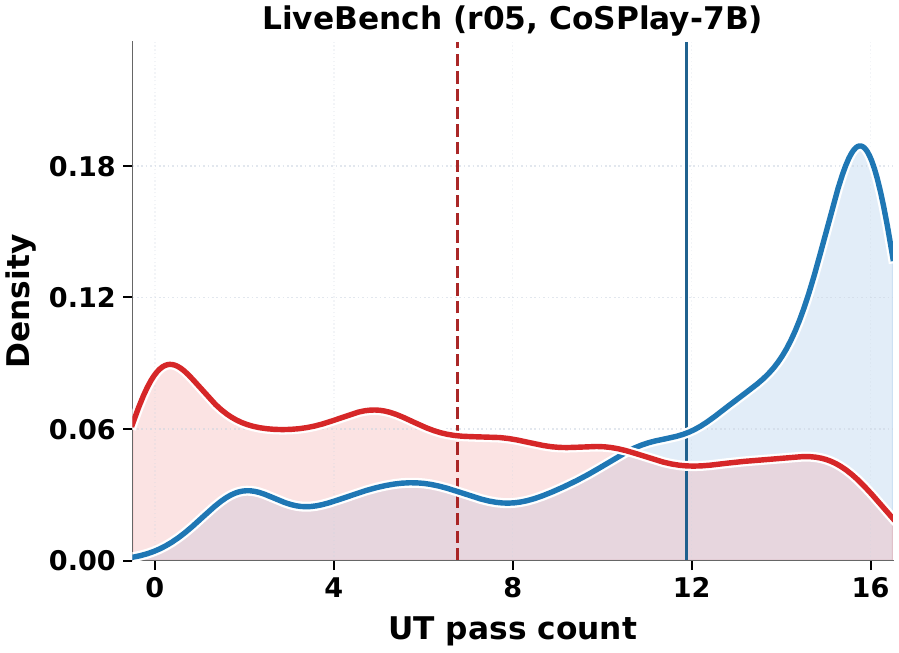}

    \caption{Density distributions of UT pass counts for correct (blue) and wrong (red) UT candidates during self-play on \textbf{LiveBench} with the \textbf{7B model}. The top row shows Round 0-2, and the bottom row shows Round 3-5. Vertical lines mean the average values.}
    \label{fig:ut_pass_count_density_7b_livebench}
\end{figure}

\begin{figure}[!t]
    \centering
    \includegraphics[width=0.28\textwidth]{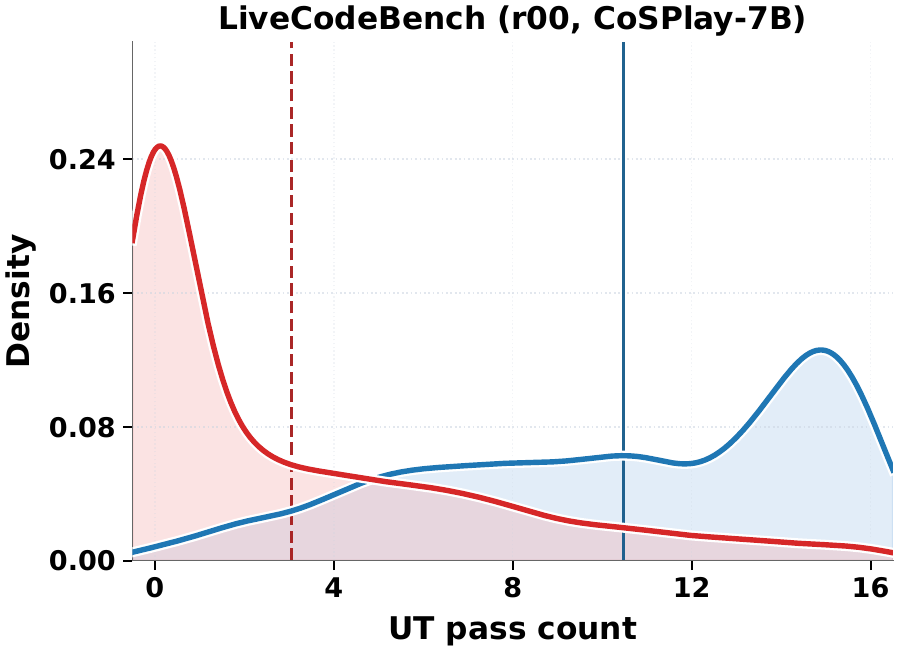}\hfill
    \includegraphics[width=0.28\textwidth]{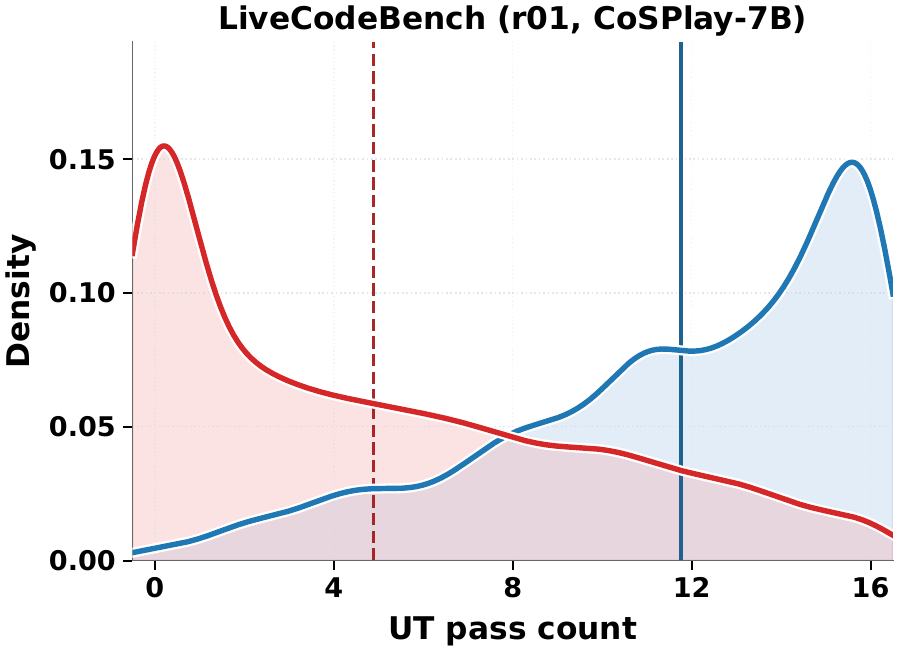}\hfill
    \includegraphics[width=0.28\textwidth]{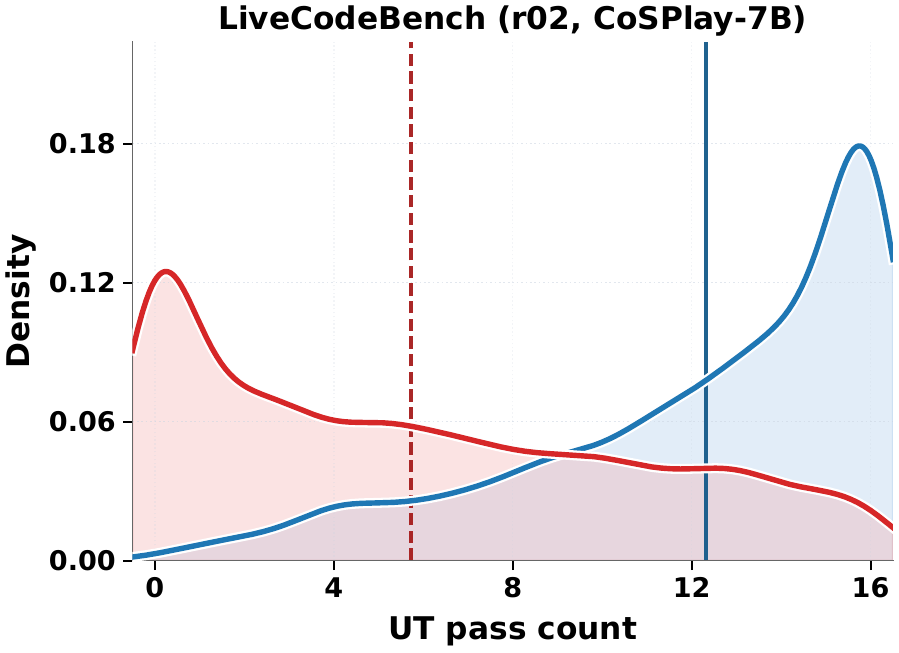}
    \vspace{-1.0em}

    \includegraphics[width=0.28\textwidth]{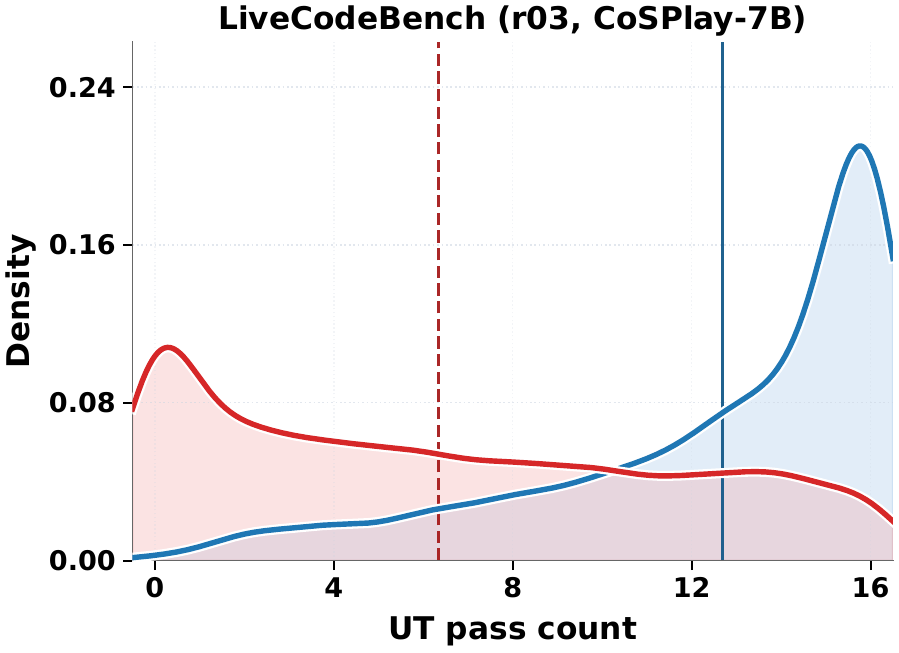}\hfill
    \includegraphics[width=0.28\textwidth]{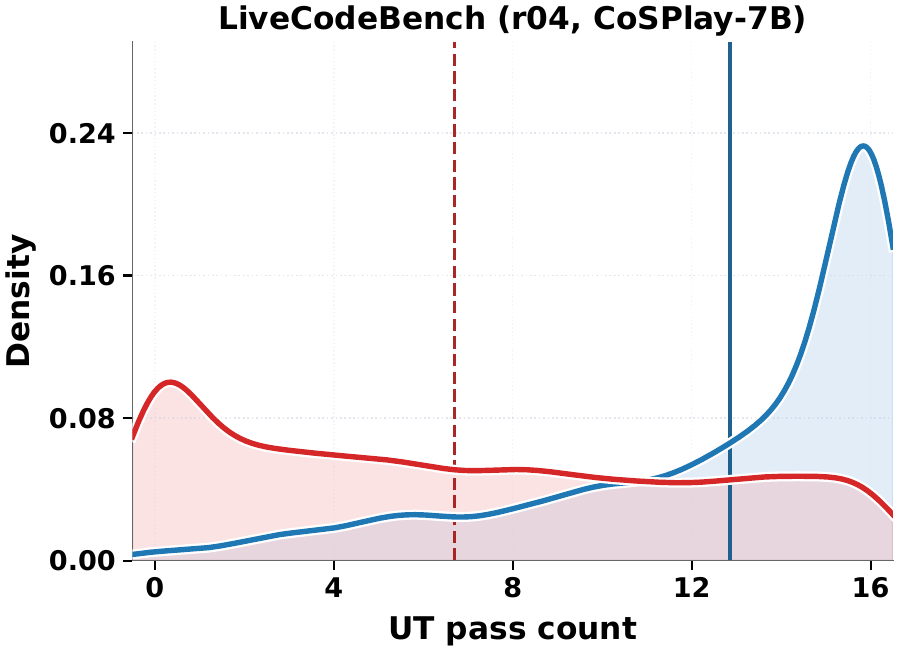}\hfill
    \includegraphics[width=0.28\textwidth]{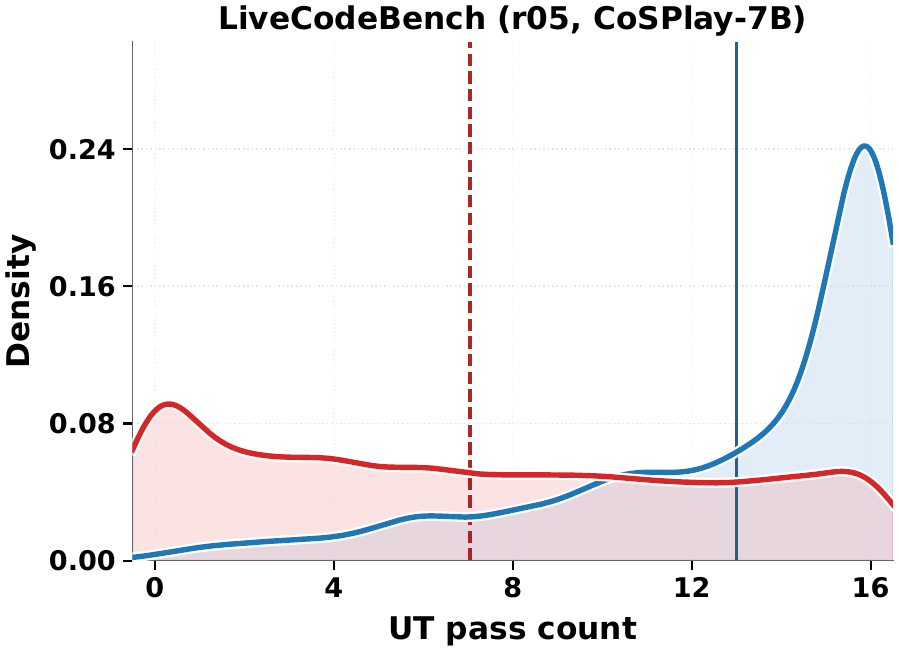}

    \caption{Density distributions of UT pass counts for correct (blue) and wrong (red)  UT candidates during self-play on \textbf{LiveCodeBench} with the \textbf{7B model}. The top row shows Round 0-2, and the bottom row shows Round 3-5. Vertical lines mean the average values.}
    \label{fig:ut_pass_count_density_7b_livecodebench}
\end{figure}

\begin{figure}[!t]
    \centering
    \includegraphics[width=0.28\textwidth]{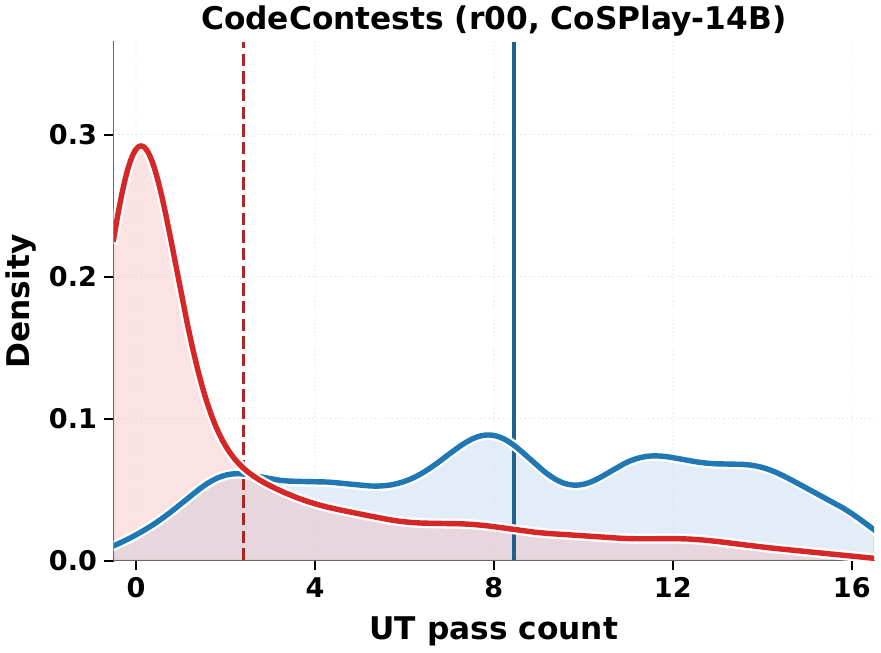}\hfill
    \includegraphics[width=0.28\textwidth]{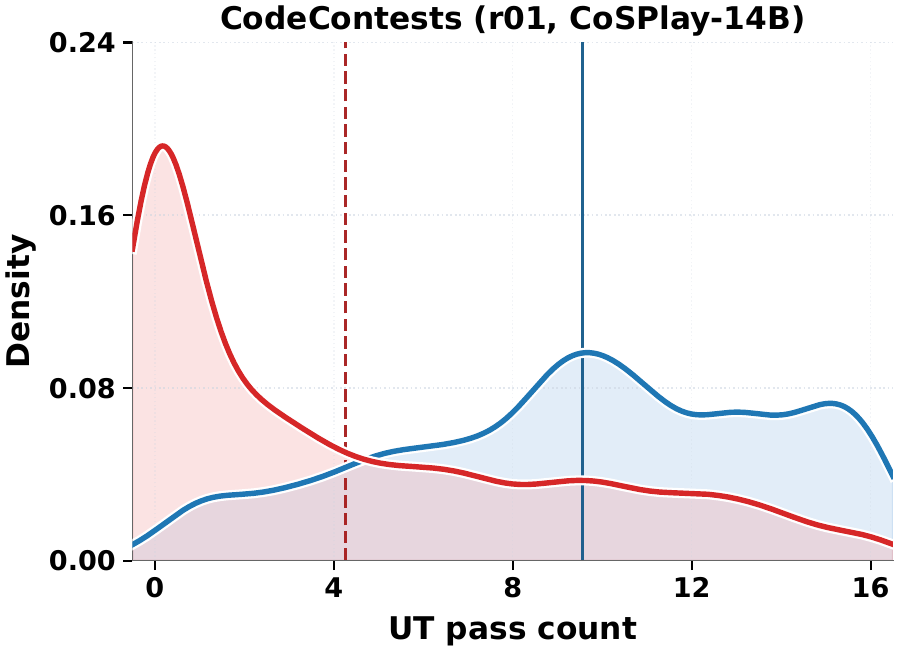}\hfill
    \includegraphics[width=0.28\textwidth]{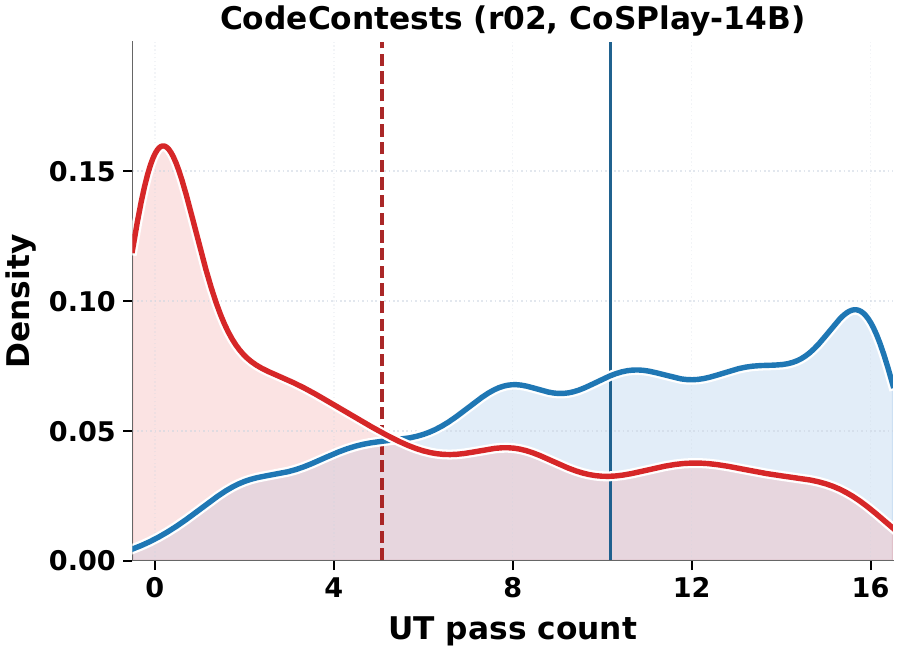}
    \vspace{-1.0em}

    \includegraphics[width=0.28\textwidth]{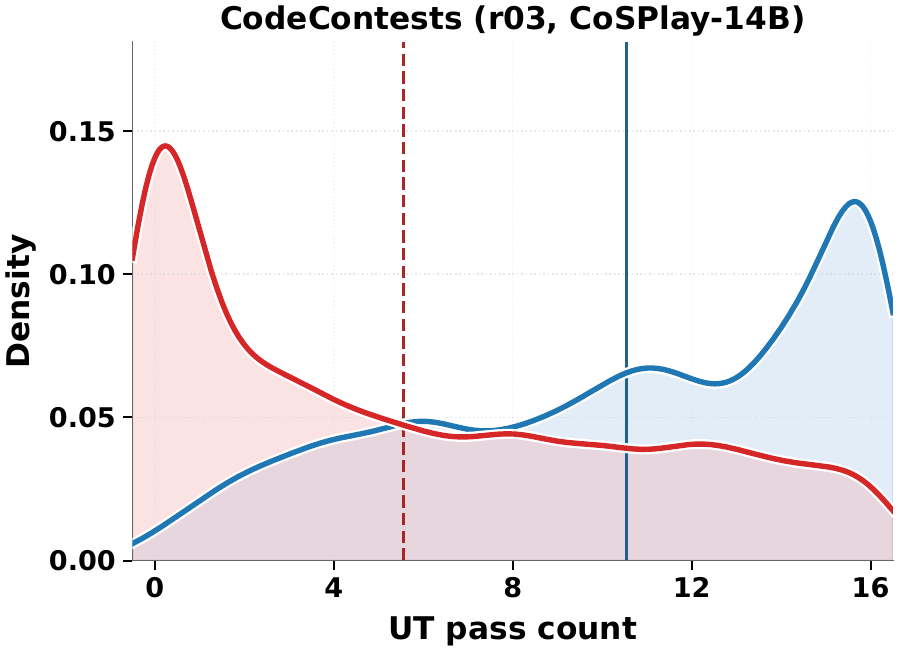}\hfill
    \includegraphics[width=0.28\textwidth]{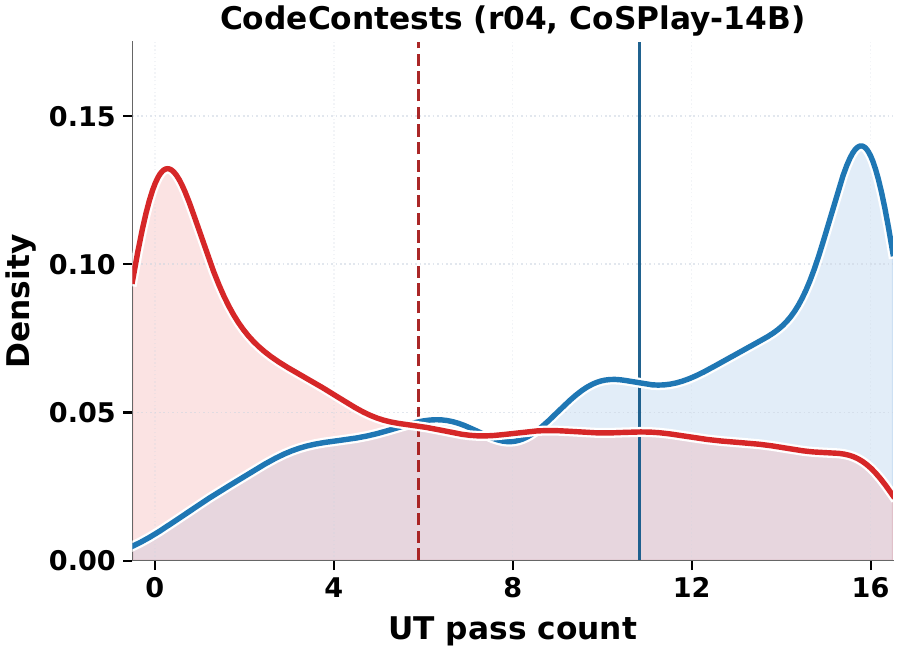}\hfill
    \includegraphics[width=0.28\textwidth]{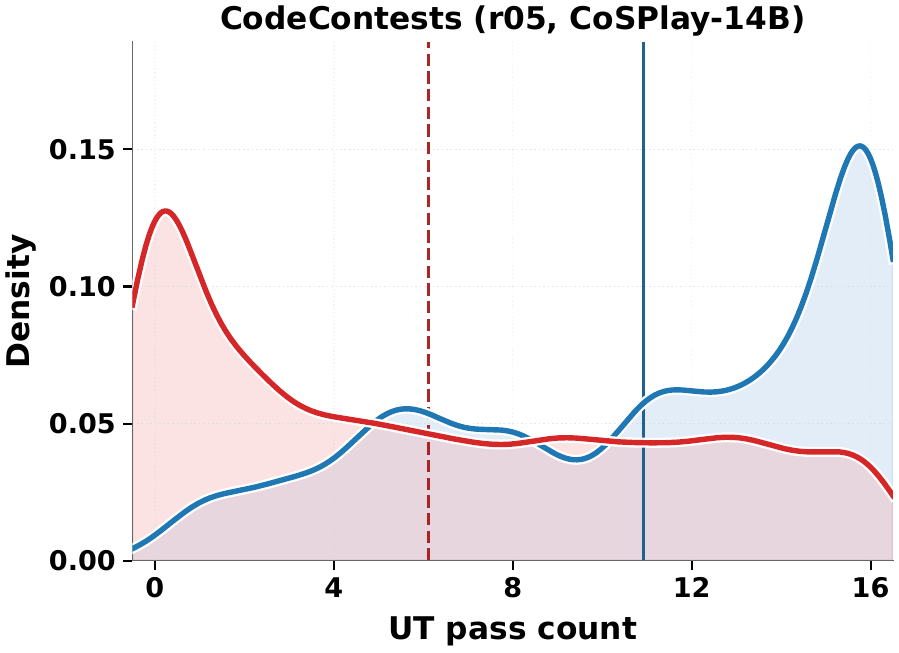}

    \caption{Density distributions of UT pass counts for correct (blue) and wrong (red)  UT candidates during self-play on \textbf{CodeContests} with the \textbf{14B model}. The top row shows Round 0-2, and the bottom row shows Round 3-5. Vertical lines mean the average values.}
    \label{fig:ut_pass_count_density_14b_codecontests}
\end{figure}

\begin{figure}[!t]
    \centering
    \includegraphics[width=0.28\textwidth]{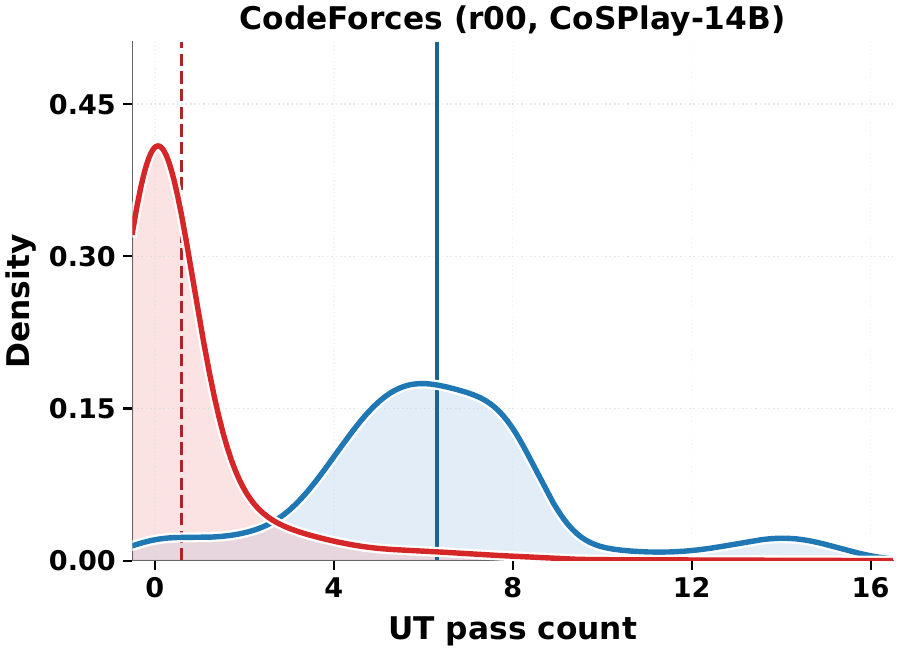}\hfill
    \includegraphics[width=0.28\textwidth]{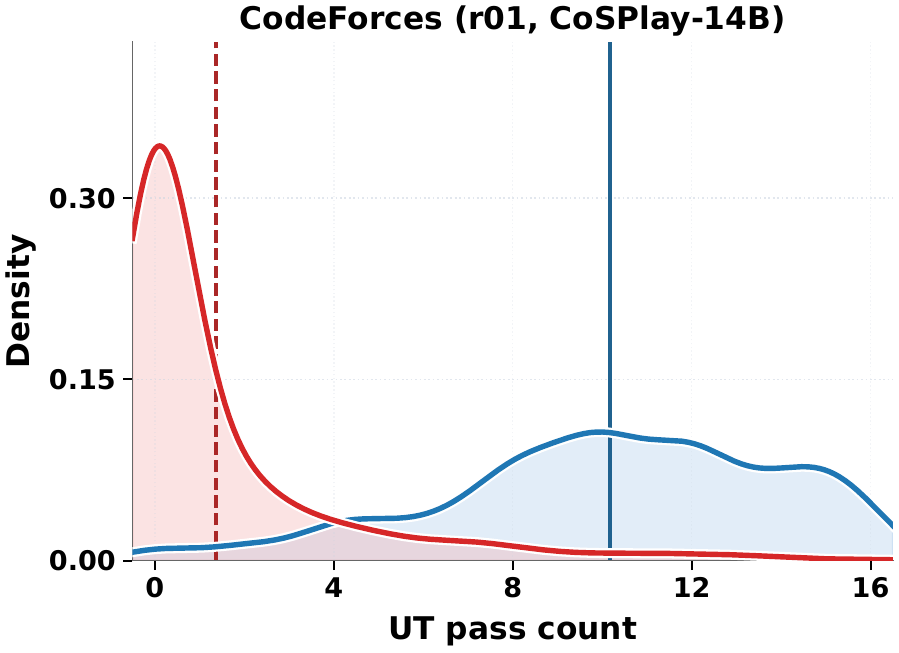}\hfill
    \includegraphics[width=0.28\textwidth]{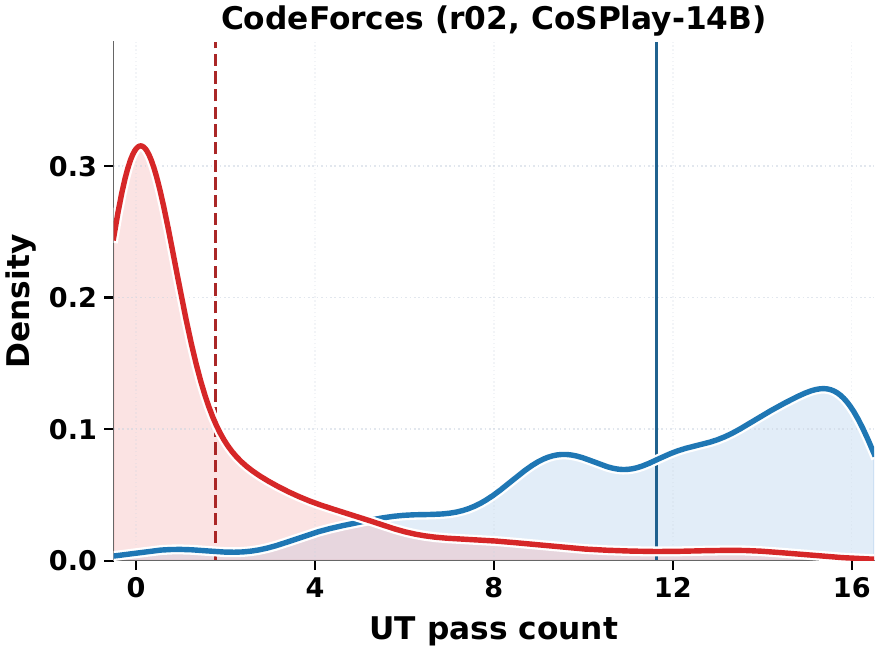}
    \vspace{-1.0em}

    \includegraphics[width=0.28\textwidth]{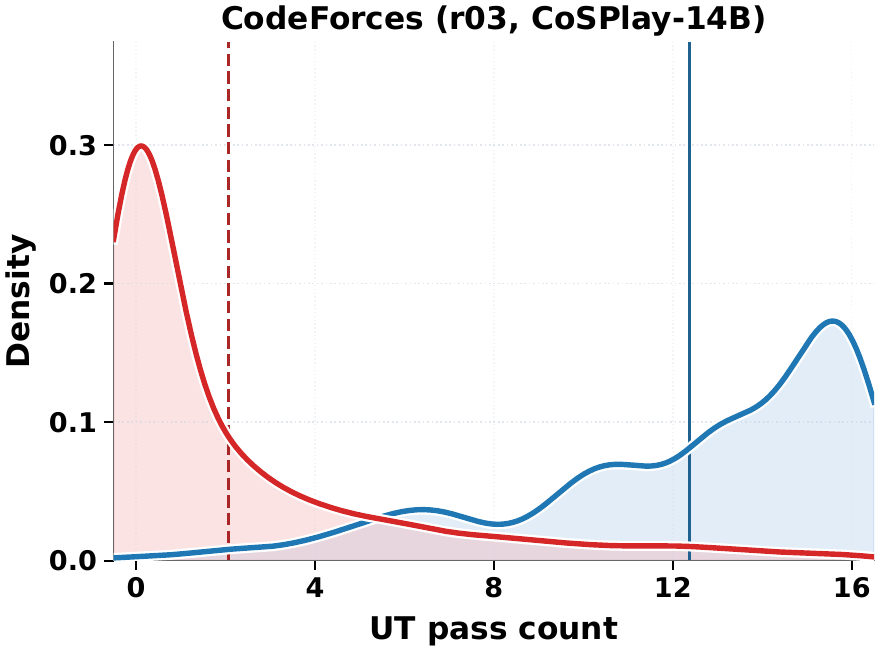}\hfill
    \includegraphics[width=0.28\textwidth]{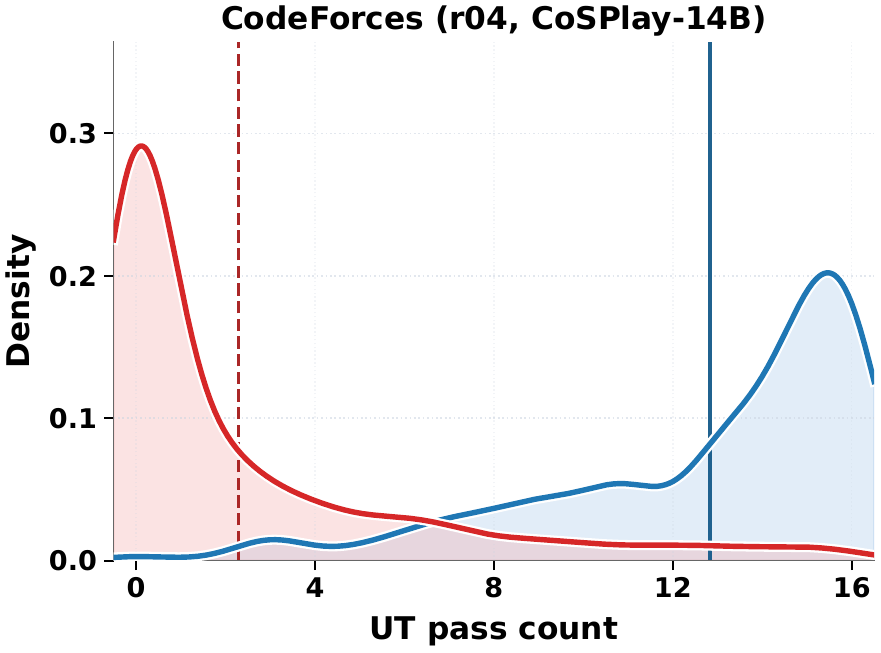}\hfill
    \includegraphics[width=0.28\textwidth]{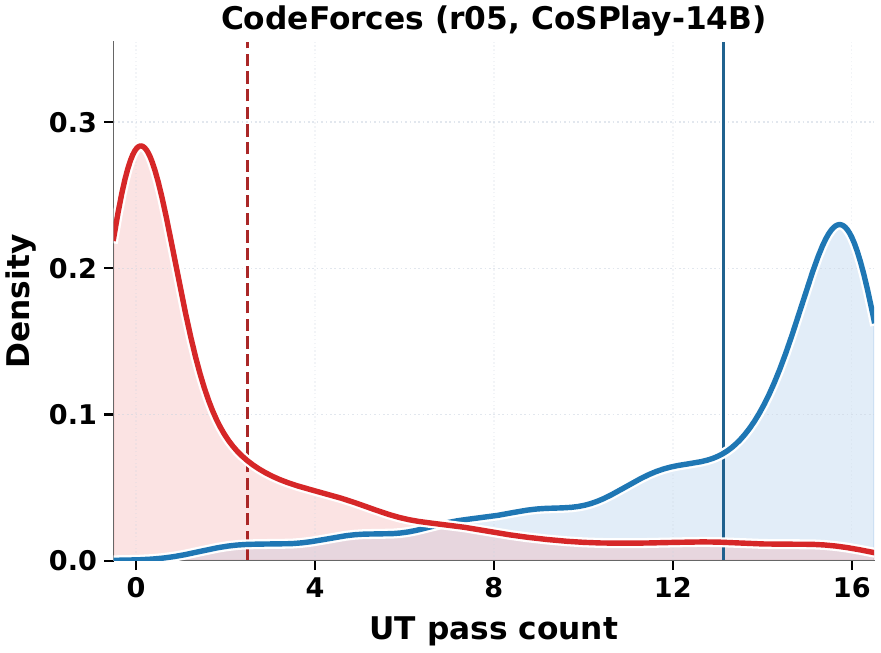}

    \caption{Density distributions of UT pass counts for correct (blue) and wrong (red)  UT candidates during self-play on \textbf{CodeForces} with the \textbf{14B model}. The top row shows Round 0-2, and the bottom row shows Round 3-5. Vertical lines mean the average values.}
    \label{fig:ut_pass_count_density_14b_codeforces}
\end{figure}

\begin{figure}[!t]
    \centering
    \includegraphics[width=0.28\textwidth]{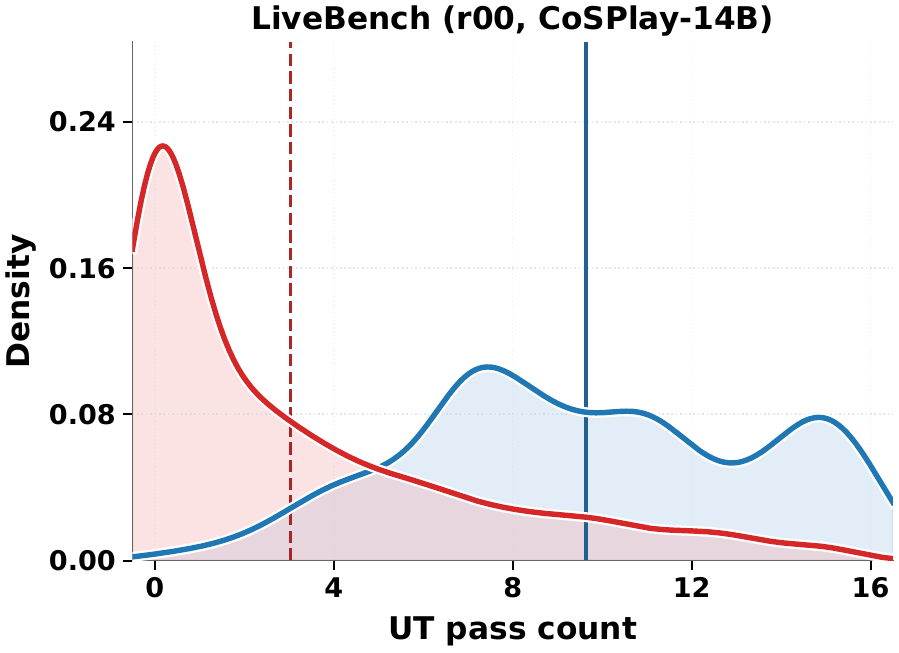}\hfill
    \includegraphics[width=0.28\textwidth]{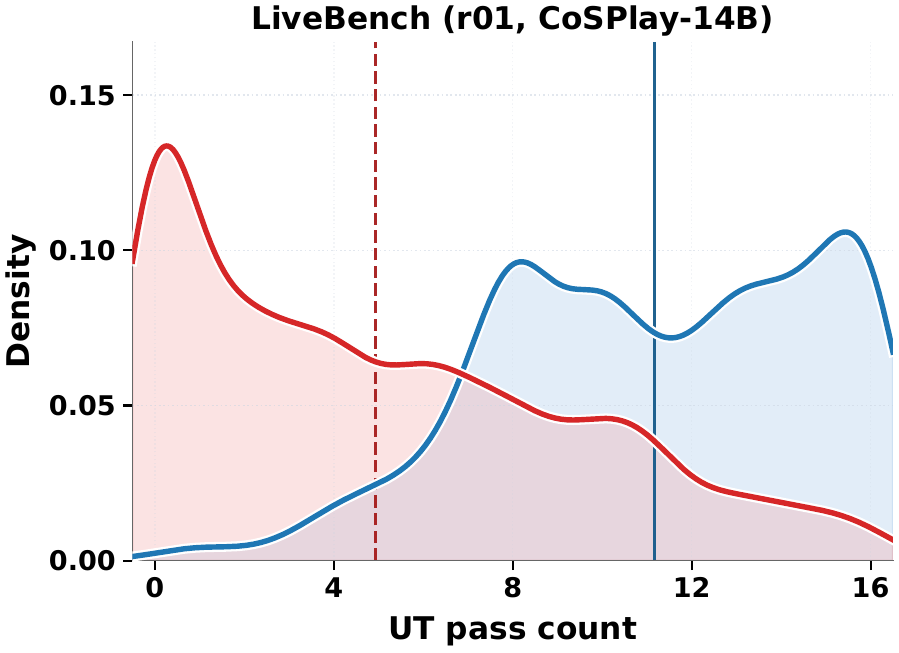}\hfill
    \includegraphics[width=0.28\textwidth]{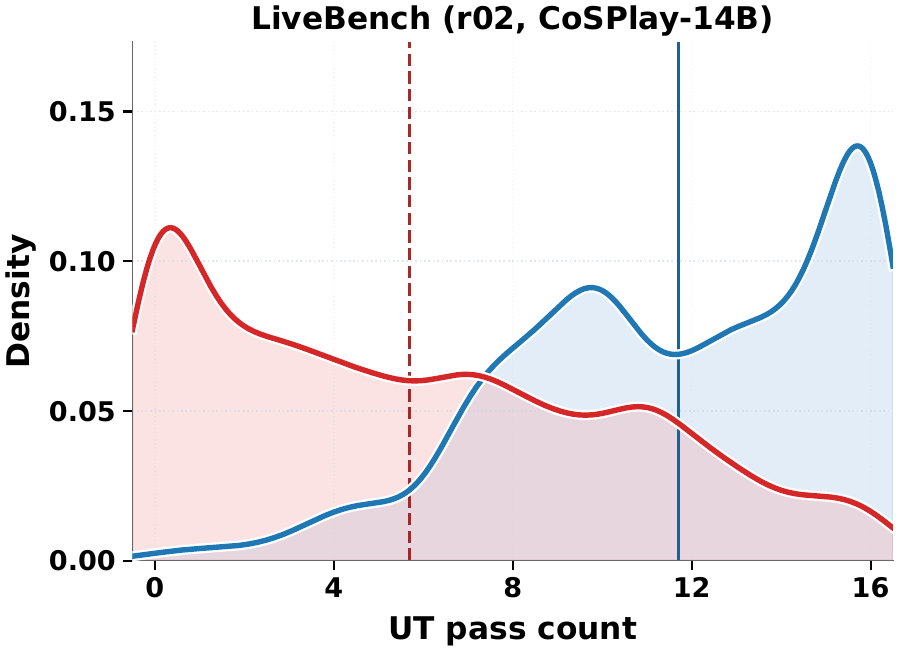}
    \vspace{-1.0em}

    \includegraphics[width=0.28\textwidth]{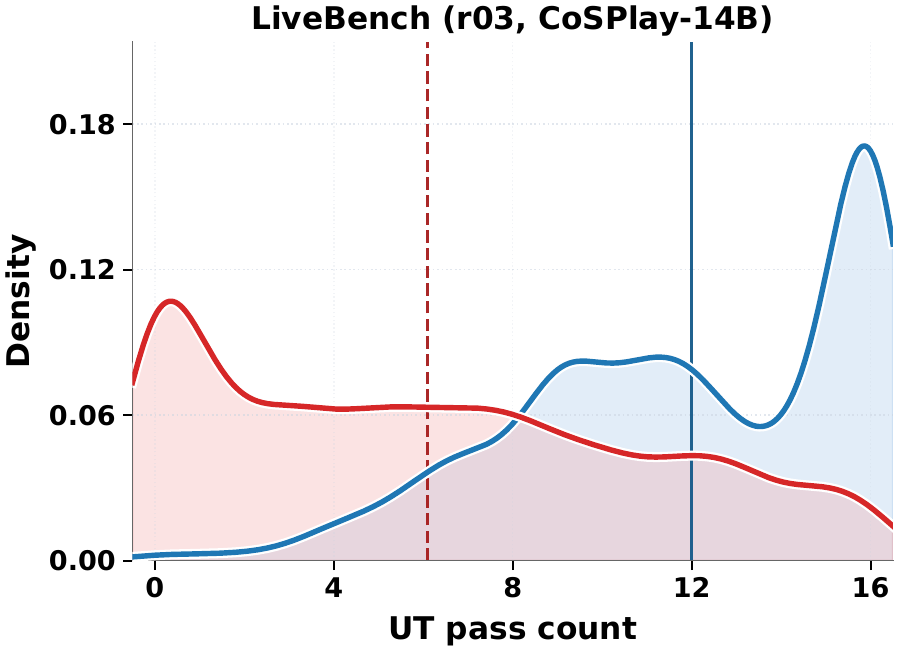}\hfill
    \includegraphics[width=0.28\textwidth]{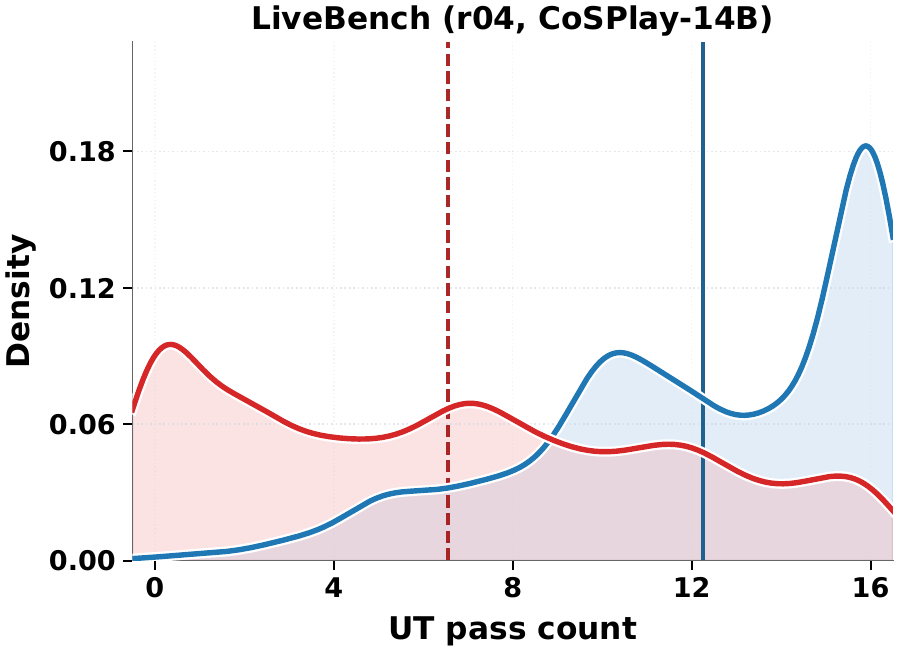}\hfill
    \includegraphics[width=0.28\textwidth]{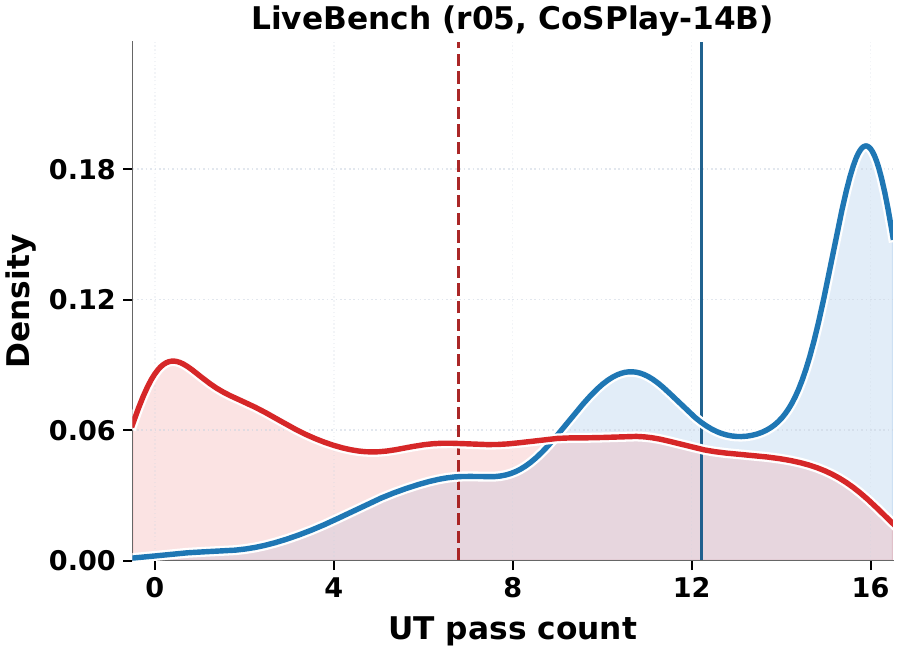}

    \caption{Density distributions of UT pass counts for correct (blue) and wrong (red)  UT candidates during self-play on \textbf{LiveBench} with the \textbf{14B model}. The top row shows Round 0-2, and the bottom row shows Round 3-5. Vertical lines mean the average values.}
    \label{fig:ut_pass_count_density_14b_livebench}
\end{figure}

\begin{figure}[!t]
    \centering
    \includegraphics[width=0.28\textwidth]{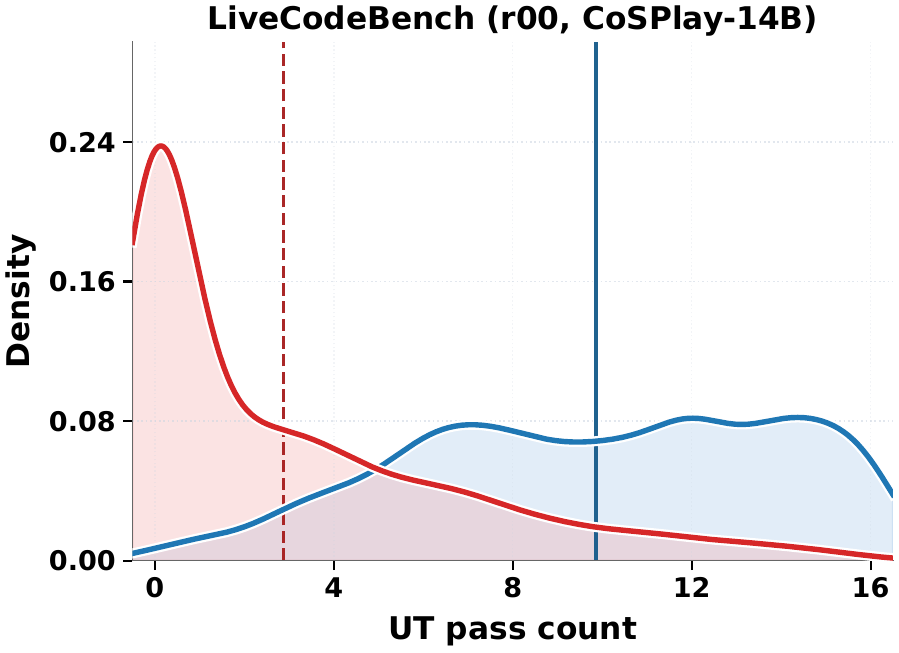}\hfill
    \includegraphics[width=0.28\textwidth]{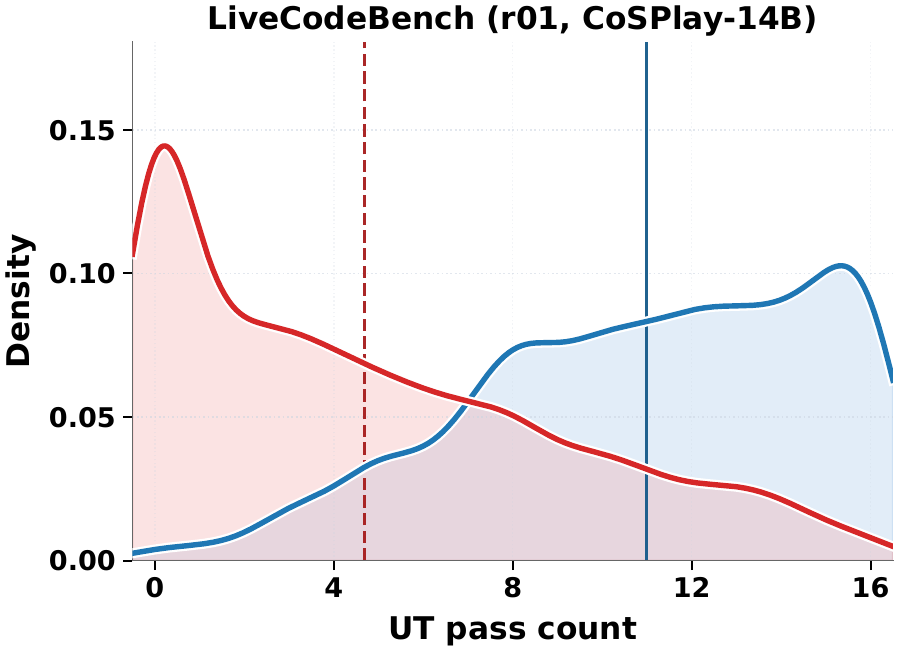}\hfill
    \includegraphics[width=0.28\textwidth]{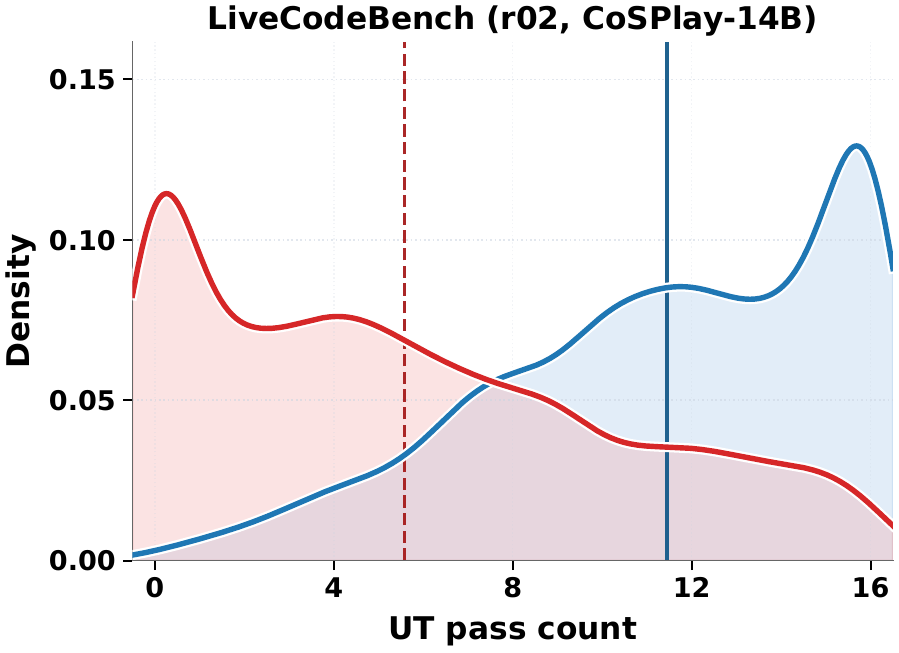}
    \vspace{-1.0em}

    \includegraphics[width=0.28\textwidth]{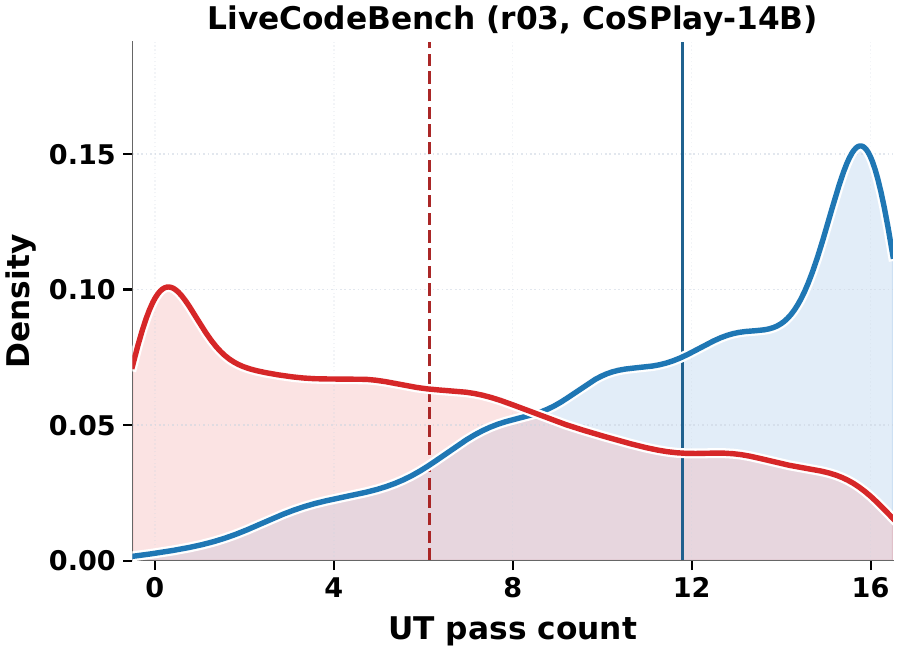}\hfill
    \includegraphics[width=0.28\textwidth]{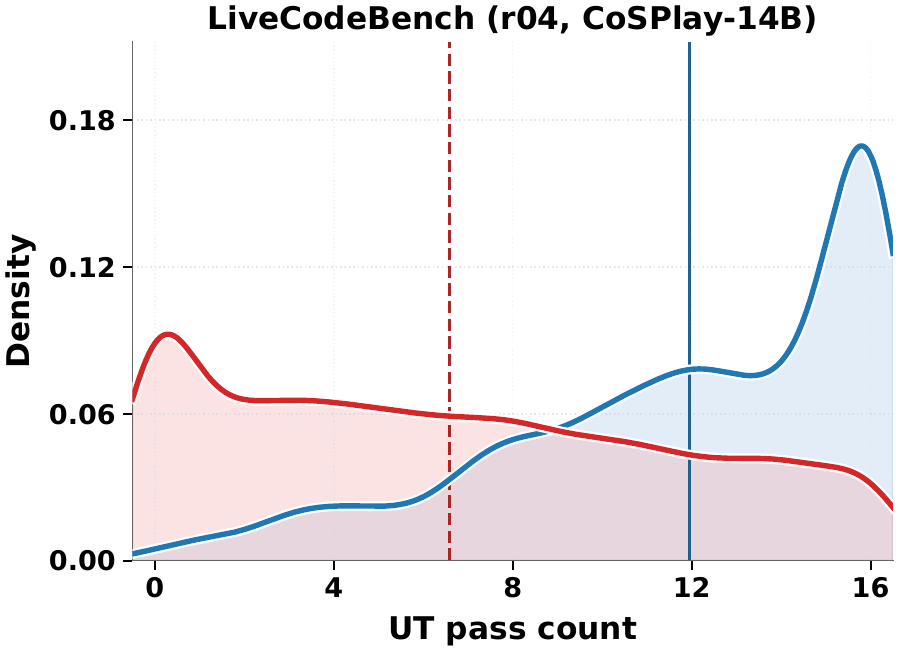}\hfill
    \includegraphics[width=0.28\textwidth]{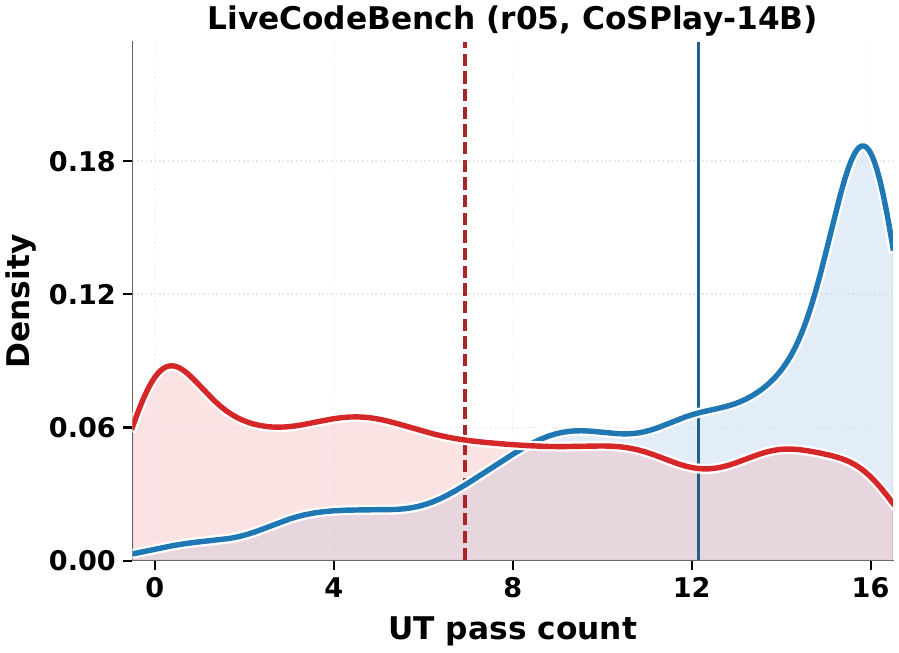}

    \caption{Density distributions of UT pass counts for correct (blue) and wrong (red)  UT candidates during self-play on \textbf{LiveCodeBench} with the \textbf{14B model}. The top row shows Round 0-2, and the bottom row shows Round 3-5. Vertical lines mean the average values.}
    \label{fig:ut_pass_count_density_14b_livecodebench}
\end{figure}

\clearpage
\section{Detailed Evolution of Code Pass-Count Density}
\label{app:code pass count density}
Figures~\ref{fig:code_pass_count_density_7b_codecontests}-\ref{fig:code_pass_count_density_14b_livecodebench} show the per-dataset and per-round plots of code pass-count density. The code pass count measures how many generated UTs a code candidate passes. Across datasets and model scales, the density progressively moves away from low pass-count regions and shifts toward higher pass-count regions, indicating that more code candidates satisfy a larger fraction of the evolved UT pool after self-play.

This pattern provides fine-grained evidence for the effectiveness of the iterative code cleaning and code fixing steps. All-failing or weak candidates are removed or refined, while the improved UT pool provides increasingly useful supervision for subsequent refinement. Together with the positive correlation between code pass count and true code correctness, these density shifts show that CoSPlay improves not only the UT quality but also the quality distribution of the code pool, which ultimately supports stronger BoN selection.

\vspace{1em}
\begin{figure}[H]
    \centering
    \includegraphics[width=0.28\textwidth]{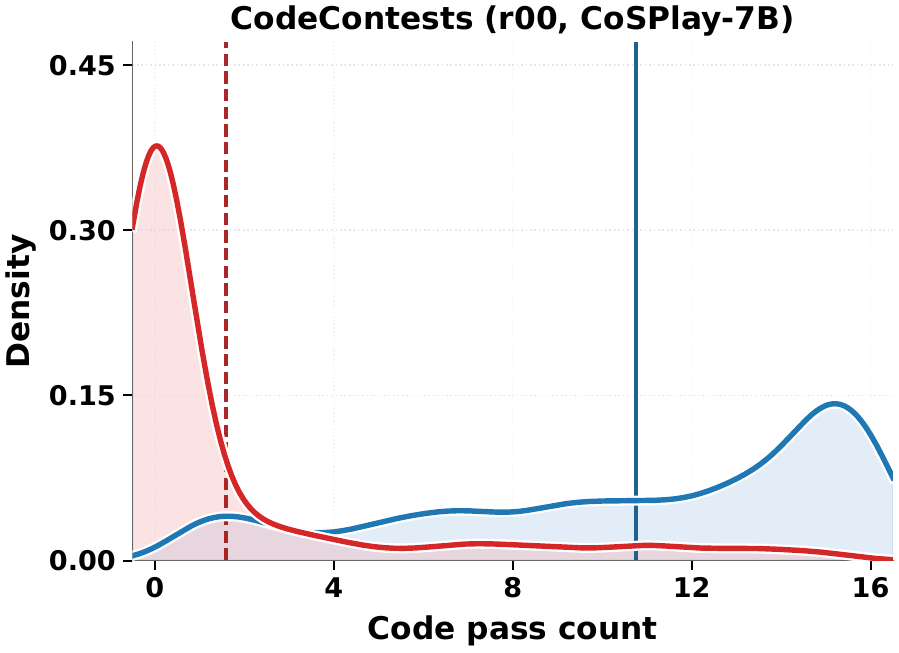}\hfill
    \includegraphics[width=0.28\textwidth]{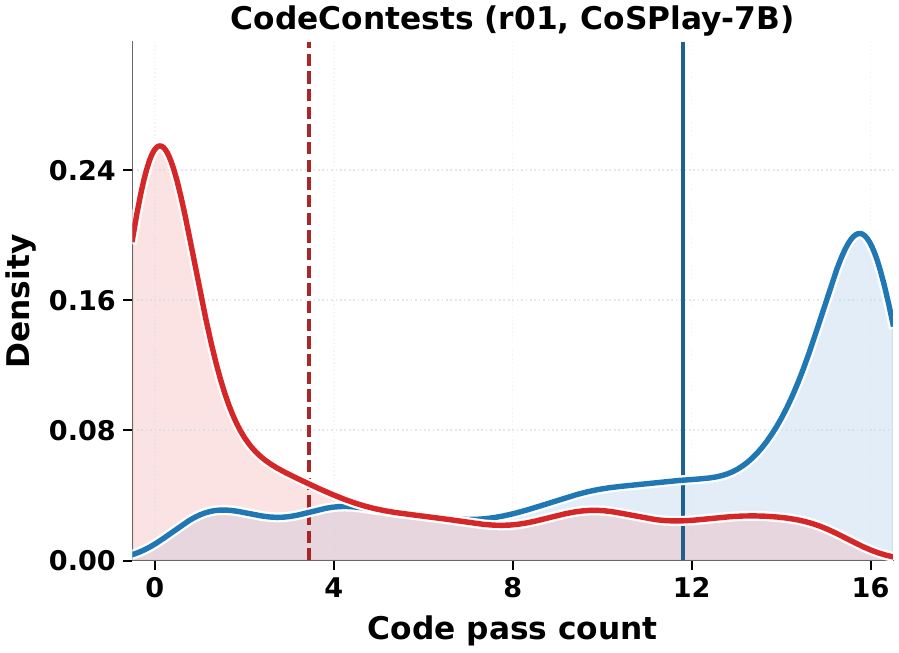}\hfill
    \includegraphics[width=0.28\textwidth]{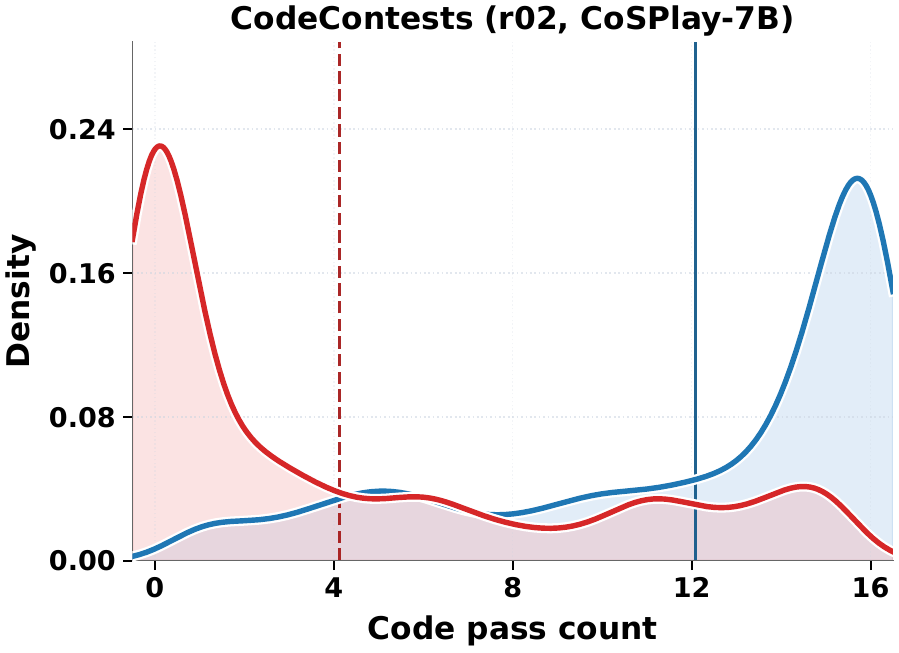}
    \vspace{-1.0em}

    \includegraphics[width=0.28\textwidth]{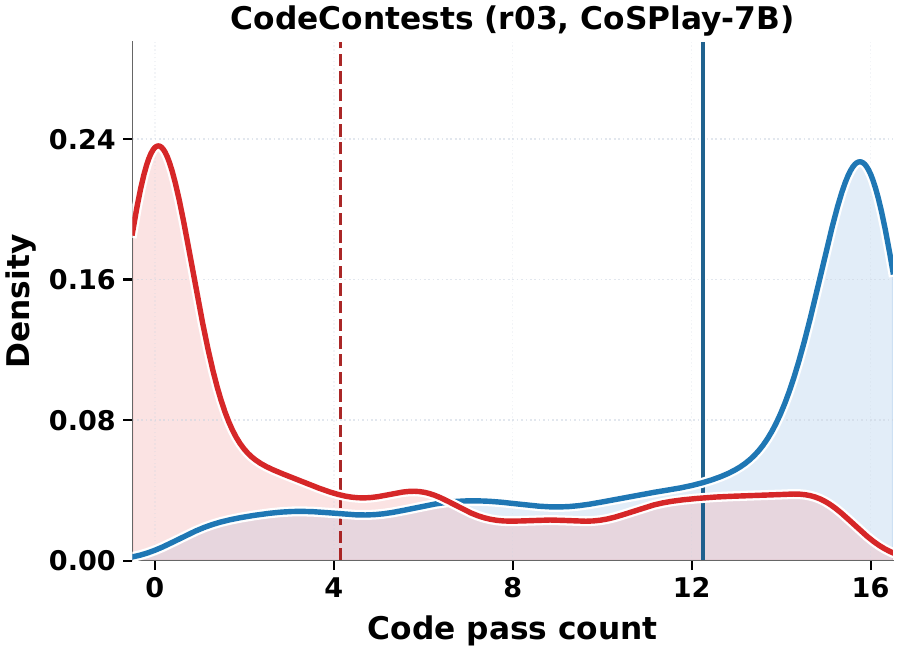}\hfill
    \includegraphics[width=0.28\textwidth]{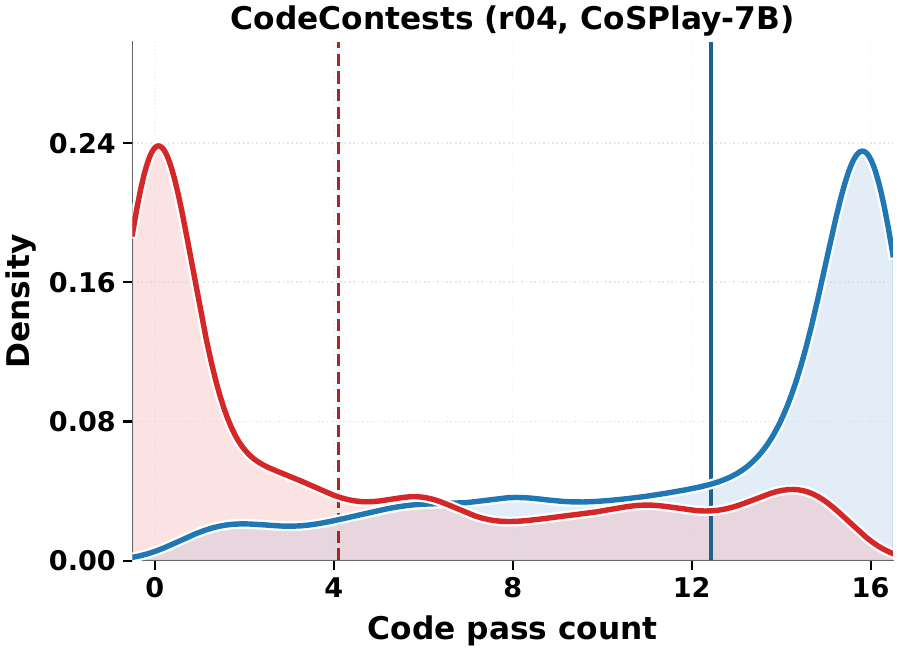}\hfill
    \includegraphics[width=0.28\textwidth]{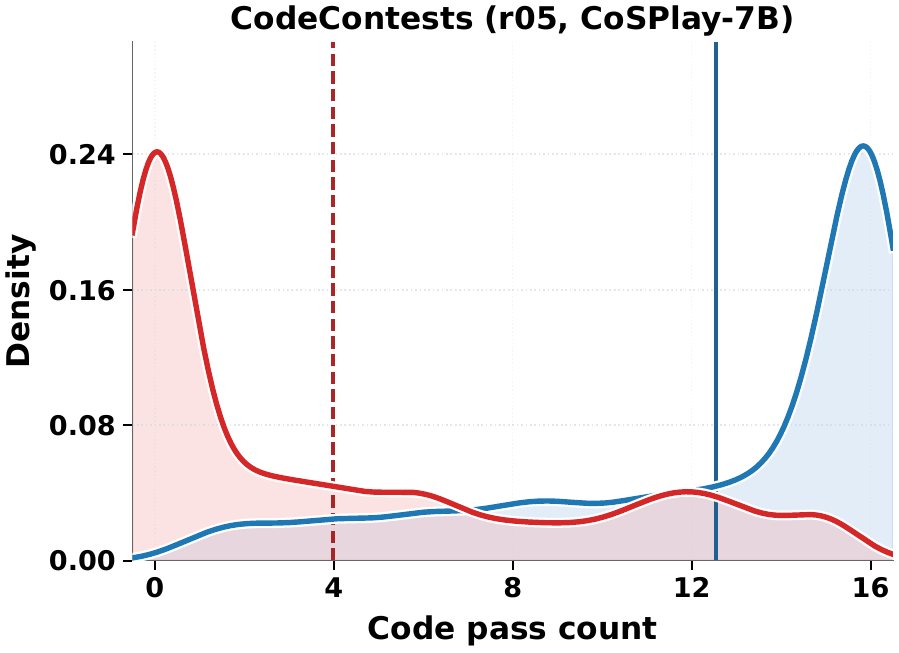}

    \caption{Density distributions of code pass counts for correct (blue) and incorrect (red) code candidates during self-play on \textbf{CodeContests} with the \textbf{7B model}. The top row shows Round 0-2, and the bottom row shows Round 3-5. Vertical lines mean the average values.}
    \label{fig:code_pass_count_density_7b_codecontests}
\end{figure}

\begin{figure}[H]
    \centering
    \includegraphics[width=0.28\textwidth]{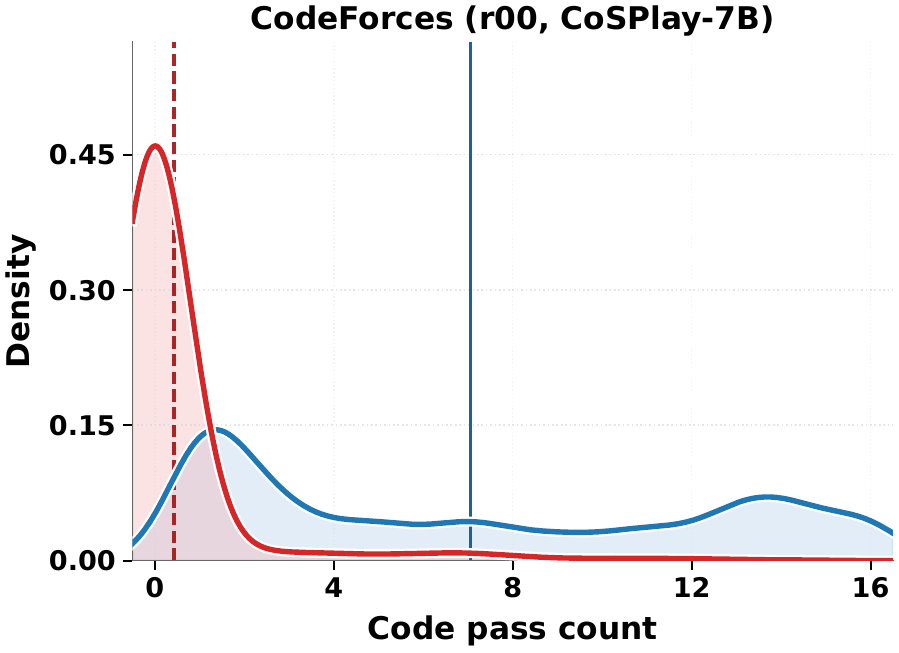}\hfill
    \includegraphics[width=0.28\textwidth]{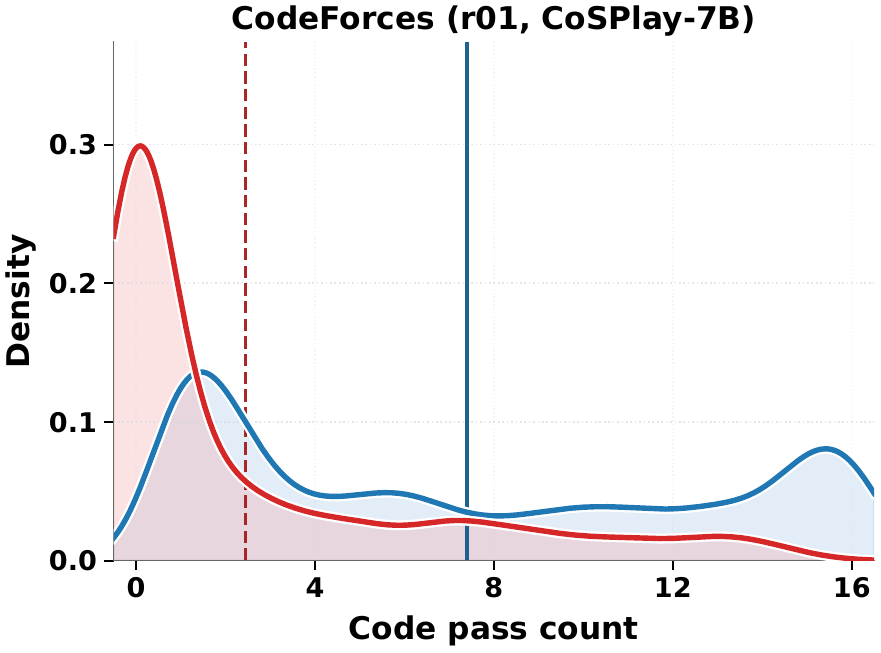}\hfill
    \includegraphics[width=0.28\textwidth]{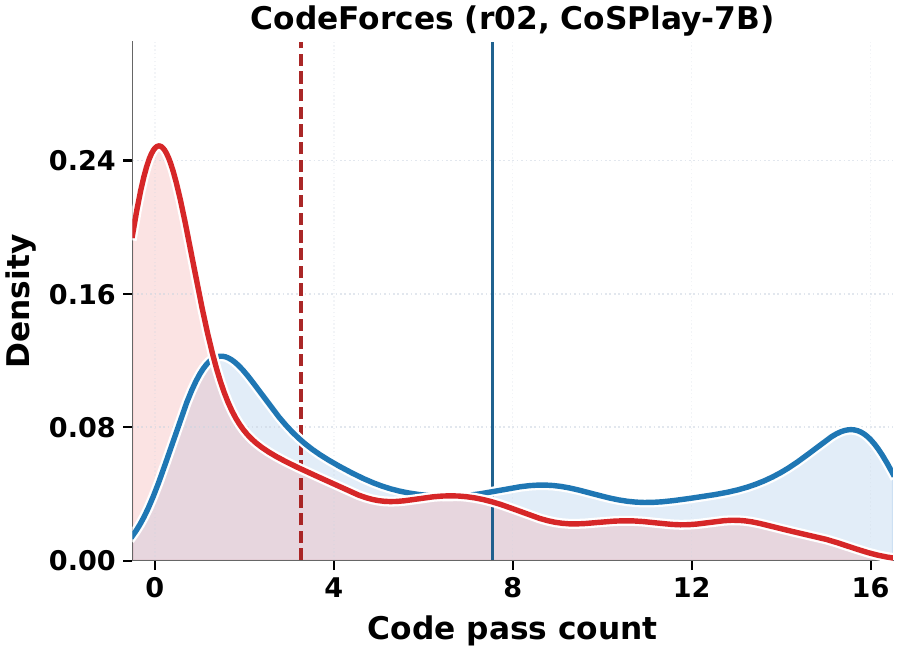}
    \vspace{-1.0em}

    \includegraphics[width=0.28\textwidth]{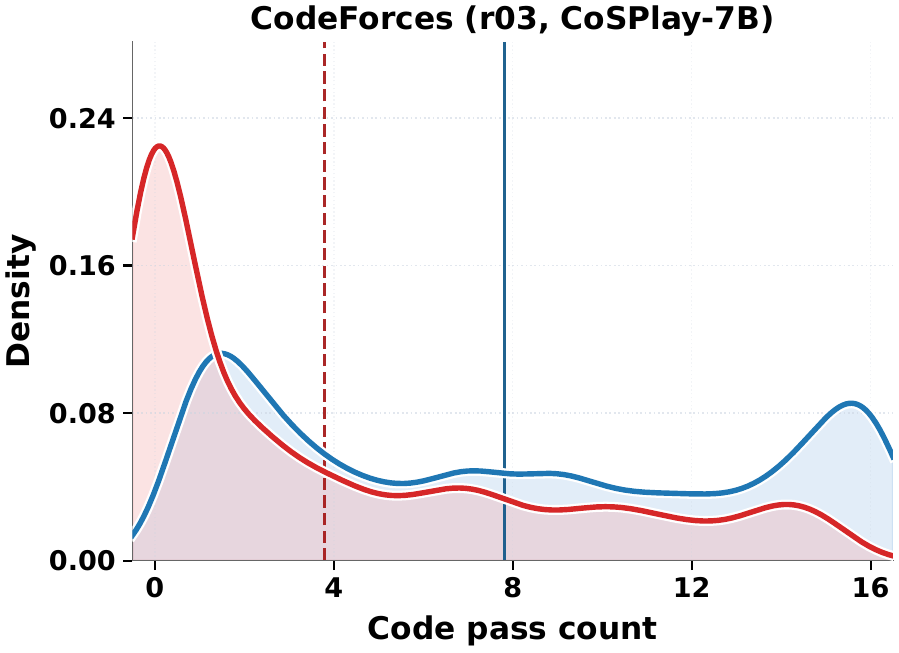}\hfill
    \includegraphics[width=0.28\textwidth]{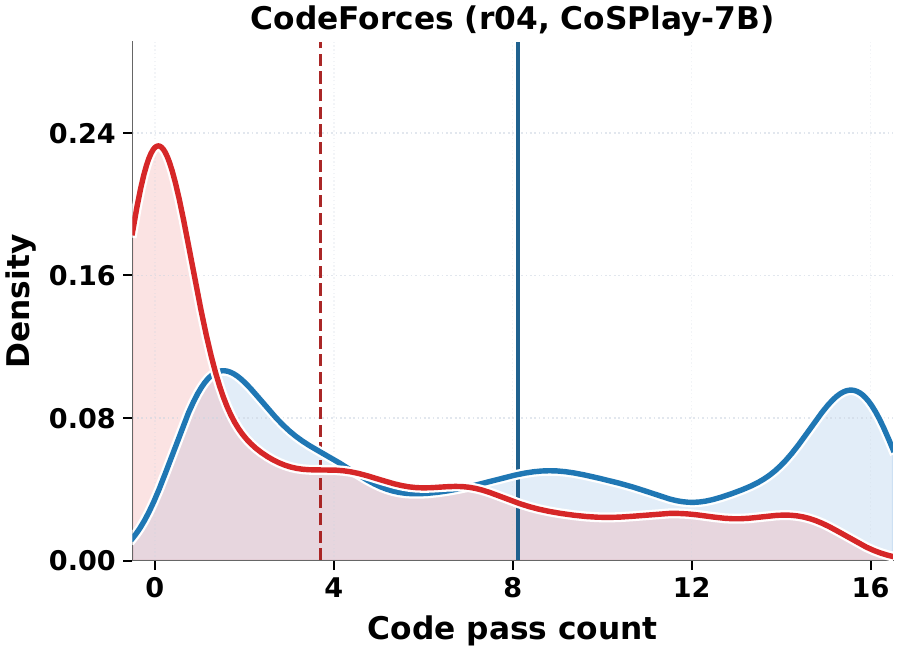}\hfill
    \includegraphics[width=0.28\textwidth]{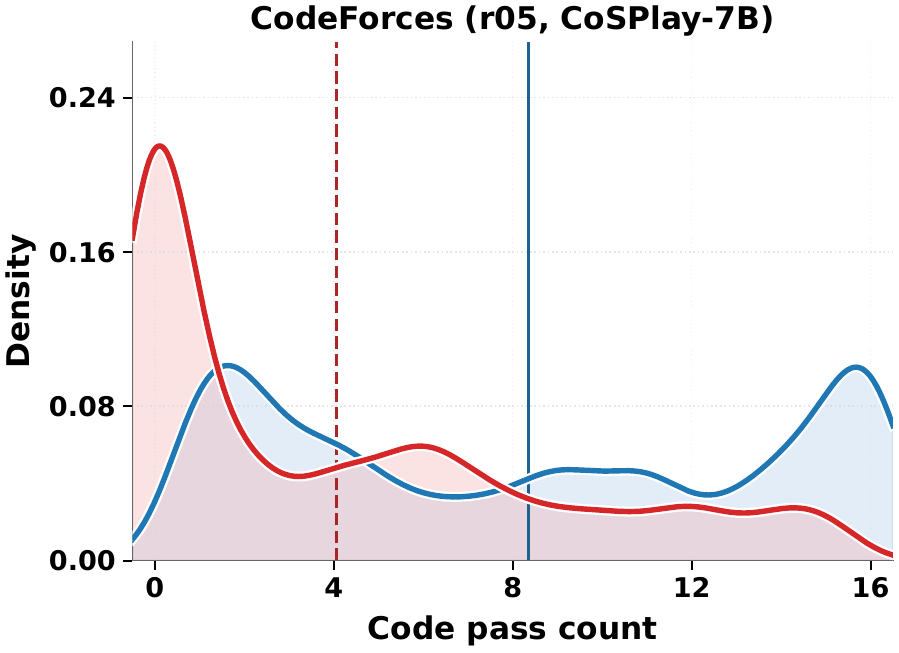}

    \caption{Density distributions of code pass counts for correct (blue) and incorrect (red) code candidates during self-play on \textbf{CodeForces} with the \textbf{7B model}. The top row shows Round 0-2, and the bottom row shows Round 3-5. Vertical lines mean the average values.}
    \label{fig:code_pass_count_density_7b_codeforces}
\end{figure}

\begin{figure}[!t]
    \centering
    \includegraphics[width=0.28\textwidth]{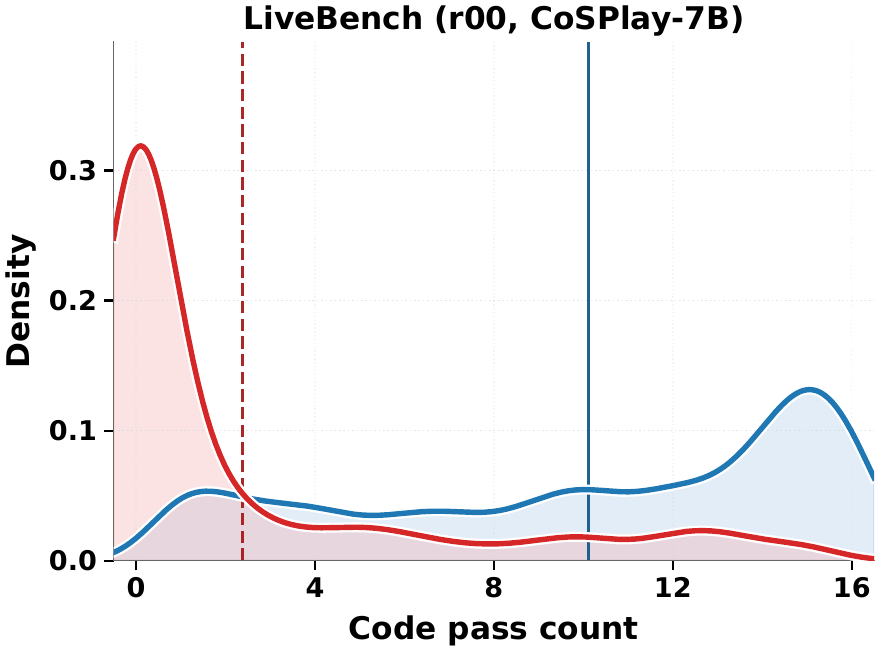}\hfill
    \includegraphics[width=0.28\textwidth]{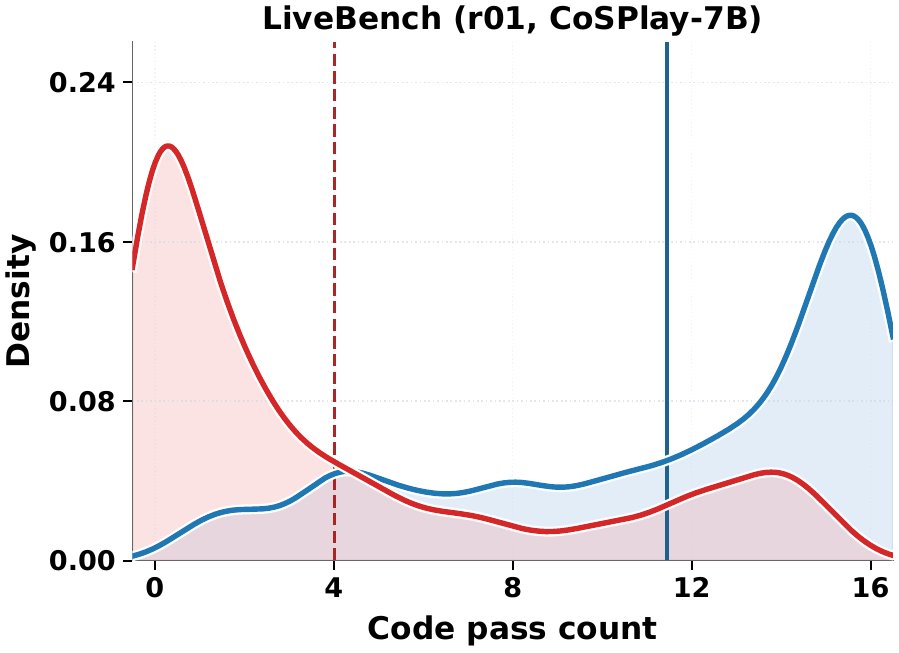}\hfill
    \includegraphics[width=0.28\textwidth]{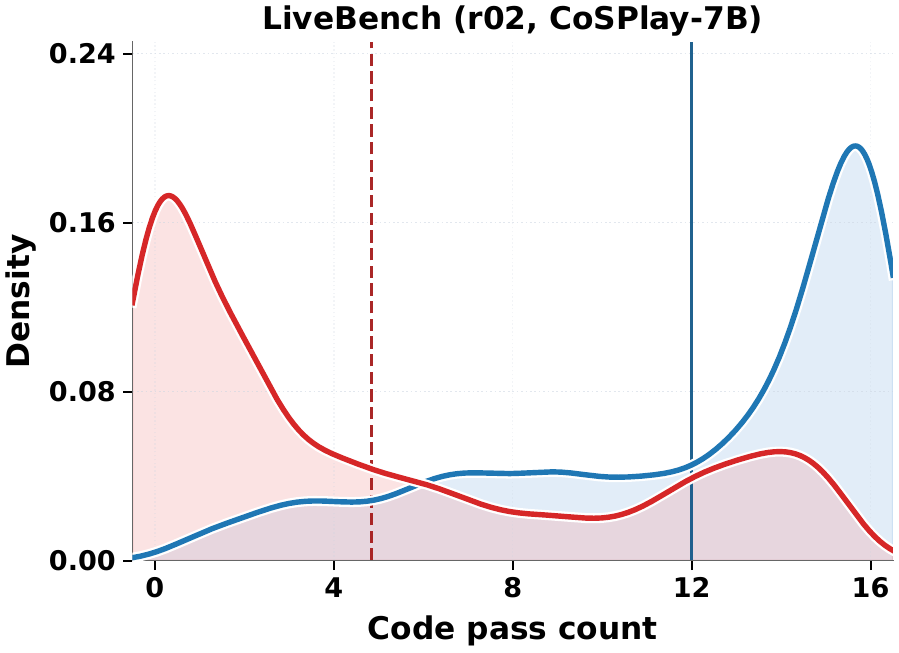}
    \vspace{-1.0em}

    \includegraphics[width=0.28\textwidth]{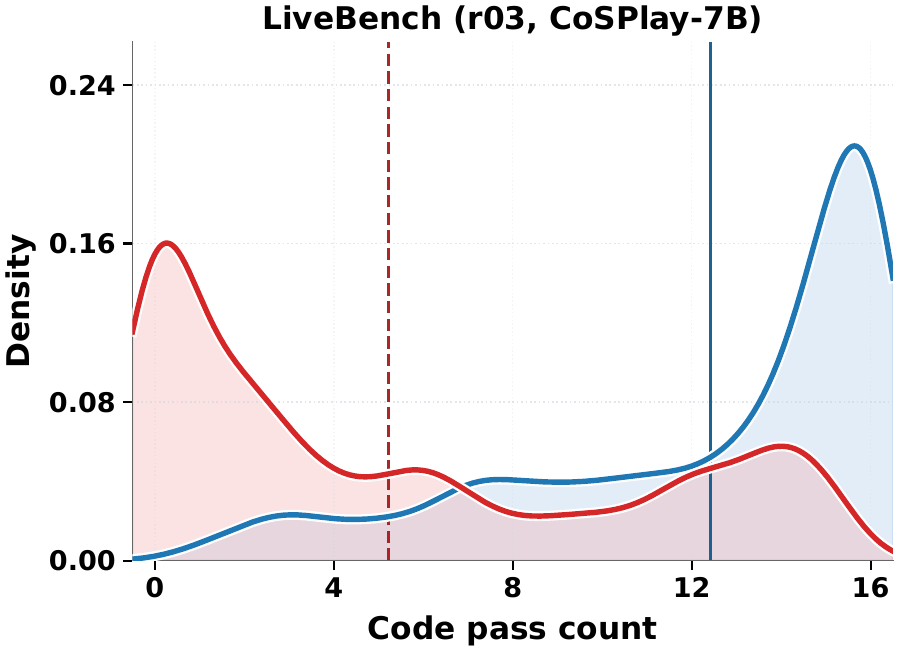}\hfill
    \includegraphics[width=0.28\textwidth]{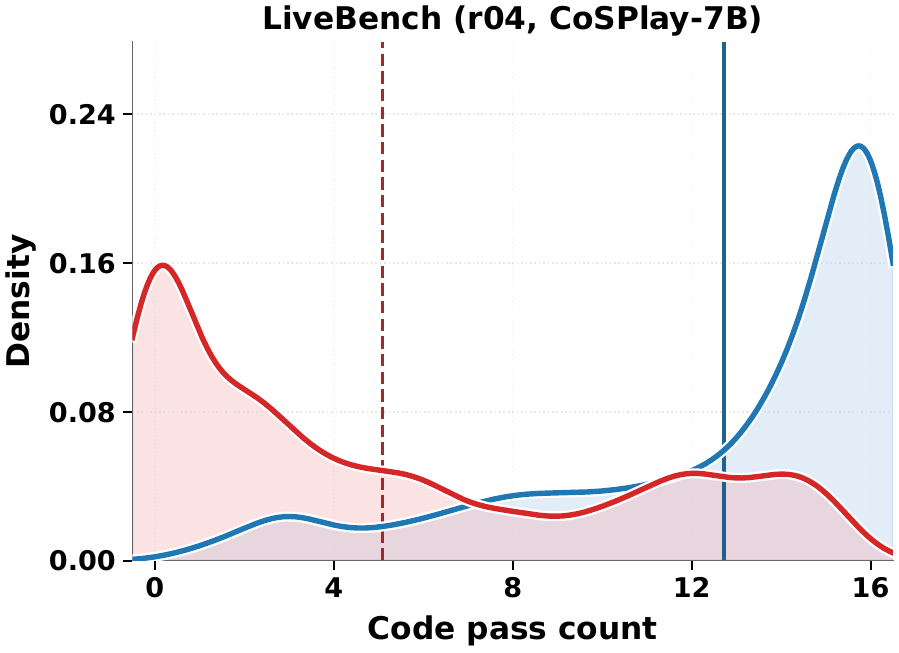}\hfill
    \includegraphics[width=0.28\textwidth]{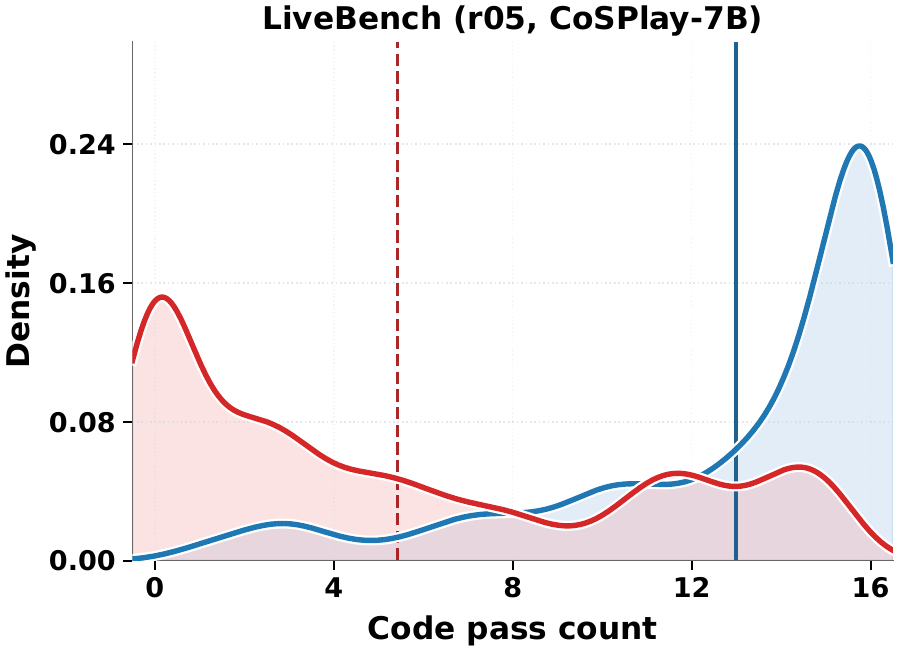}

    \caption{Density distributions of code pass counts for correct (blue) and incorrect (red) code candidates during self-play on \textbf{LiveBench} with the \textbf{7B model}. The top row shows Round 0-2, and the bottom row shows Round 3-5. Vertical lines mean the average values.}
    \label{fig:code_pass_count_density_7b_livebench}
\end{figure}

\begin{figure}[!t]
    \centering
    \includegraphics[width=0.28\textwidth]{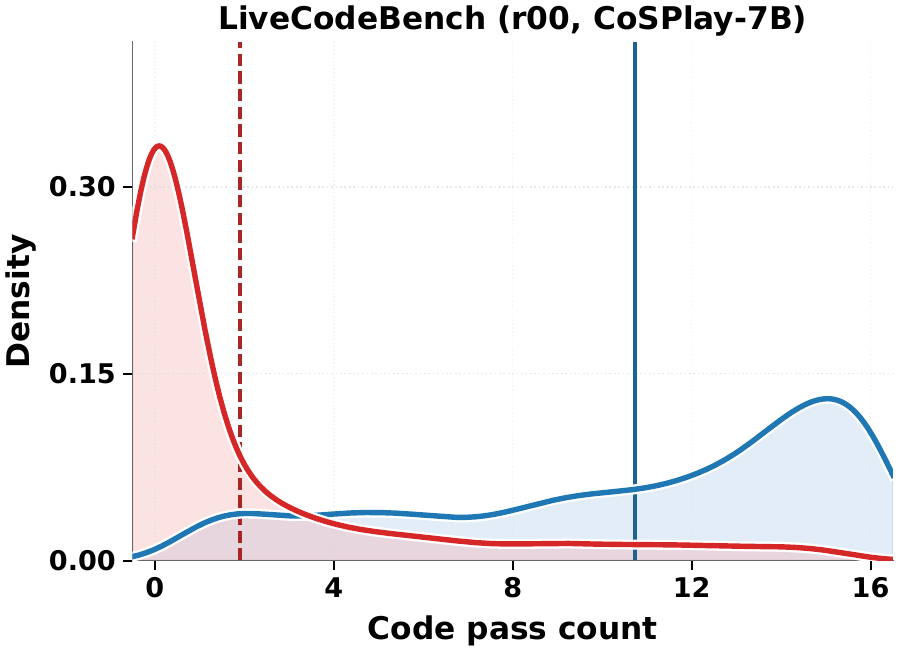}\hfill
    \includegraphics[width=0.28\textwidth]{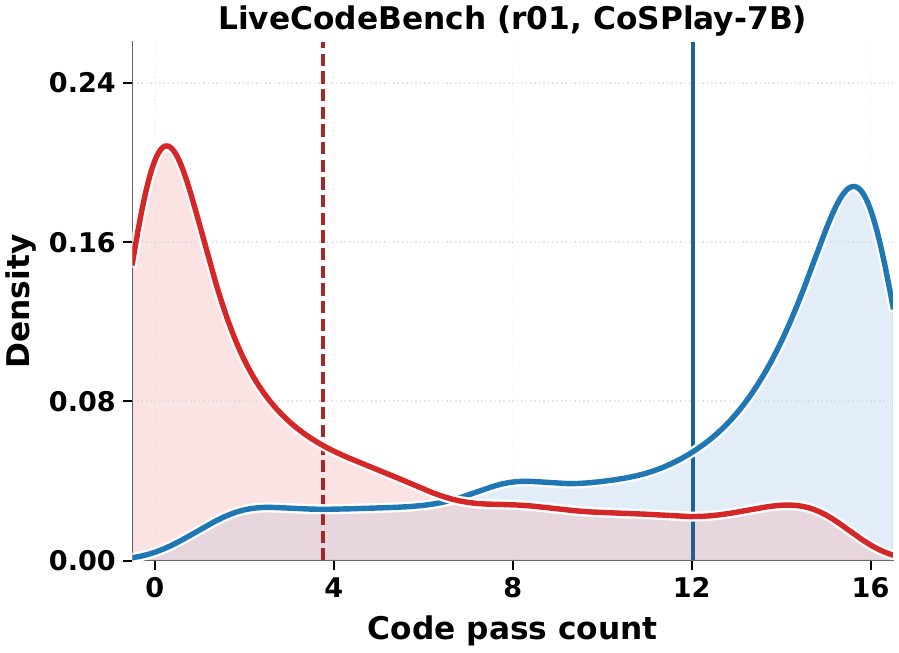}\hfill
    \includegraphics[width=0.28\textwidth]{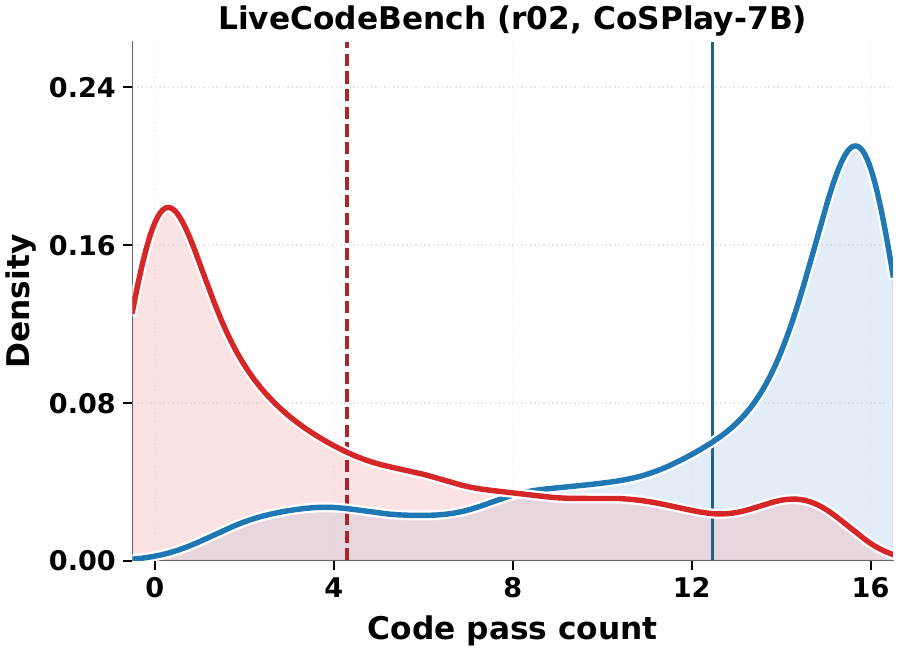}
    \vspace{-1.0em}

    \includegraphics[width=0.28\textwidth]{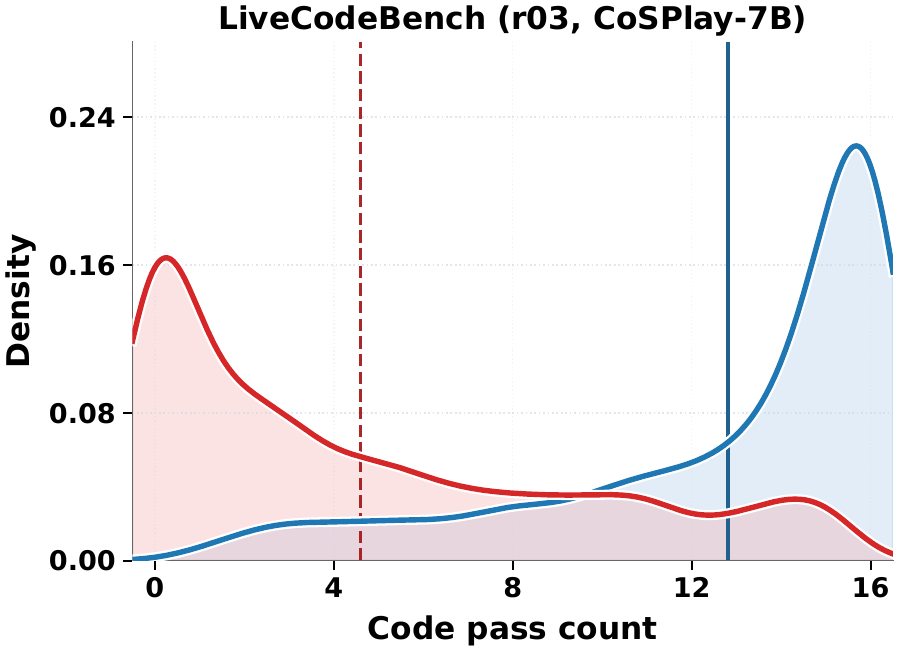}\hfill
    \includegraphics[width=0.28\textwidth]{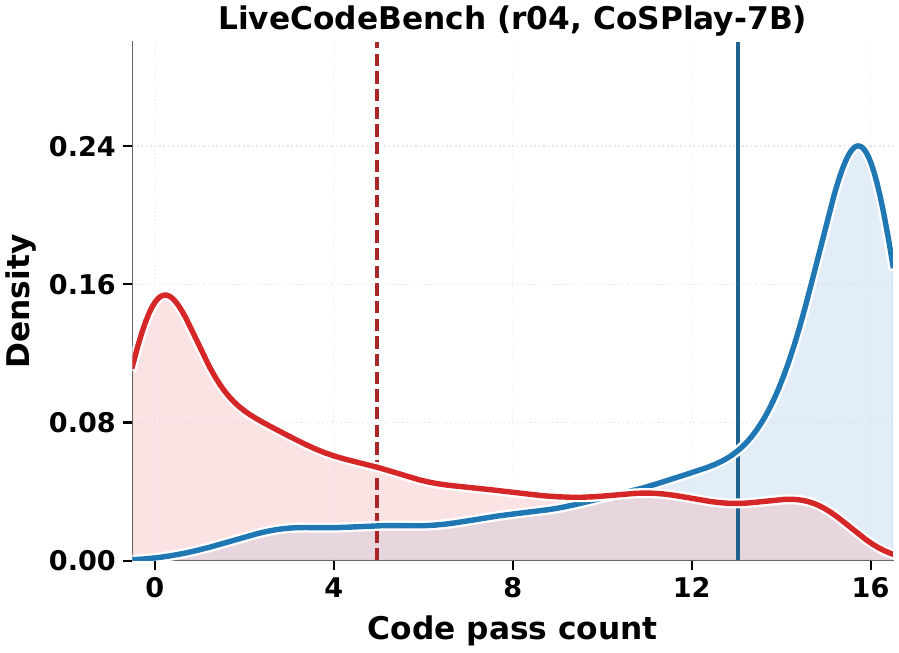}\hfill
    \includegraphics[width=0.28\textwidth]{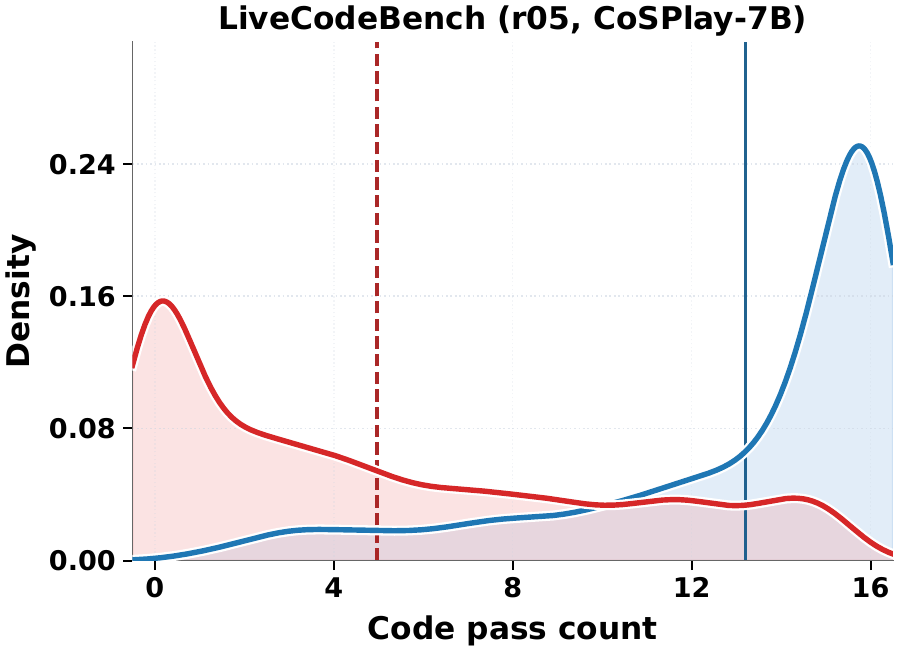}

    \caption{Density distributions of code pass counts for correct (blue) and incorrect (red) code candidates during self-play on \textbf{LiveCodeBench} with the \textbf{7B model}. The top row shows Round 0-2, and the bottom row shows Round 3-5. Vertical lines mean the average values.}
    \label{fig:code_pass_count_density_7b_livecodebench}
\end{figure}

\begin{figure}[!t]
    \centering
    \includegraphics[width=0.28\textwidth]{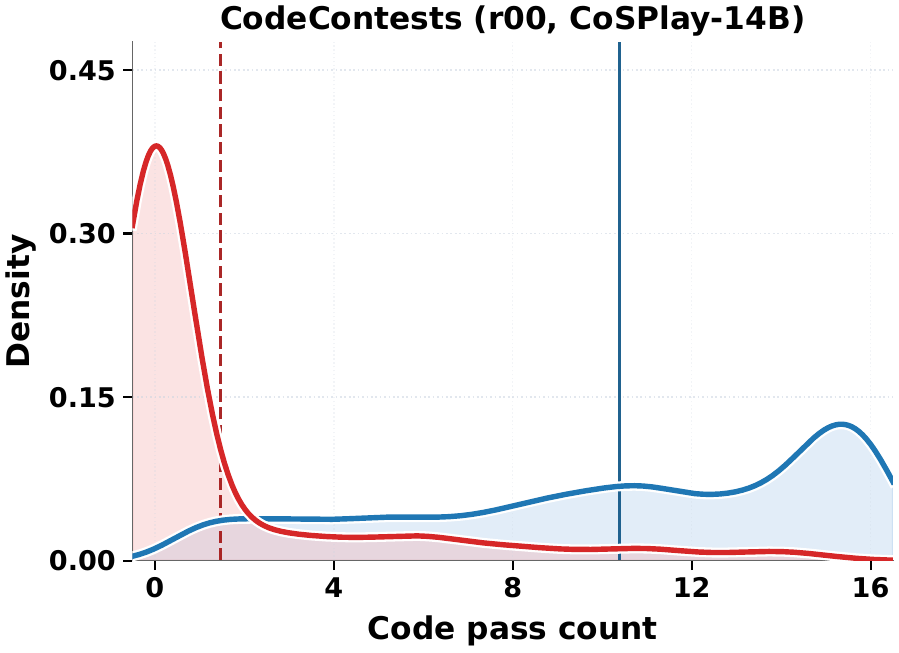}\hfill
    \includegraphics[width=0.28\textwidth]{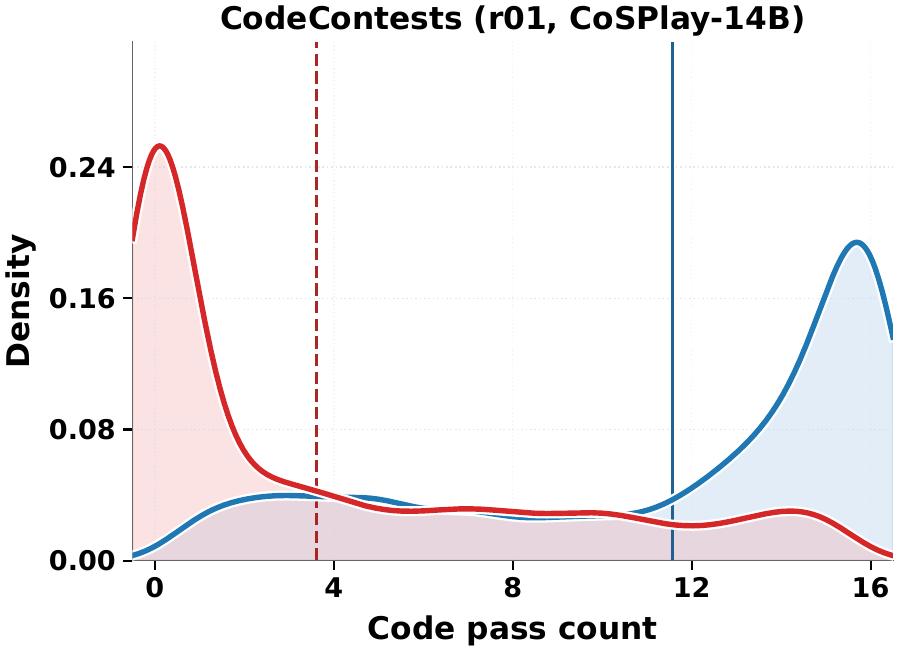}\hfill
    \includegraphics[width=0.28\textwidth]{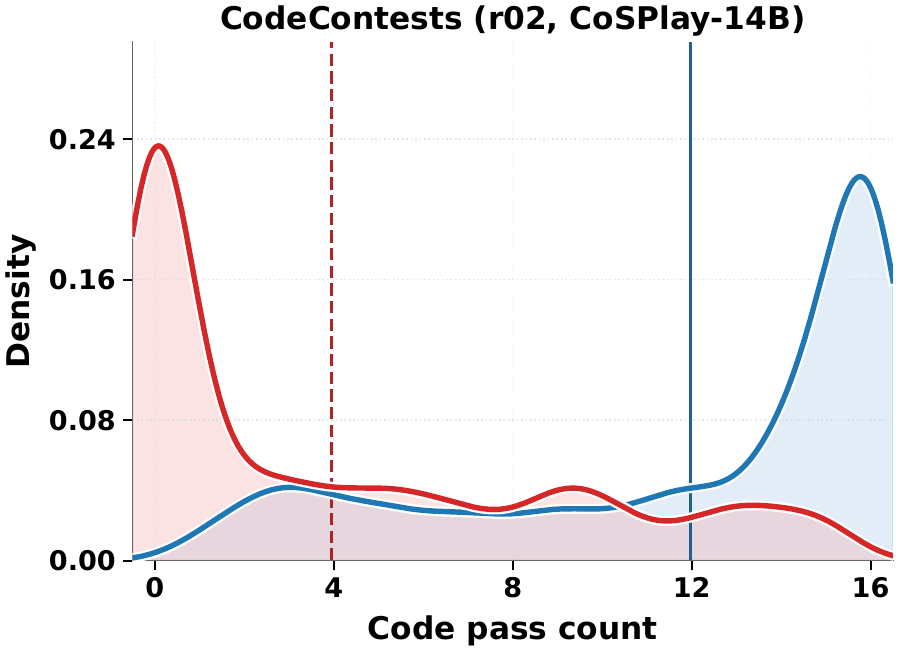}
    \vspace{-1.0em}

    \includegraphics[width=0.28\textwidth]{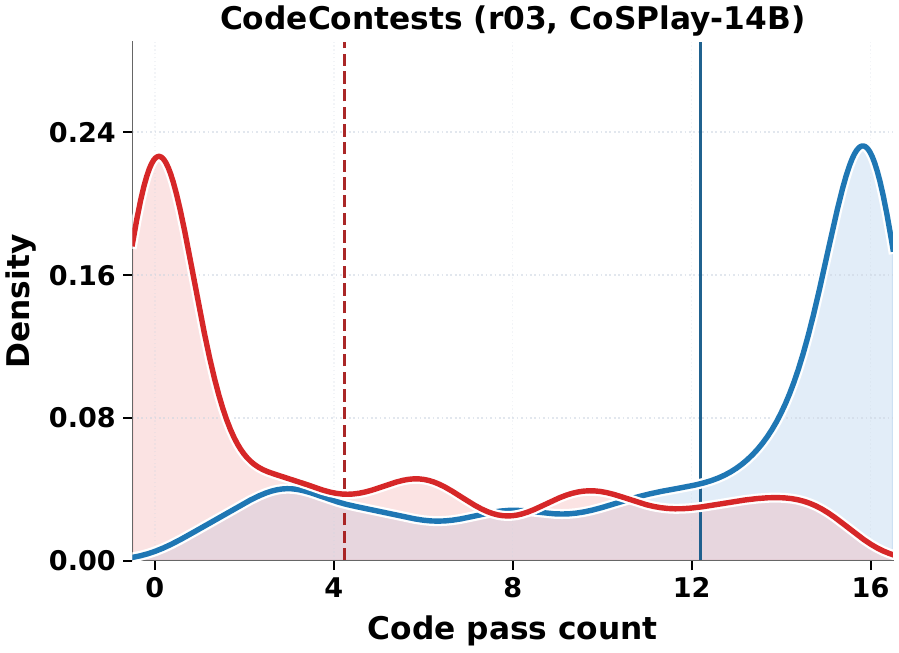}\hfill
    \includegraphics[width=0.28\textwidth]{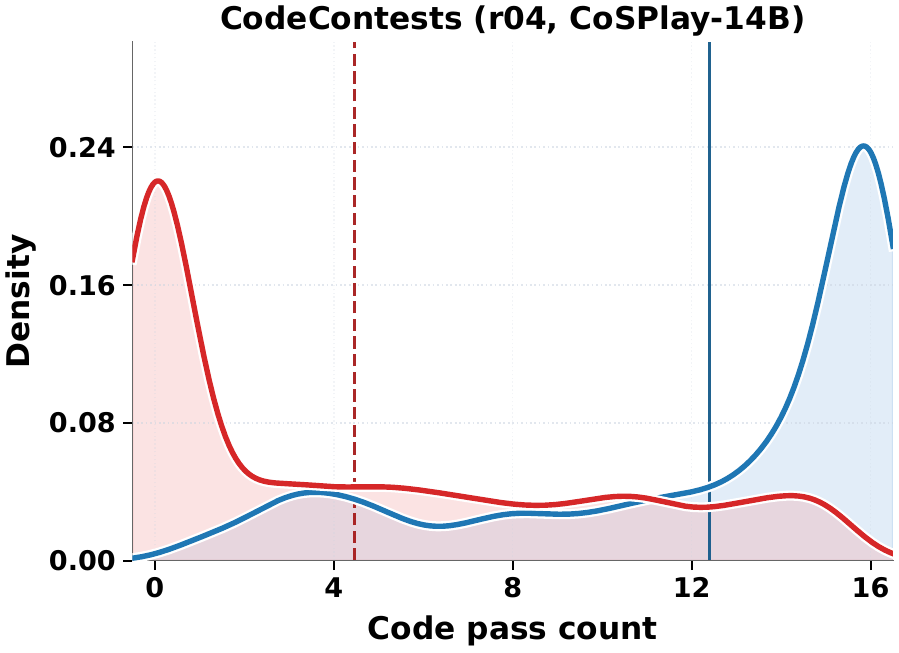}\hfill
    \includegraphics[width=0.28\textwidth]{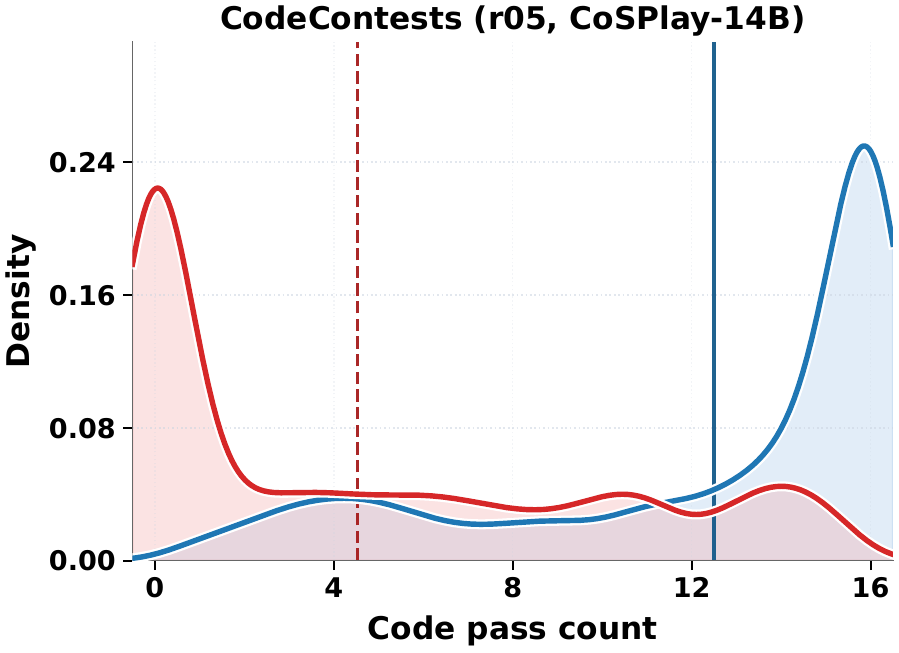}

    \caption{Density distributions of code pass counts for correct (blue) and incorrect (red) code candidates during self-play on \textbf{CodeContests} with the \textbf{14B model}. The top row shows Round 0-2, and the bottom row shows Round 3-5. Vertical lines mean the average values.}
    \label{fig:code_pass_count_density_14b_codecontests}
\end{figure}

\begin{figure}[!t]
    \centering
    \includegraphics[width=0.28\textwidth]{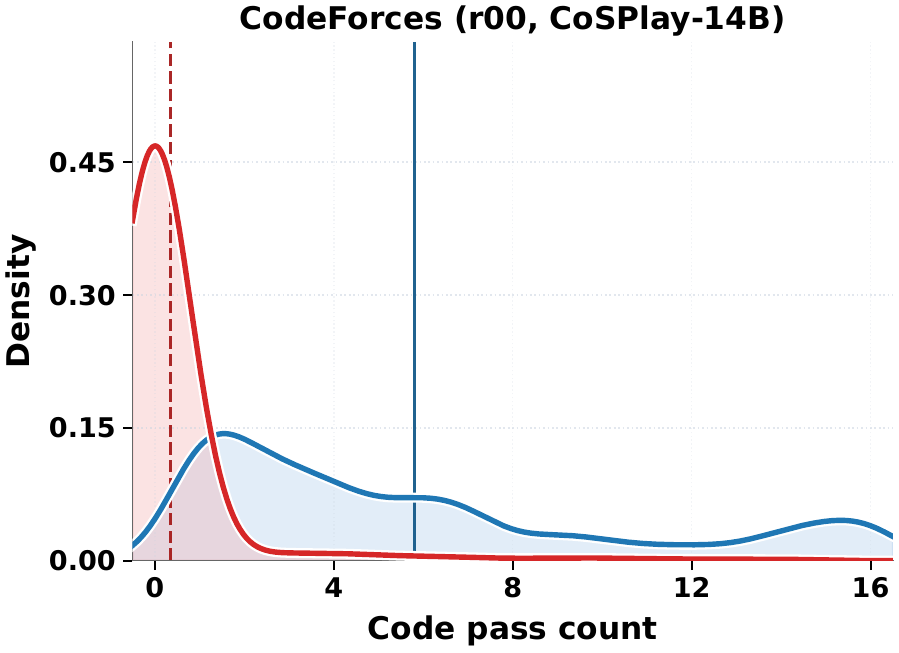}\hfill
    \includegraphics[width=0.28\textwidth]{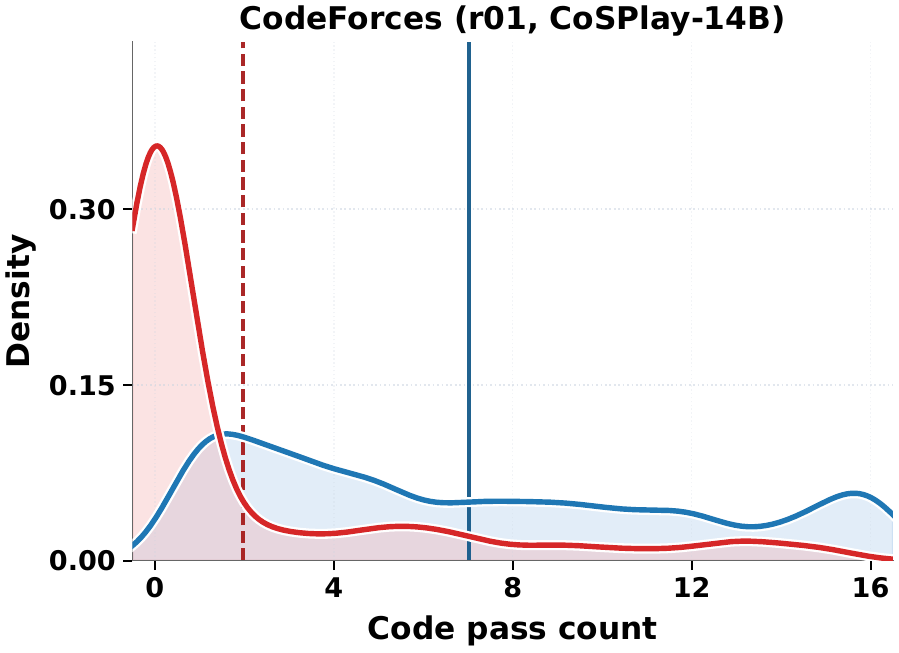}\hfill
    \includegraphics[width=0.28\textwidth]{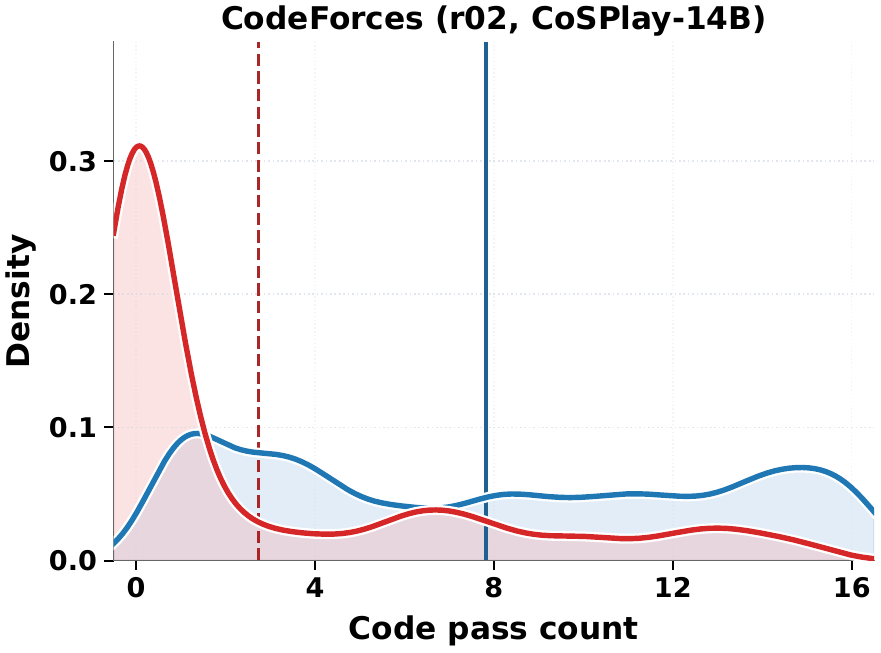}
    \vspace{-1.0em}

    \includegraphics[width=0.28\textwidth]{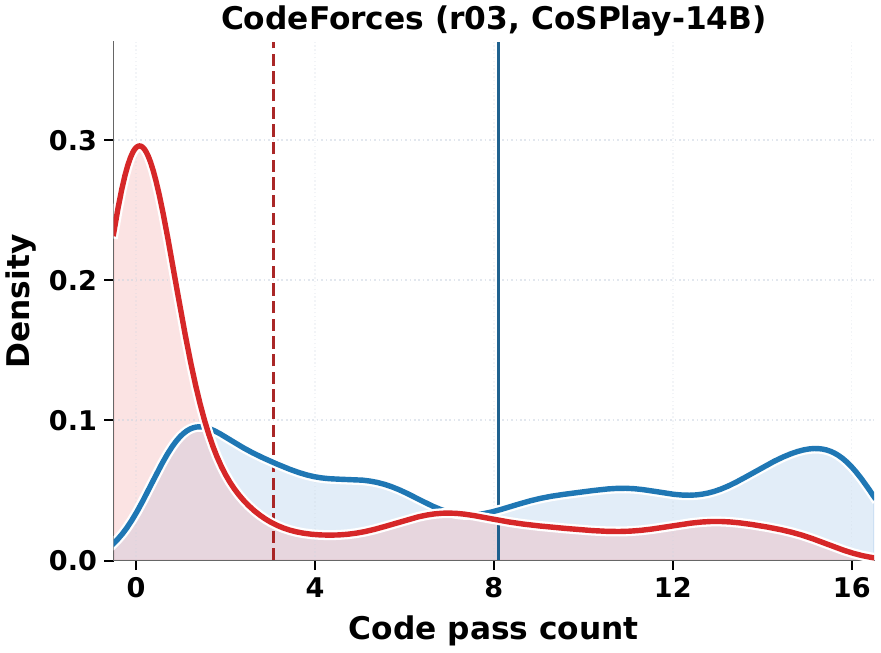}\hfill
    \includegraphics[width=0.28\textwidth]{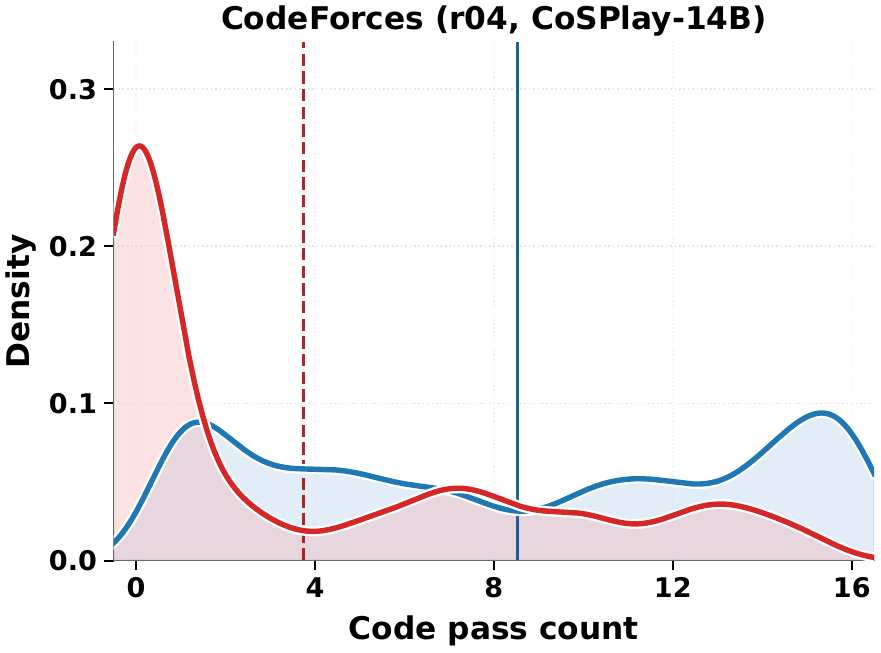}\hfill
    \includegraphics[width=0.28\textwidth]{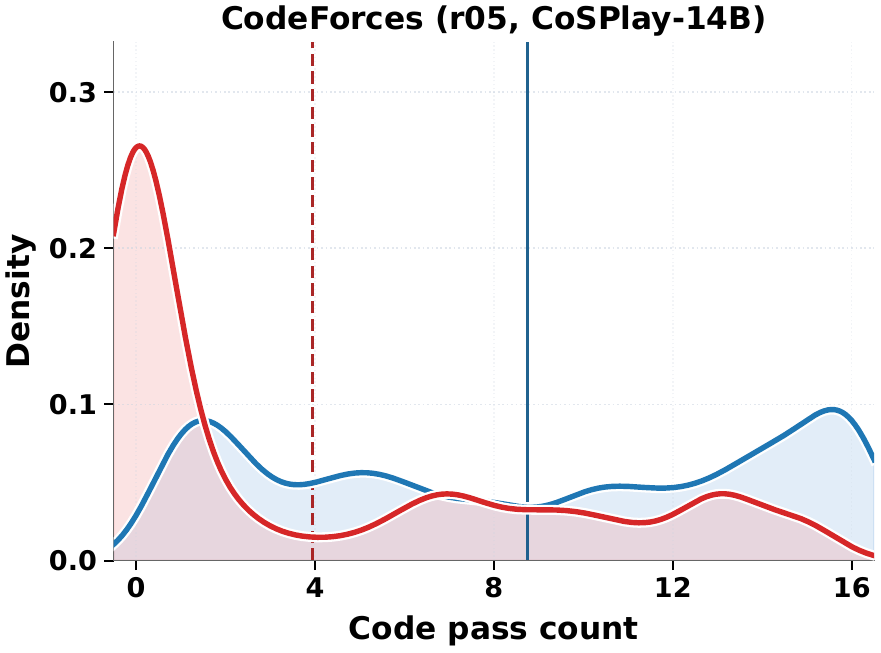}

    \caption{Density distributions of code pass counts for correct (blue) and incorrect (red) code candidates during self-play on \textbf{CodeForces} with the \textbf{14B model}. The top row shows Round 0-2, and the bottom row shows Round 3-5. Vertical lines mean the average values.}
    \label{fig:code_pass_count_density_14b_codeforces}
\end{figure}

\begin{figure}[!t]
    \centering
    \includegraphics[width=0.28\textwidth]{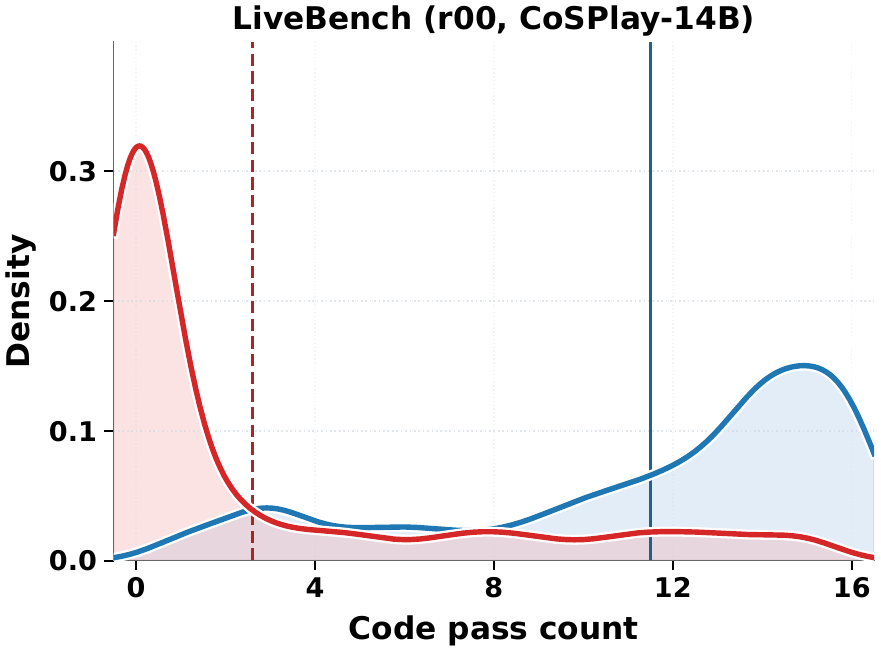}\hfill
    \includegraphics[width=0.28\textwidth]{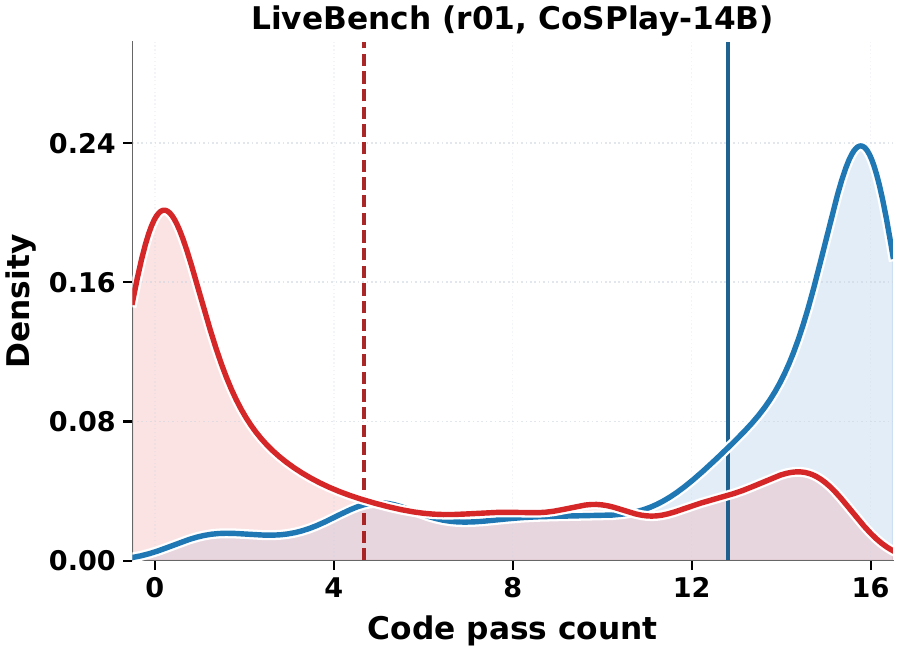}\hfill
    \includegraphics[width=0.28\textwidth]{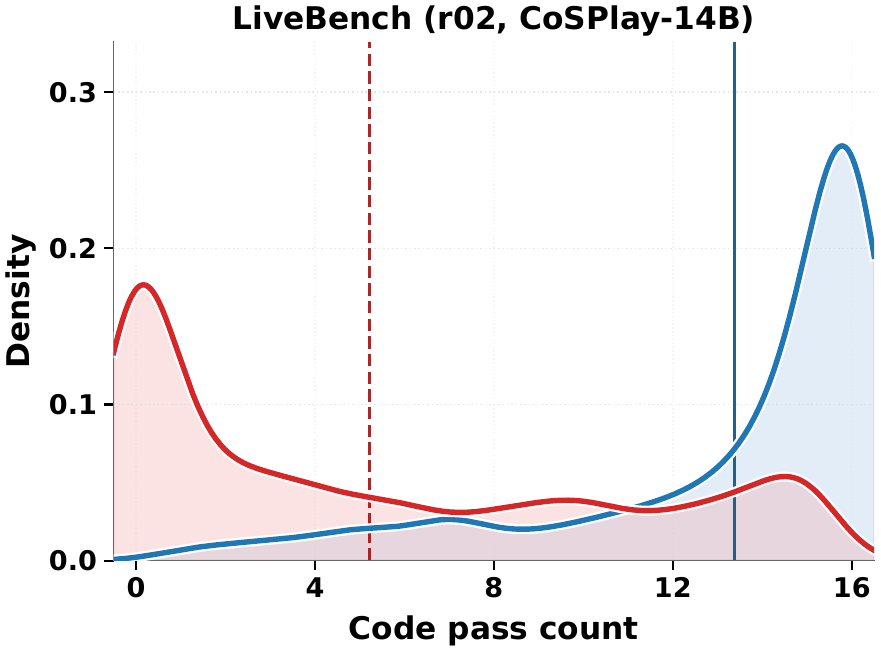}
    \vspace{-1.0em}

    \includegraphics[width=0.28\textwidth]{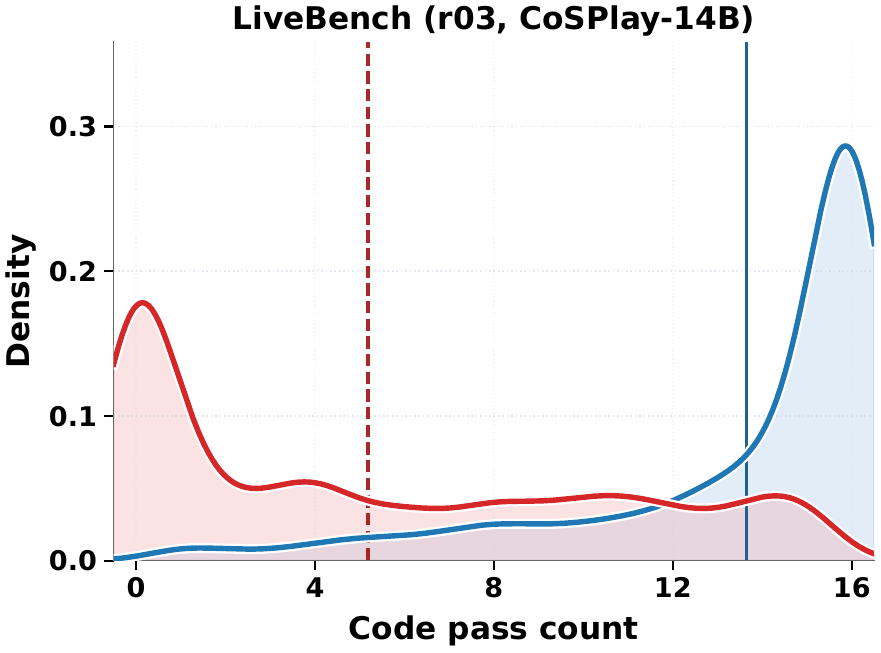}\hfill
    \includegraphics[width=0.28\textwidth]{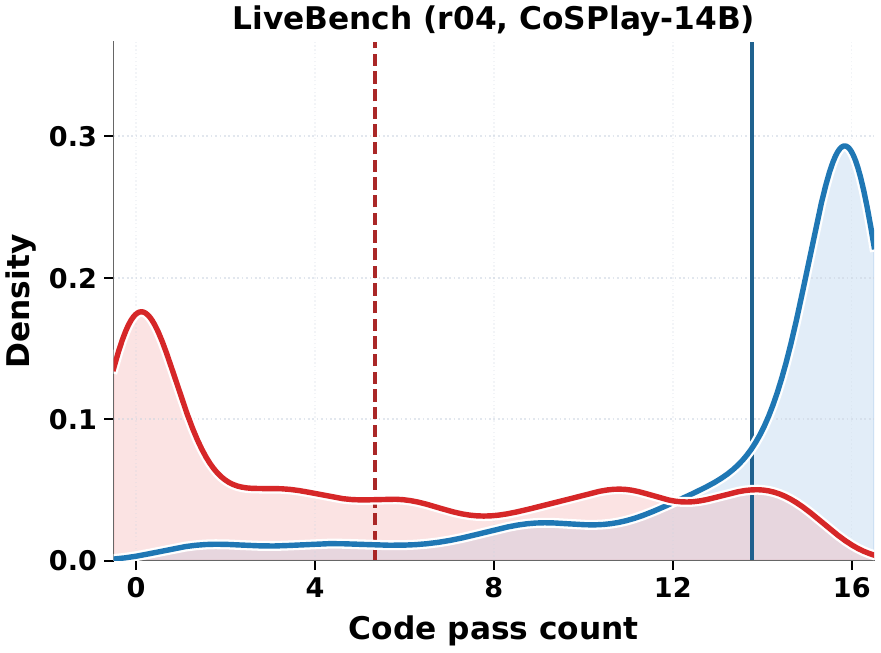}\hfill
    \includegraphics[width=0.28\textwidth]{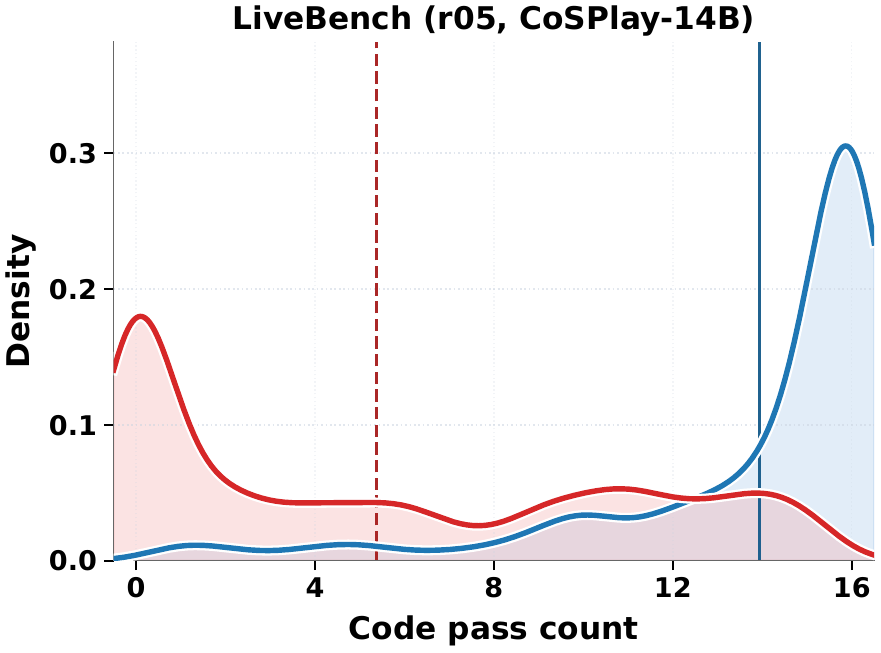}

    \caption{Density distributions of code pass counts for correct (blue) and incorrect (red) code candidates during self-play on \textbf{LiveBench} with the \textbf{14B model}. The top row shows Round 0-2, and the bottom row shows Round 3-5. Vertical lines mean the average values.}
    \label{fig:code_pass_count_density_14b_livebench}
\end{figure}
\begin{figure}[!t]
    \centering
    \includegraphics[width=0.28\textwidth]{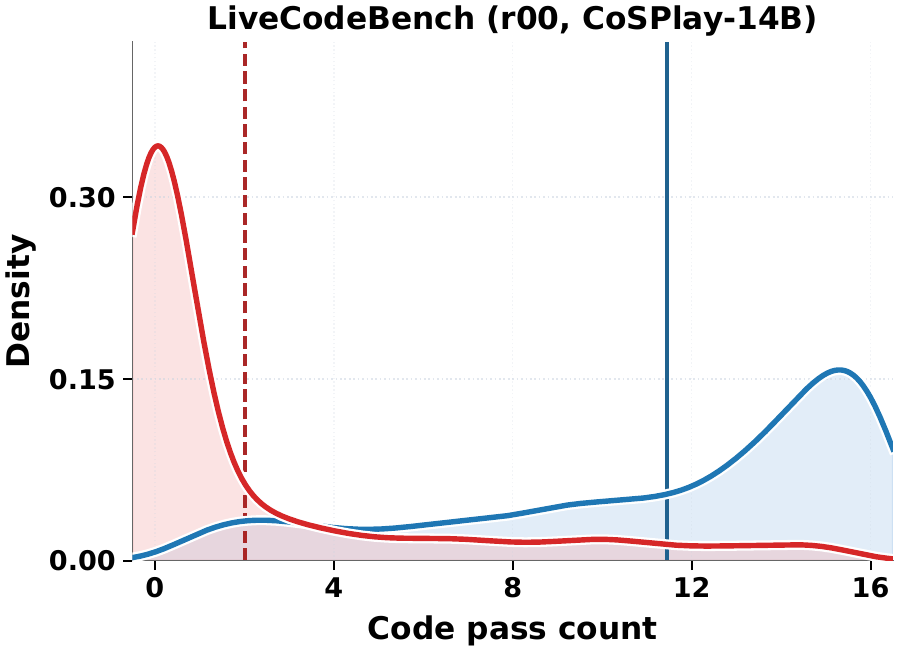}\hfill
    \includegraphics[width=0.28\textwidth]{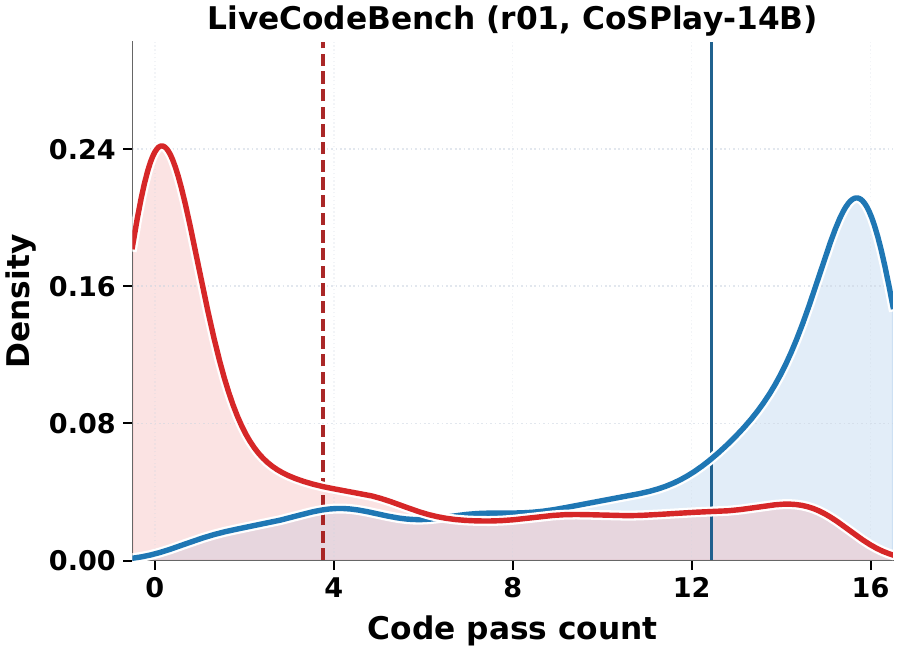}\hfill
    \includegraphics[width=0.28\textwidth]{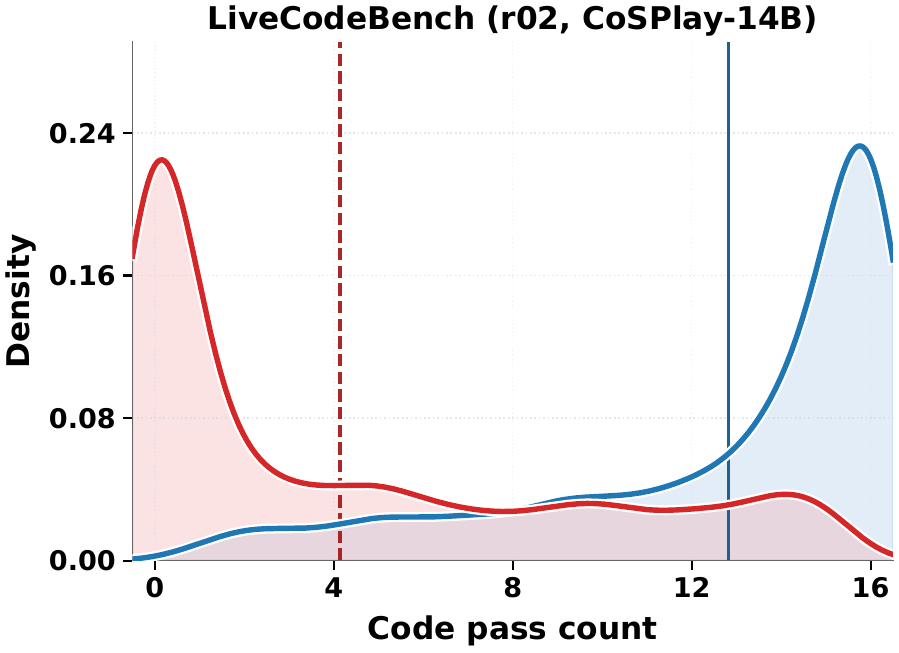}
    \vspace{-1.0em}

    \includegraphics[width=0.28\textwidth]{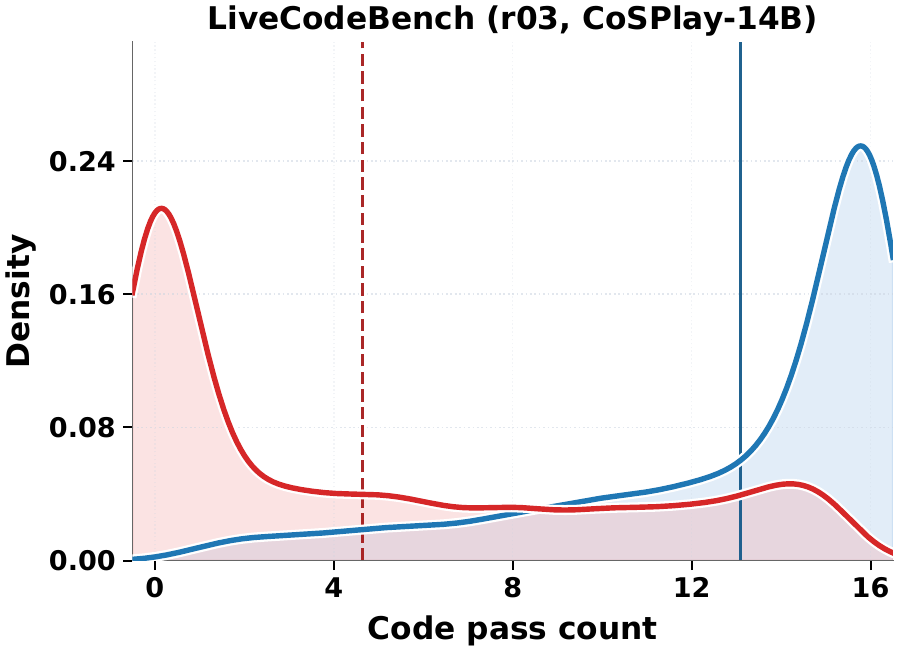}\hfill
    \includegraphics[width=0.28\textwidth]{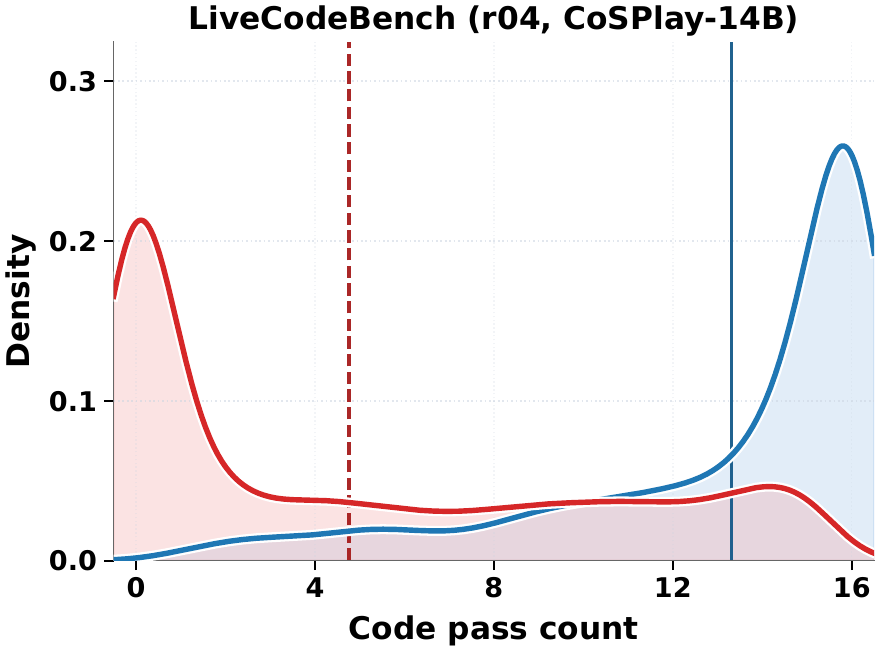}\hfill
    \includegraphics[width=0.28\textwidth]{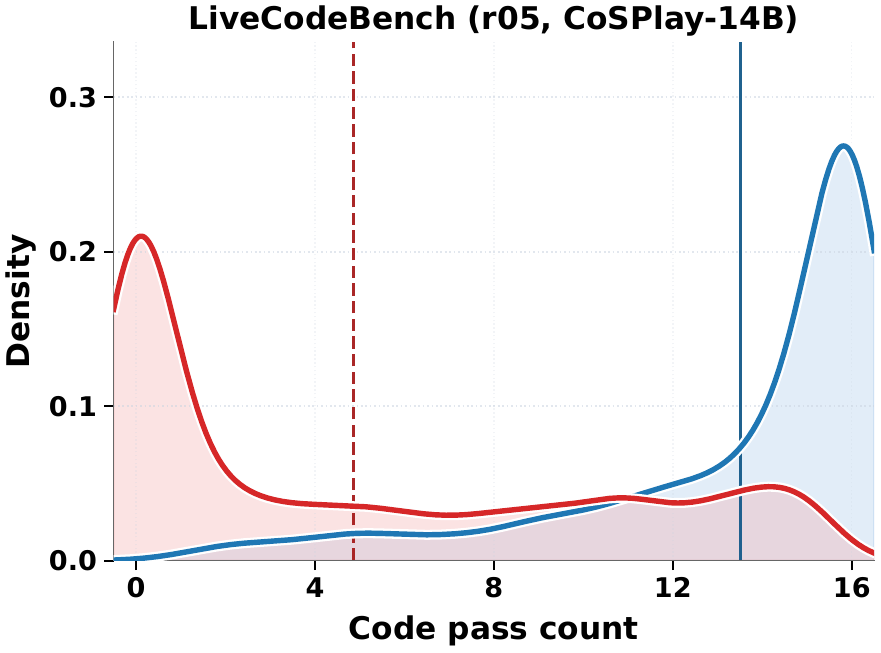}

    \caption{Density distributions of code pass counts for correct (blue) and incorrect (red) code candidates during self-play on \textbf{LiveCodeBench} with the \textbf{14B model}. The top row shows Round 0-2, and the bottom row shows Round 3-5. Vertical lines mean the average values.}
    \label{fig:code_pass_count_density_14b_livecodebench}
\end{figure}

\clearpage

\section{Detailed Evolution of average true accuracy vs. cluster size}
\label{app:cluster size vs average true acc}
Figure~\ref{fig:sub_codecontests_cluster_size_vs_code_true_acc}-\ref{fig:sub_livecodebench_cluster_size_vs_code_true_acc} provide per-dataset and per-round results for the relationship between cluster size and average true code accuracy. Across datasets and model scales, larger clusters consistently correspond to higher average correctness, supporting the assumption that output consensus is a useful GT-free proxy for functional correctness.This trend also remains stable throughout self-play rounds. As the code and UT pools are iteratively refined, high-accuracy candidates increasingly concentrate in larger consensus clusters, while small clusters are more likely to contain wrong or unstable solutions. These results provide fine-grained evidence for our execution-consensus-based cluster selection strategy.

\begin{figure}[H]
    \centering
    \includegraphics[width=0.27\textwidth,trim=0 0 0 0pt,clip]{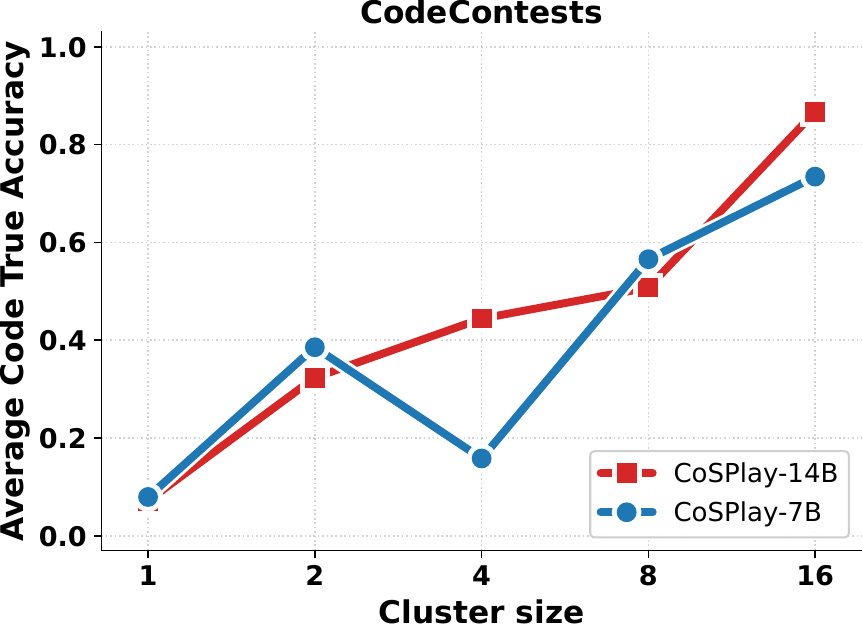}\hfill%
    \includegraphics[width=0.27\textwidth,trim=0 0 0 0pt,clip]{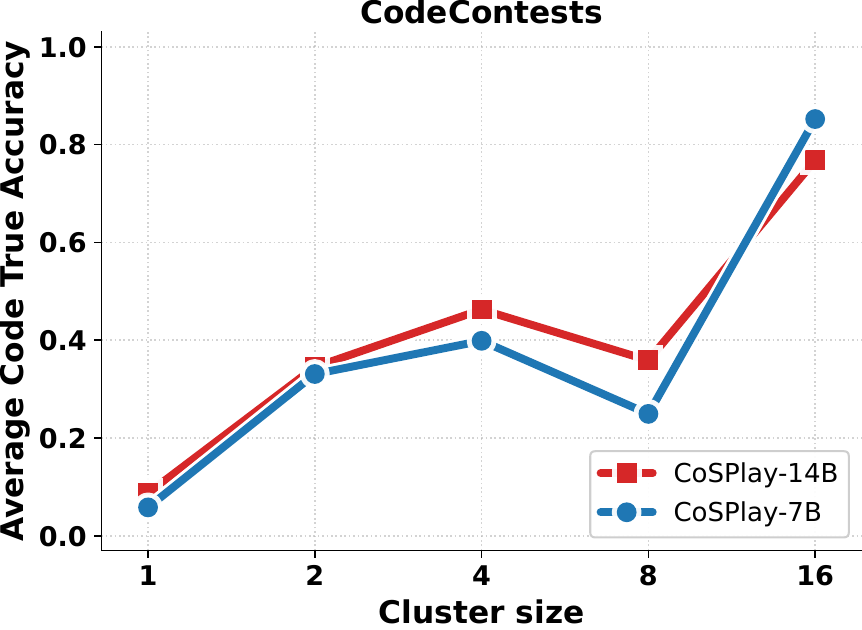}\hfill%
    \includegraphics[width=0.27\textwidth,trim=0 0 0 0pt,clip]{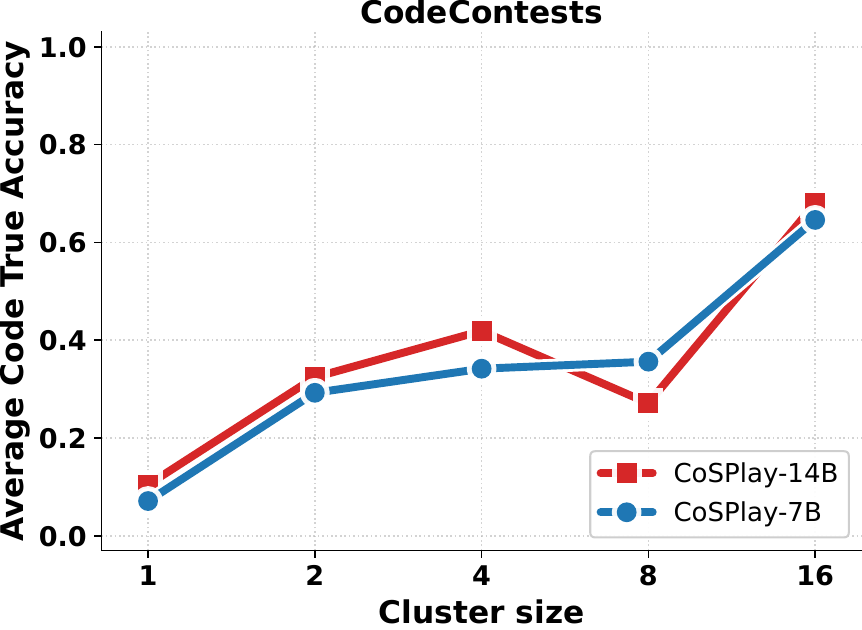}


    \includegraphics[width=0.27\textwidth,trim=0 0 0 0pt,clip]{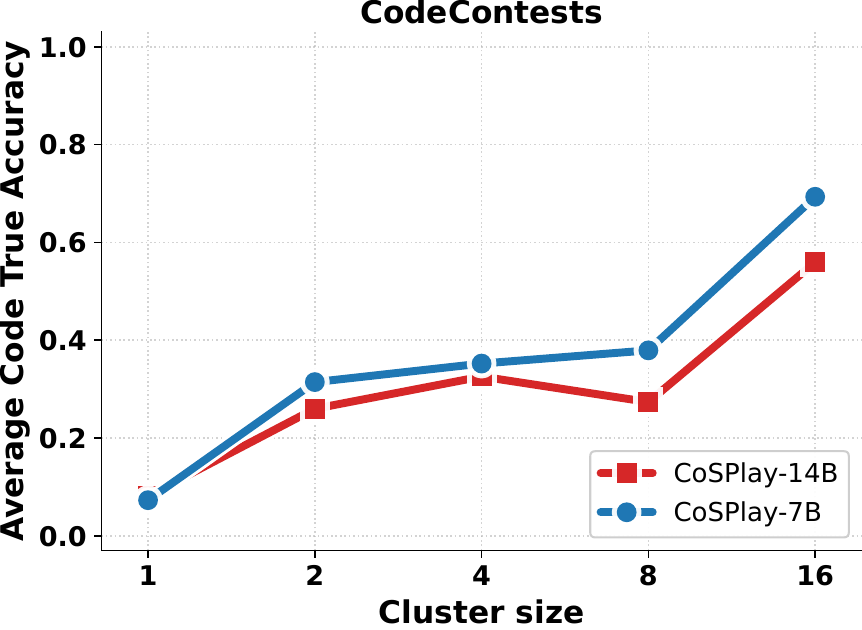}\hfill%
    \includegraphics[width=0.27\textwidth,trim=0 0 0 0pt,clip]{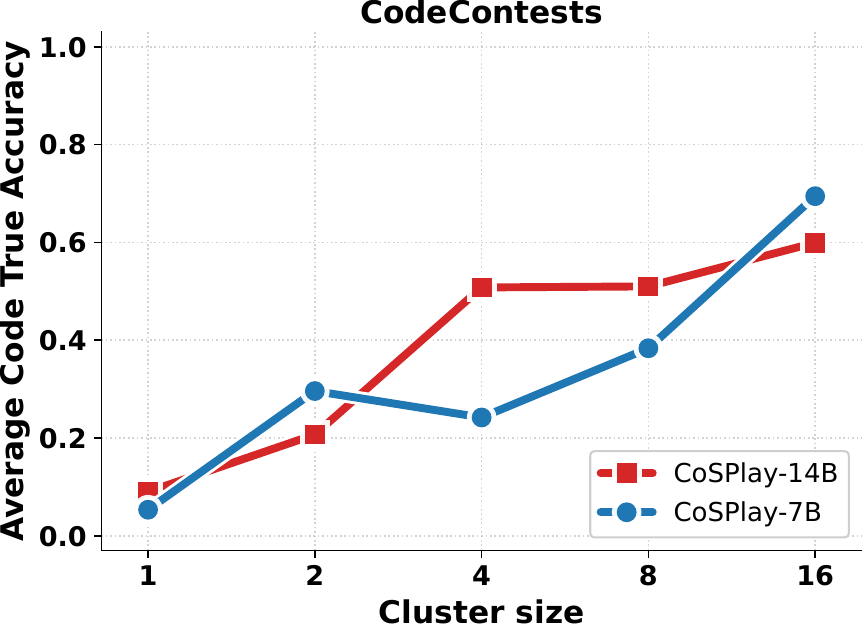}\hfill%
    \includegraphics[width=0.27\textwidth,trim=0 0 0 0pt,clip]{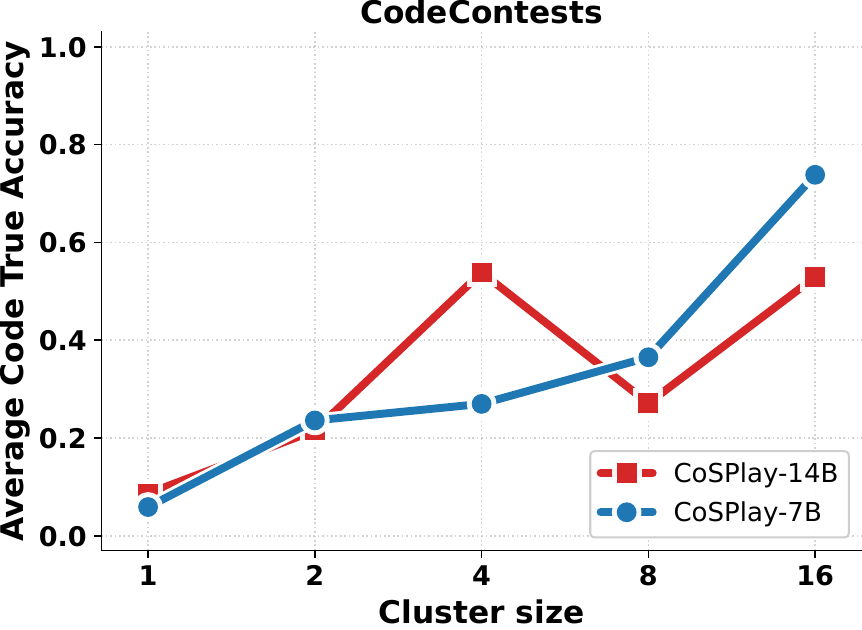}


    \caption{The relationship between cluster size and average code true accuracy during self-play on CodeContests for both 7B and 14B models. The top row shows Round 0-2, and the bottom row shows Round 3-5.}
    \label{fig:sub_codecontests_cluster_size_vs_code_true_acc}
\end{figure}

\vspace{-2.5em}

\begin{figure}[H]
    \centering
    \includegraphics[width=0.27\textwidth,trim=0 0 0 0pt,clip]{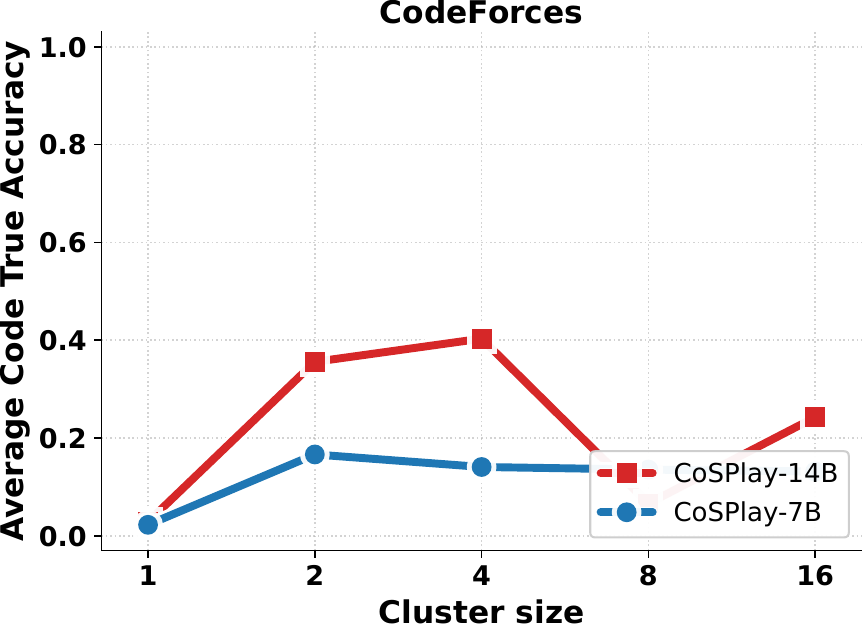}\hfill%
    \includegraphics[width=0.27\textwidth,trim=0 0 0 0pt,clip]{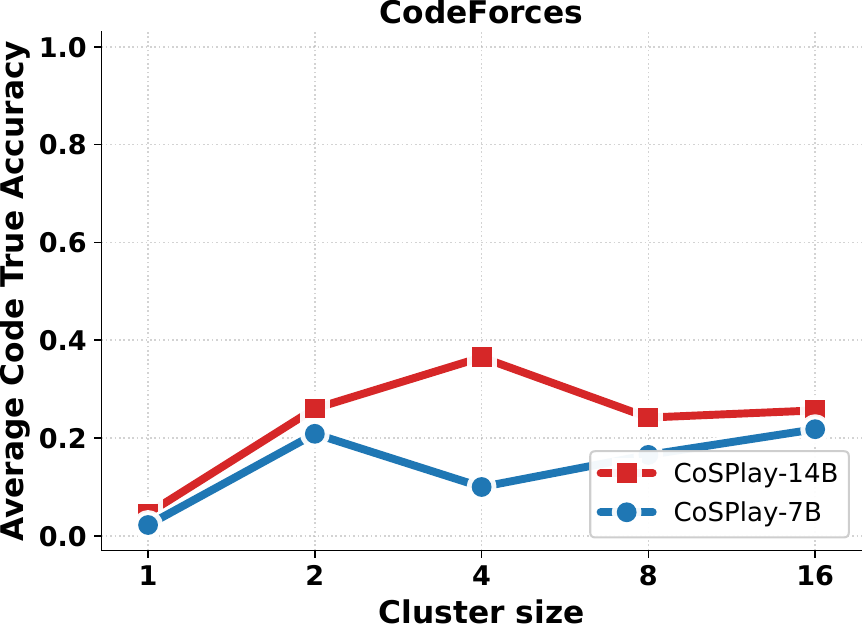}\hfill%
    \includegraphics[width=0.27\textwidth,trim=0 0 0 0pt,clip]{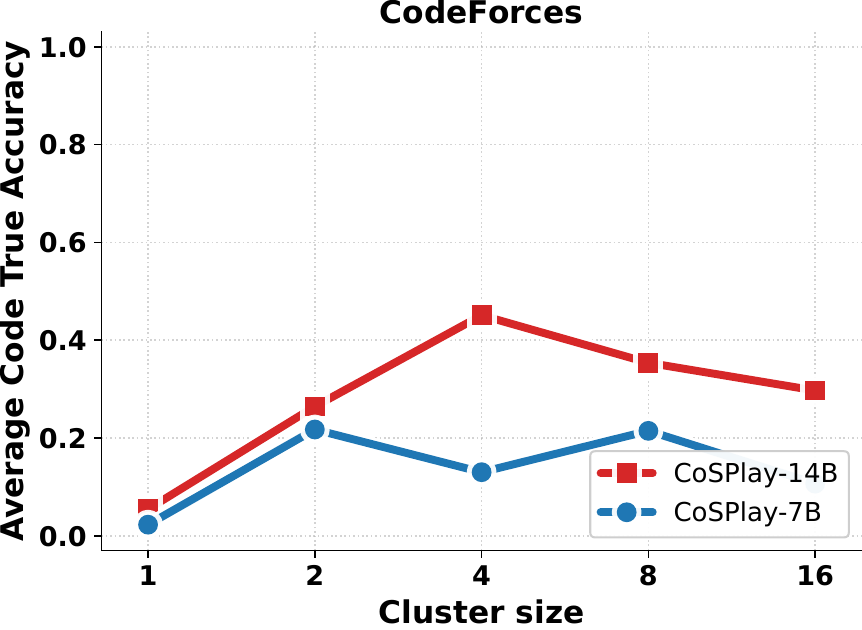}


    \includegraphics[width=0.27\textwidth,trim=0 0 0 0pt,clip]{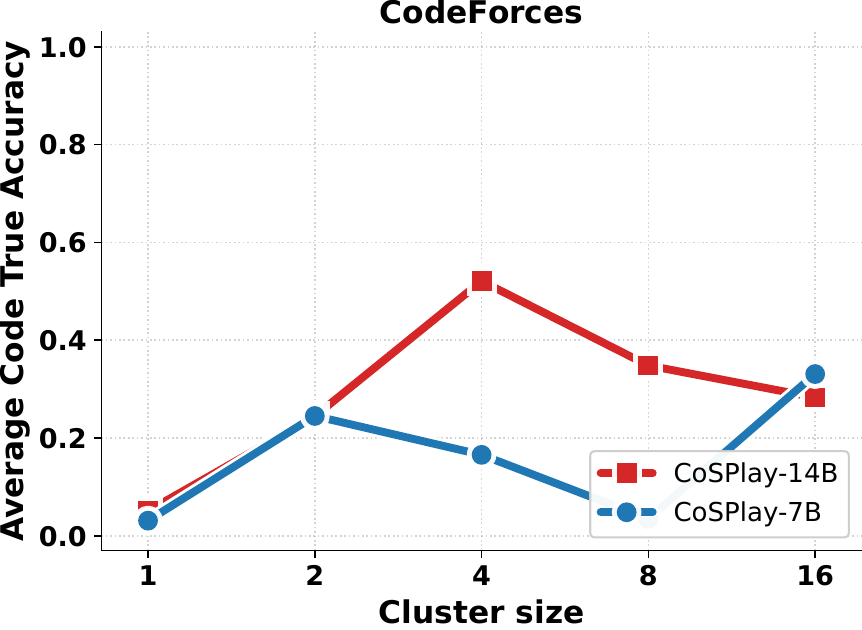}\hfill%
    \includegraphics[width=0.27\textwidth,trim=0 0 0 0pt,clip]{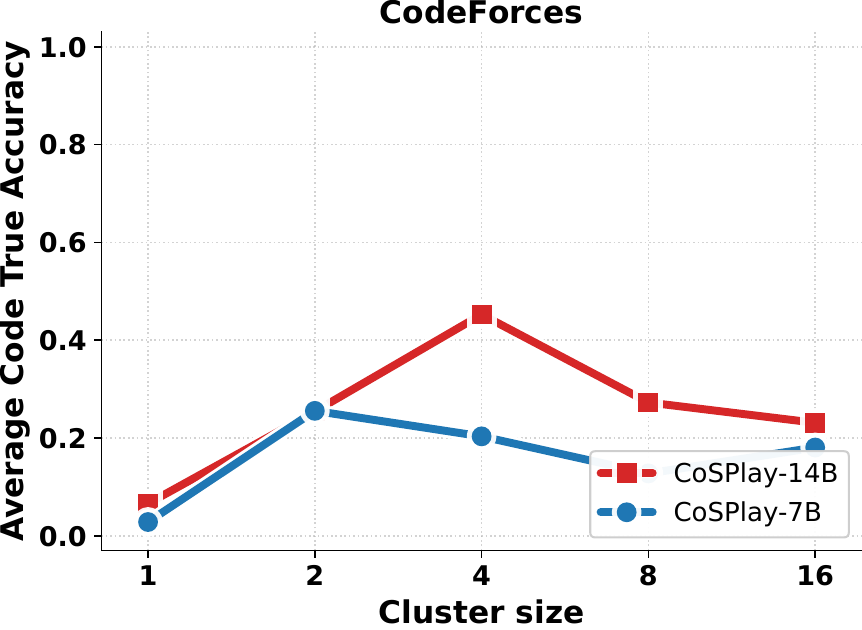}\hfill%
    \includegraphics[width=0.27\textwidth,trim=0 0 0 0pt,clip]{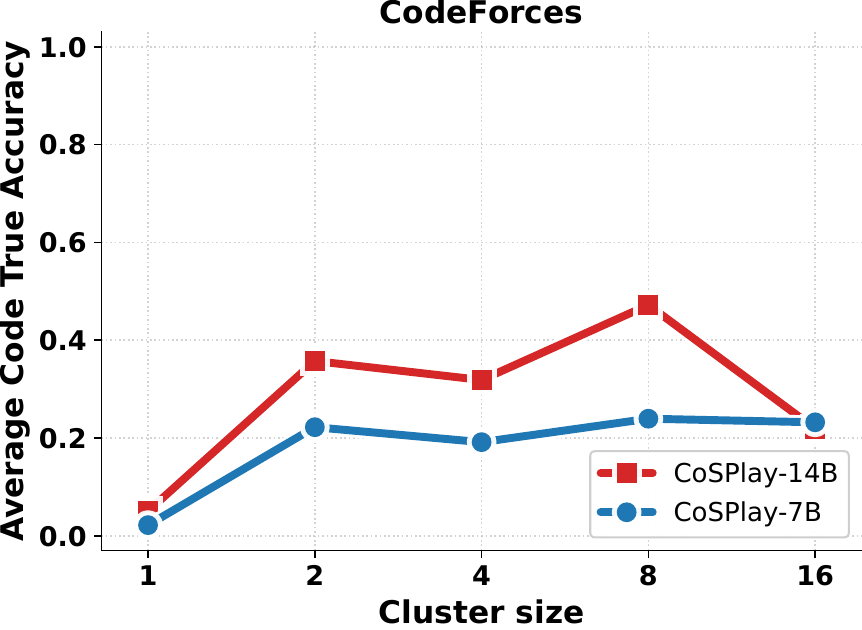}


    \caption{The relationship between cluster size and average code true accuracy during self-play on CodeForces for both 7B and 14B models. The top row shows Round 0-2, and the bottom row shows Round 3-5.}
    \label{fig:sub_codeforces_cluster_size_vs_code_true_acc}
\end{figure}

\vspace{-2.5em}

\begin{figure}[H]
    \centering
    \includegraphics[width=0.27\textwidth,trim=0 0 0 0pt,clip]{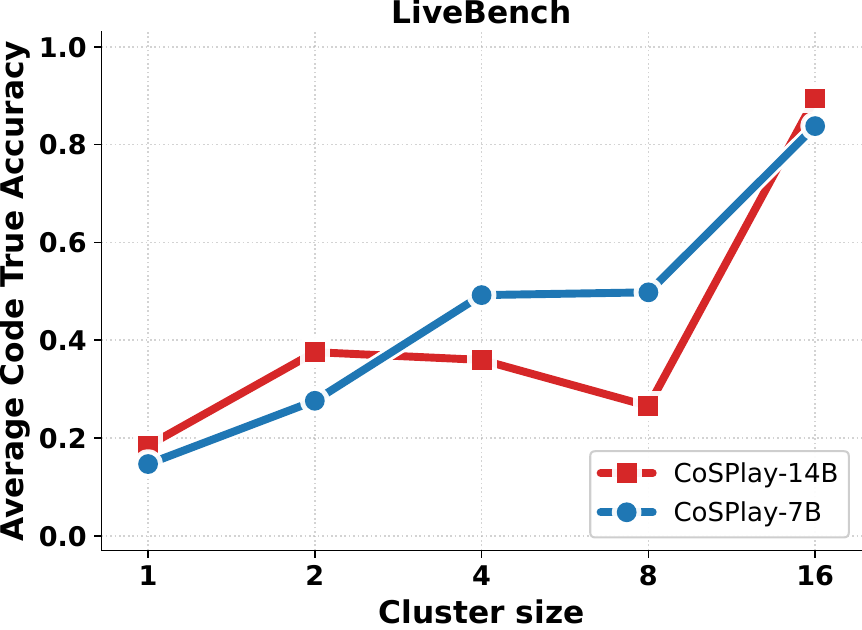}\hfill%
    \includegraphics[width=0.27\textwidth,trim=0 0 0 0pt,clip]{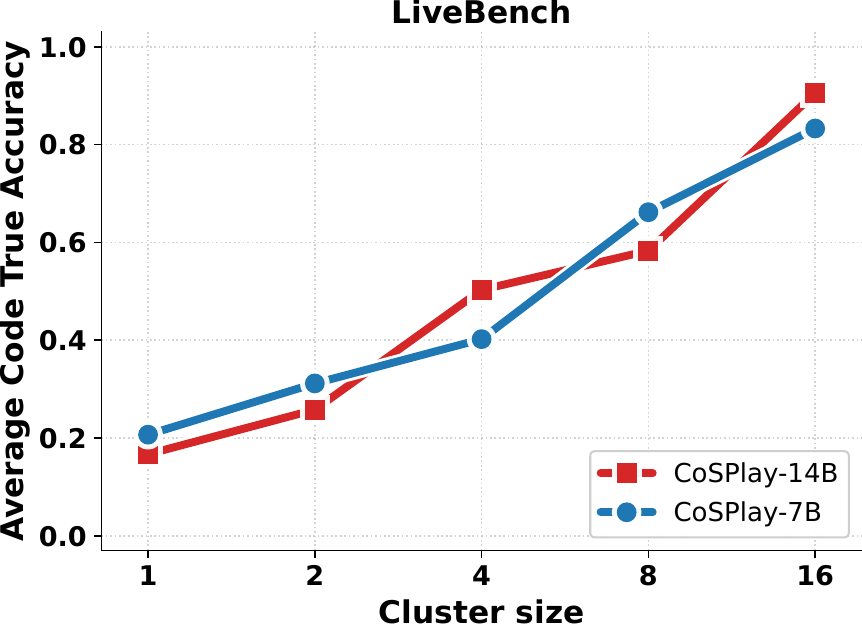}\hfill%
    \includegraphics[width=0.27\textwidth,trim=0 0 0 0pt,clip]{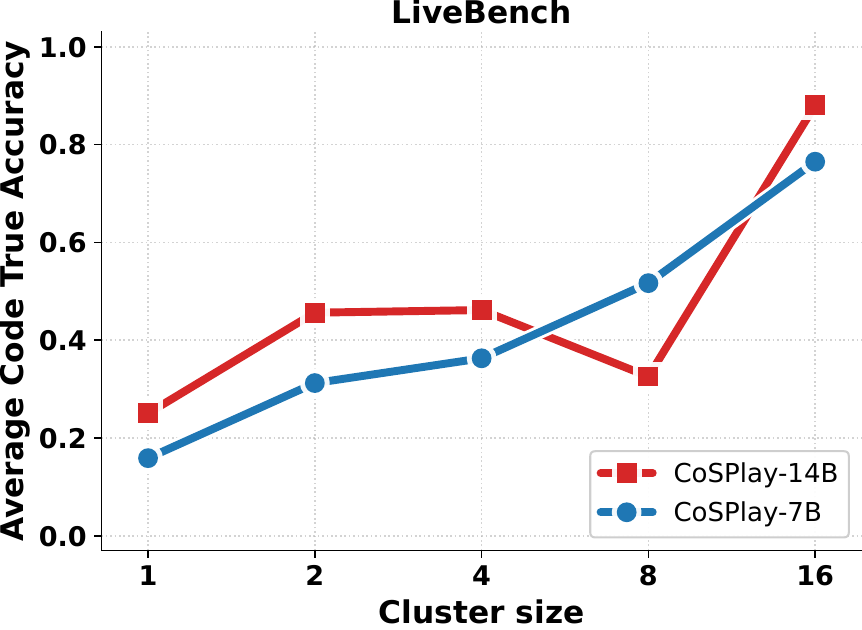}


    \includegraphics[width=0.27\textwidth,trim=0 0 0 0pt,clip]{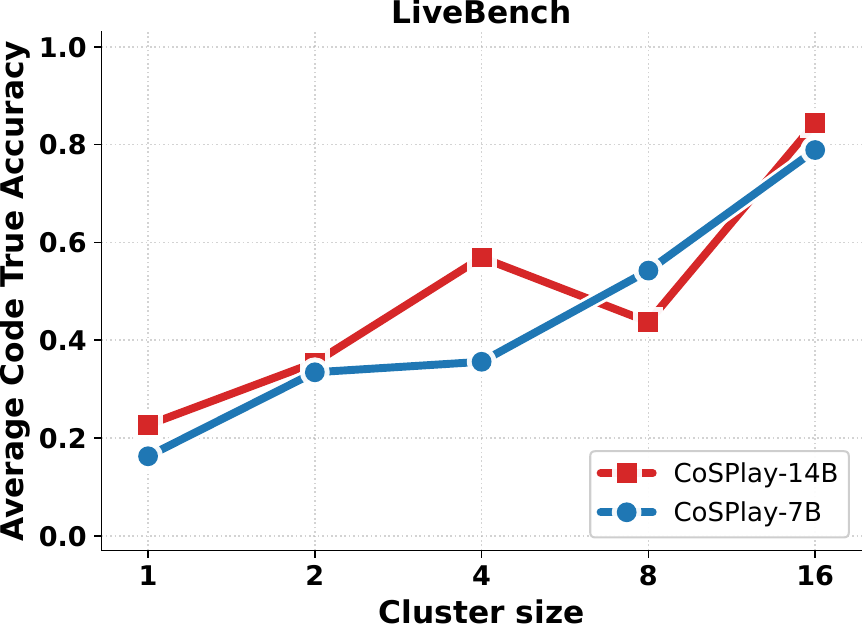}\hfill%
    \includegraphics[width=0.27\textwidth,trim=0 0 0 0pt,clip]{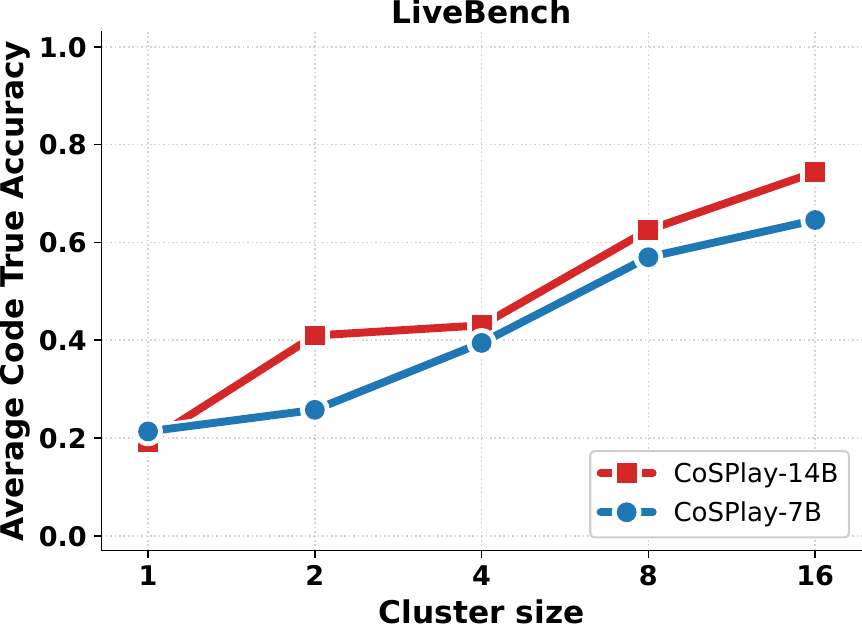}\hfill%
    \includegraphics[width=0.27\textwidth,trim=0 0 0 0pt,clip]{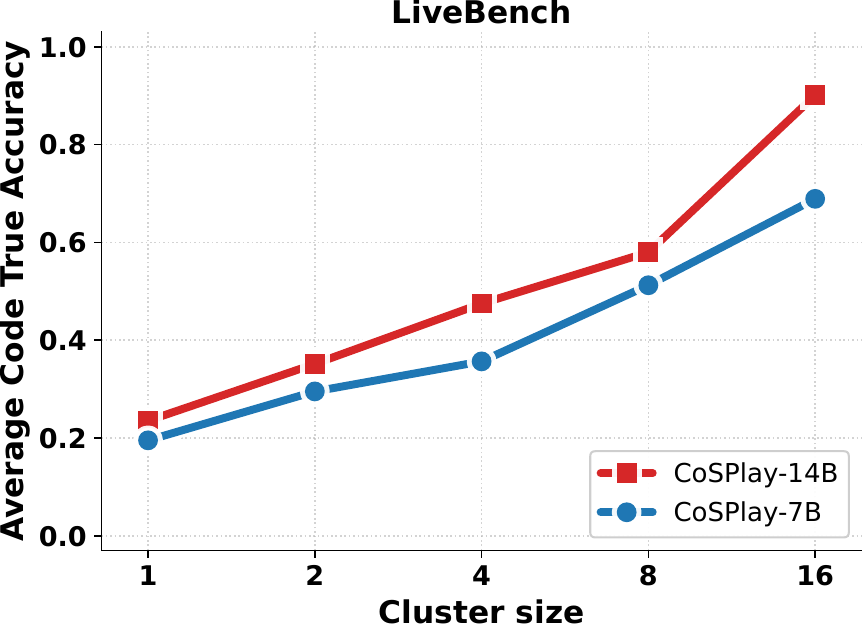}


    \caption{The relationship between cluster size and average code true accuracy during self-play on LiveBench for both 7B and 14B models. The top row shows Round 0-2, and the bottom row shows Round 3-5.}
    \label{fig:sub_livebench_cluster_size_vs_code_true_acc}
\end{figure}

\vspace{-2.5em}

\begin{figure}[H]
    \centering
    \includegraphics[width=0.27\textwidth,trim=0 0 0 0pt,clip]{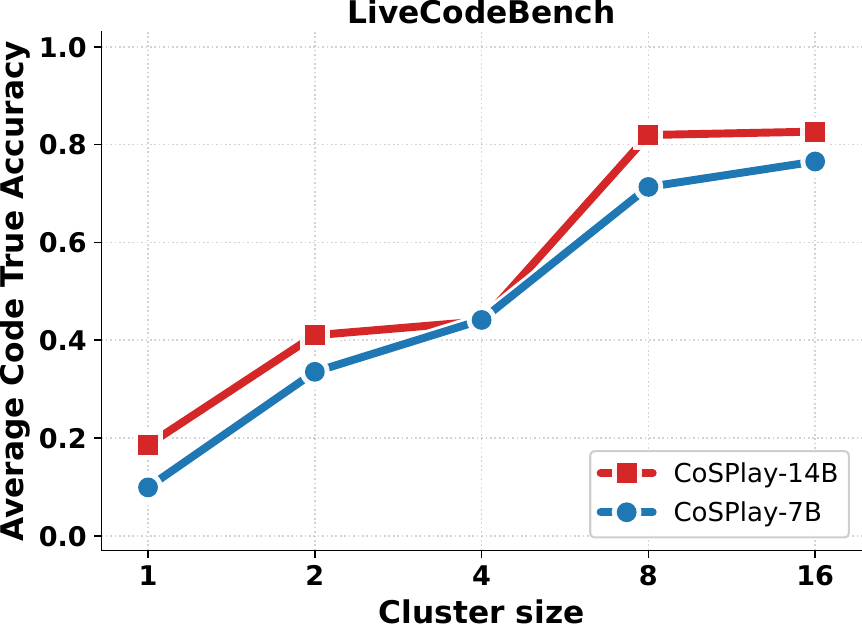}\hfill%
    \includegraphics[width=0.27\textwidth,trim=0 0 0 0pt,clip]{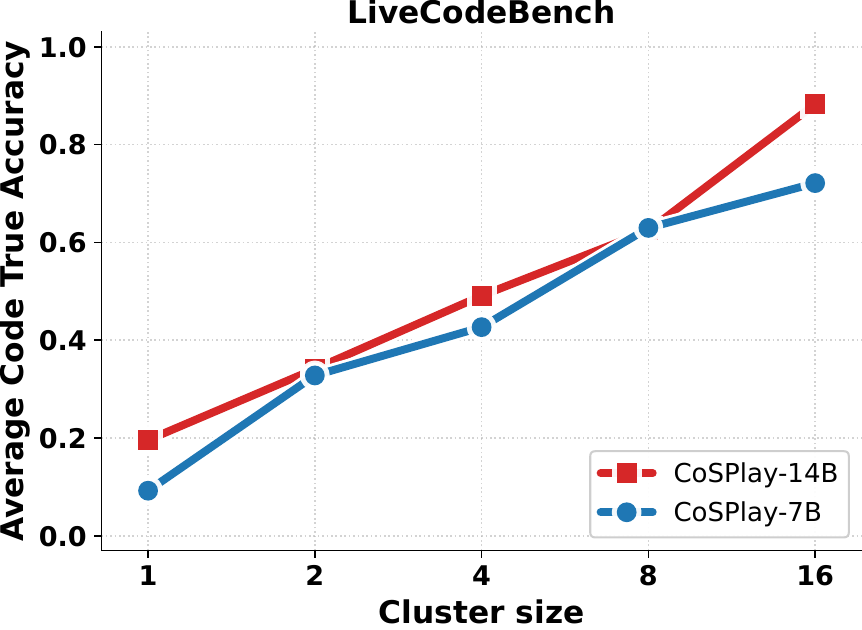}\hfill%
    \includegraphics[width=0.27\textwidth,trim=0 0 0 0pt,clip]{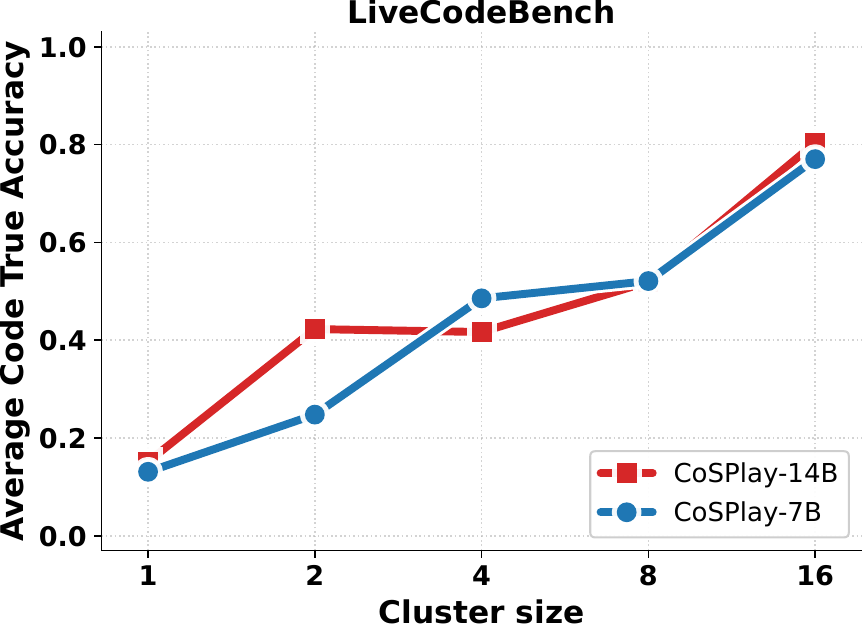}


    \includegraphics[width=0.27\textwidth,trim=0 0 0 0pt,clip]{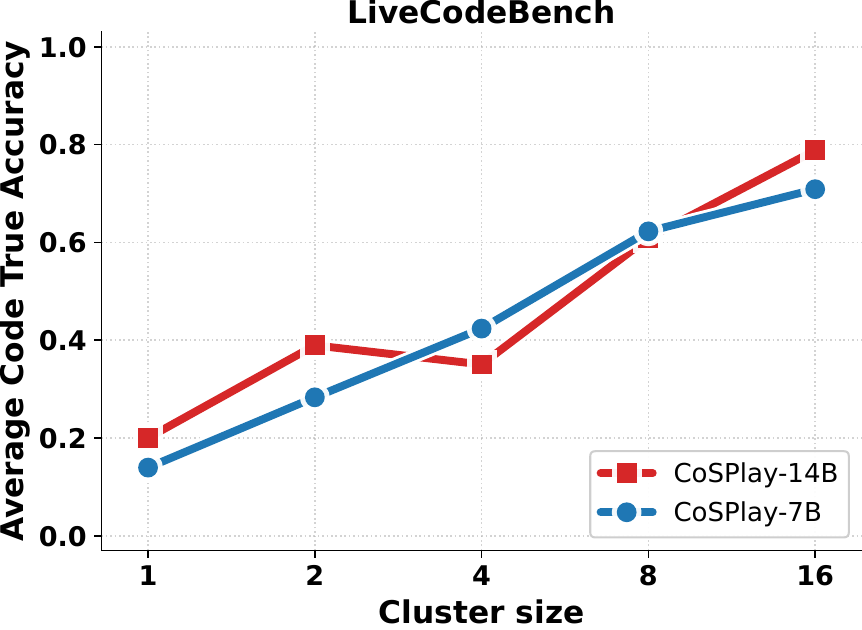}\hfill%
    \includegraphics[width=0.27\textwidth,trim=0 0 0 0pt,clip]{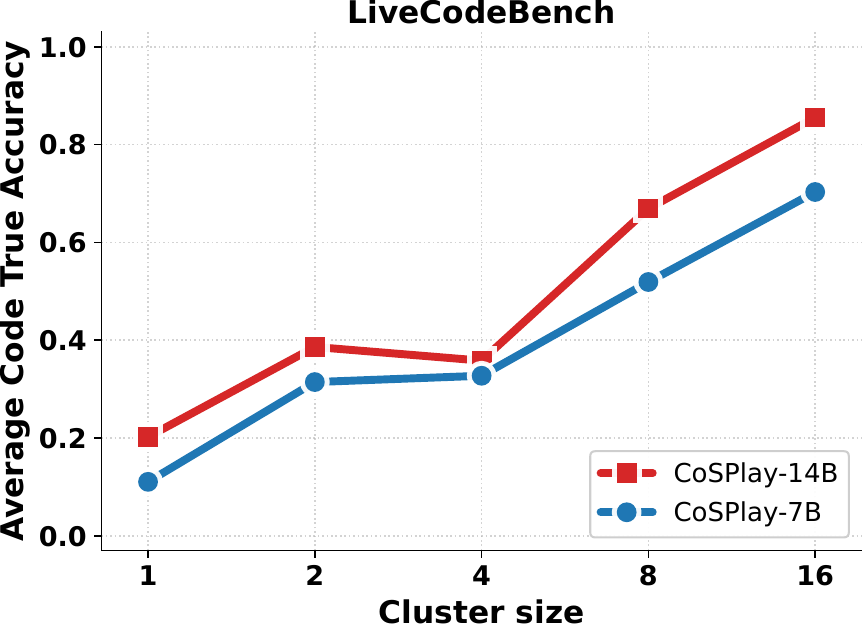}\hfill%
    \includegraphics[width=0.27\textwidth,trim=0 0 0 0pt,clip]{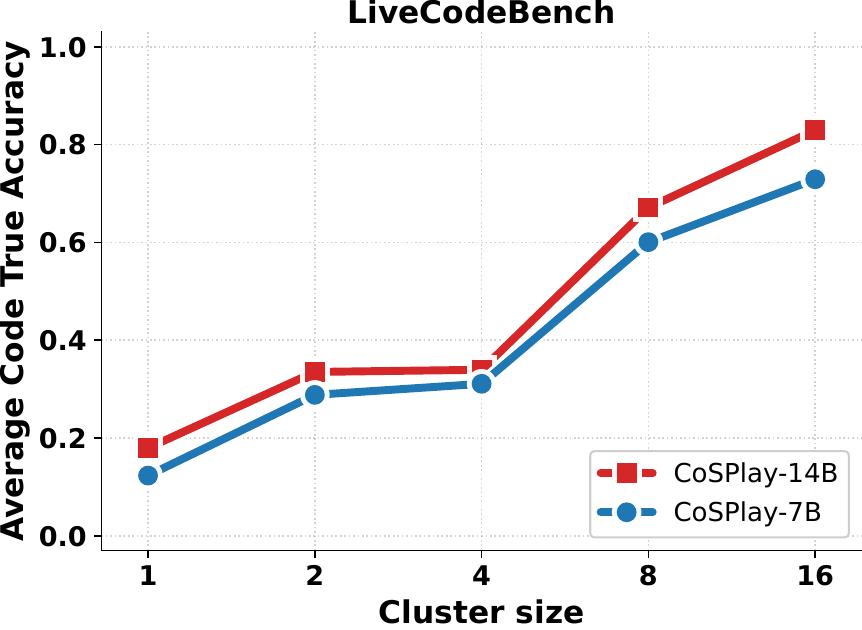}


    \caption{The relationship between cluster size and average code true accuracy during self-play on LiveCodeBench for both 7B and 14B models. The top row shows Round 0-2, and the bottom row shows Round 3-5.}
    \label{fig:sub_livecodebench_cluster_size_vs_code_true_acc}
\end{figure}

\section{Detailed Breakdown of the Relationship between UT Pass Count and True Accuracy}
\label{app: detailed UT pass count on generated codes and true average accuarcy}

This section provides the detailed data supporting the correlation analysis presented in the main text. Figure~\ref{fig:sub_codecontests_acc_14b}-\ref{fig:sub_livecodebench_acc_14b} display the relationship between UT pass counts and their true accuracy for each dataset individually. 

We observe that the strong positive correlation discussed in the main text is not an artifact of aggregation; rather, it holds consistently across all benchmarks. As the pass count increases, the reliability of the generated UTs improves uniformly across all tasks. Additionally, the performance gap between the 7B and 14B models is evident in these detailed plots, where the 14B model consistently achieves higher accuracy rates for the same pass counts. This detailed view reinforces the stability of using pass count as a filtering metric across diverse coding challenges.
\begin{figure}[H] 
    \centering
    \includegraphics[width=0.27\textwidth,trim=0 0 0 0pt,clip]{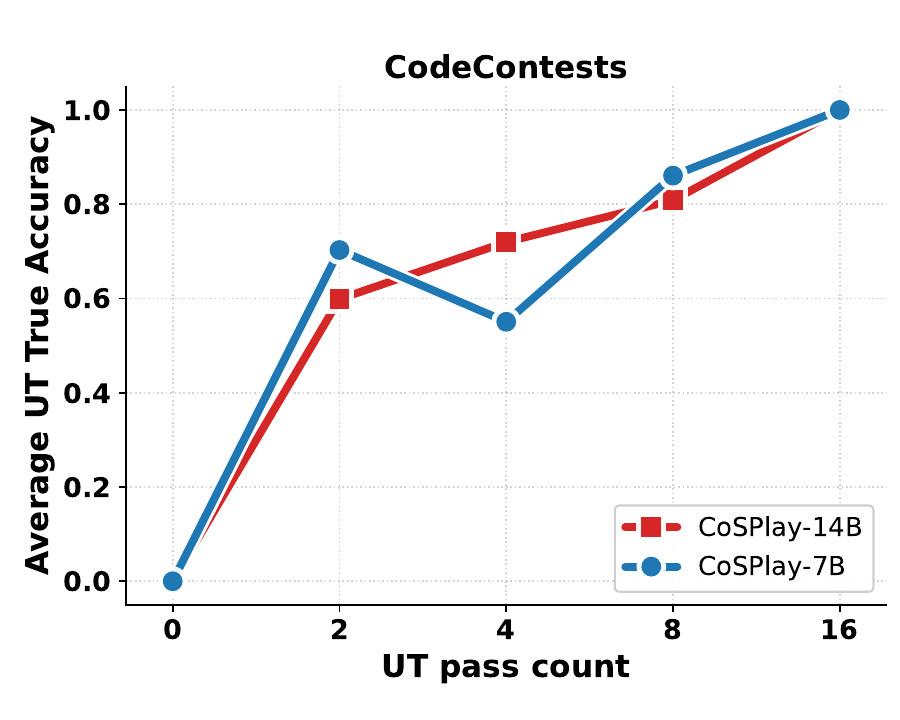}\hfill%
    \includegraphics[width=0.27\textwidth,trim=0 0 0 0pt,clip]{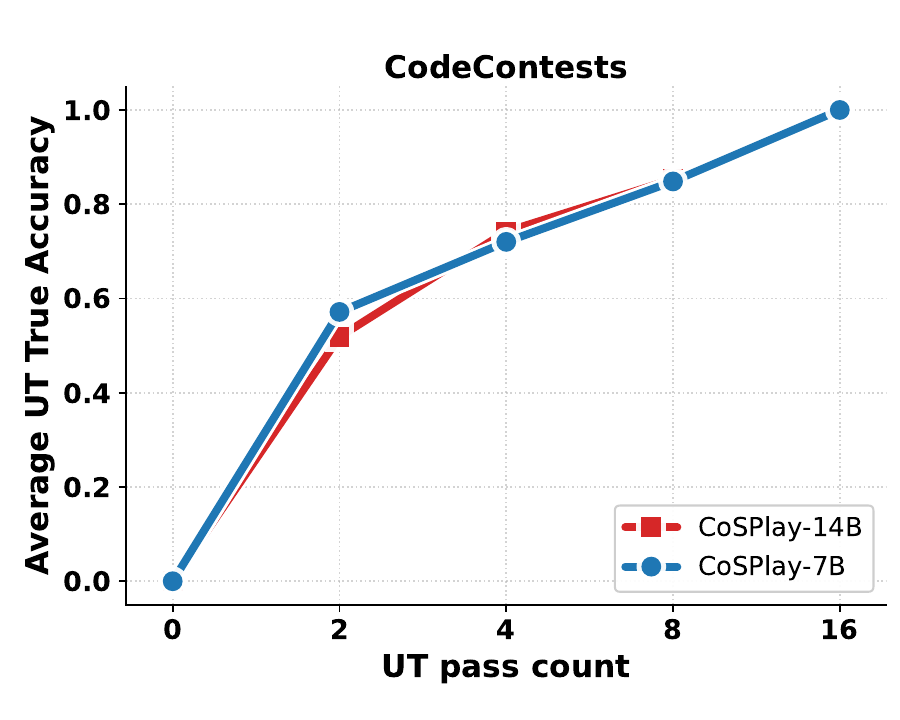}\hfill%
    \includegraphics[width=0.27\textwidth,trim=0 0 0 0pt,clip]{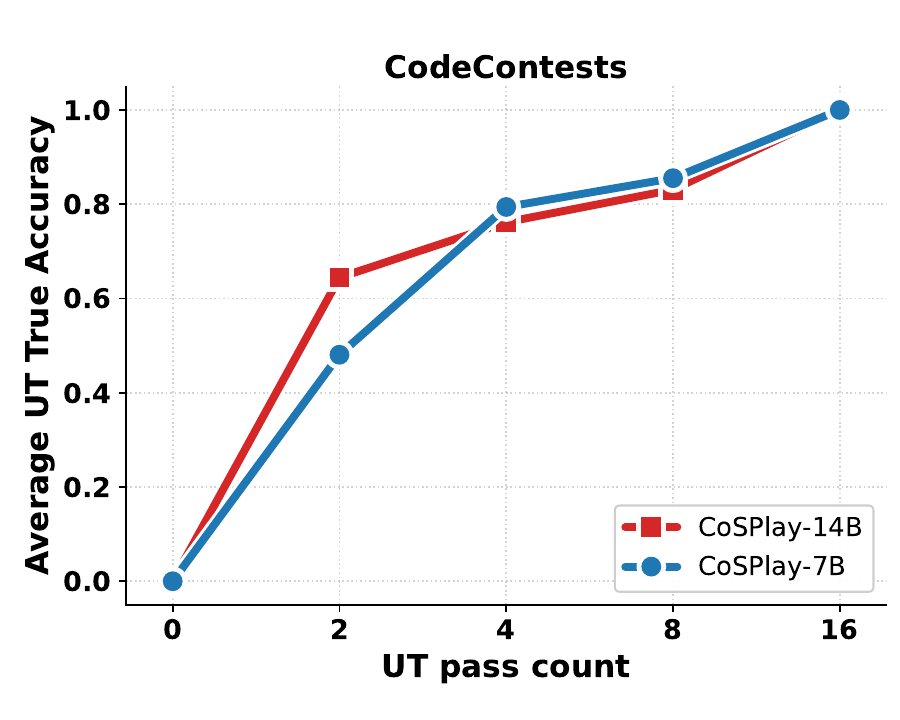}
    
    \vspace{-0.5em} 
    
    \includegraphics[width=0.27\textwidth,trim=0 0 0 0pt,clip]{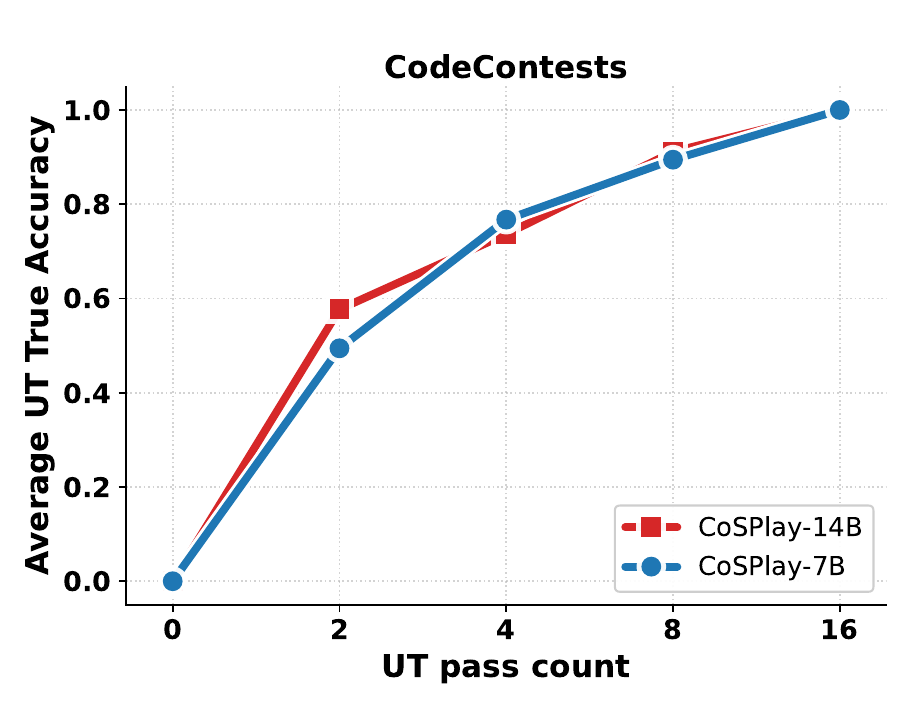}\hfill%
    \includegraphics[width=0.27\textwidth,trim=0 0 0 0pt,clip]{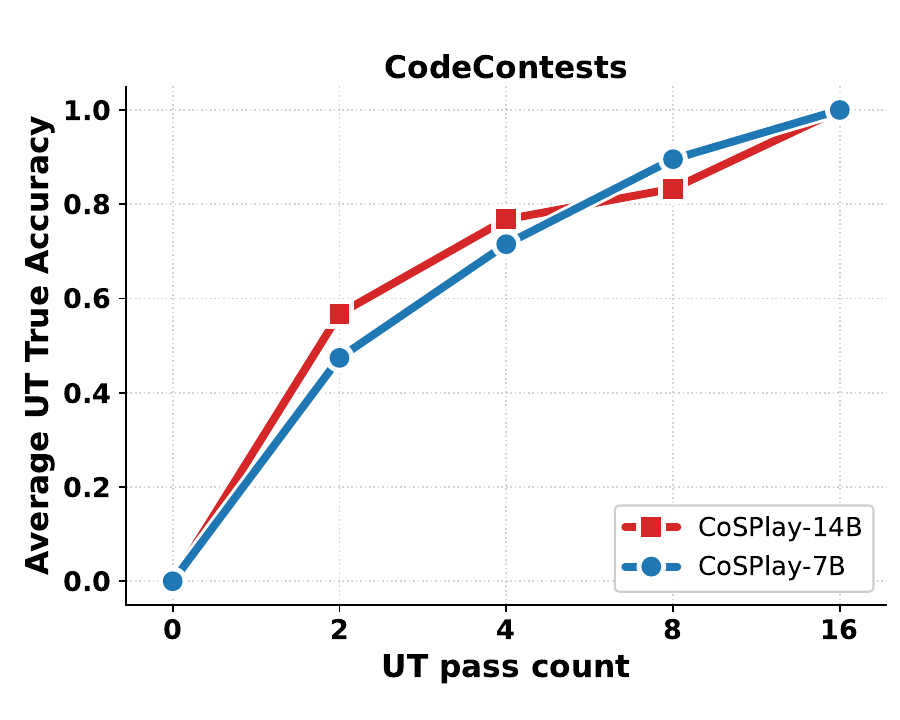}\hfill%
    \includegraphics[width=0.27\textwidth,trim=0 0 0 0pt,clip]{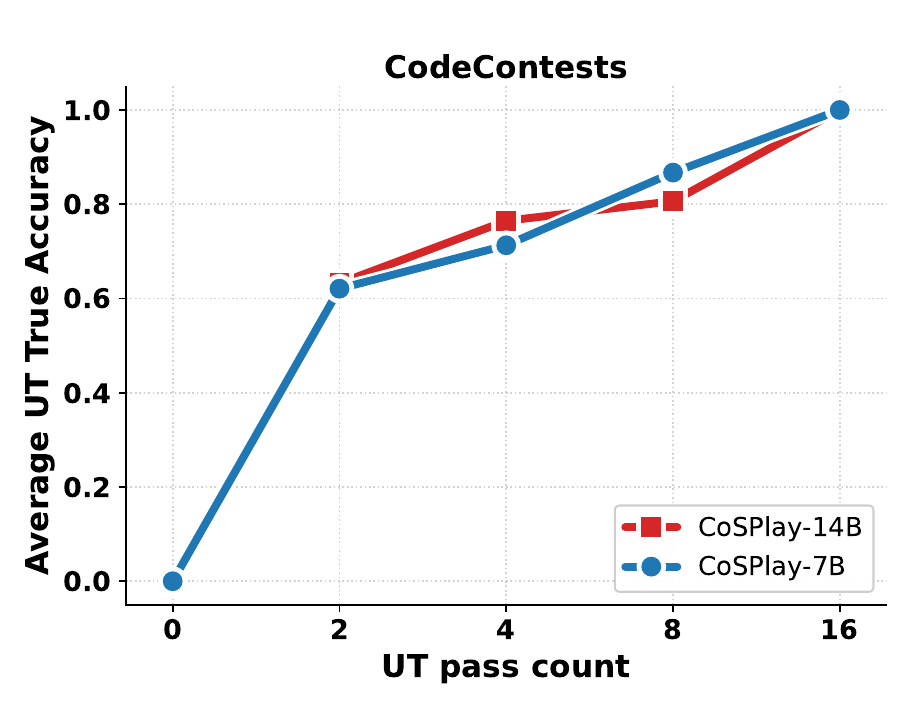}
    \vspace{-0.5em} 
    
    \caption{The relationship between UT pass counts on generated codes and average true accuracy for both 7B and 14B models on CodeContests. The top row shows Round 0-2, and the bottom row shows Round 3-5.}
    \label{fig:sub_codecontests_acc_14b}
\end{figure}

    \vspace{-2.5em} 
\begin{figure}[H] 
    \centering
    \includegraphics[width=0.27\textwidth,trim=0 0 0 0pt,clip]{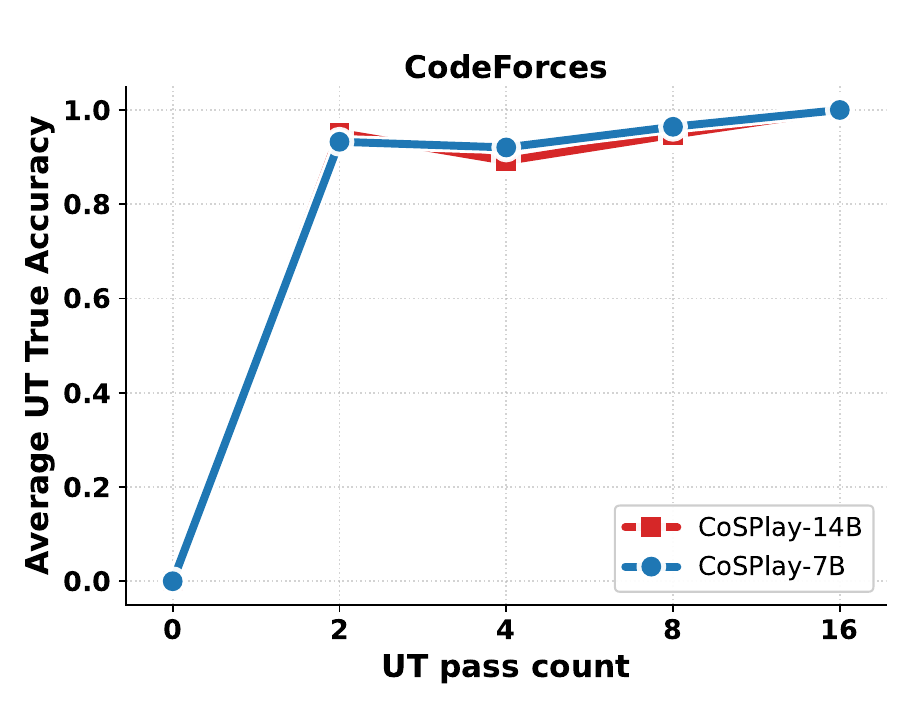}\hfill%
    \includegraphics[width=0.27\textwidth,trim=0 0 0 0pt,clip]{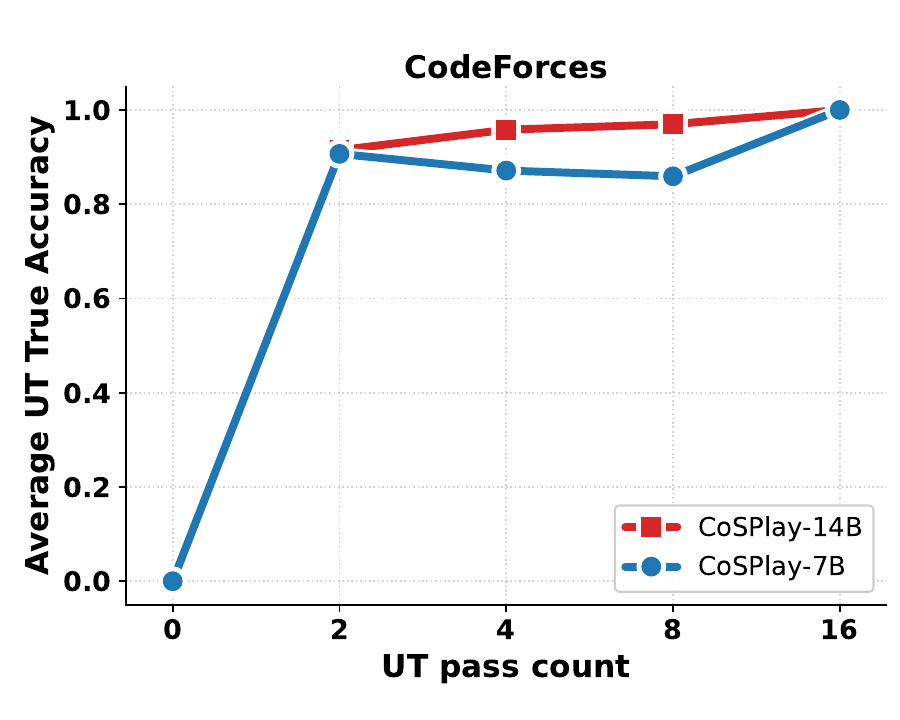}\hfill%
    \includegraphics[width=0.27\textwidth,trim=0 0 0 0pt,clip]{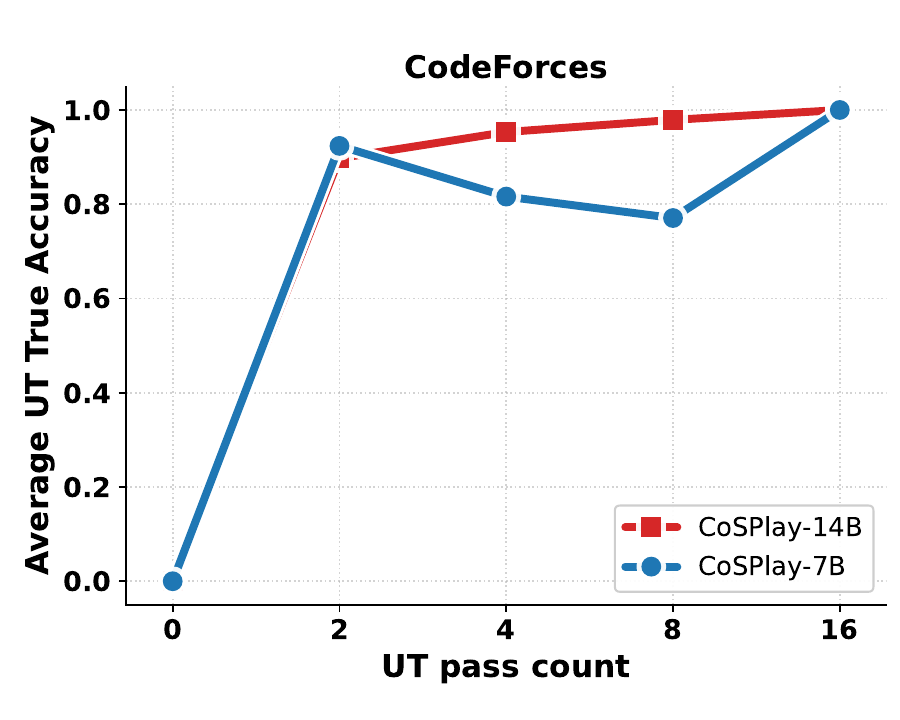}
    
    \vspace{-0.5em} 
    
    \includegraphics[width=0.27\textwidth,trim=0 0 0 0pt,clip]{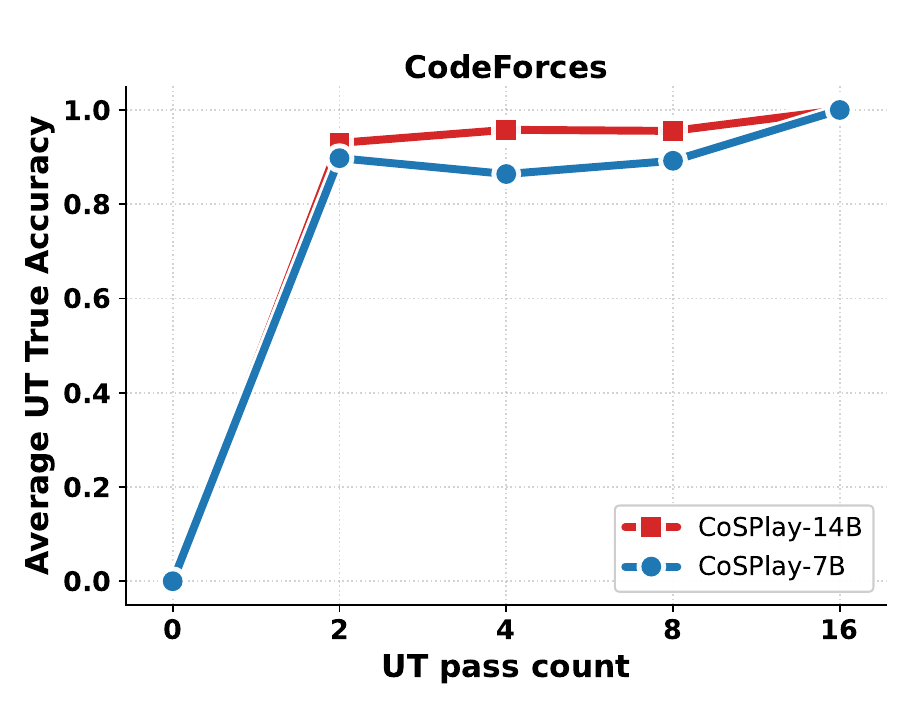}\hfill%
    \includegraphics[width=0.27\textwidth,trim=0 0 0 0pt,clip]{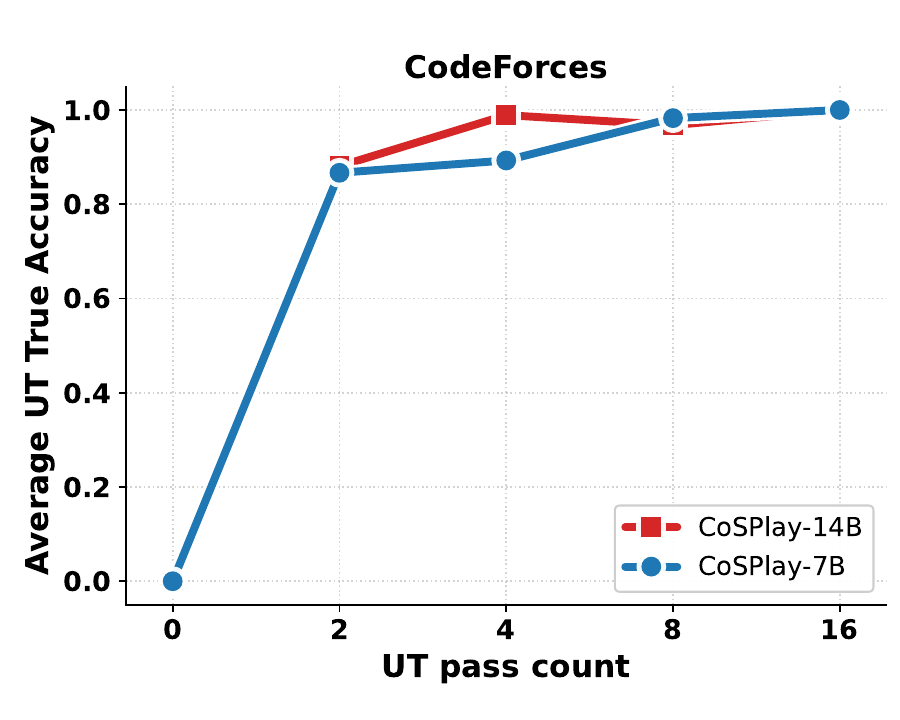}\hfill%
    \includegraphics[width=0.27\textwidth,trim=0 0 0 0pt,clip]{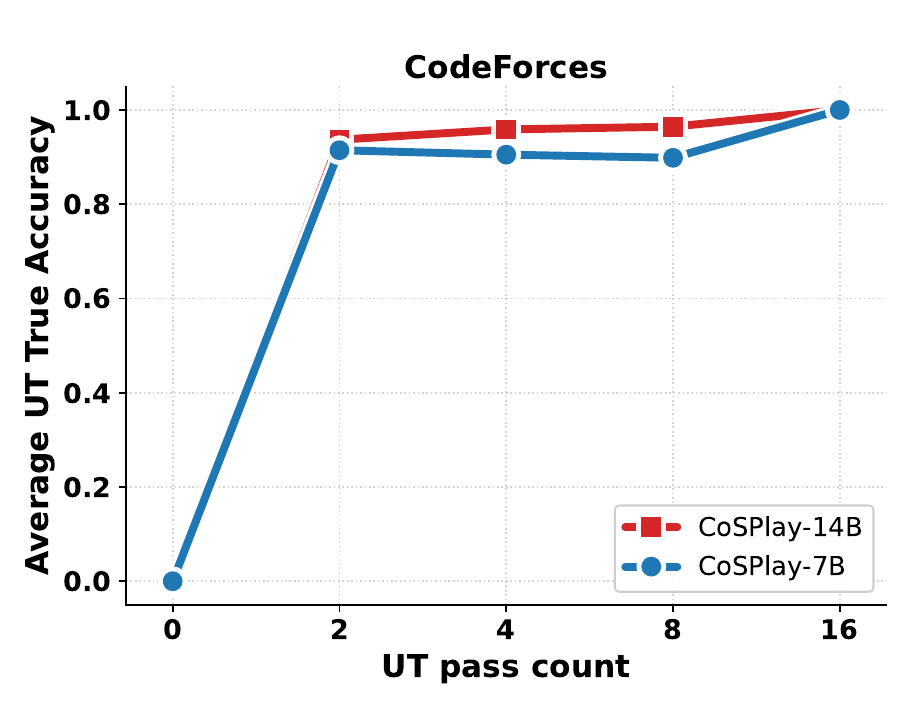}
    \vspace{-0.5em} 
    \caption{The relationship between UT pass counts on generated codes and average true accuracy for both 7B and 14B models on CodeForces. The top row shows Round 0-2, and the bottom row shows Round 3-5.}
    \label{fig:sub_codeforces_acc_14b}
\end{figure}
    \vspace{-2.5em} 

\begin{figure}[H]
    \centering
    \includegraphics[width=0.27\textwidth,trim=0 0 0 0pt,clip]{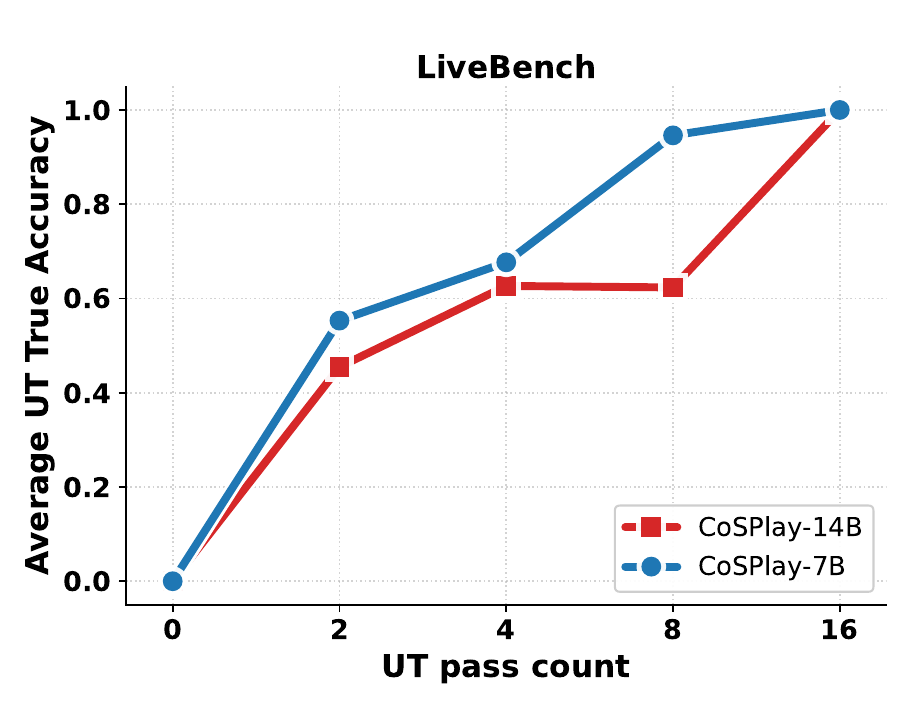}\hfill%
    \includegraphics[width=0.27\textwidth,trim=0 0 0 0pt,clip]{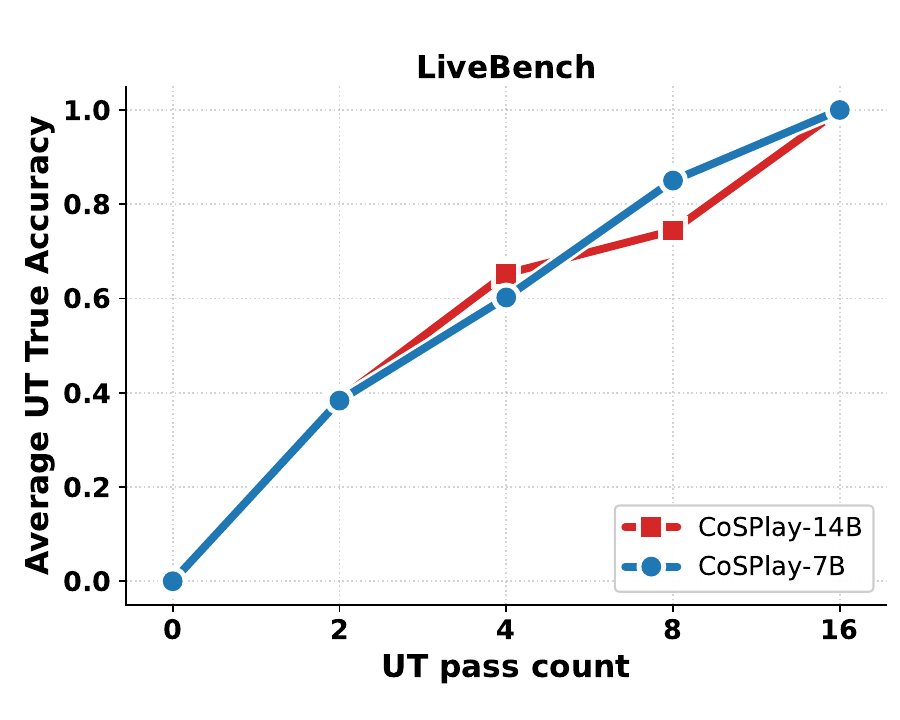}\hfill%
    \includegraphics[width=0.27\textwidth,trim=0 0 0 0pt,clip]{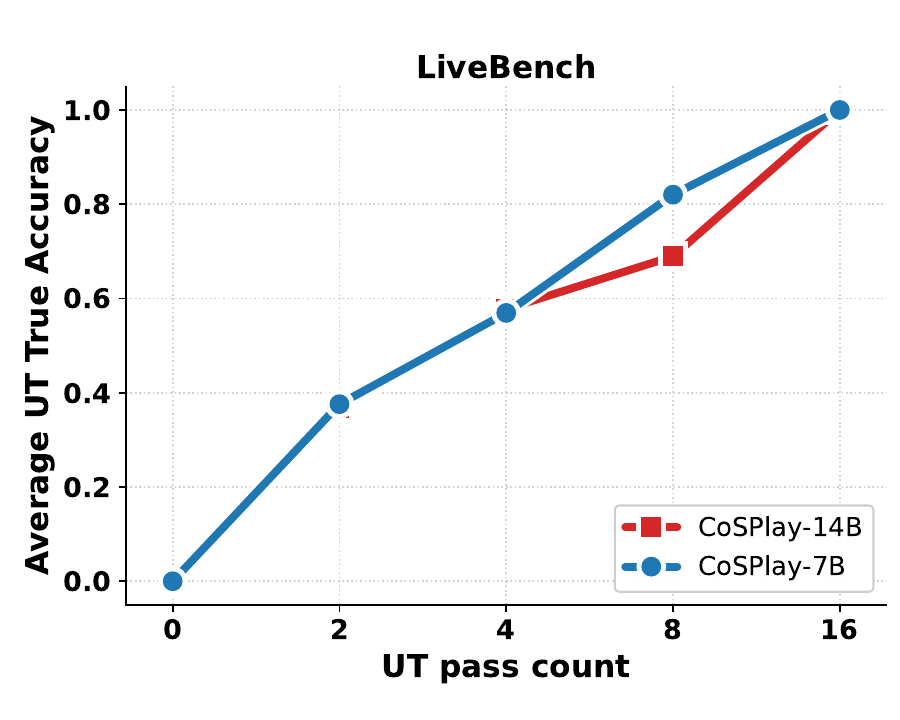}
    
    \vspace{-0.5em} 
    
    \includegraphics[width=0.27\textwidth,trim=0 0 0 0pt,clip]{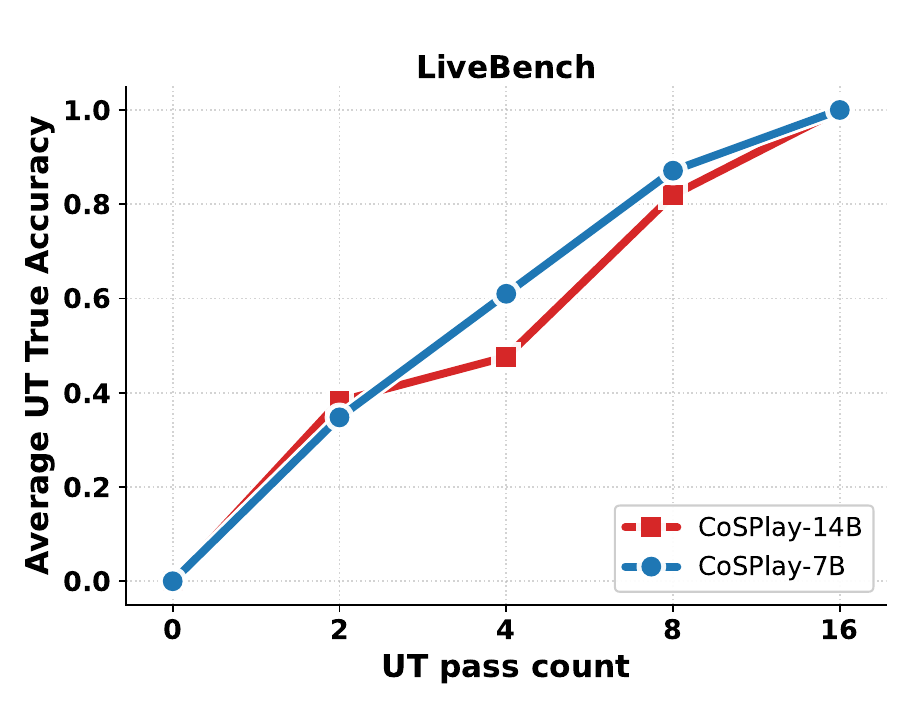}\hfill%
    \includegraphics[width=0.27\textwidth,trim=0 0 0 0pt,clip]{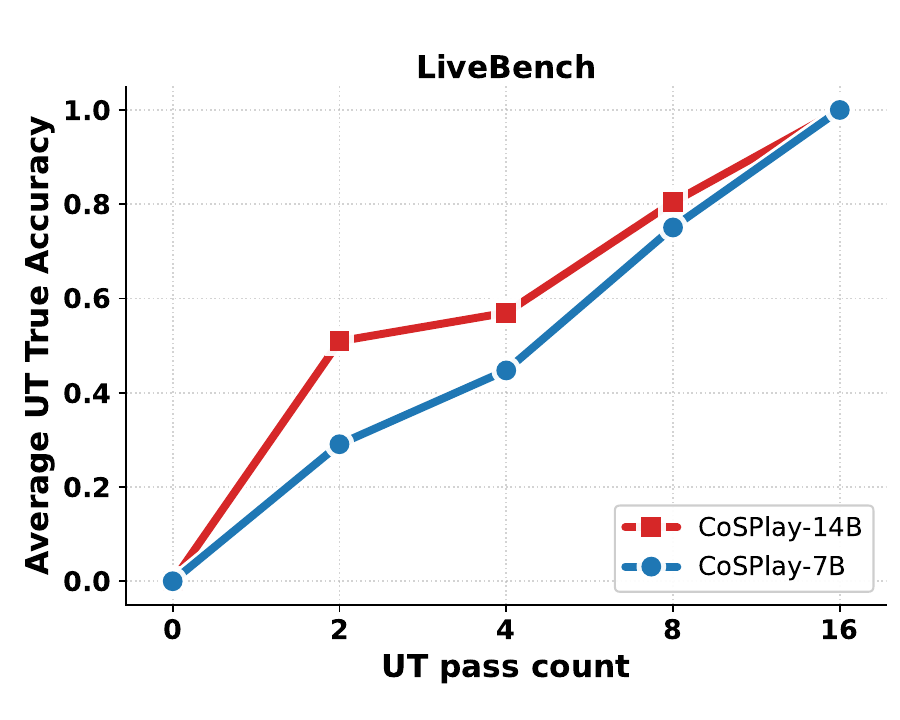}\hfill%
    \includegraphics[width=0.27\textwidth,trim=0 0 0 0pt,clip]{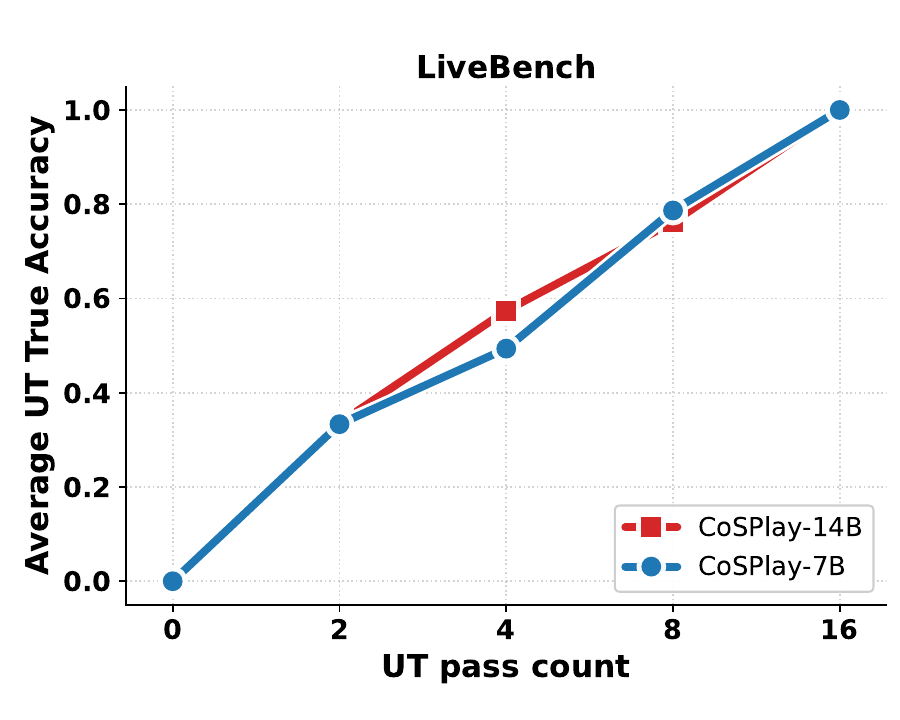}
    \vspace{-0.5em} 

    \caption{The relationship between UT pass counts on generated codes and average true accuracy for both 7B and 14B models on LiveBench. The top row shows Round 0-2, and the bottom row shows Round 3-5.}
    \label{fig:sub_livebench_acc_14b}
\end{figure}

\begin{figure}[H]
    \centering
    \includegraphics[width=0.27\textwidth,trim=0 0 0 0pt,clip]{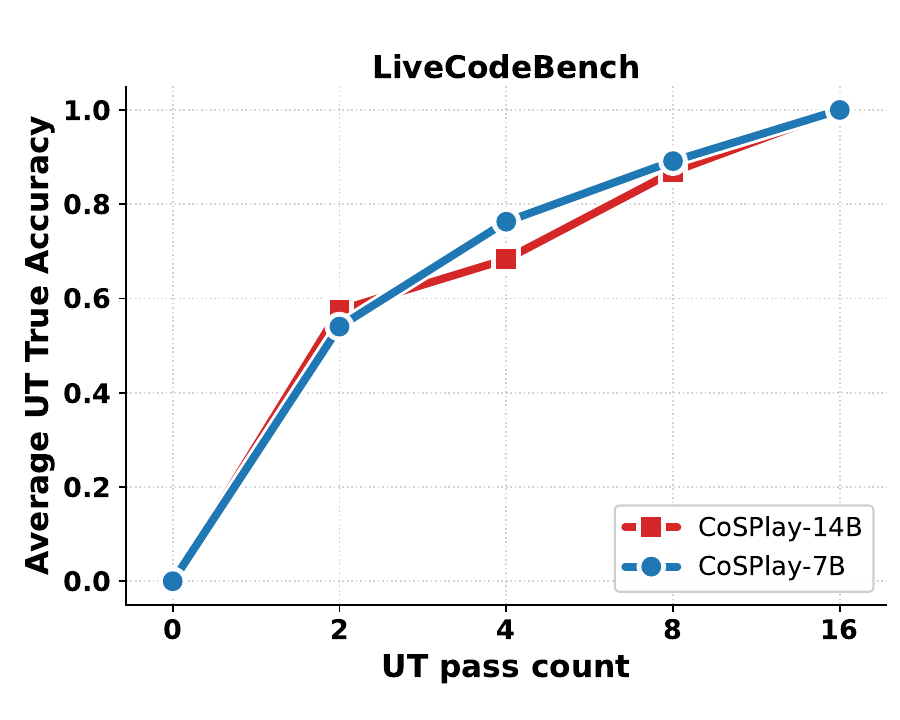}\hfill%
    \includegraphics[width=0.27\textwidth,trim=0 0 0 0pt,clip]{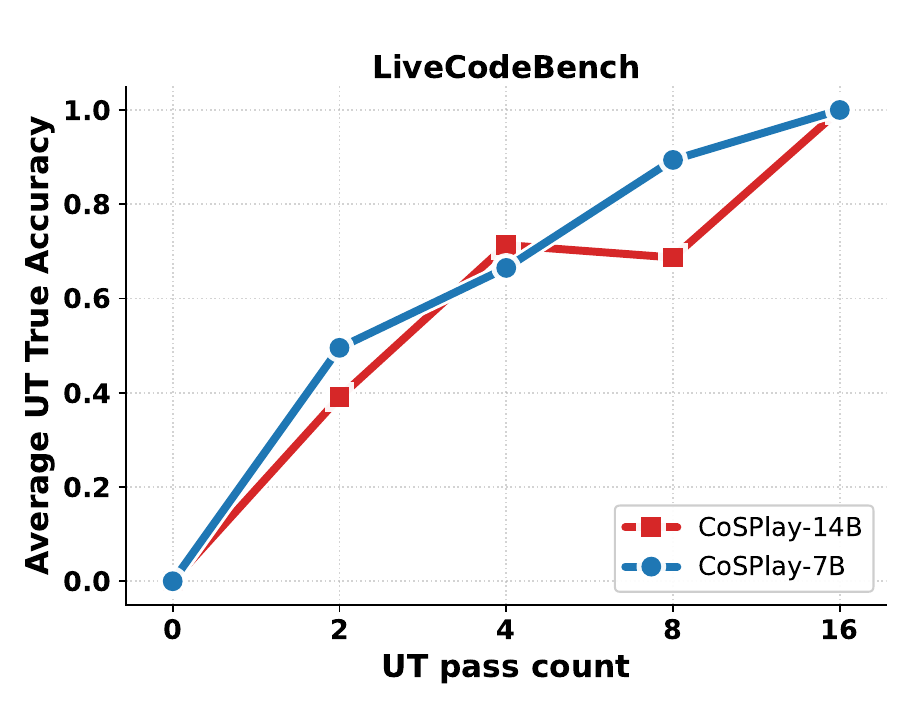}\hfill%
    \includegraphics[width=0.27\textwidth,trim=0 0 0 0pt,clip]{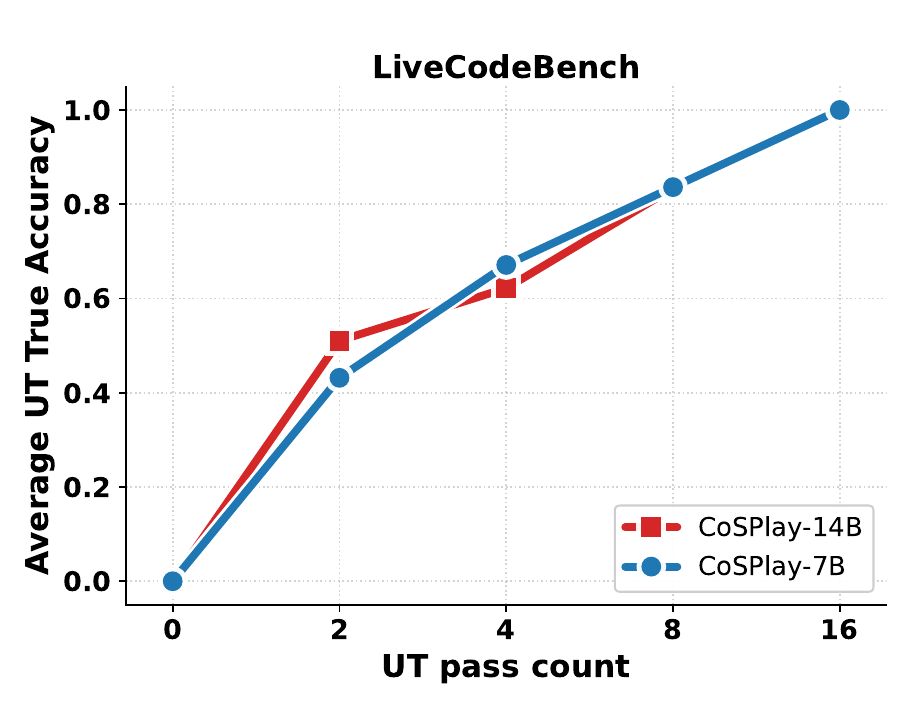}
    
    \vspace{-0.5em} 
    
    \includegraphics[width=0.27\textwidth,trim=0 0 0 0pt,clip]{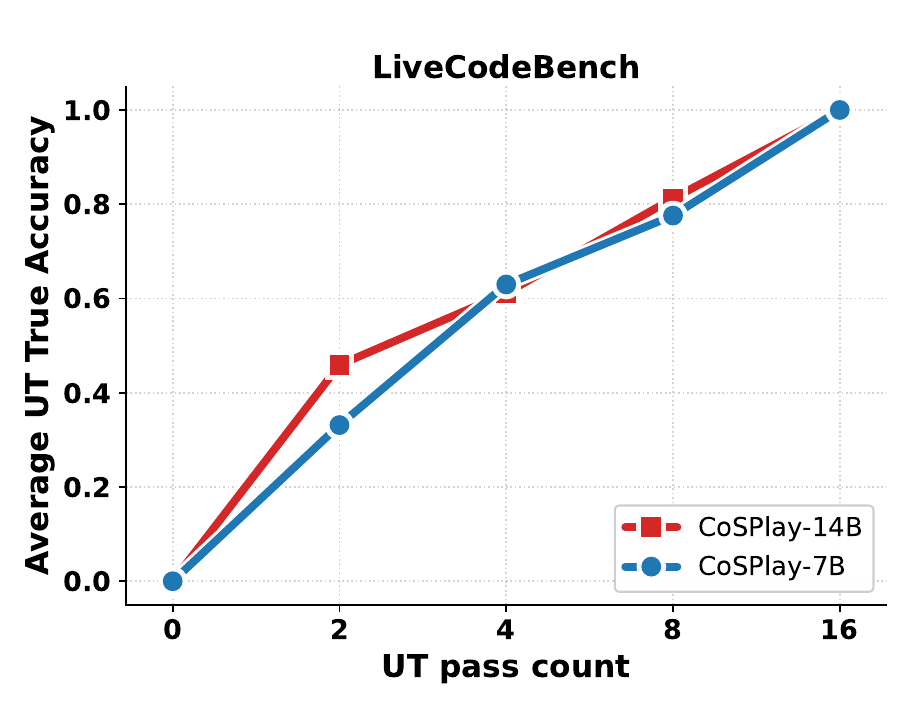}\hfill%
    \includegraphics[width=0.27\textwidth,trim=0 0 0 0pt,clip]{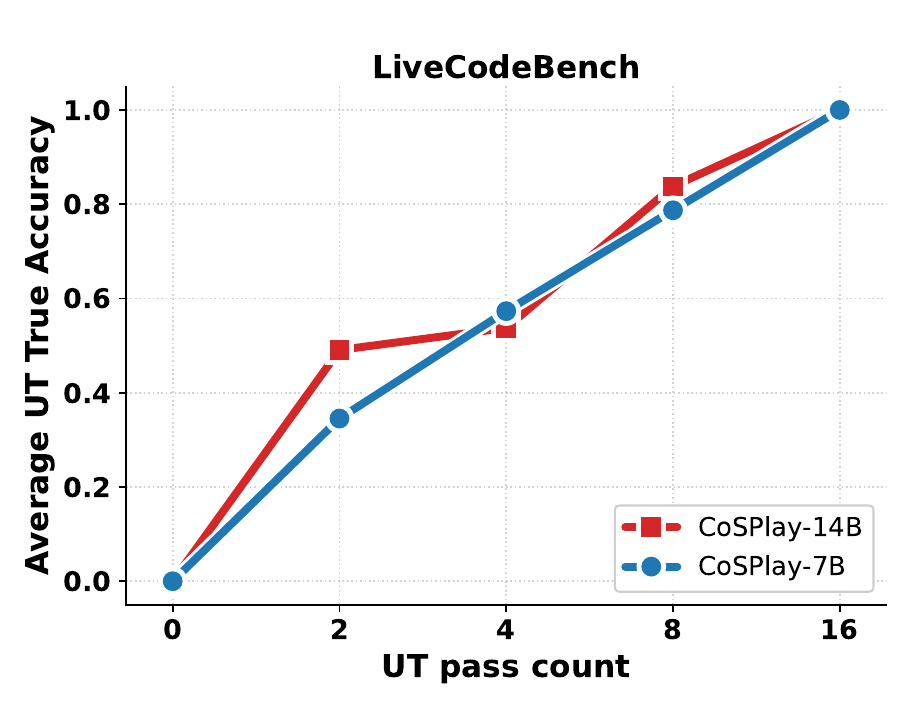}\hfill%
    \includegraphics[width=0.27\textwidth,trim=0 0 0 0pt,clip]{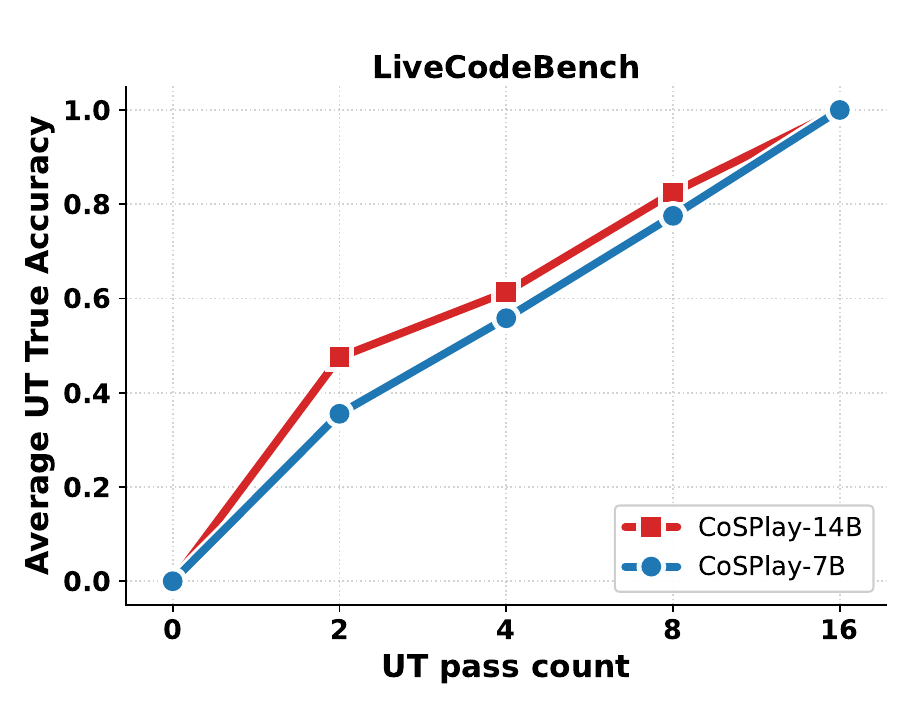}
    
    \caption{The relationship between UT pass counts on generated codes and average true accuracy for both 7B and 14B models on LiveCodeBench. The top row shows Round 0-2, and the bottom row shows Round 3-5.}
    \vspace{-0.5em} 
    \label{fig:sub_livecodebench_acc_14b}
\end{figure}

\section{Detailed Breakdown of the Relationship between Code Pass Count and True Accuracy}
\label{app: detailed Code pass count on generated codes and true average accuarcy}

This section provides the detailed data supporting the correlation analysis presented in the main text. Figure~\ref{fig:sub_code_codecontests_acc_14b}-\ref{fig:sub_code_livecodebench_acc_14b} display the relationship between code pass counts and their true accuracy for each dataset individually. 

We observe that the strong positive correlation discussed in the main text is not an artifact of aggregation; rather, it holds consistently across all benchmarks. As the pass count increases, the reliability of the generated UTs improves uniformly across all tasks. Additionally, the performance gap between the 7B and 14B models is evident in these detailed plots, where the 14B model consistently achieves higher accuracy rates for the same pass counts. This detailed view reinforces the stability of using pass count as a filtering metric across diverse coding challenges.

\begin{figure}[H] 
    \centering
    \includegraphics[width=0.27\textwidth,trim=0 10pt 0 10pt,clip]{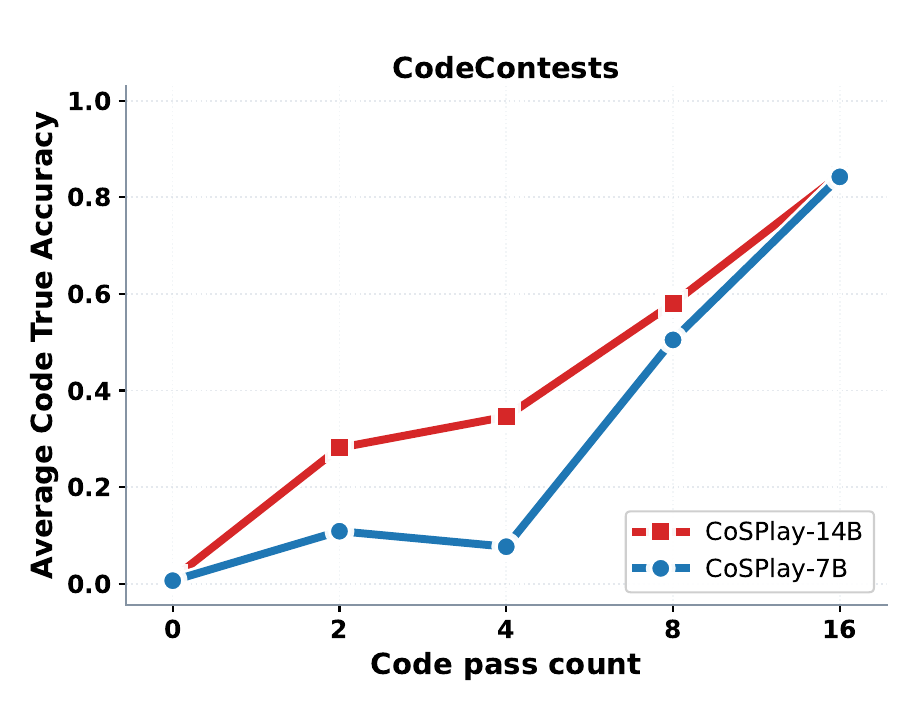}\hfill%
    \includegraphics[width=0.27\textwidth,trim=0 10pt 0 10pt,clip]{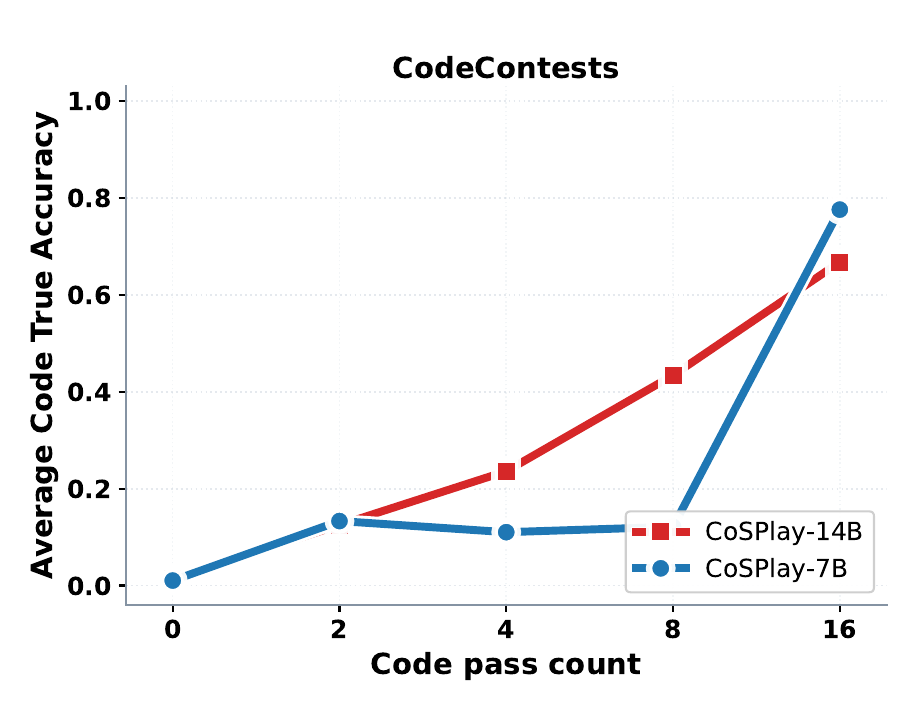}\hfill%
    \includegraphics[width=0.27\textwidth,trim=0 10pt 0 10pt,clip]{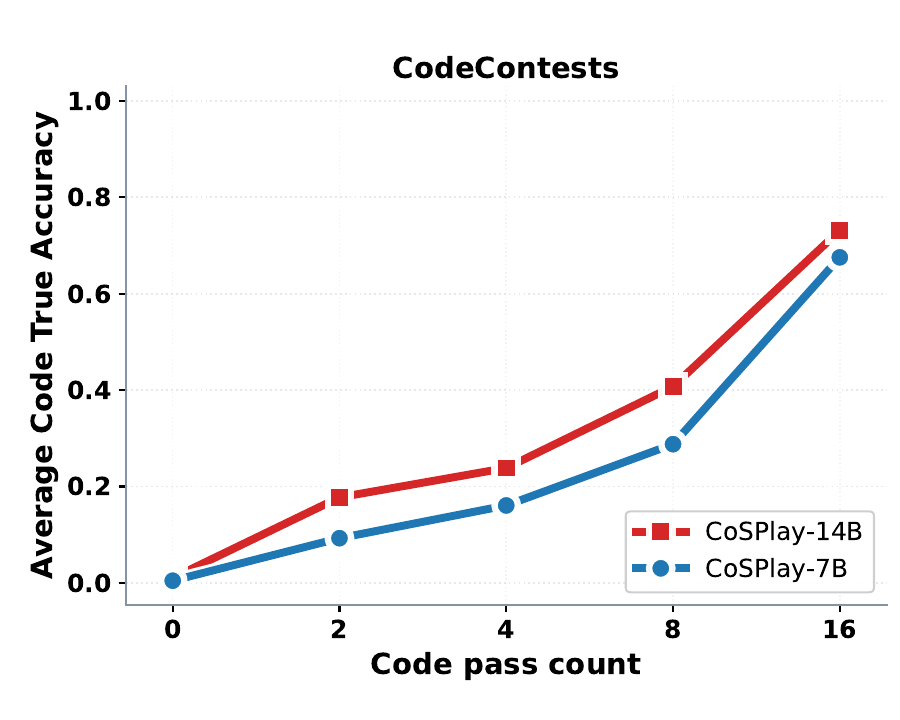}
    
    \vspace{-1.0em}
    
    \includegraphics[width=0.27\textwidth,trim=0 10pt 0 10pt,clip]{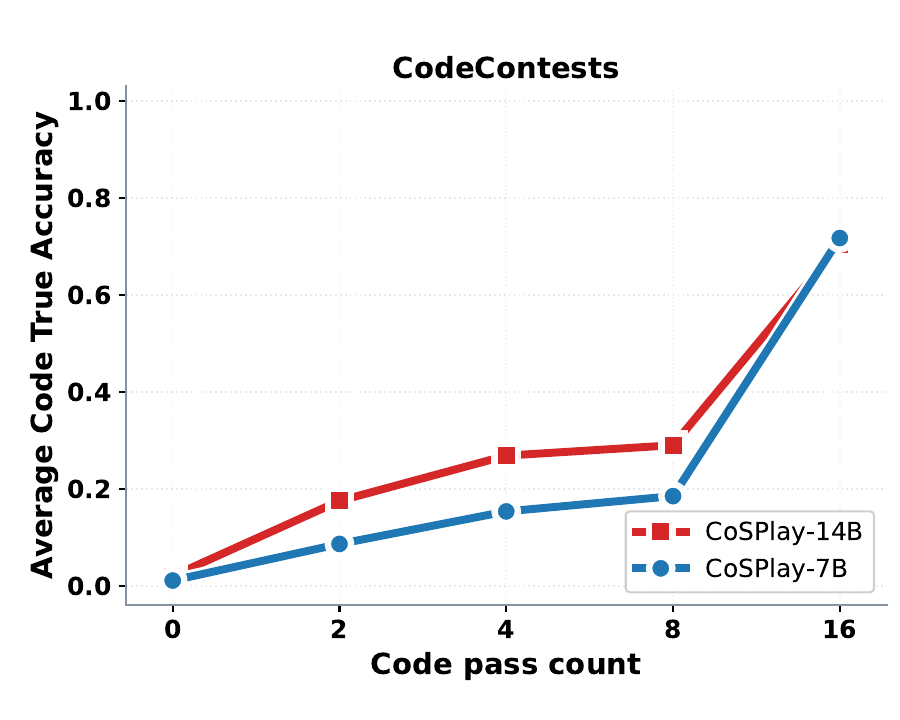}\hfill%
    \includegraphics[width=0.27\textwidth,trim=0 10pt 0 10pt,clip]{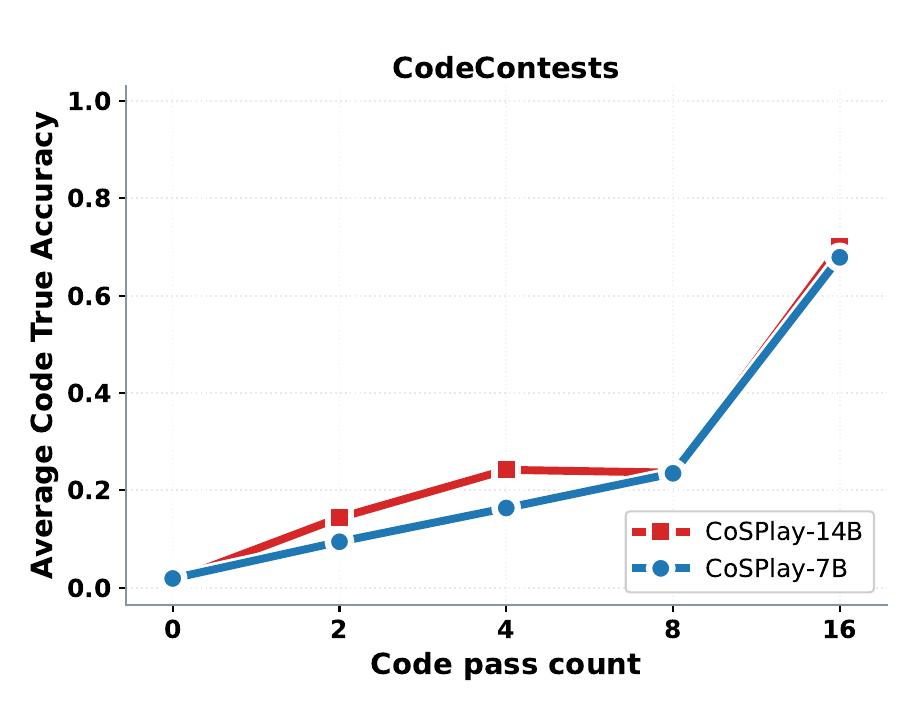}\hfill%
    \includegraphics[width=0.27\textwidth,trim=0 10pt 0 10pt,clip]{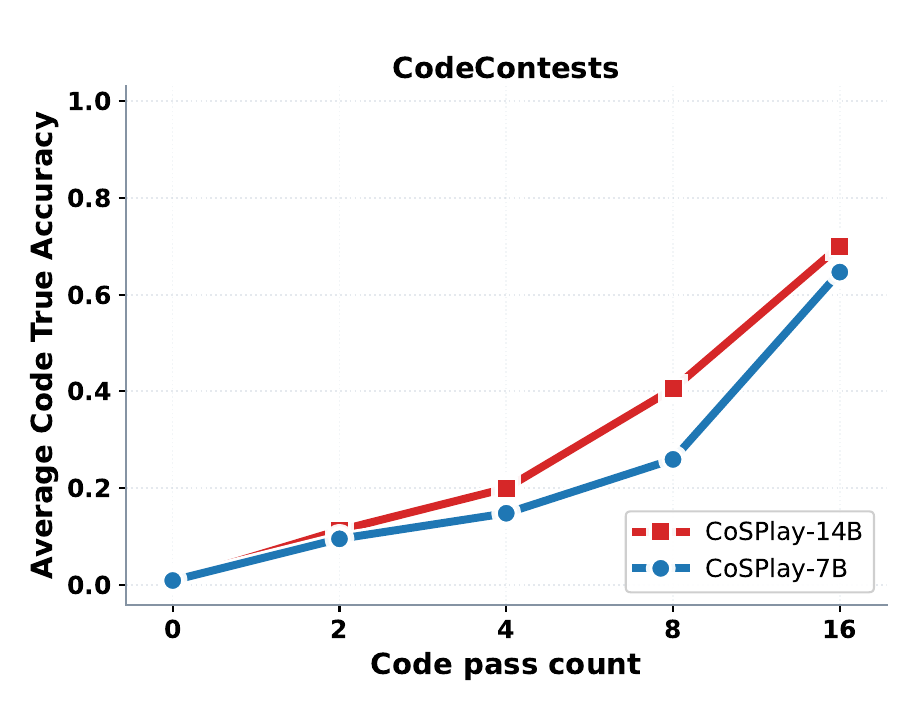}
    \vspace{-0.5em}
    
    \caption{The relationship between code pass counts and average true accuracy for both 7B and 14B models on CodeContests. The top row shows Round 0-2, and the bottom row shows Round 3-5.}
    \label{fig:sub_code_codecontests_acc_14b}
\end{figure}
\vspace{-1.5em}

\begin{figure}[H] 
    \centering
    \includegraphics[width=0.27\textwidth,trim=0 10pt 0 10pt,clip]{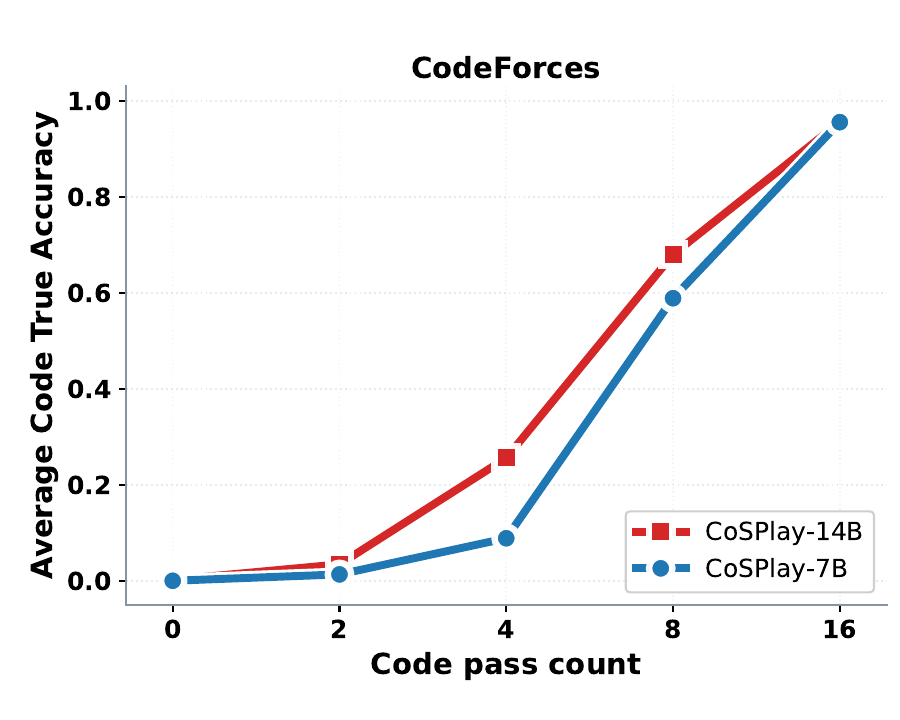}\hfill%
    \includegraphics[width=0.27\textwidth,trim=0 10pt 0 10pt,clip]{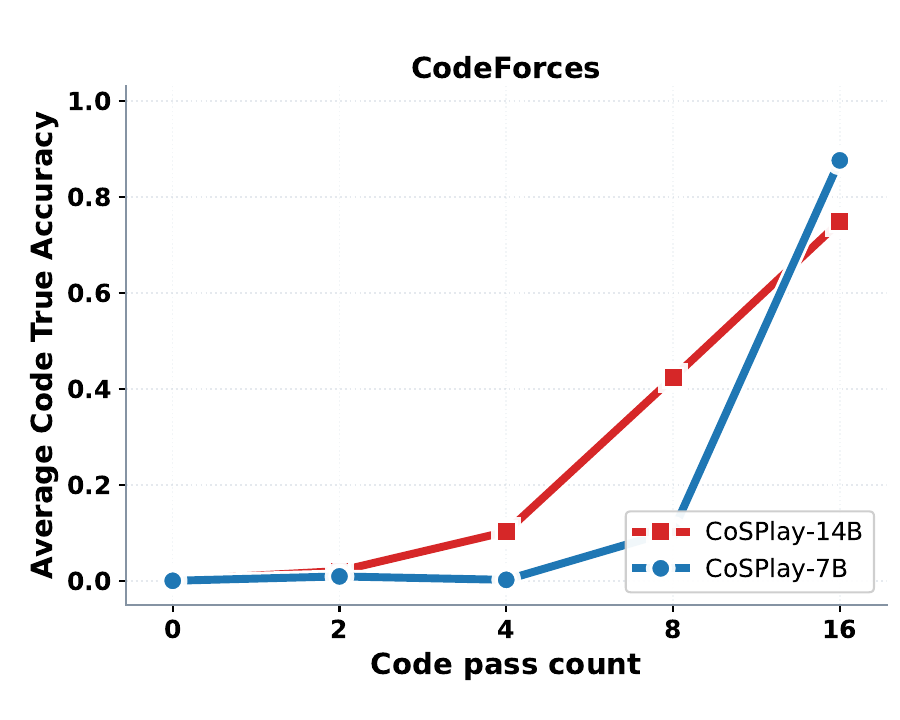}\hfill%
    \includegraphics[width=0.27\textwidth,trim=0 10pt 0 10pt,clip]{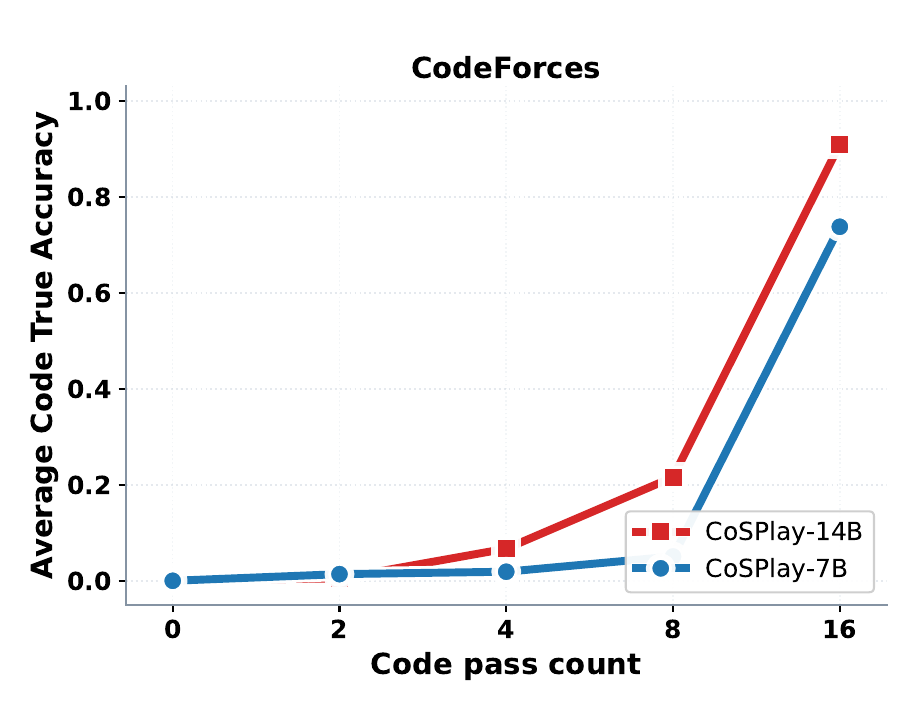}
    
    \vspace{-1.0em}
    
    \includegraphics[width=0.27\textwidth,trim=0 10pt 0 10pt,clip]{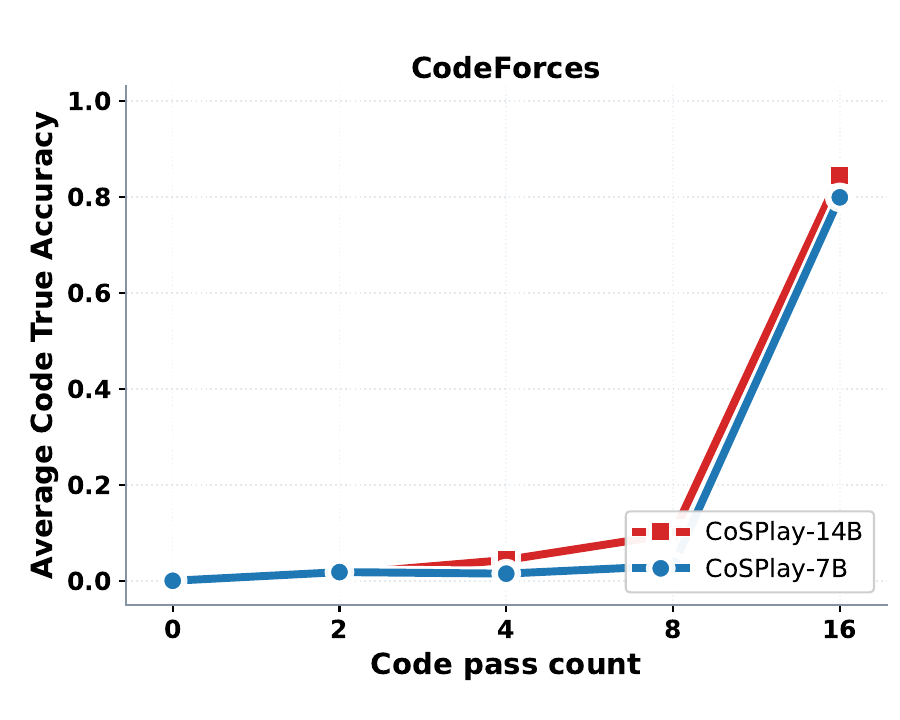}\hfill%
    \includegraphics[width=0.27\textwidth,trim=0 10pt 0 10pt,clip]{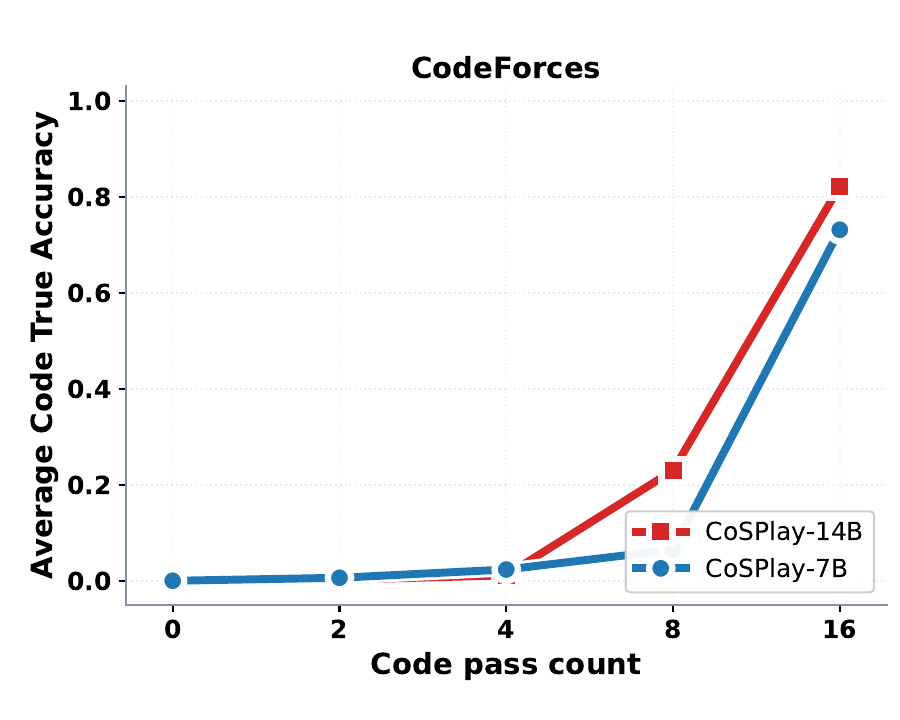}\hfill%
    \includegraphics[width=0.27\textwidth,trim=0 10pt 0 10pt,clip]{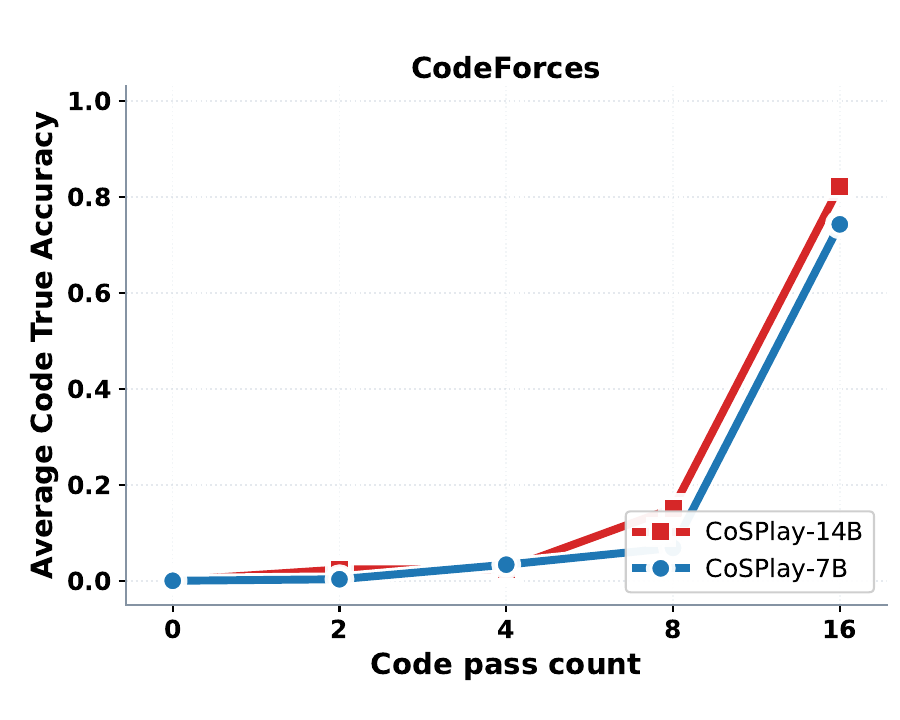}
    \vspace{-0.5em}
    
    \caption{The relationship between code pass counts and average true accuracy for both 7B and 14B models on CodeForces. The top row shows Round 0-2, and the bottom row shows Round 3-5.}
    \label{fig:sub_code_codeforces_acc_14b}
\end{figure}
\vspace{-1.5em}

\begin{figure}[H]
    \centering
    \includegraphics[width=0.27\textwidth,trim=0 10pt 0 10pt,clip]{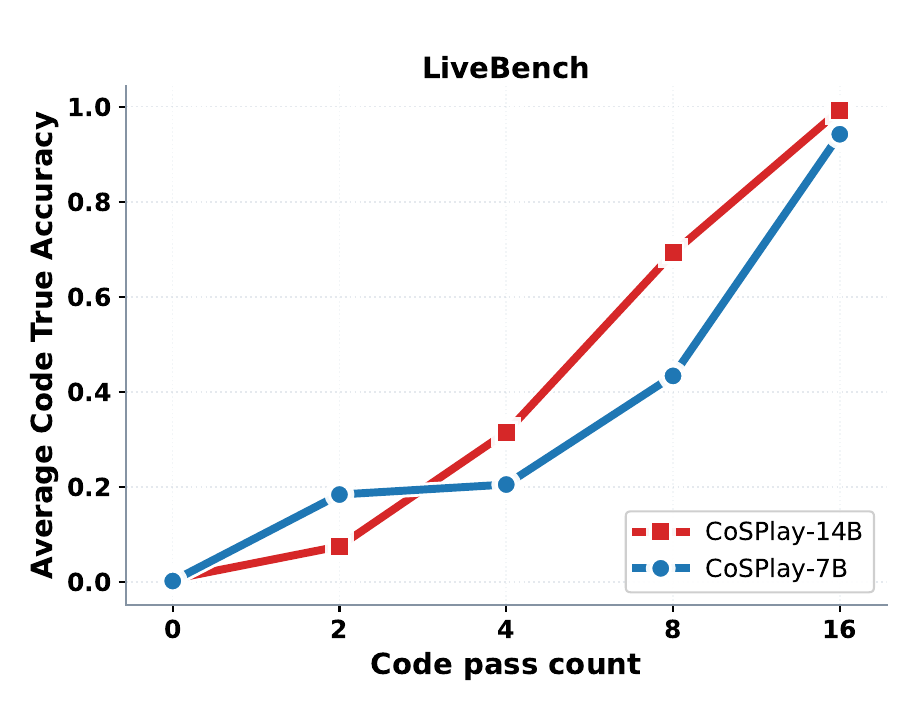}\hfill%
    \includegraphics[width=0.27\textwidth,trim=0 10pt 0 10pt,clip]{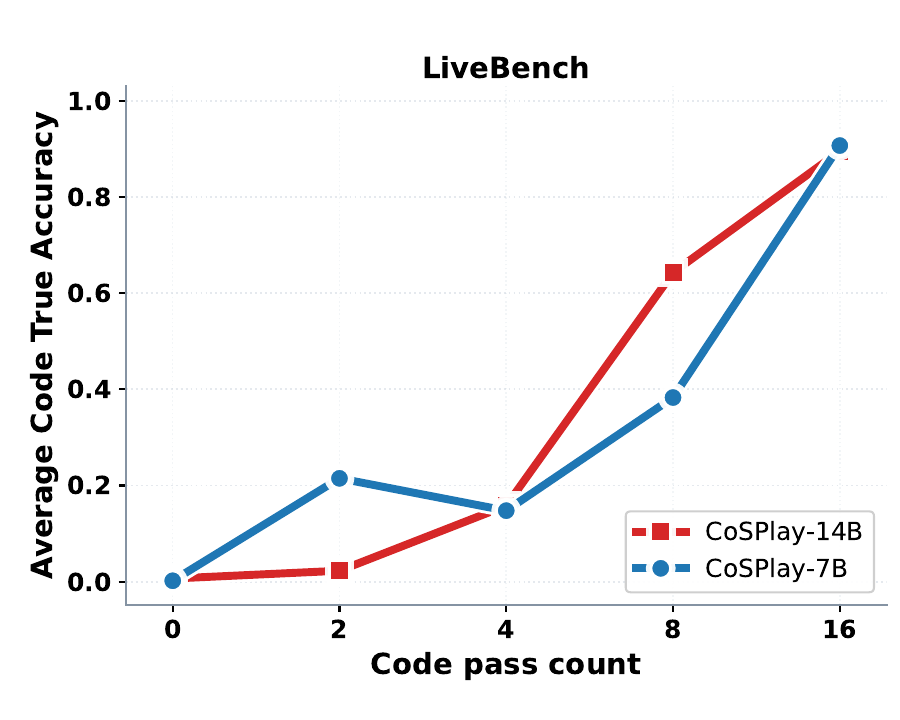}\hfill%
    \includegraphics[width=0.27\textwidth,trim=0 10pt 0 10pt,clip]{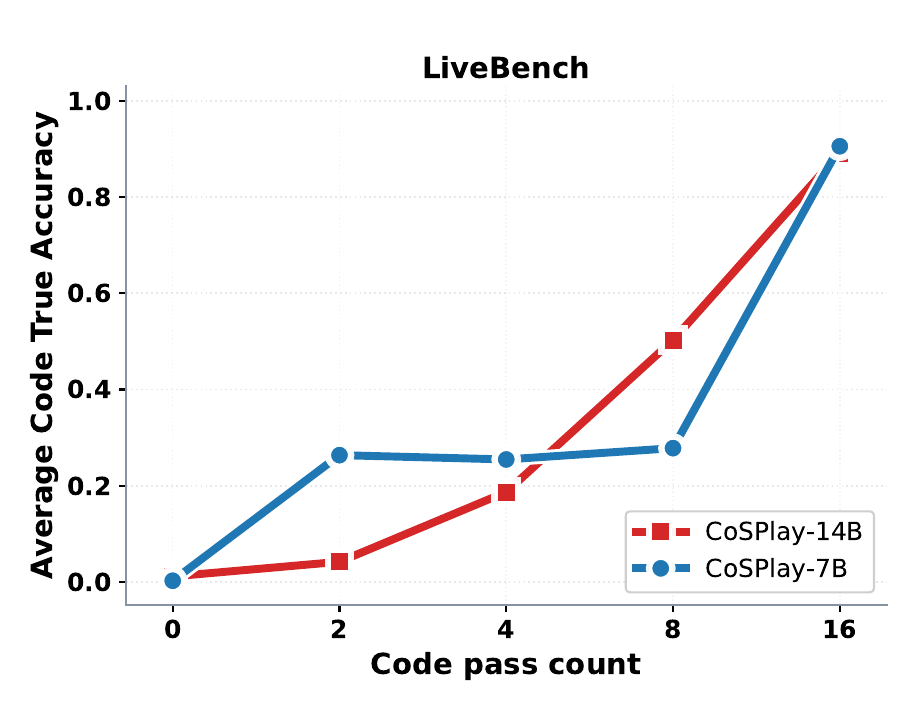}
    
    \vspace{-1.0em}
    
    \includegraphics[width=0.27\textwidth,trim=0 10pt 0 10pt,clip]{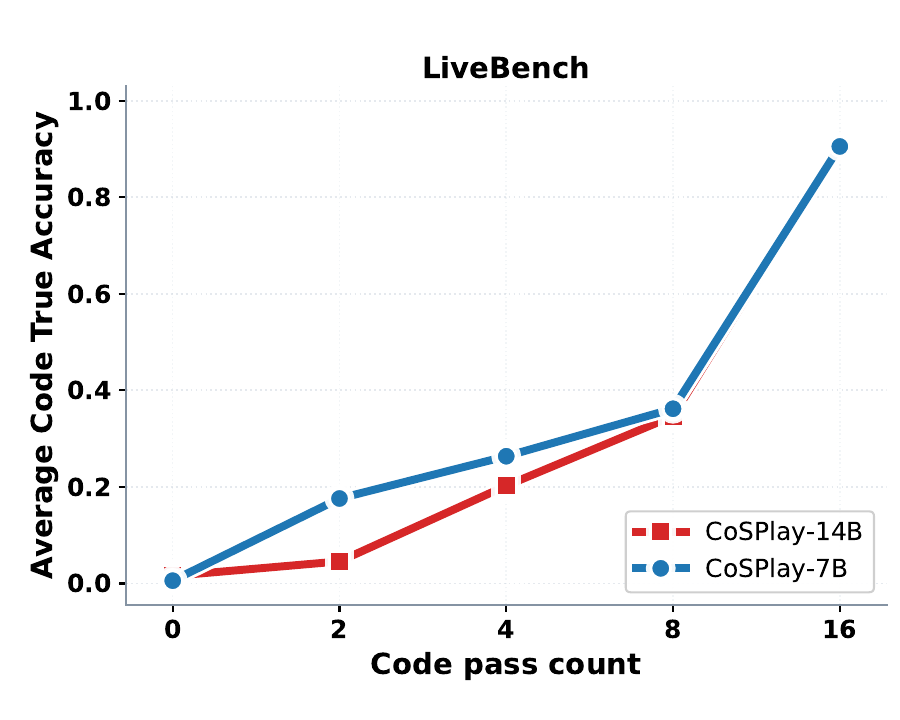}\hfill%
    \includegraphics[width=0.27\textwidth,trim=0 10pt 0 10pt,clip]{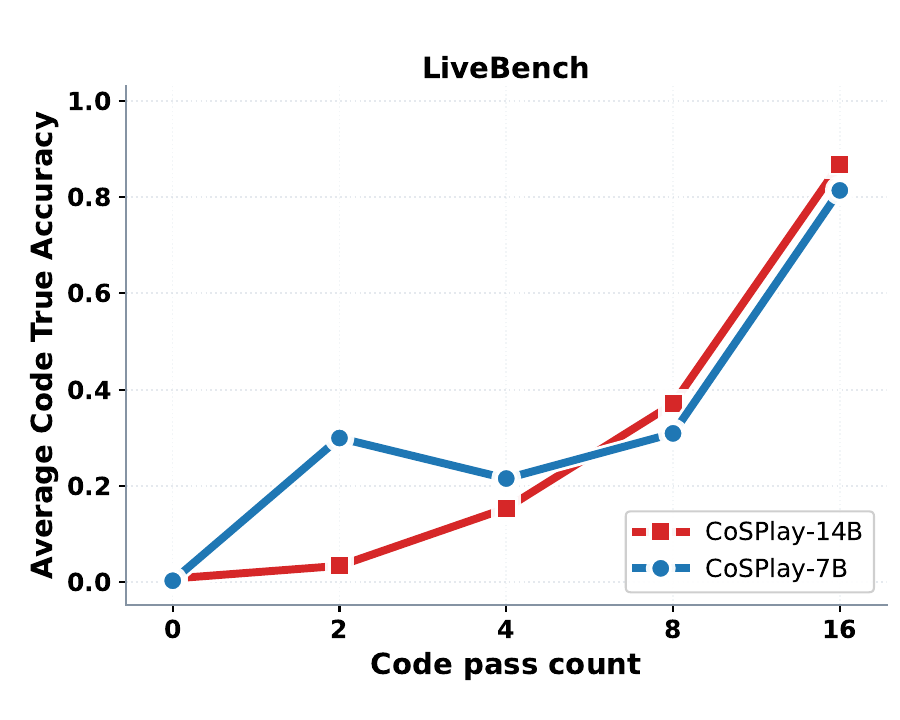}\hfill%
    \includegraphics[width=0.27\textwidth,trim=0 10pt 0 10pt,clip]{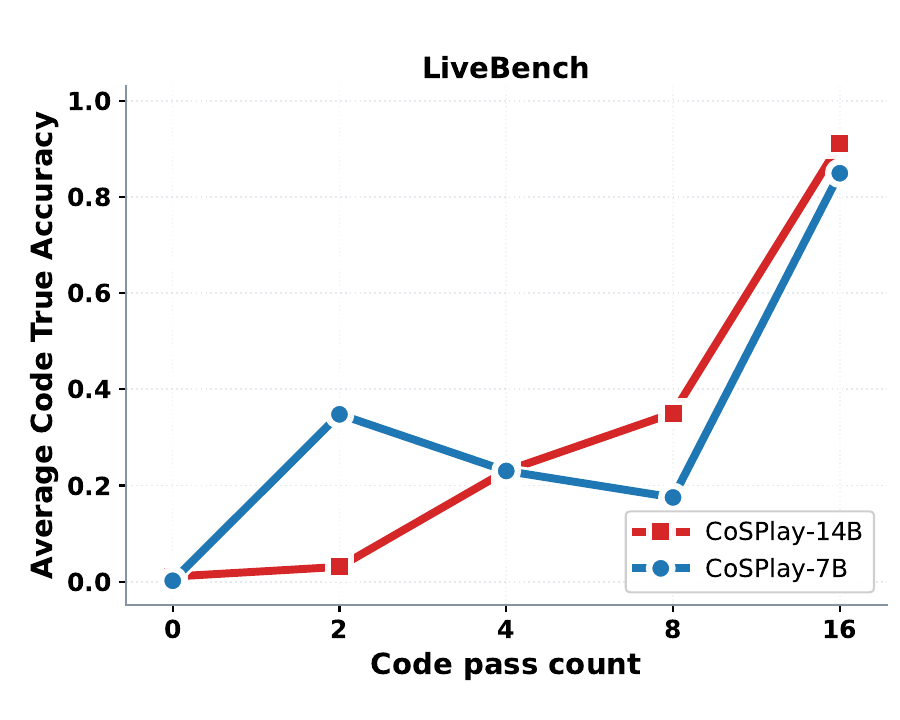}
    \vspace{-0.5em}
    
    \caption{The relationship between code pass counts and average true accuracy for both 7B and 14B models on LiveBench. The top row shows Round 0-2, and the bottom row shows Round 3-5.}
    \label{fig:sub_code_livebench_acc_14b}
\end{figure}
\vspace{-1.5em}

\begin{figure}[H]
    \centering
    \includegraphics[width=0.27\textwidth,trim=0 10pt 0 10pt,clip]{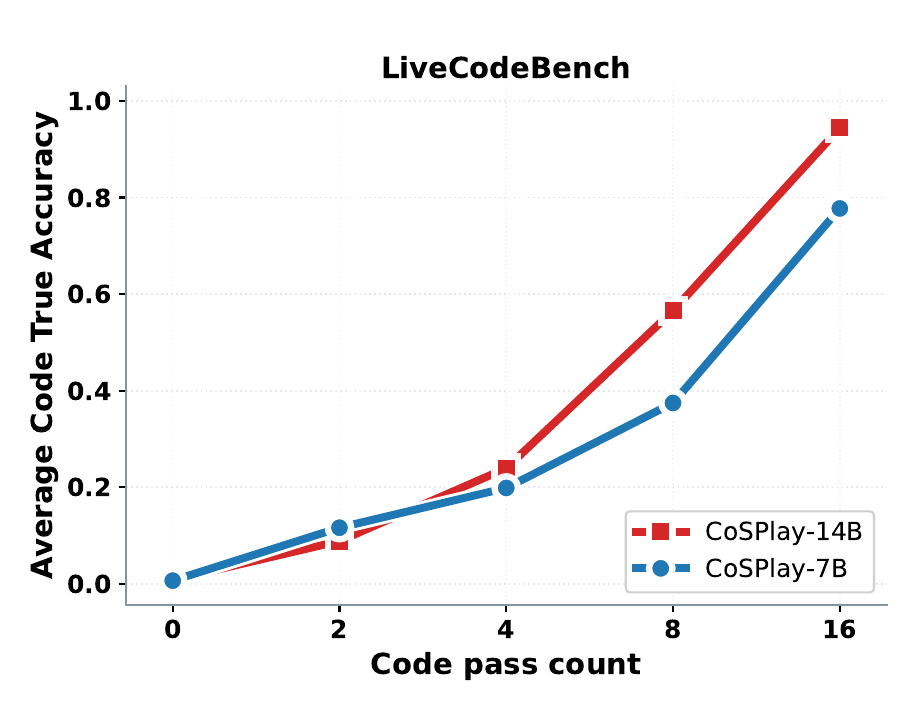}\hfill%
    \includegraphics[width=0.27\textwidth,trim=0 10pt 0 10pt,clip]{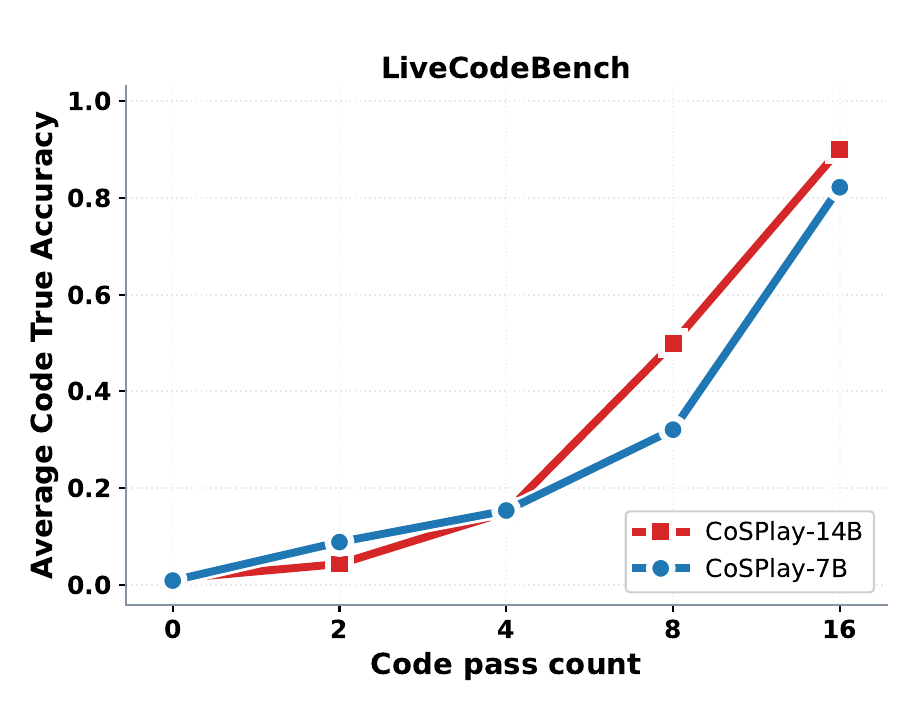}\hfill%
    \includegraphics[width=0.27\textwidth,trim=0 10pt 0 10pt,clip]{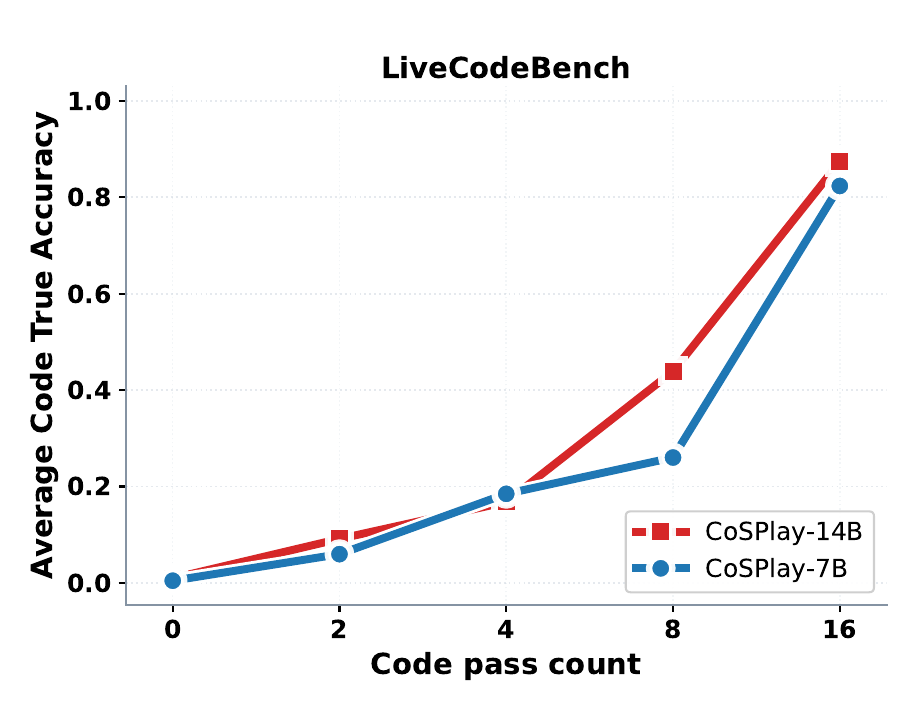}
    
    \vspace{-1.0em}
    
    \includegraphics[width=0.27\textwidth,trim=0 10pt 0 10pt,clip]{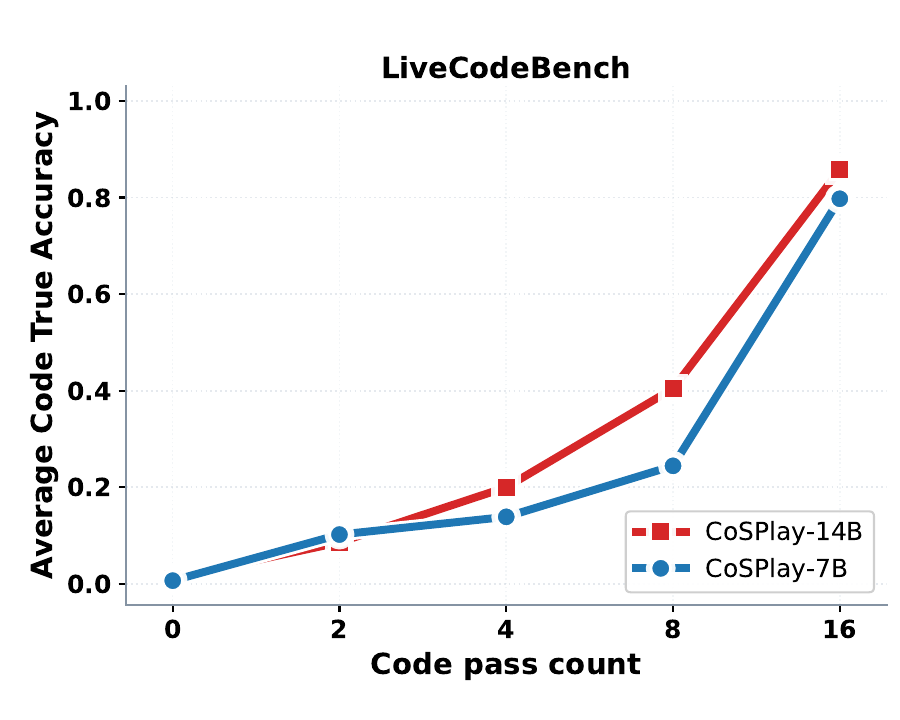}\hfill%
    \includegraphics[width=0.27\textwidth,trim=0 10pt 0 10pt,clip]{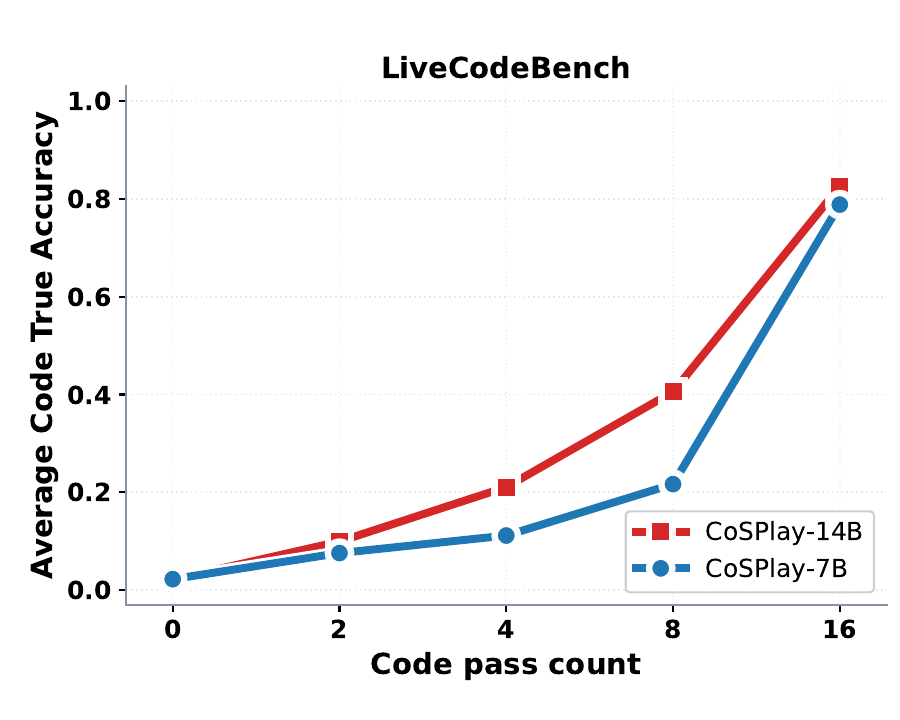}\hfill%
    \includegraphics[width=0.27\textwidth,trim=0 10pt 0 10pt,clip]{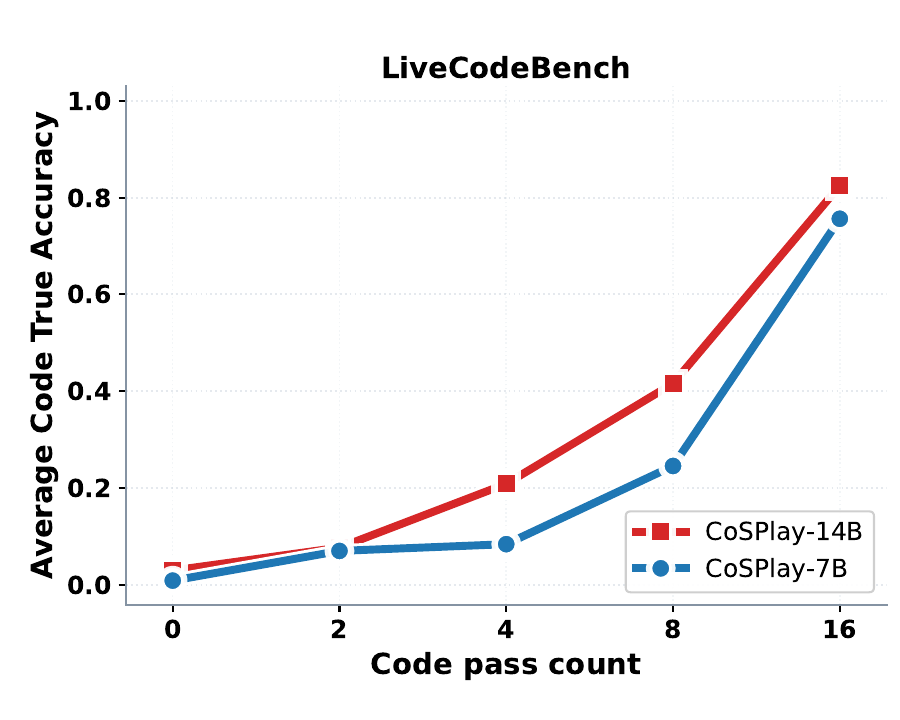}
    \vspace{-0.5em}
    
    \caption{The relationship between code pass counts and average true accuracy for both 7B and 14B models on LiveCodeBench. The top row shows Round 0-2, and the bottom row shows Round 3-5.}
    \label{fig:sub_code_livecodebench_acc_14b}
\end{figure}






\section{Average evolution of the UT and code pass-count distributions across iterative self-play rounds}
\label{app:pass_count_distribution_evolution}

Figures~\ref{fig:average_ut_passrate_ratio}-\ref{fig:average_code_passrate_ratio_14B} provide per-dataset results for the evolution of UT and code pass-count distributions at the 7B and 14B scales. Consistent with the averaged trends in the main text, self-play gradually shifts probability mass from low pass-count regions to high pass-count regions, indicating that the execution matrix becomes more structured over iterations. For UTs, the reduction of low pass-count density suggests that self-play removes or refreshes overly hard, invalid, or spuriously coupled tests that provide little reliable supervision. The increased density in high pass-count regions shows that more UTs become broadly satisfiable by the evolving code pool, making them more useful for verification and selection. For codes, the rightward shift indicates that more candidates pass a larger number of UTs after iterative refinement and replacement. This reflects a positive feedback loop between the two pools: better UTs provide more informative execution feedback, while stronger code candidates help validate and refine the UT pool. The consistent trend across datasets supports that CoSPlay improves both candidate quality and selection reliability during self-play.







\begin{figure}[!htbp]
    \centering
    \includegraphics[width=0.94\textwidth]{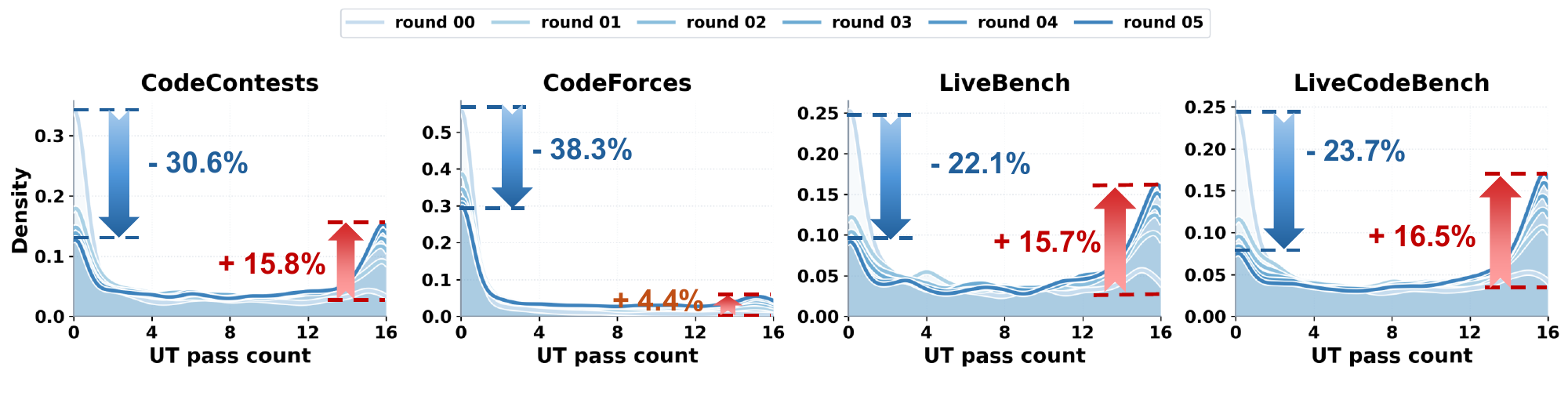}
    \vspace{-0.5em}
    \caption{Evolution of UT pass-count distributions during self-play with the 7B model. Curves show per-round density changes across four benchmarks.}
    \vspace{-1.5em}
    \label{fig:average_ut_passrate_ratio}
\end{figure}

\begin{figure}[!htbp]
    \centering
    \includegraphics[width=0.94\textwidth]{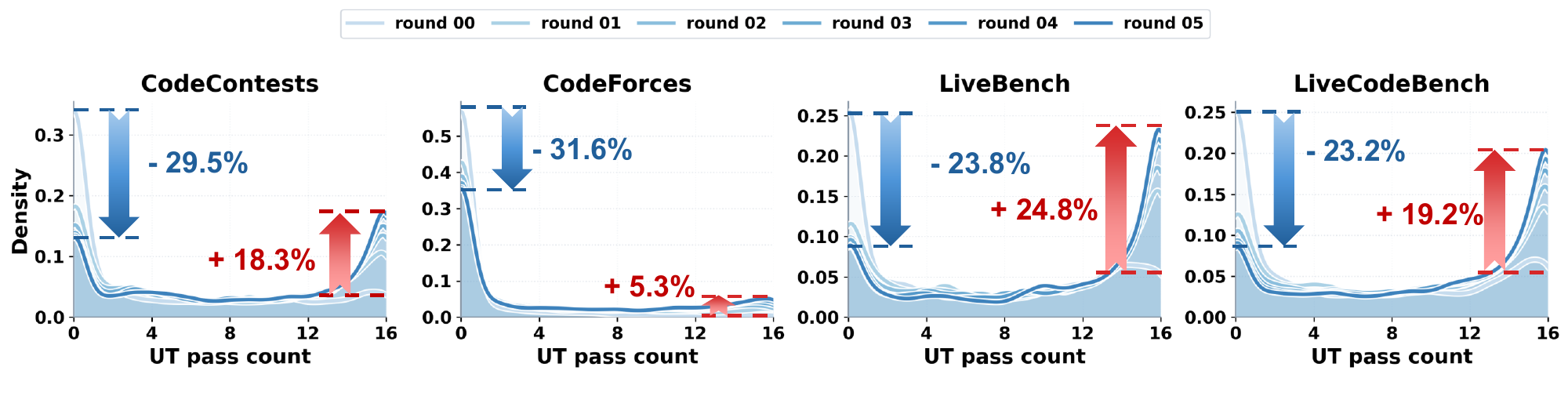}
    \vspace{-0.5em}
    \caption{Evolution of UT pass-count distributions during self-play with the 14B model. Curves show per-round density changes across four benchmarks.}
    \vspace{-1.5em}
    \label{fig:average_ut_passrate_ratio_14B}
\end{figure}

\begin{figure}[!htbp]
    \centering
    \includegraphics[width=0.94\textwidth]{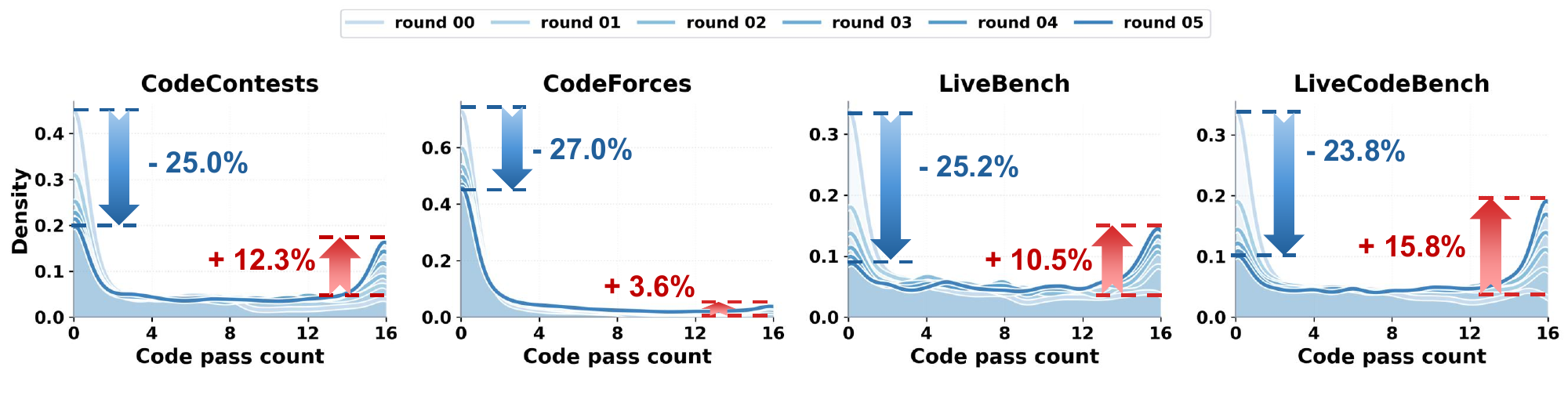}
    \vspace{-0.5em}
    \caption{Evolution of code pass-count distributions during self-play with the 7B model. Curves show per-round density changes across four benchmarks.}
    \vspace{-1.5em}
    \label{fig:average_code_passrate_ratio}
\end{figure}

\begin{figure}[!htbp]
    \centering
    \includegraphics[width=0.94\textwidth]{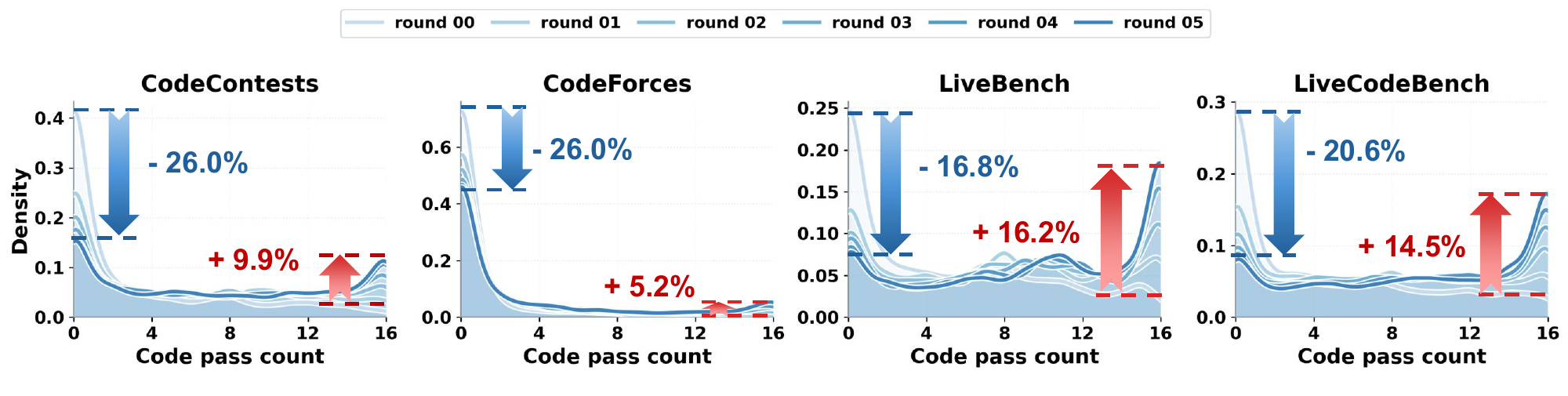}
    \vspace{-0.5em}
    \caption{Evolution of code pass-count distributions during self-play with the 14B model. Curves show per-round density changes across four benchmarks.}
    \vspace{-1.5em}
    \label{fig:average_code_passrate_ratio_14B}
\end{figure}
\section{Detailed metrics evolution during self-play stage}
\label{app:Detailed metrics evolution during self-play stage}



\subsection{Detailed Signal accuracy evolution during self-play rounds}
\begin{figure}[t]
    \centering
    \includegraphics[width=1.0\textwidth]{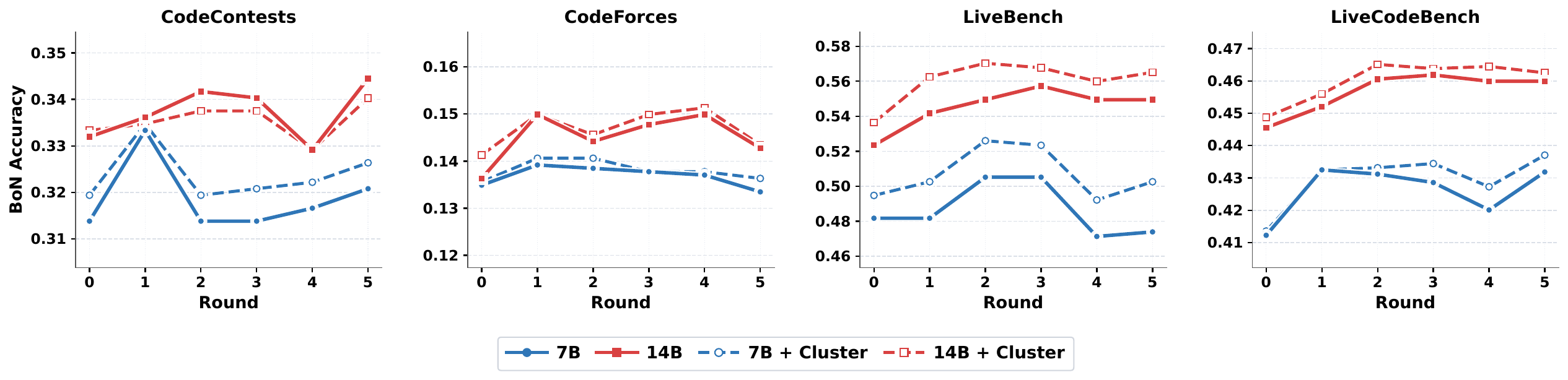}       
    \caption{Evolution of \textbf{Signal Accuracy} across iterative self-play rounds.}
    \label{fig:metric_signal_acc}
\end{figure}
\begin{figure*}[!t]
    \centering

        \centering
        \includegraphics[width=\textwidth]{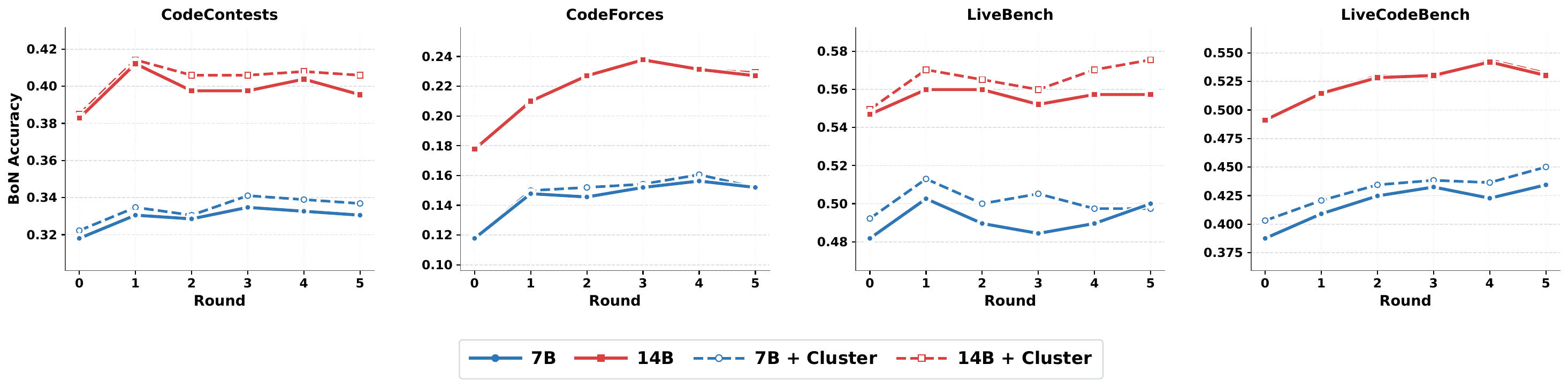}
        \label{fig:metric_bon_7b}



    \caption{Evolution of \textbf{Best-of-N (BoN) accuracy} evaluated on four benchmarks during self-play rounds.}
    \label{fig:metric_bon}
\end{figure*}
\begin{figure}[!t]
    \centering
    \includegraphics[width=1.0\textwidth]{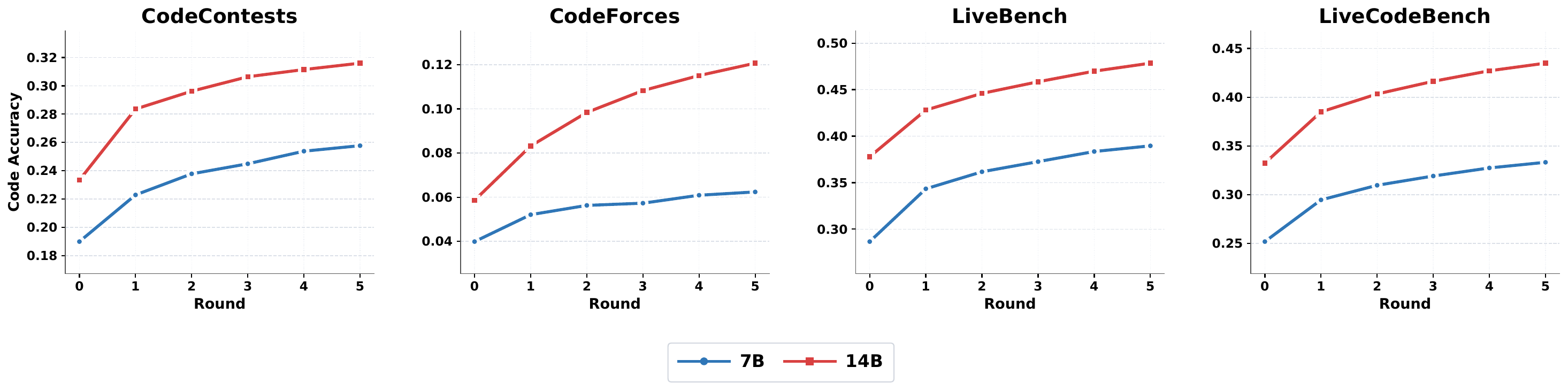}   
    \caption{Evolution of \textbf{Code Accuracy} across iterative self-play rounds.}
    \label{fig:metric_code_acc}
\end{figure}
\begin{figure}[!t]
    \centering
    \includegraphics[width=1.0\textwidth]{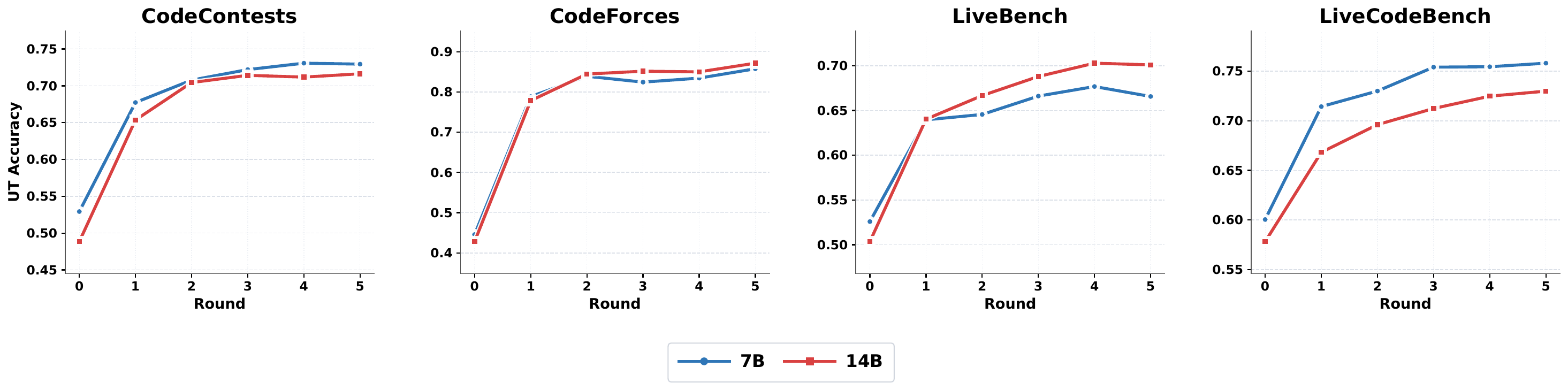}       
    \caption{Evolution of \textbf{Unit Test (UT) Accuracy} across iterative self-play rounds}
    \label{fig:metric_ut_acc}
\end{figure}
Figure~\ref{fig:metric_signal_acc} illustrates the evolution of Signal Accuracy, defined as the BoN accuracy of CoSPlay-generated UTs when used to select Qwen2.5-7B-Instruct-generated code candidates. Both the 7B and 14B exhibit a generally upward trend across all datasets, confirming that iterative self-play progressively enhances the discriminative power of the evolved UT pool. Introducing cluster-based selection (+ Cluster) consistently yields higher Signal Accuracy than the BoN-only variant, demonstrating that cluster selection further improves the reliability of UT-guided candidate selection. The 14B model achieves higher Signal Accuracy than the 7B model across all datasets and rounds, reflecting the benefit of greater model capacity in generating more accurate and discriminative UTs. Among datasets, \textit{CodeForces} remains the most challenging, showing the lowest absolute Signal values throughout.


\subsection{Detailed BoN evolution during self-play rounds}
\label{app:detailed_code_bon_change}

Figure~\ref{fig:metric_bon} illustrates the evolution of BoN Accuracy across self-play rounds. Both the 7B and 14B models show a consistent upward trend across all four datasets, confirming that iterative self-play steadily improves the quality of the code candidate pool. The most notable gains occur in the early rounds, after which performance continues to grow at a more gradual pace. Applying cluster-based selection (+ Cluster) consistently improves BoN Accuracy over the plain BoN variant, validating the effectiveness of execution-consensus-based selection as a complementary mechanism. As expected, the 14B model outperforms the 7B model across all datasets and rounds. \textit{CodeForces} remains the most challenging benchmark with the lowest absolute BoN scores, while \textit{LiveBench} and \textit{LiveCodeBench} yield comparatively higher accuracy for both models.

\subsection{Detailed Code accuracy evolution during self-play rounds}
\label{app:detailed_code_acc_change}
Figure~\ref{fig:metric_code_acc} illustrates the trajectory of Code Accuracy. Both the 7B and 14B models show a consistent upward trend across all datasets, confirming the effectiveness of iterative self-play. The most substantial gain occurs between Round 0 and Round 1, followed by steady, gradual growth. Notably, both models achieve their highest accuracy on \textit{LiveBench} and \textit{LiveCodeBench}, while \textit{CodeForces} remains the most challenging with the lowest overall scores. As expected, the 14B model consistently outperforms the 7B model across all stages.

\subsection{Detailed UT accuracy evolution during self-play rounds}
\label{app:detailed_ut_acc_change}
Figure~\ref{fig:metric_ut_acc} presents the evolution of Unit Test (UT) Accuracy. The defining characteristic across all datasets is a dramatic surge from Round 0 to Round 1, indicating that our self-play stage can filter wrong UT and inject correct UT into the UT pool. This leap is most prominent on \textit{CodeForces}, which jumps from the lowest to the highest accuracy. Following this initial surge, performance largely plateaus with only marginal fluctuations. The overall trajectories are remarkably similar for both models, though the 14B model achieves slightly higher peak values.

\newpage


\clearpage

\section{Case study}
\label{app:cosplay-case}
\subsection{Case Study of successful code fixing}

\subsubsection{CoSPlay-7B}

\tcbinputlisting{
styleA,
title={Question for CoSPlay-7B},
listing file={case_study/7B_second_high/problem.txt},
label={lst:7b_problem}
}


\tcbinputlisting{
styleA,
title={Buggy code from CoSPlay-7B},
listing file={case_study/7B_second_high/original_code.txt},
label={lst:7b_buggy_code}
}


\tcbinputlisting{
styleA,
title={Generated UT from CoSPlay-7B},
listing file={case_study/7B_second_high/case_original_idea.txt},
label={lst:7b_ut}
}


\tcbinputlisting{
styleA,
title={Fixed Code from CoSPlay-7B},
listing file={case_study/7B_second_high/fix_code.txt},
label={lst:7b_fixed_code}
}


\subsubsection{CoSPlay-14B}
\label{subsubsec: case_study_second_high_content_14B}

\tcbinputlisting{
styleA,
title={Question for CoSPlay-14B},
listing file={case_study/14B_second_high/problem.txt},
label={lst:problem}
}


\tcbinputlisting{
styleA,
title={Buggy Code from CoSPlay-14B},
listing file={case_study/14B_second_high/original_code.txt},
label={lst:buggy_code}
}


\tcbinputlisting{
styleA,
title={Generated UT from CoSPlay-14B},
listing file={case_study/14B_second_high/case_original_idea.txt},
label={lst:ut}
}


\tcbinputlisting{
styleA,
title={Fixed Code from CoSPlay-14B},
listing file={case_study/14B_second_high/fix_code.txt},
label={lst:fixed_code}
}

\label{appendix:case_study_step3}

Figure~\ref{fig:case_study_second_high} presents two concrete examples of successful code refinement driven by the non-trivial best UT selected in Step~3.

\textbf{Example 1: Balance Scale Comparison.}
The buggy code attempts to solve a simple balance-scale problem, where the correct decision should be based on comparing the total weights on the two sides, i.e., \texttt{A + B} versus \texttt{C + D}. However, the implementation mistakenly compares the sums of reciprocals, \texttt{1/A + 1/B} and \texttt{1/C + 1/D}, which can lead to an incorrect ordering for many inputs. The generated UT directly targets this semantic error with \texttt{A = 10}, \texttt{B = 10}, \texttt{C = 1}, and \texttt{D = 1}. The expected output is \texttt{Left}, since the left side has total weight \texttt{20} while the right side has total weight \texttt{2}. Nevertheless, the buggy code outputs \texttt{Right}, because its reciprocal-based comparison gives \texttt{0.2} on the left and \texttt{2} on the right. The refinement guidance correctly localizes this erroneous comparison and replaces the reciprocal sums with direct weight sums, producing a semantically correct fixed code.

\textbf{Example 2: Two-Row Monkey Seating.}
The buggy code correctly caps the row-preference groups by their row capacities, but then miscomputes the remaining capacity as \texttt{m - (a + b)}, effectively treating the classroom as if it had only one row. The generated UT exposes this bug with \texttt{m = 10}, \texttt{a = 4}, \texttt{b = 4}, and \texttt{c = 11}, where the expected answer is \texttt{19} because there are \texttt{20} total seats and all \texttt{11} no-preference monkeys can still be seated after placing the \texttt{4} row-1 and \texttt{4} row-2 monkeys. However, the buggy code outputs \texttt{10}, since it incorrectly believes only \texttt{2} seats remain. The refinement guidance fixes the capacity calculation by replacing \texttt{m - (a + b)} with \texttt{2 * m - (a + b)}, producing a semantically correct fixed code.

\textbf{Summary.}
Both examples show that Step~3 can provide precise refinement signals when the selected UT exposes a concrete semantic mismatch between the buggy implementation and the problem requirement. In the balance-scale example, the UT reveals that the code compares reciprocal values instead of total weights; in the monkey-seating example, the UT reveals that the code computes remaining capacity using only one row instead of two. In both cases, the failure is not caused by formatting or invalid inputs, but by a targeted test that activates the exact erroneous logic. This allows the LLM to perform a localized correction---replacing the wrong comparison formula in the first case and the wrong capacity formula in the second---rather than rewriting the entire solution from scratch, demonstrating the effectiveness of Step~3's targeted refinement strategy.

\begin{figure}[H]
    \centering
    \includegraphics[width=1.0\textwidth]{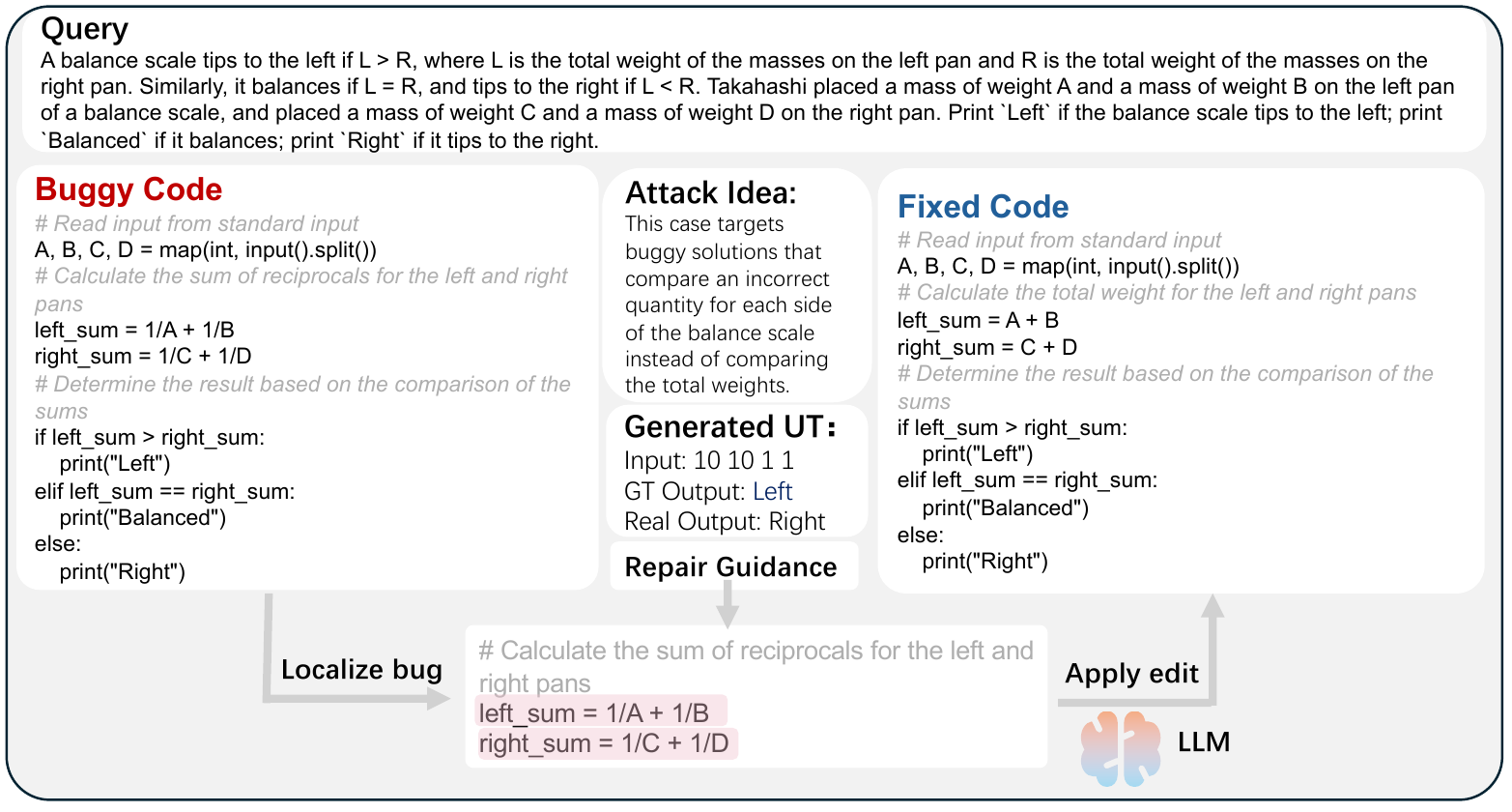}
    \hfill
    \includegraphics[width=1.0\textwidth] {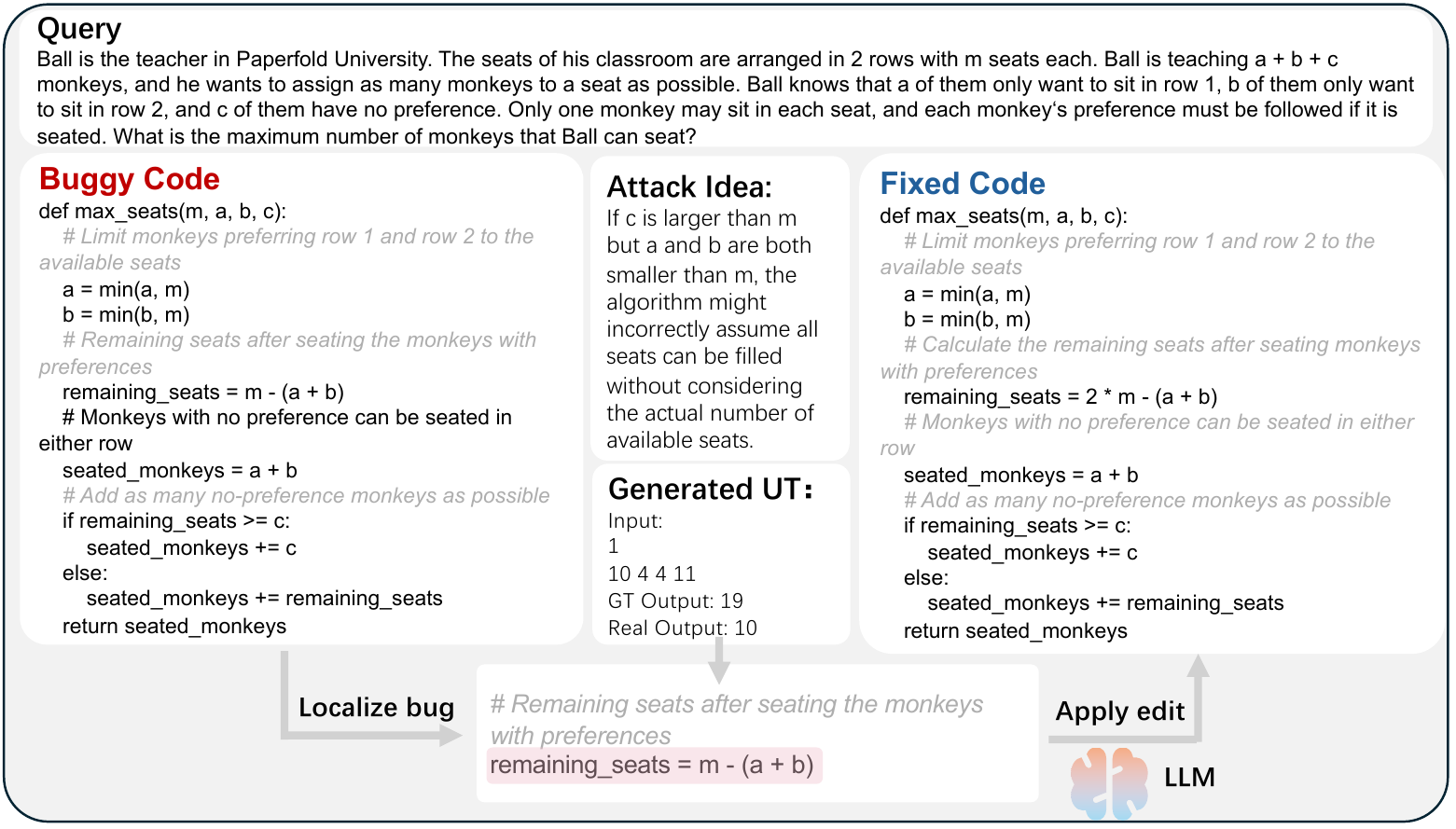}

    \caption{Case study of successful code fix.}
    \label{fig:case_study_second_high}
\end{figure}

\subsection{Case study of successful locating Code–UT Coupling}

\subsubsection{CoSPlay-7B}

\tcbinputlisting{
styleA,
title={Question for CoSPlay-7B},
listing file={case_study/7B_second_low/question.txt},
label={lst:question_cosplay}
}

\tcbinputlisting{
styleA,
title={Wrong Code for CoSPlay-7B},
listing file={case_study/7B_second_low/wrong_code.txt},
label={lst:wrong_code_cosplay}
}

\tcbinputlisting{
styleA,
title={Erroneously coupled UT for CoSPlay-7B},
listing file={case_study/7B_second_low/wrong_ut.txt},
label={lst:wrong_ut_cosplay}
}

\tcbinputlisting{
styleA,
title={New generated UT for CoSPlay-7B},
listing file={case_study/7B_second_low/right_ut.txt},
label={lst:right_ut_cosplay}
}

\subsubsection{CoSPlay-14B}

\tcbinputlisting{
styleA,
title={Question for CoSPlay-14B},
listing file={case_study/14B_second_low/question.txt},
label={lst:question_cosplay_14b}
}

\tcbinputlisting{
styleA,
title={Wrong Code for CoSPlay-14B},
listing file={case_study/14B_second_low/wrong_code.txt},
label={lst:wrong_code_cosplay_14b}
}
\tcbinputlisting{
styleA,
title={Erroneously coupled UT for CoSPlay-14B},
listing file={case_study/14B_second_low/wrong_ut.txt},
label={lst:wrong_ut_cosplay_14b}
}

\tcbinputlisting{
styleA,
title={New generated UT for CoSPlay-14B},
listing file={case_study/14B_second_low/right_ut.txt},
label={lst:right_ut_cosplay_14b}
}

\subsubsection{Analysis for case study of UT-code coupling}
\label{appendix:case_study_step2}
\begin{figure*}[!t]
    \centering

    \begin{subfigure}[t]{0.98\textwidth}
        \centering
        \includegraphics[width=\textwidth]{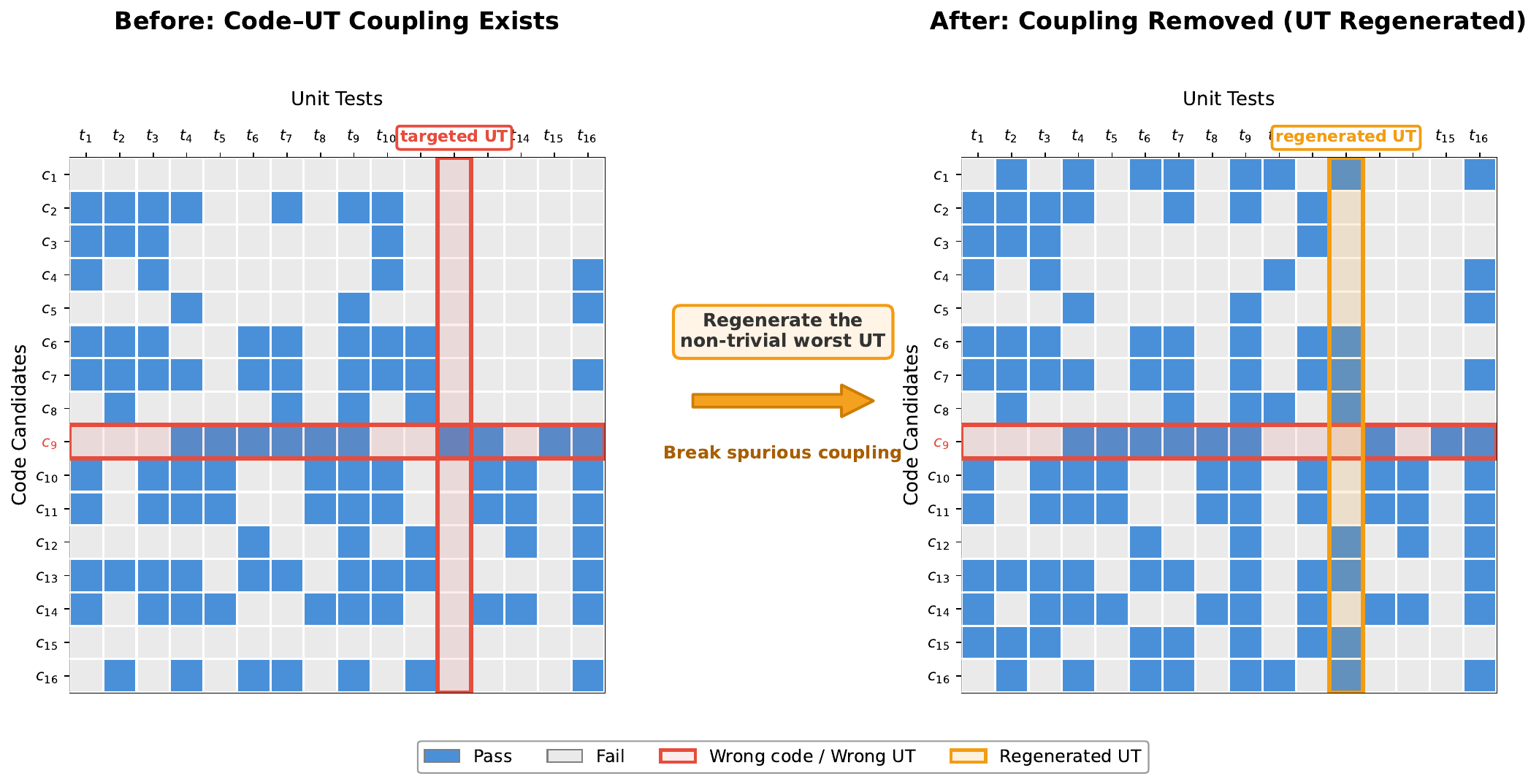}
        \caption{CoSPlay-7B}
        \label{fig:case_study_second_low_7B}
    \end{subfigure}

    \vspace{0.5em}

    \begin{subfigure}[t]{0.98\textwidth}
        \centering
        \includegraphics[width=\textwidth]{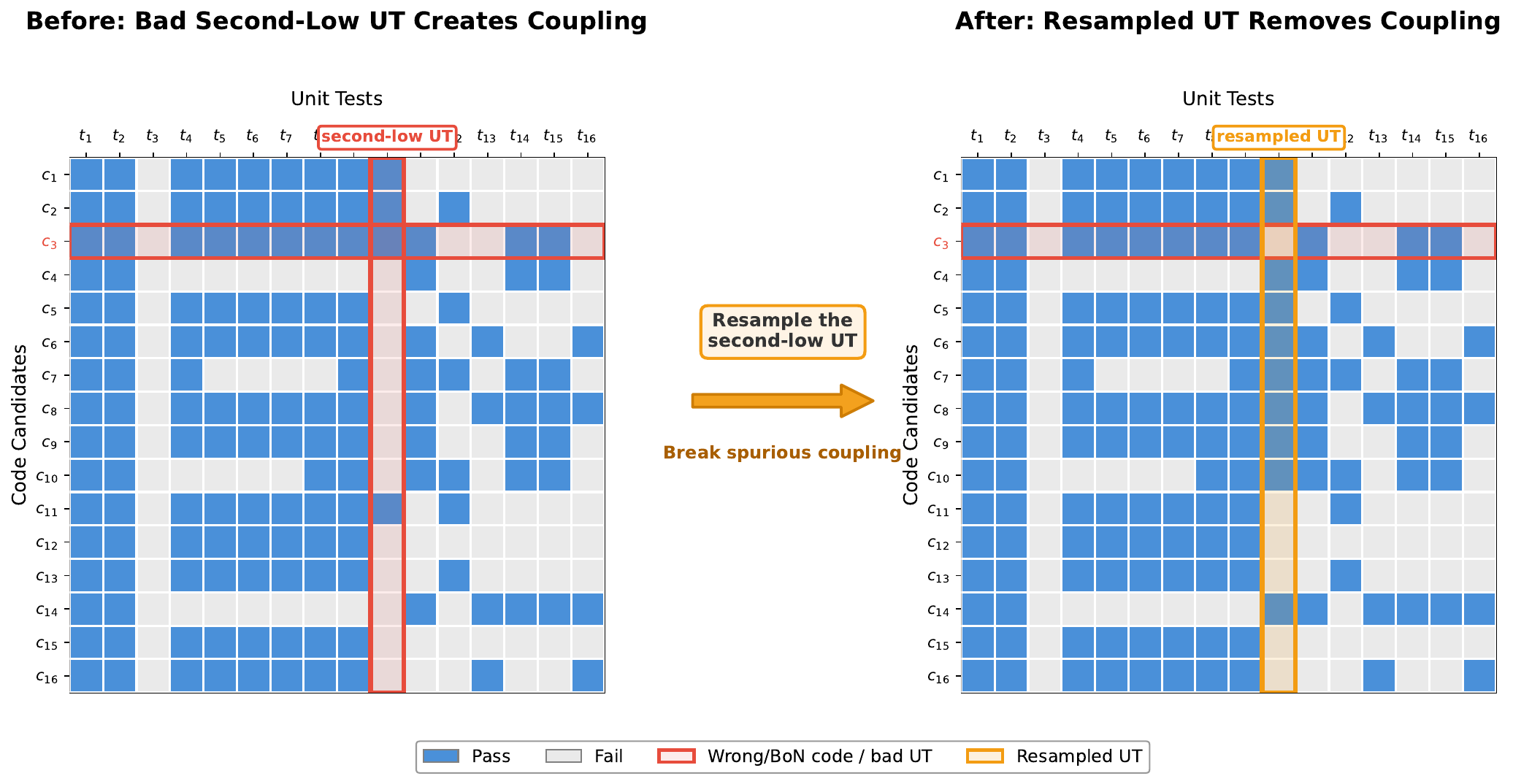}
        \caption{CoSPlay-14B}
        \label{fig:case_study_second_low_14B}
    \end{subfigure}

    \caption{
Execution matrices demonstrating the resolution of Code--UT coupling in CoSPlay.
Each row represents a generated code candidate, and each column represents a generated unit test (UT). 
A blue cell indicates that the code passes the UT, while a grey cell indicates failure. 
In each subfigure, the left panel shows the execution matrix before Step~2, where the highlighted red row marks a suspicious wrong or BoN-selected code and the highlighted red column marks a suspicious low-but-nonzero-pass-rate UT. 
Their highlighted intersection corresponds to a false positive: the flawed code passes the suspicious UT, creating spurious Code--UT coupling. 
The right panel shows the matrix after replacing this suspicious UT with a regenerated or resampled UT, highlighted in orange. 
After replacement, the same highlighted flawed code no longer passes the new UT, indicating that the accidental agreement between the code and the previous UT has been removed. 
Subfigure~(a) shows the CoSPlay-7B case, where regenerating the non-trivial worst UT breaks the coupling between a wrong code and a wrong UT. 
Subfigure~(b) shows the CoSPlay-14B case, where resampling the bad second-low UT removes misleading support for the highlighted wrong/BoN-selected code. 
Overall, Step~2 performs a targeted intervention by replacing a suspicious UT column while keeping the remaining UT pool unchanged, thereby eliminating the highlighted false positive and improving the discriminative signal for subsequent repair and selection.
}
    \label{fig:case_study_second_low_bool}
\end{figure*}

\begin{figure*}[!t]
    \centering
    \includegraphics[width=1.0\textwidth]{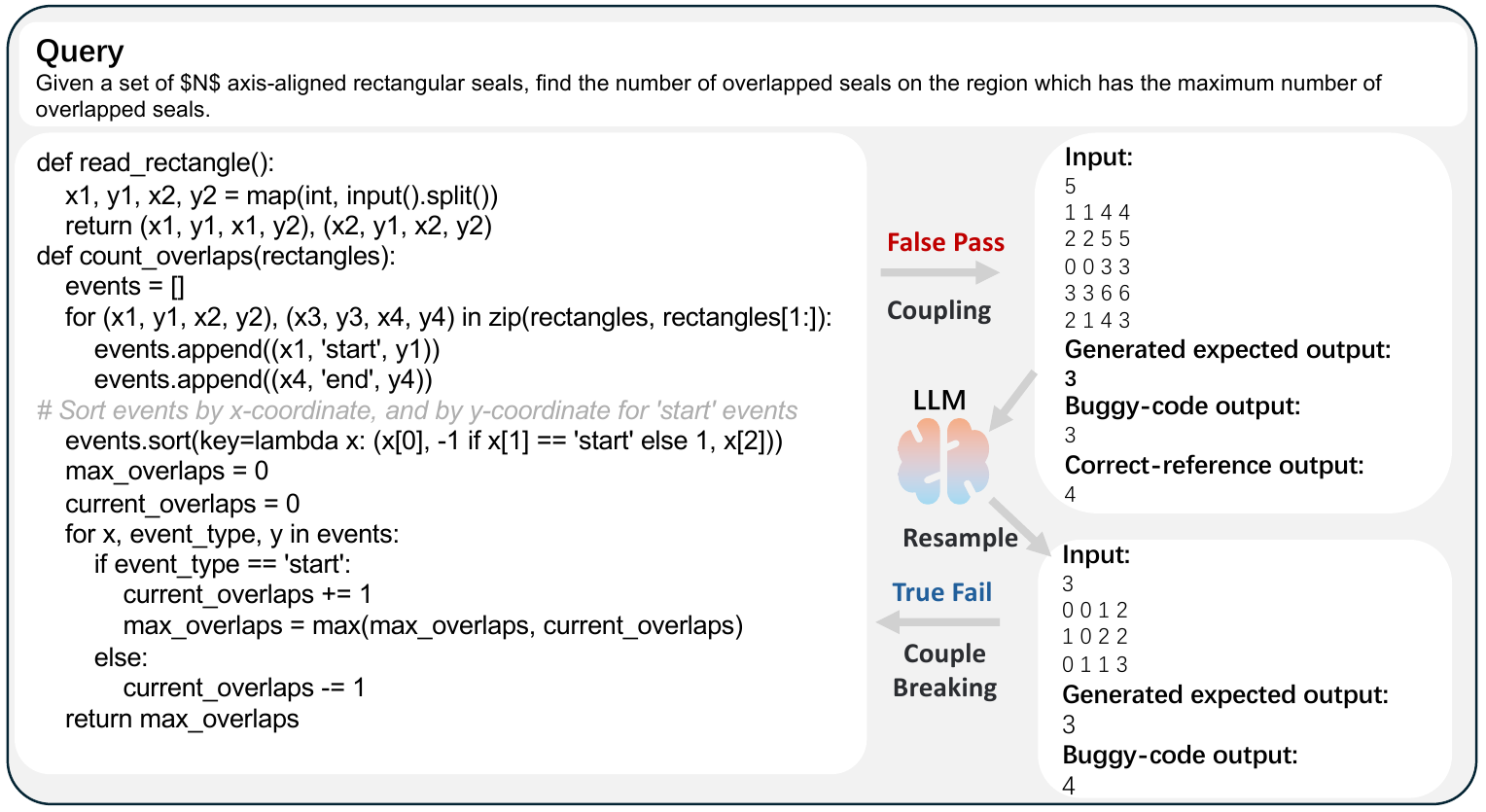}
    \hfill
    \includegraphics[width=1.0\textwidth]{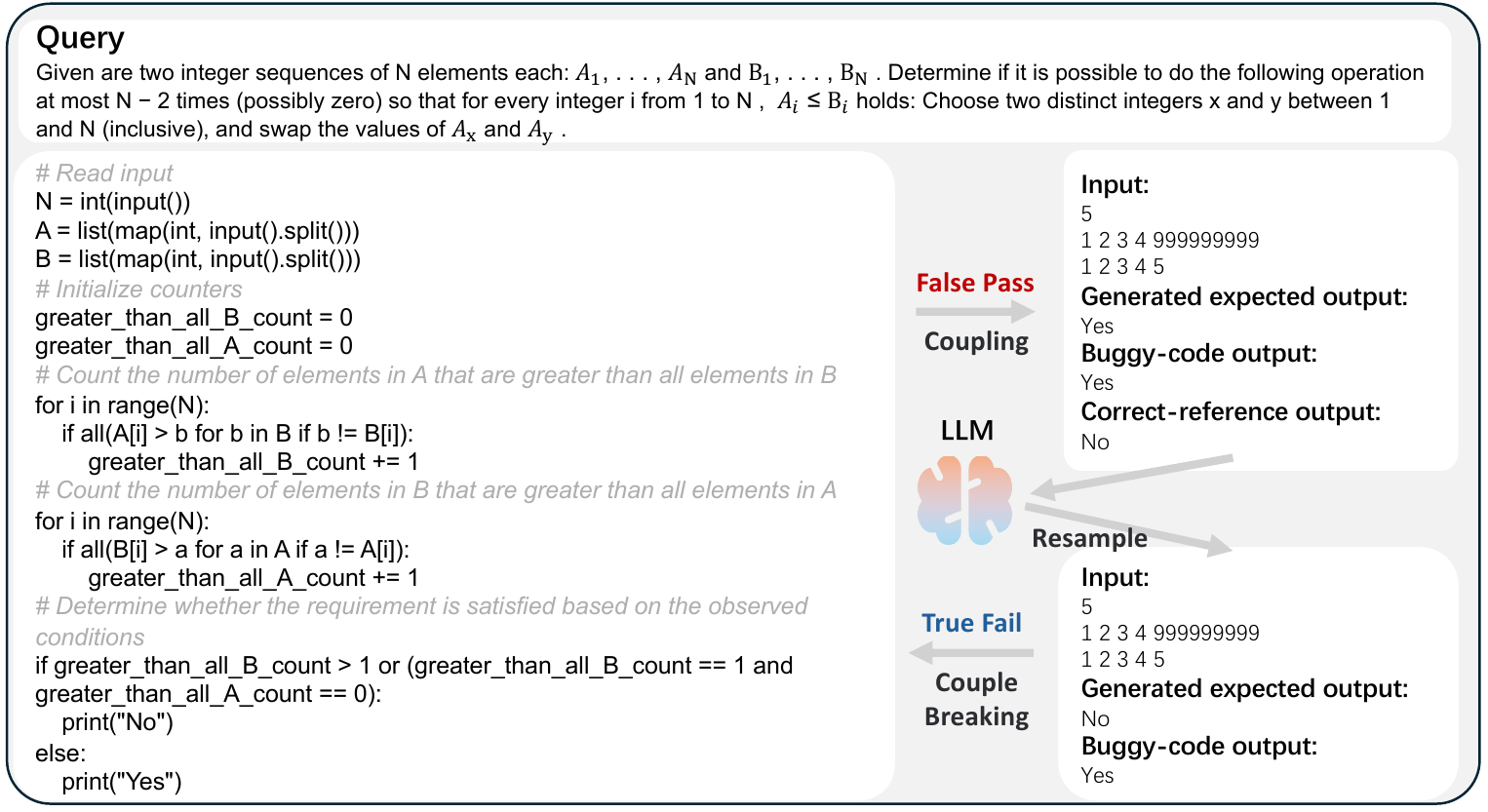}
    
    \caption{Case study of Code-UT coupling.}
    \label{fig:case_study_second_low}
\end{figure*}
Figure~\ref{fig:case_study_second_low} presents two concrete examples illustrating how spurious Code-UT coupling arises in practice and how Step~2 of our iterative self-play framework detects and breaks it. Figure~\ref{fig:case_study_second_low_bool} further visualizes the same process through the lens of the binary execution matrix, offering a complementary macro-level view of how the UT pool's discriminative structure changes before and after regeneration.

\textbf{Execution Matrix Analysis (Figure~\ref{fig:case_study_second_low_bool}).}
Each row of the matrix corresponds to a code candidate $c_i$ and each column to a unit test $t_j$; a blue cell indicates a pass ($\mathbf{M}_{ij}=1$) and a grey cell indicates a failure ($\mathbf{M}_{ij}=0$). The highlighted red row marks the suspicious wrong or BoN-selected code, while the highlighted red column marks the suspicious UT before regeneration. The highlighted orange column marks the regenerated or resampled UT after Step~2.

In the 7B case, the before panel shows a clean Code-UT coupling pattern: the highlighted wrong code passes the highlighted low-pass UT, creating a false positive that can inflate the pass count of an wrong solution. After regeneration, the corresponding orange UT column no longer accepts the same wrong code, thereby removing this spurious agreement.

In the 14B case, the highlighted second-low UT also accepts the highlighted wrong/BoN-selected code before resampling, indicating that this UT can still provide misleading support for an wrong candidate. Unlike the 7B case, this column is not a unique-pass column; rather, it is a low-pass suspicious UT whose agreement with the highlighted wrong code contributes to the coupling phenomenon. After resampling, the highlighted wrong code fails the new UT, showing that the accidental agreement between the code and the previous UT has been broken.

Overall, the execution matrices show that Step~2 operates in a targeted manner: it replaces a suspicious low-but-nonzero-pass-rate UT column, while leaving the rest of the UT pool unchanged. This intervention removes the highlighted false positive and restores a more useful discriminative signal for subsequent refinement and selection.

\textbf{Example 1: Rectangle Overlap Counting.}
The problem asks to count the maximum number of overlapping axis-aligned rectangles over any region. The shown candidate code is semantically flawed because it reduces the two-dimensional overlap problem to a one-dimensional event-style count and fails to correctly model overlaps along both axes. Nevertheless, it passes the initial \emph{False Pass} UT because the UT is spuriously aligned with the same error. Specifically, the expected answer of the initial UT is \texttt{3}, while the true maximum overlap is \texttt{4}. The wrong code also outputs \texttt{3}, which leads to a false positive. CoSPlay then resamples the suspicious UT using the spuriously passed code as a negative reference. The regenerated UT is still not oracle-correct, since its expected output remains imperfect. However, it successfully breaks the previous code and UT coupling because the same buggy implementation no longer agrees with the regenerated UT and is therefore filtered out. This example shows that the resampling step is not meant to guarantee an immediately perfect UT. Instead, it serves as a decoupling mechanism that disrupts accidental agreement between wrong code and wrong tests and recovers useful discriminative signal for selection.

\textbf{Example 2: Swap-Based Array Dominance.}
The second example considers whether the sequence $A$ can be permuted by swapping elements within $A$ at most $N-2$ times so that $A_i \leq B_i$ holds for every position $i$. The generated \emph{False Pass} UT is a valid input, but its oracle is incorrect. Specifically, for the input \texttt{A = [1,2,3,4,999999999]} and \texttt{B = [1,2,3,4,5]}, the generated expected output is \texttt{Yes}. However, the true answer is \texttt{No}, because the element \texttt{999999999} in $A$ is larger than every element in $B$. Since the allowed operation only permutes elements of $A$ and cannot change their values, this element must occupy some position $i$ and will necessarily violate $A_i \leq B_i$. Thus, the instance is impossible even with unlimited swaps, and therefore also impossible under the stricter $N-2$ swap limit.

The buggy code does not directly check whether the multiset of $A$ can be matched to $B$ under the coordinate-wise inequality. Instead, it uses an indirect heuristic based on counting elements that are larger than almost all elements in the other sequence. On this UT, the heuristic incorrectly prints \texttt{Yes}, matching the wrong generated oracle. In contrast, the semantically correct reference code sorts both sequences and correctly detects that the largest element of $A$ still exceeds the largest element of $B$, outputting \texttt{No}. As a result, this wrong UT creates a clean Code-UT coupling: the wrong code passes the wrong test, while the correct code fails it. This is precisely the type of false positive that can inflate the pass count of an incorrect solution and mislead final selection.

\textbf{Summary.}
Taken together, Figure~\ref{fig:case_study_second_low} and Figure~\ref{fig:case_study_second_low_bool} illustrate the Code-UT coupling problem at two complementary levels of granularity. The former exposes the concrete input-output misalignment that creates a false positive, while the latter reveals its global footprint in the execution matrix as a low-pass-rate UT column spuriously aligned with a wrong code. The rectangle-overlap example shows how targeted UT regeneration can disrupt an accidental agreement between a flawed implementation and an imperfect test. The swap-based array example further isolates an even cleaner failure mode: the generated UT has a valid input but an incorrect oracle, causing the wrong code to pass while a correct reference solution fails.

These cases demonstrate that self-generated UTs can become misleading not only because their inputs are insufficiently diverse, but also because their expected outputs may accidentally agree with a specific erroneous implementation. Step~2 addresses this issue by identifying suspicious low-but-nonzero-pass-rate UTs and regenerating them with the incorrectly passing code as a negative reference. This targeted regeneration discourages repeated agreement with the same flawed logic, breaks spurious Code-UT coupling, and helps recover useful discriminative signal for refinement and selection without relying on ground-truth supervision.

\section{Prompts used in our method}
\label{app:prompt}
\tcbinputlisting{
styleA,                                    
title={Prompt for solution hints generation},       
listing file={prompts/first_order_obs.txt} 
}

\tcbinputlisting{
styleA,                                    
title={Prompt for specific solution idea generation},      
listing file={prompts/second_order_obs.txt}
}

\tcbinputlisting{
styleA,                                    
title={Prompt for code generation from solution idea},         
listing file={prompts/idea_to_code.txt}    
}

\tcbinputlisting{
styleA,                                    
title={Prompt for UT attack idea generation},            
listing file={prompts/attack_idea.txt}     
}

\tcbinputlisting{
styleA,                                    
title={Prompt for UT input generation from UT attack idea},               
listing file={prompts/ut_input.txt}        
}

\tcbinputlisting{
styleA,                                    
title={Prompt for UT output generation from specific input},              
listing file={prompts/ut_output.txt}       
}

\tcbinputlisting{
styleA,                                    
title={Prompt for generating non-coupling UT},                       
listing file={prompts/refine.txt}        
}

\tcbinputlisting{
styleA,                                    
title={Prompt for code fixing},                       
listing file={prompts/fix_code.txt}        
}


\tcbinputlisting{
styleA,                                    
title={Prompt for random UT input generation},        
listing file={prompts/random_input.txt}    
}

\tcbinputlisting{
styleA,                                    
title={Prompt for direct code generation},            
listing file={prompts/resample_code.txt}   
}

\tcbinputlisting{
styleA,                                    
title={Prompt for direct UT generation},              
listing file={prompts/reample_UT.txt}      
}


\section{Acknowledgement}

This work was supported by the Guangdong Basic and Applied Basic Research Foundation (Grant No. 2026A1515011579), the High-performance Computing Platform of The Hong Kong University of Science and Technology (Guangzhou), the HKUST-HKUST(GZ) 1+1+1 Joint Funding Program (Grant No. C\_2025\_031), and the Guangzhou-HKUST(GZ) Joint Funding Program (Grant No. 2023A03J0008), Education Bureau of Guangzhou Municipality. This work was also supported by Jiangsu Industrial Technology Research Institute (JITRI) and Wuxi National High-Tech District (WND).

\end{document}